\documentclass[colorlinks]{article}
\usepackage[T1]{fontenc}
\usepackage[utf8]{inputenc}
\usepackage{color}
\usepackage{array}
\usepackage{wrapfig}
\usepackage{booktabs}
\usepackage{amsmath}
\usepackage{amsthm}
\usepackage{graphicx}
\usepackage[numbers]{natbib}
\usepackage{microtype}
\usepackage[unicode=true,
 bookmarks=true,bookmarksnumbered=true,bookmarksopen=true,bookmarksopenlevel=1,
 breaklinks=false,pdfborder={0 0 0},pdfborderstyle={},backref=false,colorlinks=true]
 {hyperref}

\makeatletter

\providecommand{\tabularnewline}{\\}

\theoremstyle{definition}
\newtheorem*{defn*}{\protect\definitionname}
\theoremstyle{plain}
\newtheorem{thm}{\protect\theoremname}
\theoremstyle{definition}
\newtheorem{defn}[thm]{\protect\definitionname}

\usepackage{xcolor}
\definecolor{mycol}{rgb}{0,0,0.65}
\hypersetup{
    colorlinks,
    linkcolor={mycol},
    citecolor={mycol},
    urlcolor={mycol}
}

\usepackage[final]{neurips_2019}

\usepackage{url}
\usepackage{amsfonts,bm,amsmath}
\usepackage{nicefrac}

\author{Justin Domke$^1$ and Daniel Sheldon$^{1,2}$\\
$^1$ College of Information and Computer Sciences, University of Massachusetts Amherst \\
$^2$ Department of Computer Science, Mount Holyoke College
}

\usepackage{algorithm}
\usepackage{wrapfig}

\usepackage{thmtools}
\usepackage{thm-restate}

\usepackage[shortlabels]{enumitem}
\setlist[itemize]{leftmargin=18pt}

\DeclareMathAlphabet{\mathbfsf}{\encodingdefault}{\sfdefault}{bx}{n}
\newcommand{\upgreektemplate}[2]{#2{
\renewcommand{\alpha}{\upalpha}
\renewcommand{\beta}{\upbeta}
\renewcommand{\theta}{\uptheta}
\renewcommand{\gamma}{\upgamma}
\renewcommand{\lambda}{\uplambda}
\renewcommand{\delta}{\updelta}
\renewcommand{\nu}{\bm \upnu}
\renewcommand{\phi}{\upphi}
\renewcommand{\zeta}{\upzeta}
\renewcommand{\omega}{\bm \upomega}
\renewcommand{\Lambda}{\Uplambda}
\renewcommand{\Gamma}{\Upgamma}
\renewcommand{\Delta}{\Updelta}
\renewcommand{\Theta}{\Uptheta}
#1
}}
\newcommand{\upgreek}[1]{\upgreektemplate{#1}{\mathsf}}
\newcommand{\bupgreek}[1]{\upgreektemplate{#1}{\mathbfsf}}
\usepackage{upgreek}

\usepackage[nameinlink]{cleveref}
\AtBeginDocument{%
\let\ref\Cref 
}
\Crefname{equation}{Eq.}{Eqs.}
\Crefname{figure}{Fig.}{Figs.}
\Crefname{lem1}{Lem.}{Lems.}
\Crefname{thm1}{Thm.}{Thms.}
\Crefname{clm1}{Claim}{Claims}
\Crefname{section}{Sec.}{Secs.}

\definecolor{darkgreen}{rgb}{0.0,0.5,0.0}
\definecolor{darkred}{rgb}{0.5,0.0,0.0}

\usepackage{cancel}

\let\mid\vert 

\makeatother

\providecommand{\definitionname}{Definition}
\providecommand{\theoremname}{Theorem}

\begin{document}
\global\long\def\argmin{\operatornamewithlimits{argmin}}%

\global\long\def\argmax{\operatornamewithlimits{argmax}}%

\global\long\def\prox{\operatornamewithlimits{prox}}%

\global\long\def\diag{\operatorname{diag}}%

\global\long\def\lse{\operatorname{lse}}%

\global\long\def\R{\mathbb{R}}%

\global\long\def\E{\operatornamewithlimits{\mathbb{E}}}%

\global\long\def\P{P}%

\global\long\def\V{\operatornamewithlimits{\mathbb{V}}}%

\global\long\def\N{\mathcal{N}}%

\global\long\def\L{\mathcal{L}}%

\global\long\def\C{\mathbb{C}}%

\global\long\def\tr{\operatorname{tr}}%

\global\long\def\norm#1{\left\Vert #1\right\Vert }%

\global\long\def\norms#1{\left\Vert #1\right\Vert ^{2}}%

\global\long\def\pars#1{\left(#1\right)}%

\global\long\def\pp#1{(#1)}%

\global\long\def\bracs#1{\left[#1\right]}%

\global\long\def\bb#1{[#1]}%

\global\long\def\verts#1{\left\vert #1\right\vert }%

\global\long\def\Verts#1{\left\Vert #1\right\Vert }%

\global\long\def\angs#1{\left\langle #1\right\rangle }%

\global\long\def\KL#1{[#1]}%

\global\long\def\KL#1#2{\mathrm{KL}\bracs{#1\middle\Vert#2}}%

\global\long\def\div{\text{div}}%

\global\long\def\erf{\text{erf}}%

\global\long\def\vvec{\text{vec}}%

\global\long\def\b#1{\bm{#1}}%

\global\long\def\r#1{\upgreek{#1}}%

\global\long\def\br#1{\bupgreek{\bm{#1}}}%

\global\long\def\T{t}%

\global\long\def\ep{\bm{\varepsilon}}%

\global\long\def\rep{\bm{\upvarepsilon}}%

\global\long\def\marker{\checkmark}%

\global\long\def\zo{\b z^{*}}%

\global\long\def\z{\b z}%

\global\long\def\x{\b x}%

\global\long\def\u{\b u}%

\global\long\def\ur{\r u}%

\global\long\def\qmc{Q}%

\global\long\def\pmc{P^{\textrm{MC}}}%

\global\long\def\eqind{\hspace{0.5cm}}%

\global\long\def\ed{\overset{d}{=}}%

\global\long\def\elbo#1#2{\mathrm{ELBO}\bracs{#1\middle\Vert#2}}%

\global\long\def\sm#1{\mathrm{Norm}\pars{#1}}%

\global\long\def\Sm#1#2{\sm{#1\mid#2}}%

\global\long\def\tform#1{\mathcal{T}\pars{#1}}%

\global\long\def\Tform#1#2{\tform{#1,#2}}%

\global\long\def\mean{\mathrm{mean}}%

\global\long\def\zr{\r z}%

\global\long\def\wr{\r{\omega}}%

\global\long\def\ccr{\r{\chi}}%

\global\long\def\vr{\r{\nu}}%

\title{Divide and Couple: Using Monte Carlo Variational Objectives for Posterior
Approximation}

\maketitle
\maketitle
\begin{abstract}
Recent work in variational inference (VI) uses ideas from Monte Carlo
estimation to tighten the lower bounds on the log-likelihood that
are used as objectives. However, there is no systematic understanding
of how optimizing different objectives relates to approximating the
posterior distribution. Developing such a connection is important
if the ideas are to be applied to inference---i.e., applications
that require an approximate posterior and not just an approximation
of the log-likelihood. Given a VI objective defined by a Monte Carlo
estimator of the likelihood, we use a \textquotedbl divide and couple\textquotedbl{}
procedure to identify augmented proposal and target distributions.
The divergence between these is equal to the gap between the VI objective
and the log-likelihood. Thus, after maximizing the VI objective, the
augmented variational distribution may be used to approximate the
posterior distribution.
\end{abstract}

\section{Introduction}

Variational inference (VI) is a leading approximate inference method
in which a posterior distribution $p(z\mid x)$ is approximated by
a simpler distribution $q(z)$ from some approximating family. The
procedure to select $q$ is based on the decomposition that \citep{Saul_1996_MeanFieldTheory}\vspace{-0.25cm}

\begin{equation}
\log p(x)=\E_{q(\zr)}\bigg[\log\frac{p(\zr,x)}{q(\zr)}\bigg]+\text{KL}[q(\zr)\|p(\zr\mid x)].\label{eq:ELBO-decomp}
\end{equation}
The first term is the \emph{evidence lower bound} ($\text{ELBO}$)
\citep{Blei_2017_VariationalInferenceReview}. Selecting $q$ to maximize
the ELBO tightens the lower bound on $\log p(x)$ and simultaneously
minimizes the KL-divergence in the second term. This dual view is
important because minimizing the KL-divergence justifies using $q$
to approximate the posterior for making predictions.

Recent work has investigated tighter objectives \citep{Burda_2015_ImportanceWeightedAutoencoders,Naesseth_2018_VariationalSequentialMonte,Maddison_2017_FilteringVariationalObjectives,Le_2018_AutoEncodingSequentialMonte,Nowozin_2018_DEBIASINGEVIDENCEAPPROXIMATIONS,Rainforth_2018_TighterVariationalBounds},\marginpar{}
based on the following principle: Let $R$ be an estimator of the
likelihood---i.e., a nonnegative random variable with $\E R=p(x)$.
By Jensen's inequality, $\log p(x)\geq\E\log R,$ so $\log R$ is
a stochastic lower bound on the log-likelihood. Parameters of the
estimator can be optimized to tighten the bound. Standard VI is the
case when $R=p(\zr,x)/q(\zr)$ and $\zr\sim q$, which is parameterized
in terms of $q$. Importance-weighted autoencoders \citep[IWAEs;][]{Burda_2015_ImportanceWeightedAutoencoders}
essentially use $R=\frac{1}{M}\sum_{m=1}^{M}p(\zr_{m},x)/q(\zr_{m})$
where $\zr_{1}\ldots\zr_{M}\sim q$ are iid. Sequential Monte Carlo
(SMC) also gives a variational objective \citep{Naesseth_2018_VariationalSequentialMonte,Maddison_2017_FilteringVariationalObjectives,Le_2018_AutoEncodingSequentialMonte}.
The principle underlying these works is that likelihood estimators
that are more concentrated lead to tighter bounds, because the gap
in Jensen's inequality is smaller. To date, the main application has
been for learning parameters of the generative model $p$.

\begin{figure}
\begin{centering}
\includegraphics[width=0.5\columnwidth]{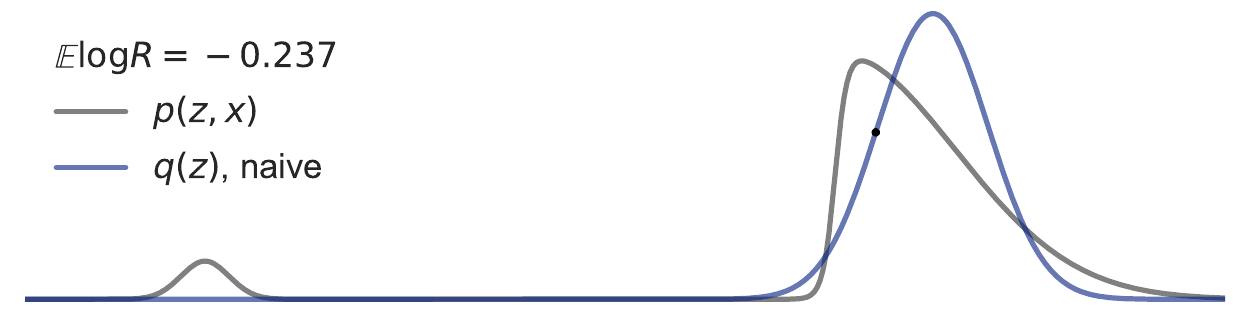}\includegraphics[width=0.5\columnwidth]{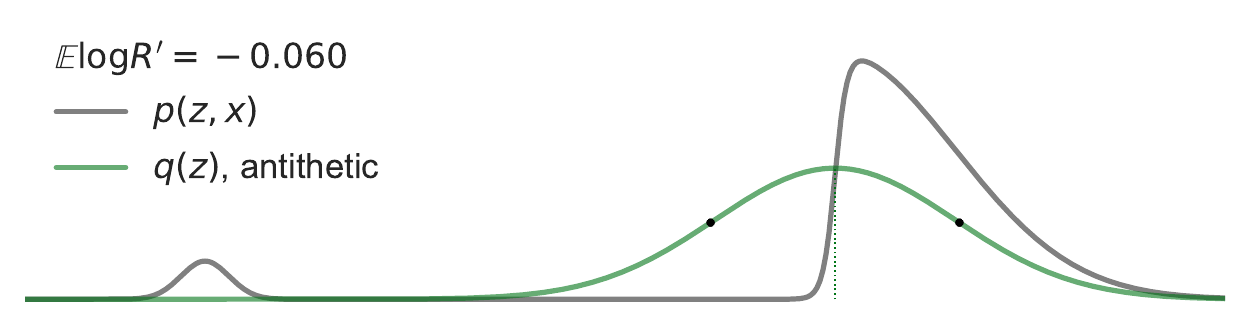}
\par\end{centering}
\caption{Left: Naive VI on the running example. Right: A tighter bound using
antithetic sampling.\label{fig:Naive-VI-running}}
\end{figure}

Our key question is: \emph{what are the implications of modified variational
objectives for probabilistic inference}? \ref{eq:ELBO-decomp} relates
the standard ELBO to $\text{KL}\bb{q\Vert p}$, which justifies using
$q$ for posterior inference. If we optimize a variational objective
obtained from a different estimator, does this still correspond to
minimizing some KL-divergence? It has been shown \citep{Cremer_2017_ReinterpretingImportanceWeightedAutoencoders,Naesseth_2018_VariationalSequentialMonte,Domke_2018_ImportanceWeightingVariational}
that maximizing the IWAE objective corresponds to minimizing (an upper
bound to) $\text{KL}[q^{\text{IS}}(\zr)\|p(\zr\mid x)]$, where $q^{\text{IS}}$
is a version of $q$ that is \begin{wrapfigure}[7]{o}{0.5\columnwidth}%
\begin{centering}
\vspace{-0.2cm}
\includegraphics[width=0.5\columnwidth]{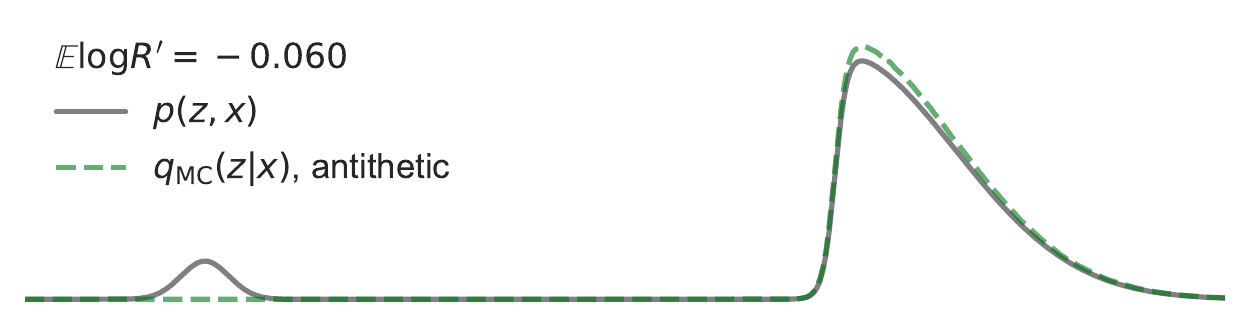}\vspace{-0.2cm}
\par\end{centering}
\caption{A kernel density approximation of $\protect\qmc\protect\pp z$ for
the antithetic estimator in \ref{fig:Naive-VI-running}.\label{fig:antithetic-running-samps}}
\end{wrapfigure}%
\textquotedbl corrected\textquotedbl{} toward $p$ using importance
sampling; this justifies using $q^{\text{IS}}$ to approximate the
posterior. \citet{Naesseth_2018_VariationalSequentialMonte} also
show that performing VI with an SMC objective can be seen as minimizing
(an upper bound to) a divergence from the SMC sampling distribution
$q^{\text{SMC}}(z)$ to $p(z|x)$.\marginpar{} For an arbitrary estimator,
however, there is little understanding.

We establish a deeper connection between variational objectives and
approximating families. Given a non-negative Monte Carlo (MC) estimator
$R$ such that $\E R=p(x)$, we show how to find a distribution $\qmc(z)$
such that the divergence between $\qmc(z)$ and $p(z\mid x)$ is at
most the gap between $\E\log R$ and $\log p(x)$. Thus, better estimators
mean better posterior approximations. The approximate posterior $\qmc\pp z$
can be found by a two-step \textquotedbl divide and couple\textquotedbl{}
procedure. The ``divide'' step follows \citet{Le_2018_AutoEncodingSequentialMonte}
and connects maximizing $\E\log R$ to minimizing a divergence between
two distributions, but not necessarily involving $p(z|x)$. The ``couple''
step shows how to find $Q\pp z$ such that the divergence is an upper
bound to $\KL{Q(\zr)}{p(\zr|x)}$. We show how a range of ideas from
the statistical literature--- such as antithetic sampling, stratified
sampling, and quasi Monte Carlo \citep{Owen_2013_MonteCarlotheory}---can
produce novel variational objectives; then, using the divide and
couple framework, we describe efficient-to-sample approximate posteriors
$\qmc\pp z$ for each of these objectives. We contribute mathematical
tools for deriving new estimators and approximating distributions
within this framework. Experiments show that the novel objectives
enabled by this framework can lead to improved likelihood bounds and
posterior approximations.

There is a large body of work using MC techniques to reduce the variance
of \emph{gradient estimators} of the standard variational objective
\citep{Ranganath_2014_BlackBoxVariational,Buchholz_2018_QuasiMonteCarloVariational,Miller_2017_ReducingReparameterizationGradient,Titsias_2015_LocalExpectationGradients,Geffner_2018_UsingLargeEnsembles}.
The aims of this paper are different: we use MC techniques to change
$R$ to get a tighter objective.

\section{Setup and Motivation}

Imagine we have some distribution $p(z,x)$. After observing data
$x$, we wish to approximate the posterior $p(z|x)$. Traditional
VI tries to both bound $\log p(x)$ and approximate $p(z|x)$ using
the ``ELBO decomposition'' of \ref{eq:ELBO-decomp}. We already
observed that a similar lower bound can be obtained from any non-negative
random variable $R$ with $\E\ R=p(x),$ since by Jensen's inequality,
\[
\log p(x)\geq\E\ \log R.
\]
Traditional VI can be seen as defining $R=p\pp{\zr,x}/q\pp{\zr}$
for $\zr\sim q$ and then optimizing the parameters of $q$ to maximize
$\E\log R$. Many other estimators $R$ of $p(x)$ can be designed
and their parameters optimized to make $\E\log R$ as large as possible.
We want to know: what relationship does this have to approximating
the posterior $p(z|x)$?

\subsection{Example\label{subsec:RunningExample}}

\Cref{fig:Naive-VI-running} shows a one dimensional target distribution
$p(z,x)$ as a function of $z$, and the Gaussian $q(z)$ obtained
by standard VI, i.e. maximizing $\E\log R$ for $R=p(\zr,x)/q\pp{\zr}$.
The resulting bound is $\E\log R\approx-0.237$, while the true value
is $\log p(x)=0.$ By \ref{eq:ELBO-decomp}, the KL-divergence from
$q(z)$ to $p(z|x)$ is $0.237$. Tightening the likelihood bound
has made $q$ close to $p.$

A Gaussian cannot exactly represent the main mode of $p(z,x)$, since
it is asymmetric. Antithetic sampling \citep{Owen_2013_MonteCarlotheory}
can exploit this. Define\vspace{-0.1cm}
\begin{equation}
R'=\frac{1}{2}\pars{\frac{p\pp{\zr,x}+p\pp{T\pp{\zr},x}}{q\pp{\zr}}},\ \zr\sim q,\label{eq:antithetic}
\end{equation}
where $T\pp z=\mu-\pp{z-\mu}$ is $z$ ``reflected'' around the
mean $\mu$ of $q.$ This is a valid estimator since $q\pp z$ is
constant under reflection. Tightening this bound over Gaussians $q$
gives $\E\log R'\approx-0.060$. This is better, intuitively, since
the right half of $q$ is a good match to the main mode of $p$, i.e.
since $q(z)\approx\frac{1}{2}p(z,x)$ for $z$ in that region. 

What about $p(z|x)$? It is \emph{not} true that antithetic sampling
gives a $q\pp z$ with lower divergence (it is around 7.34). After
all, naive VI already found the optimal Gaussian. Is there some \emph{other}
distribution that is close to $p\pp{z|x}$? How can we find it? These
questions motivate this paper.

\subsection{Notation and Conventions}

We use sans-serif font for random variables. All estimators $R$ may
depend on the input $x$, but we suppress this for simplicity. Similarly,
$a\pp{z|\omega}$ may depend on $x$. Proofs for all results are in
the supplement. Objects such as $\pmc$, $\qmc$, $p$, $q$ are distributions
of random variables and will be written like densities: $Q(\omega)$,
$p(z,x)$. However, the results are more general: the supplement includes
a more rigorous version of our main results using probability measures.
We write densities with Dirac delta functions. These are not Lebesgue-integrable,
but can be interpreted unambiguosuly as Dirac measures: e.g., $a(z|\omega)=\delta(z-\omega)$
means the conditional distribution of $\zr$ given $\wr=\omega$ is
the Dirac measure $\delta_{\omega}$. Throughout, $x$ is fixed, so
$p(z,x)$ and $\pmc(\omega,z,x)$ are unnormalized distributions over
the other variables, and $p(z|x)$ and $\pmc(\omega,z|x)$ are the
corresponding normalized distributions.

\section{The Divide-and-Couple Framework\label{sec:The-Divide-and-Couple-Framework}}

In this section we identify a correspondence between maximizing a
likelihood bound and posterior inference for general non-negative
estimators using a two step ``divide'' and then ``couple'' construction.

\subsection{Divide}

Let $R(\wr)$ be a positive function of $\wr\sim Q(\omega)$ such
that $\E_{Q(\wr)}R(\wr)=p(x)$, i.e., $R$ is an unbiased likelihood
estimator with sampling distribution $Q(\omega)$. The ``divide''
step follows \citep[Claim 1]{Le_2018_AutoEncodingSequentialMonte}:
we can interpret $\E_{Q(\wr)}\log R(\wr)$ as an ELBO by defining
$\pmc$ so that $R\pp{\omega}=\pmc\pp{\omega,x}/\qmc\pp{\omega}$.
That is, $\pmc$ and $\qmc$ ``divide'' to produce $R$. Specifically:

\begin{restatable}{lem1}{trivdivide}\label{lem:lemdivide}Let $\wr$
be a random variable with distribution $Q(\omega)$ and let $R(\wr)$
be a positive estimator such that $\E_{Q(\wr)}R(\wr)=p(x)$. Then
\begin{eqnarray*}
\pmc\pp{\omega,x} & = & Q\pp{\omega}R\pp{\omega}
\end{eqnarray*}
is an unnormalized distribution over $\omega$ with normalization
constant $p(x)$ and $R\pp{\omega}=\pmc\pp{\omega,x}/\qmc\pp{\omega}$
for $Q\pp{\omega}>0$. Furthermore as defined above,
\begin{equation}
\log p(x)=\E_{Q\pp{\wr}}\log R\pp{\wr}+\KL{\qmc(\wr)}{\pmc(\wr|x)}.\label{eq:divide-elbo}
\end{equation}
\end{restatable}While this shows that it is easy to connect a stochastic
likelihood bound to minimizing a KL-divergence, this construction
alone is not useful for probabilistic inference, since neither $\qmc\pp{\omega}$
nor $\pmc\pp{\omega,x}$ make any reference to $z$. Put another way:
Even if $\KL{\qmc(\wr)}{\pmc(\wr|x)}$ is small, so what? This motivates
the coupling step below.\marginpar{} More generally, $\text{\ensuremath{\pmc}}$
is defined by letting $R=d\pmc/dQ$ be the Radon-Nikodym derivative,
a change of measure from $Q$ to $\pmc$; see supplement.

\subsection{Couple}

If the distributions identified in the above lemma are going to be
useful for approximating $p(z|x)$, they must be connected somehow
to $z$. In this section, we suggest \emph{coupling} $\pmc\pp{\omega,x}$
and $p\pp{z,x}$ into some new distribution $\pmc\pp{\omega,z,x}$
with $\pmc(z,x)=p(z,x)$. In practice, it is convenient to describe
couplings via a conditional distribution $a\pp{z|\omega}$ that augments
$\pmc\pp{\omega,x}$. There is a straightforward condition for when
$a$ is valid: for the augmented distribution $\pmc(z,\omega,x)=\pmc(\omega,x)a(z|\omega)$
to be a valid coupling, we require that $\int\pmc(\omega,x)a(z|\omega)d\omega=p(z,x)$.
An equivalent statement of this requirement is as follows.
\begin{defn*}
\label{def:valid-pair}An estimator $R\pp{\omega}$ and a distribution
$a\pp{z|\omega}$ are a \emph{valid estimator-coupling pair} under
distribution $Q(\omega)$ if
\begin{equation}
\E_{Q\pp{\wr}}R(\wr)a(z|\wr)=p(z,x).\label{eq:valid_est_coupling_pair}
\end{equation}
\end{defn*}
The definition implies that $\E_{Q(\wr)}R(\wr)=p(x)$ as may be seen
by integrating out $z$ from both sides. That is, the definition implies
that $R$ is an unbiased estimator for $p\pp x.$

Now, suppose we have a valid estimator/coupling pair and that $R$
is a good (low variance) estimator. How does this help us to approximate
the posterior $p\pp{z|x}$? The following theorem gives the ``divide
and couple'' framework.\footnote{With measures instead of densities would write $Q(z\in B,\omega\in A)=\int_{A}a(z\in B|\omega)dQ(\omega)$
where $a$ is a Markov kernel; see supplement.}

\begin{restatable}{thm1}{divandcouple}\label{thm:div-and-couple}Suppose
that $R\pp{\omega}$ and $a\pp{z|\omega}$ are a valid estimator-coupling
pair under $Q(\omega)$. Then,
\begin{eqnarray}
\qmc\pp{z,\omega} & = & Q\pp{\omega}a\pp{z|\omega},\label{eq:divcouple-q}\\
\pmc\pp{z,\omega,x} & = & Q\pp{\omega}R\pp{\omega}a\pp{z|\omega},\label{eq:divcouple-p}
\end{eqnarray}
are valid distributions, $\pmc\pp{z,x}=p\pp{z,x}$, and
\begin{equation}
\log p(x)=\E_{Q\pp{\wr}}\log R\pp{\wr}+\KL{\qmc\pp{\zr}}{p\pp{\zr|x}}+\KL{\qmc\pp{\wr|\zr}}{\pmc\pp{\wr|\zr,x}}.\label{eq:div_and_couple_final_decomp}
\end{equation}

\end{restatable}

The final term in \ref{eq:div_and_couple_final_decomp} is a conditional
divergence \citep[Sec. 2.5]{Cover_2006_Elementsinformationtheory}.
This is the divergence over $\omega$ between $\qmc\pp{\omega|z}$
and $\pmc\pp{\omega|z,x}$, averaged over $z$ drawn from $Q\pp z$.
At a high level, this theorem is proved by applying \ref{lem:lemdivide},
using the chain rule of KL-divergence, and simplifying using the fact
that $a$ is a valid coupling. The point of this theorem is: if $R$
is a good estimator, then $\E\log R$ will be close to $\log p\pp x.$
Since KL-divergences are non-negative, this means that the marginal
$Q\pp z$ must be close to the target $p\pp{z|x}.$ A coupling gives
us a way to \emph{transform} $Q\pp{\omega}$ so as to approximate
$p\pp{z|x}.$

To be useful, we need to access $\qmc\pp z$, usually by sampling.
We assume it is possible to sample from $Q\pp{\omega}$ since this
is part of the estimator. The user must supply a routine to sample
from $a(z|\omega$). This is why the ``trivial coupling'' of $a\pp{z|\omega}=p\pp{z|x}$
is not helpful --- if the user could sample from $p\pp{z|x}$, the
inference problem is already solved! Some estimators may be pre-equipped
with a method to approximate $p(z|x)$. For example, in SMC, $\wr$
include particles $z_{i}$ and weights $w_{i}$ such that selecting
a particle in proportion to its weight approximates $p(z|x)$. This
can be seen as a coupling $a(z|\omega)$, which provides an alternate
interpretation of the divergence bounds of \citep{Naesseth_2018_VariationalSequentialMonte}.
The divide and couple framework is also closely related to extended-space
MCMC methods \citep{Andrieu_2010_ParticleMarkovchain}, which also
use extended target distributions that admit $p(z,x)$ as a marginal;
see also \ref{sec:Conclusions}. However, in these methods, the estimators
also seem to come with \textquotedblleft obvious\textquotedblright{}
couplings.  There is no systematic understanding of how to derive
couplings for general estimators.

\subsection{Example\label{subsec:Running-Example-2}}

Consider again the antithetic estimator from \ref{eq:antithetic}.
We saw before that the antithetic estimator gives a tighter variational
bound than naive VI under the distribution in \ref{fig:Naive-VI-running}.
However, that distribution is \emph{less} similar to the target than
the one from naive VI. To reflect this (and match our general notation)
we now\footnote{With this notation, \ref{eq:antithetic} becomes $R=\frac{1}{2}\frac{p\pp{\wr,x}+p\pp{T\pp{\wr},x}}{Q\pp{\wr}}$
where $\wr\sim Q.$ Here and in the rest of this paper, when $p\pp{\cdot,\cdot}$
has two arguments, the first always plays the role of $z$.} write $\wr\sim Q$ (instead of $\zr\sim q$).

Now, we again ask: Since $Q\pp{\omega}$ is a poor approximation of
the target, can we find some other distribution that is a good approximation?
Consider the coupling distribution
\[
a\pp{z|\omega}=\pi\pp{\omega}\ \delta\pp{z-\omega}+\pp{1-\pi\pp{\omega}}\ \delta\pp{z-T\pp{\omega}},\ \ \ \ \pi\pp{\omega}=\frac{p(\omega,x)}{p(\omega,x)+p(T\pp{\omega},x)}.
\]
Intuitively, $a\pp{z|\omega}$ is supported only on $z=\omega$ and
on $z=T\pp{\omega}$, with probability proportional to the target
distribution. It is simple to verify (by substitution, see \ref{claim:antitheticrunningvalidpair}
in supplement) that $R$ and $a$ form a valid estimator-coupling
pair. Thus, the augmented variational distribution is $Q\pp{\omega,z}=Q\pp{\omega}a\pp{z|\omega}.$
To sample from this, draw $\omega\sim Q$ and select $z=\omega$ with
probability $\pi\pp{\omega}$, or $z=T(\omega)$ otherwise. The marginal
$Q(z)$ is shown in \Cref{fig:antithetic-running-samps}. This is
a much better approximation of the target than naive VI.

\section{Deriving Couplings\label{sec:Deriving-Couplings}}

\ref{thm:div-and-couple} says that if $\E_{Q(\wr)}\log R(\wr)$
is close to $\log p(x)$ and you have a tractable coupling\emph{ $a(z|\omega)$},
then drawing $\omega\sim Q\pp{\omega}$ and then $z\sim a\pp{z|\omega}$
yields samples from a distribution $Q\pp z$ close to $p\pp{z|x}.$
But how can we find a tractable coupling?

\begin{table*}
\noindent \begin{centering}
\caption{\label{tab:methods-new}Variance reduction methods jointly transform
estimators and couplings. Take an estimator $R_{0}\protect\pp{\omega}$
with coupling $a_{0}\protect\pp{z|\omega}$, valid under $Q_{0}\protect\pp{\omega}$.
Each line shows a new estimator $R\protect\pp{\cdot}$ and coupling
$a\protect\pp{z|\cdot}$. The method to simulate $Q\protect\pp{\cdot}$
is described in the left column. Here, $F^{-1}$ is a mapping so that
if $\omega$ is uniform on $[0,1]^{d}$, then $F^{-1}(\omega)$ has
density $Q_{0}(\omega)$.}
\vspace{0.2cm}
\par\end{centering}
\noindent \centering{}%
\begin{tabular}{>{\raggedright}m{5.25cm}>{\raggedright}p{2.65cm}>{\raggedright}p{4.8cm}}
\toprule 
{\footnotesize{}Description} & {\footnotesize{}$R\pp{\cdot}$} & {\footnotesize{}$a\pp{z|\cdot}$}\tabularnewline
\midrule
\addlinespace[0.02cm]
\textbf{\footnotesize{}IID Mean}{\footnotesize\par}

{\footnotesize{}$\wr_{1}\cdots\wr_{M}\sim Q_{0}$ i.i.d.} & {\footnotesize{}${\displaystyle \frac{1}{M}\sum_{m=1}^{M}R_{0}\pp{\omega_{m}}}$} & {\footnotesize{}${\displaystyle {\displaystyle \frac{\sum_{m=1}^{M}R_{0}\pp{\omega{}_{m}}a_{0}\pp{z|\omega{}_{m}}}{\sum_{m=1}^{M}R_{0}\pp{\omega_{m}}}}}$}\tabularnewline
\addlinespace[0.07cm]
\textbf{\footnotesize{}Stratified Sampling}{\footnotesize\par}

{\footnotesize{}$\Omega_{1}\cdots\Omega_{M}$ partition $\Omega$,
$\wr_{1}^{m}\cdots\wr_{M_{n}}^{1}\sim Q_{0}$ restricted to~$\Omega_{m}$,
$\mu_{m}=Q_{0}\pp{\wr\in\Omega_{m}}$.} & {\footnotesize{}${\displaystyle \sum_{m=1}^{M}\frac{\mu_{m}}{N_{m}}\sum_{n=1}^{N_{m}}R_{0}\pars{\omega_{n}^{m}}}$} & {\footnotesize{}${\displaystyle {\displaystyle \frac{\sum_{m=1}^{M}\frac{\mu_{m}}{N_{m}}\sum_{n=1}^{N_{m}}R_{0}\pars{\omega_{n}^{m}}a_{0}\pp{z|\omega_{n}^{m}}}{\sum_{m=1}^{M}\frac{\mu_{m}}{N_{m}}\sum_{n=1}^{N_{m}}R_{0}\pars{\omega_{n}^{m}}}}}$}\tabularnewline
\addlinespace[0.07cm]
\textbf{\footnotesize{}Antithetic Sampling}{\footnotesize\par}

{\footnotesize{}$\wr\sim Q_{0}$. For all $m$, $T_{m}\pp{\wr}\overset{d}{=}\wr$.} & {\footnotesize{}${\displaystyle \frac{1}{M}\sum_{m=1}^{M}R_{0}\pars{T_{m}\pp{\omega}}}$} & {\footnotesize{}${\displaystyle \frac{\sum_{m=1}^{M}R_{0}\pp{T_{m}(\omega)}\ a_{0}\pars{z|T_{m}\pp{\omega}}}{\sum_{m=1}^{M}R\pp{T_{m}(\omega)}}}$}\tabularnewline
\addlinespace[0.07cm]
\textbf{\footnotesize{}Randomized Quasi Monte Carlo}{\footnotesize\par}

{\footnotesize{}$\wr\sim\mathrm{Unif}([0,1]^{d})$, $\bar{\omega}_{1},\cdots\bar{\omega}_{M}$
fixed, $T_{m}\pp{\omega}=F^{-1}\pars{\bar{\omega}_{m}+\omega\ \text{(mod 1})}$} & {\footnotesize{}${\displaystyle \frac{1}{M}{\displaystyle \sum_{m=1}^{M}R_{0}\pars{T_{m}\pp{\omega}}}}$} & {\footnotesize{}${\displaystyle \frac{\sum_{m=1}^{M}R_{0}\pp{T_{m}(\omega)}\ a_{0}\pars{z|T_{m}\pp{\omega}}}{\sum_{m=1}^{M}R_{0}\pp{T_{m}(\omega)}}}$}\tabularnewline
\addlinespace[0.07cm]
\textbf{\footnotesize{}Latin Hypercube Sampling}{\footnotesize\par}

{\footnotesize{}$\wr_{1},\cdots,\wr_{M}$ jointly sampled from Latin
hypercube \citep[Ch. 10.3]{Owen_2013_MonteCarlotheory}, $T=F^{-1}.$} & {\footnotesize{}${\displaystyle \frac{1}{M}{\displaystyle \sum_{m=1}^{M}R_{0}\pars{T\pp{\omega_{m}}}}}$} & {\footnotesize{}${\displaystyle \frac{\sum_{m=1}^{M}R_{0}\pp{T(\omega_{m})}\ a_{0}\pars{z|T\pp{\omega_{m}}}}{\sum_{m=1}^{M}R_{0}\pp{T(\omega_{m})}}}$}\tabularnewline
\bottomrule
\addlinespace[0.07cm]
\end{tabular}\vspace{-0.2cm}
\end{table*}

Monte Carlo estimators are often created recursively using techniques
that take some valid estimator $R$ and transform it into a new valid
estimator $R'$. These techniques (e.g. change of measure, Rao-Blackwellization,
stratified sampling) are intended to reduce variance. Part of the
power of Monte Carlo methods is that these techniques can be easily
combined. In this section, we extend some of these techniques to transform
valid \emph{estimator-coupling pairs} into new valid estimator-coupling
pairs. The hope is that the standard toolbox of variance reduction
techniques can be applied as usual, and the coupling is derived ``automatically''.

\ref{tab:methods-new} shows corresponding transformations of estimators
and couplings for several standard variance reduction techniques.
In the rest of this section, we will give two abstract tools that
can be used to create all the entries in this table. For concreteness,
we begin with a trivial ``base'' estimator-coupling pair. Take a
distribution $Q_{0}\pp{\wr}$ and let $R_{0}\pp{\omega}=p\pp{\omega,x}/Q_{0}\pp{\omega}$
and $a_{0}\pp{z|\omega}=\delta\pp{z-\omega}$ (the deterministic coupling).
It is easy to check that these satisfy \ref{eq:valid_est_coupling_pair}.\marginpar{}

\subsection{Abstract Transformations of Estimators and Couplings}

Our first abstract tool transforms an estimator-coupling pair on some
space $\Omega$ into another estimator-coupling pair on a space $\Omega^{M}\times\left\{ 1,\cdots,M\right\} $.
This can be thought of as having $M$ ``replicates'' of the $\omega$
in the original estimator, along with an extra integer-valued variable
that selects one of them. We emphasize that this result does not (by
itself) reduce variance --- in fact, $R$ has exactly the same distribution
as $R_{0}$.

\begin{restatable}{thm1}{splitting}\label{thm:splitting-estimators}Suppose
that $R_{0}\pp{\omega}$ and $a_{0}\pp{z|\omega}$ are a valid estimator-coupling
pair under $Q_{0}\pp{\omega}$. Let $Q\pp{\omega_{1},\cdots,w_{M},m}$
be any distribution such that if $\pp{\wr_{1},\cdots,\wr_{M},\r m}\sim Q$,
then $\wr_{\r m}\sim Q_{0}$. Then,
\begin{align}
R\pp{\omega_{1},\cdots,\omega_{M},m} & =R_{0}\pp{\omega_{m}}\label{eq:splitting-new-R}\\
a(z|\omega_{1},\cdots,\omega_{M},m) & =a_{0}(z|\omega_{m})\label{eq:splitting-new-a}
\end{align}
are a valid estimator-coupling pair under $Q\pp{\omega_{1},\cdots,w_{M},m}.$

\end{restatable}

\marginpar{}Rao-Blackwellization is a well-known way to transform
an estimator to reduce variance; we want to know how it affects \emph{couplings}.
Take an estimator $R_{0}\pp{\omega,\nu}$ with state space\marginpar{}
$\Omega\times N$ and distribution $Q_{0}\pp{\omega,\nu}$. A new
estimator $R(\omega)=\E_{Q_{0}(\vr\mid\omega)}R(\omega,\vr)$ that
analytically marginalizes out $\nu$ has the same expectation and
equal or lesser variance, by the Rao-Blackwell theorem. The following
result shows that if $R_{0}$ had a coupling, then it is easy to define
a new coupling for $R.$

\begin{restatable}{thm1}{raoblack}\label{thm:rao-black-new}Suppose
that $R_{0}\pp{\omega,\nu}$ and $a_{0}\pp{z|\omega,\nu}$ are a valid
estimator-coupling pair under $Q_{0}\pp{\omega,\nu}$. Then
\begin{alignat*}{1}
R\pp{\omega} & =\E_{Q_{0}\pp{\vr|\omega}}R_{0}\pars{\omega,\vr},\\
a\pp{z|\omega} & =\frac{1}{R\pp{\omega}}\E_{Q_{0}\pp{\vr|\omega}}\bracs{R_{0}\pars{\omega,\vr}a_{0}\pp{z|\omega,\vr}},
\end{alignat*}
are a valid estimator-coupling pair under $Q\pp{\omega}=\int Q_{0}\pp{\omega,\nu}d\nu$.\end{restatable}

\subsection{Specific Variance Reduction Techniques}

Each of the techniques in \ref{tab:methods-new} can be derived by
first applying \ref{thm:splitting-estimators} and then \ref{thm:rao-black-new}.
As a simple example, consider the IID mean. Suppose $R_{0}\pp{\omega}$
and $a_{0}\pp{z|\omega}$ are valid under $Q_{0}$. If we let $\wr_{1},\cdots,\wr_{M}\sim Q_{0}$
i.i.d. and $\r m$ uniform on $\left\{ 1,\cdots,M\right\} $ then
this satisfies the condition of \ref{thm:splitting-estimators} that
$\wr_{\r m}\sim Q_{0}.$ Thus we can define $R$ and $a$ as in \ref{eq:splitting-new-R}
and \ref{eq:splitting-new-a}. Applying Rao-Blackwellization to marginalize
out $\r m$ gives exactly the form for $R$ and $a$ shown in the
table. Details are in \ref{subsec:Specific-Variance-Reduction} of
the supplement

As another example, take stratified sampling. For simplicity, we assume
one sample in each strata ($N_{m}=1$). Suppose $\Omega_{1}\cdots\Omega_{M}$
partition the state-space\marginpar{} and let $\wr_{m}\sim Q_{0}(\wr\mid\wr\in\Omega_{m})$
and $\r m$ be equal to $m$ with probability $\mu_{m}=Q_{0}(\omega\in\Omega_{m})$.
It is again the case that $\wr_{\r m}\sim Q_{0}$, and applying \ref{thm:splitting-estimators}
and then \ref{thm:rao-black-new} produces the estimator-coupling
pair shown in the table. Again, details are in \ref{subsec:Specific-Variance-Reduction}
of the supplement.

\subsection{Example}

\begin{figure}
\begin{centering}
\includegraphics[width=0.5\columnwidth]{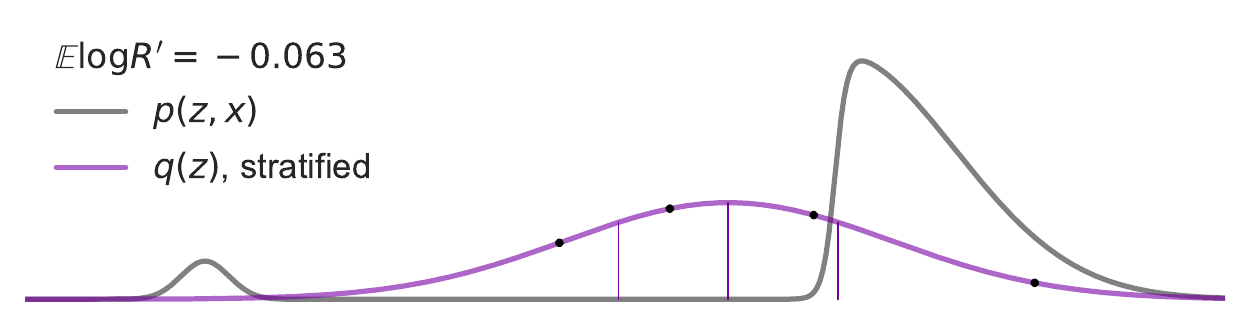}\includegraphics[width=0.5\columnwidth]{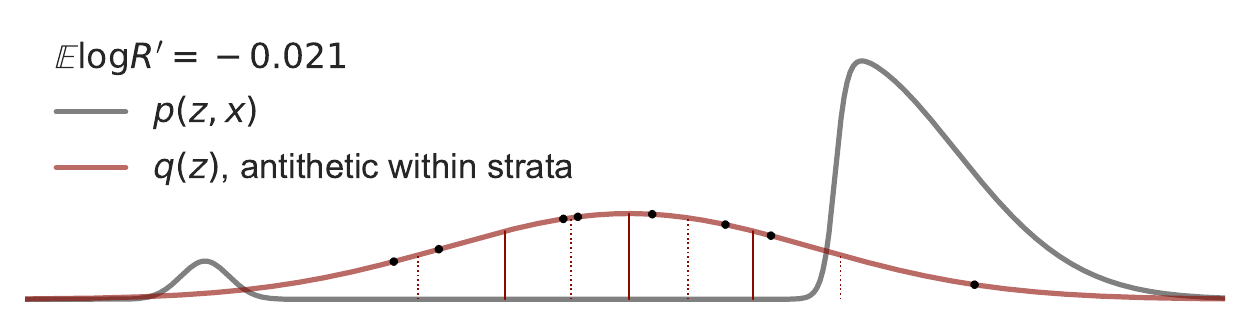}
\par\end{centering}
\caption{Stratified sampling and antithetic within stratified sampling on the
running example.\label{fig:strat-and-anti-strat-running}}
\end{figure}
\begin{figure}
\begin{centering}
\includegraphics[width=0.5\columnwidth]{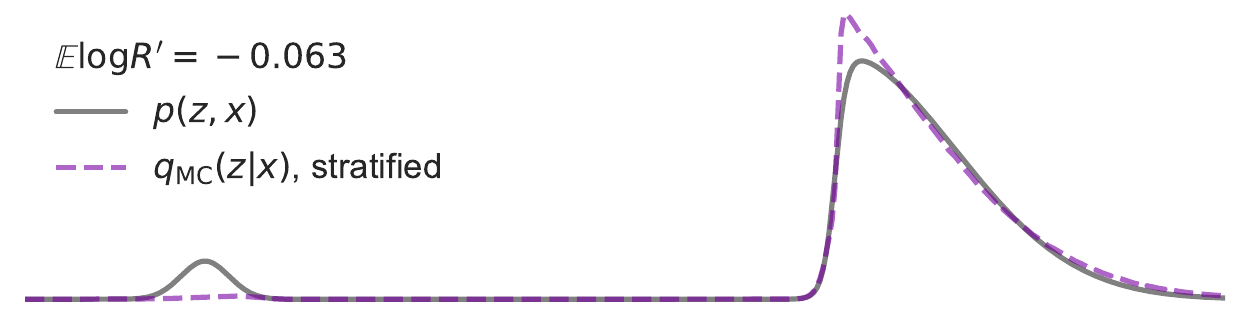}\includegraphics[width=0.5\columnwidth]{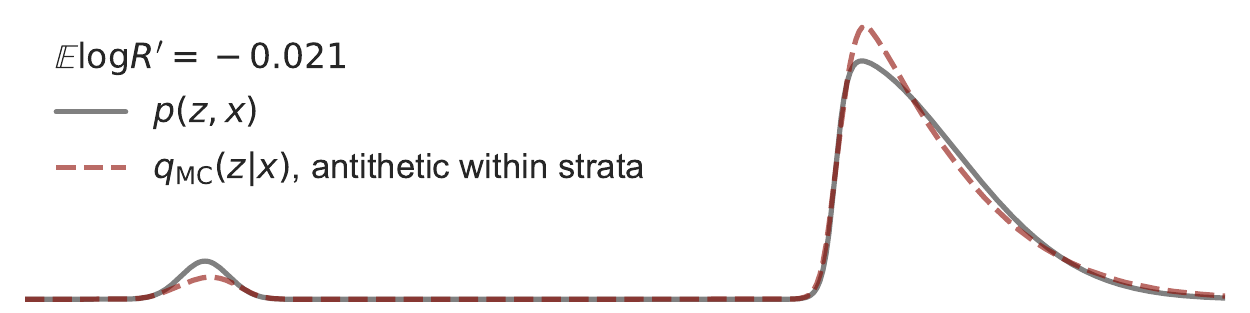}
\par\end{centering}
\caption{A kernel density approximation of $\protect\qmc\protect\pp{z|x}$
for the estimators in \ref{fig:strat-and-anti-strat-running}\label{fig:strat-and-anti-strat-running-samps}.}
\end{figure}

We return to the example from \ref{subsec:RunningExample} and \ref{subsec:Running-Example-2}.
\ref{fig:strat-and-anti-strat-running} shows the result of applying
stratified sampling to the standard VI estimator $R=p\pp{\zr,x}/q\pp{\zr}$
and then adjusting the parameters of $q$ to tighten the bound. The
bound is tighter than standard VI and slightly worse than antithetic
sampling.\marginpar{}

Why not combine antithetic and stratified sampling? \ref{fig:strat-and-anti-strat-running}
shows the result of applying antithetic sampling inside of stratified
sampling. Specifically, the estimator $R\pp{\omega^{m}}$ for each
stratum $m$ is replaced by $\frac{1}{2}\pars{R\pp{\omega^{m}}+R\pp{T_{m}(\omega^{m})}}$
where $T_{m}$ is a reflection inside the stratum that leaves the
density invariant. A fairly tight bound results. For all of antithetic
sampling (\ref{fig:antithetic-running-samps}), stratified sampling
(\ref{fig:strat-and-anti-strat-running}) and antithetic within stratified
sampling (\ref{fig:strat-and-anti-strat-running}) tightening $\E\log R$
finds $Q\pp{\omega}$ \marginpar{} such that all batches place some
density on $z$ in the main mode of $p$. Thus, the better sampling
methods permit a $q$ with some coverage of the left mode of $p$
while precluding the possibility that all samples in a batch are simultaneously
in a low-density region (which would result in $R$ near zero, and
thus a very low value for $\E\log R$). What do these estimators say
about $p(z|x)?$ \ref{fig:strat-and-anti-strat-running-samps} compares
the resulting $\qmc(z)$ for each estimator --- the similarity to
$p\pp{z|x}$ correlates with the likelihood bound.

\section{Implementation and Empirical Study\label{sec:Empirical-Study}}

\begin{figure}
\begin{centering}
\noindent\begin{minipage}[t]{1\columnwidth}%
\begin{center}
\includegraphics[width=0.21\textwidth]{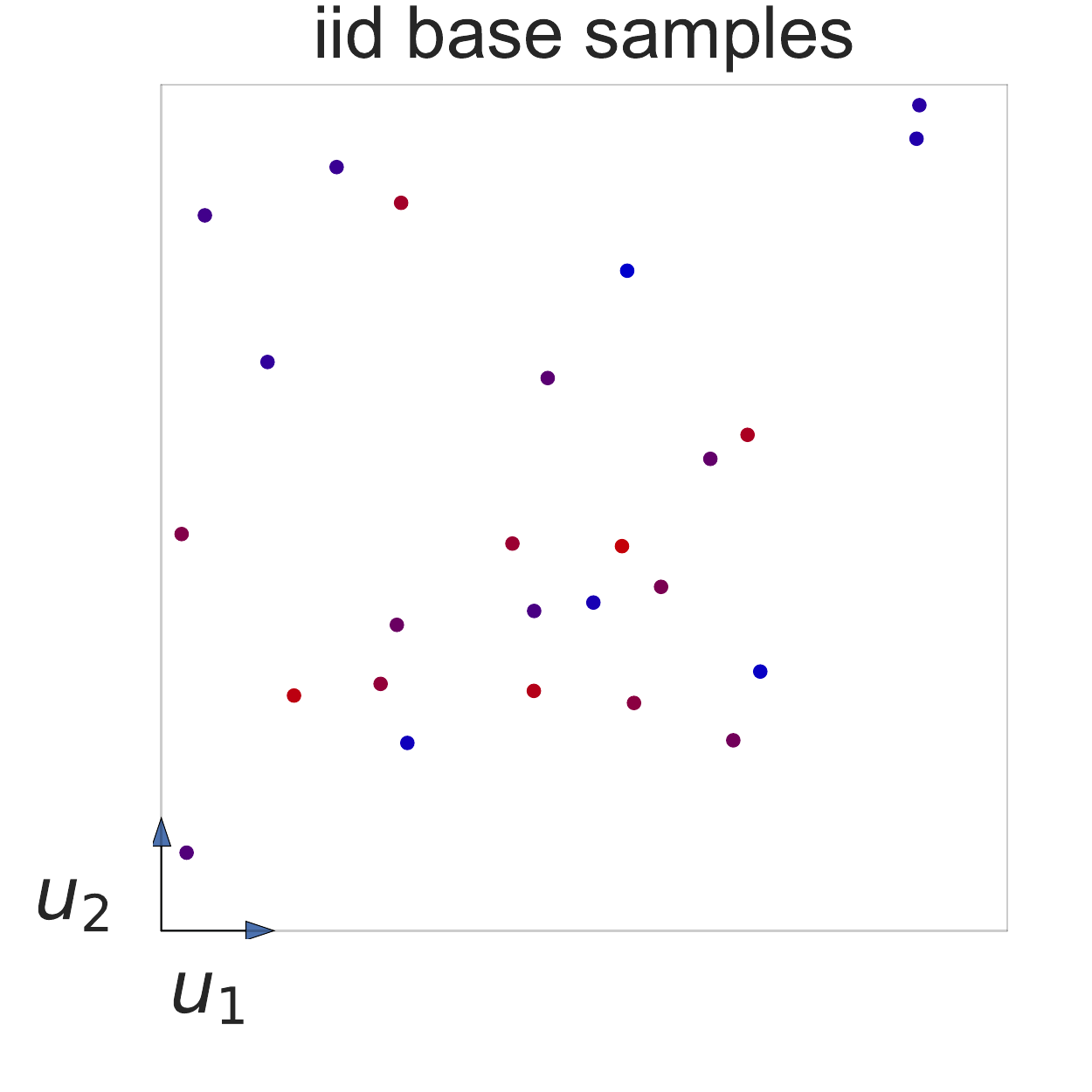}\hspace{-0.35cm}\includegraphics[width=0.21\textwidth]{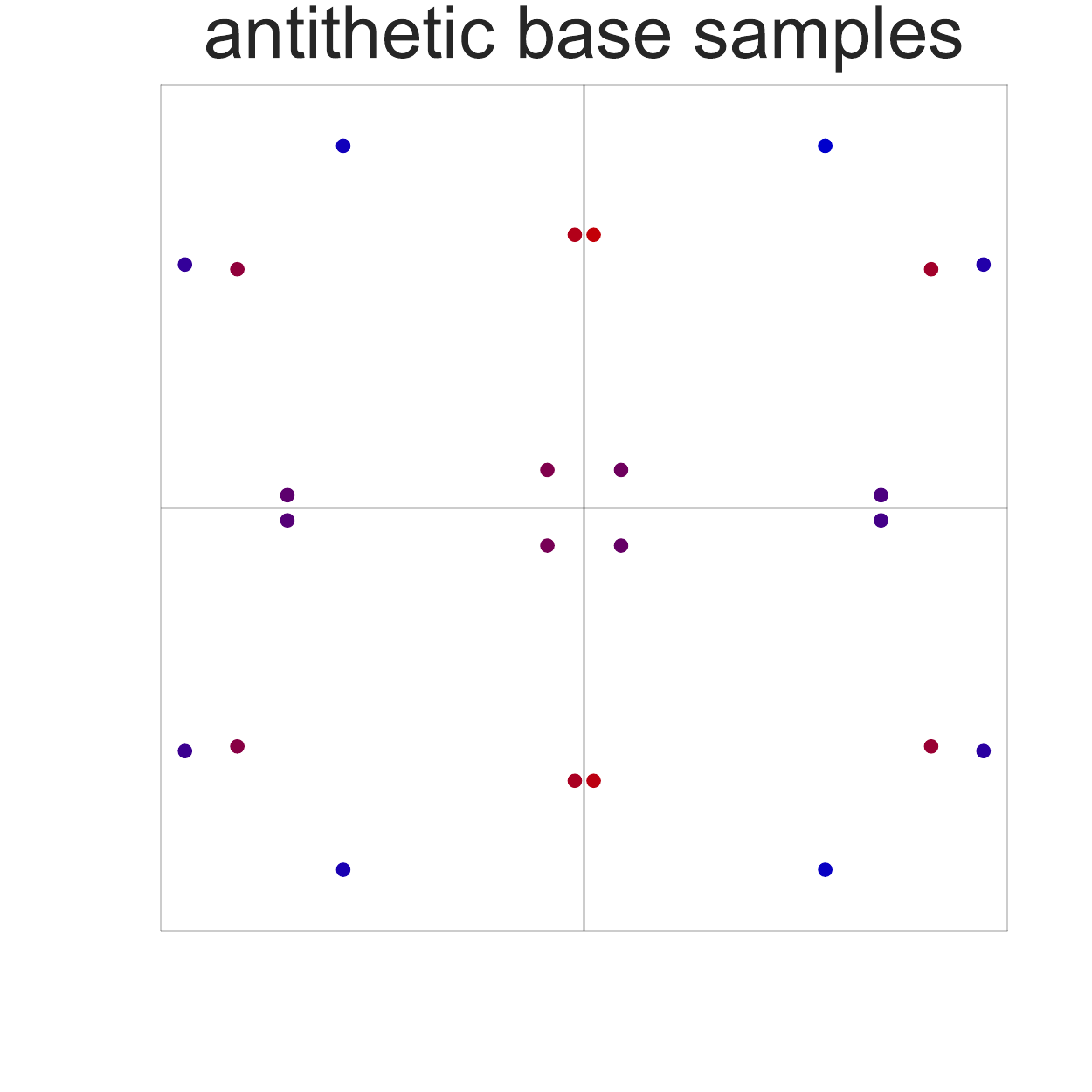}\hspace{-0.35cm}\includegraphics[width=0.21\textwidth]{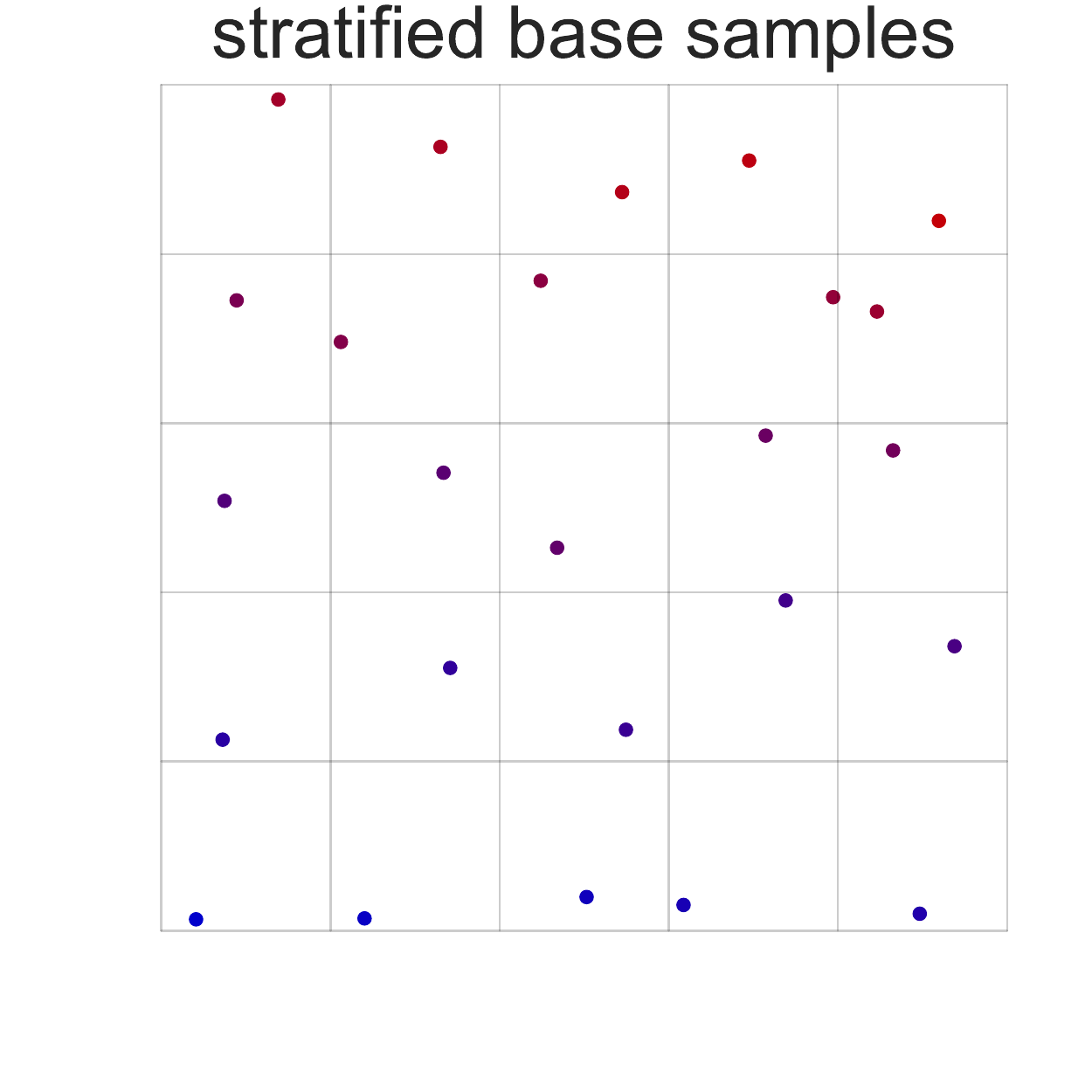}\hspace{-0.35cm}\includegraphics[width=0.21\textwidth]{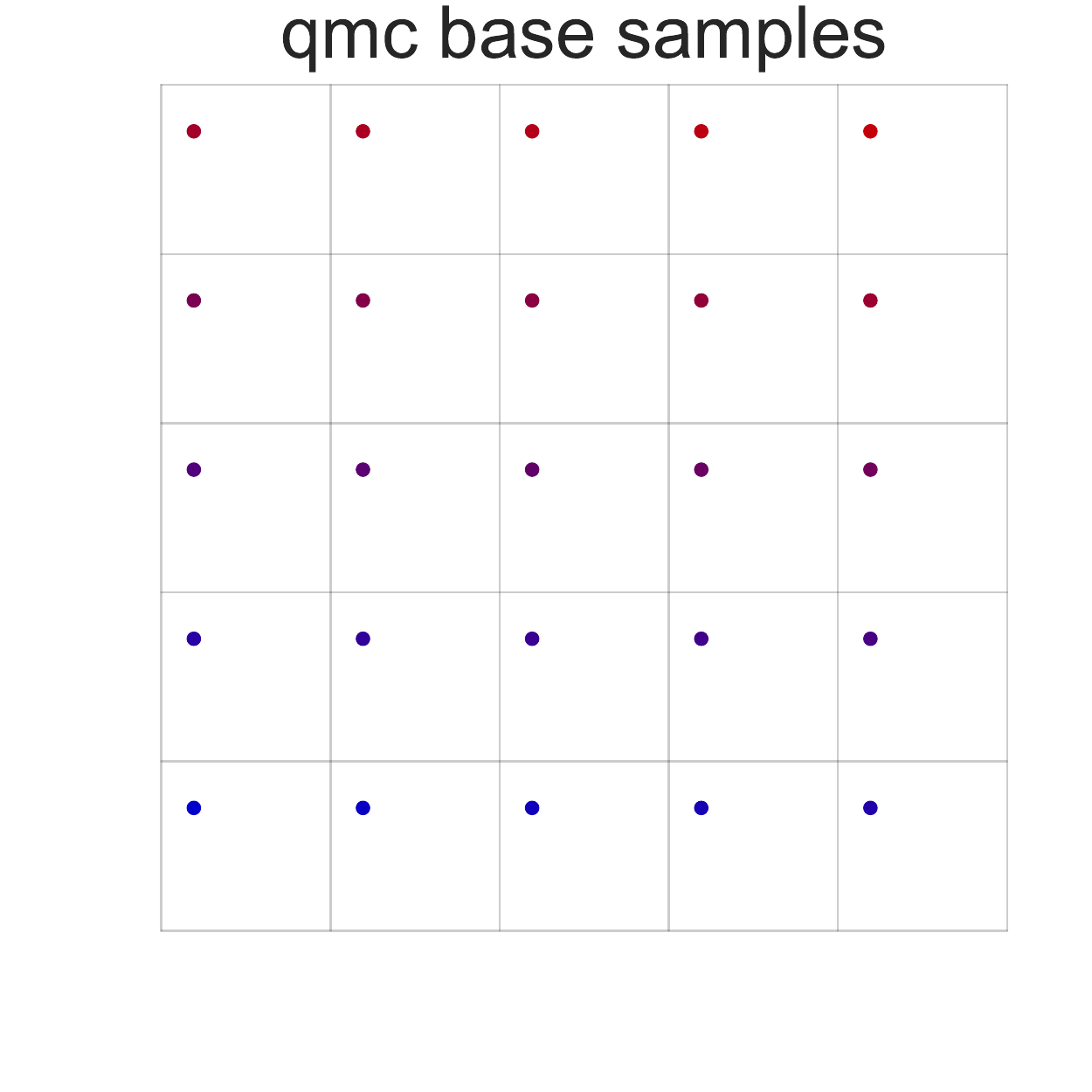}\hspace{-0.35cm}\includegraphics[width=0.21\textwidth]{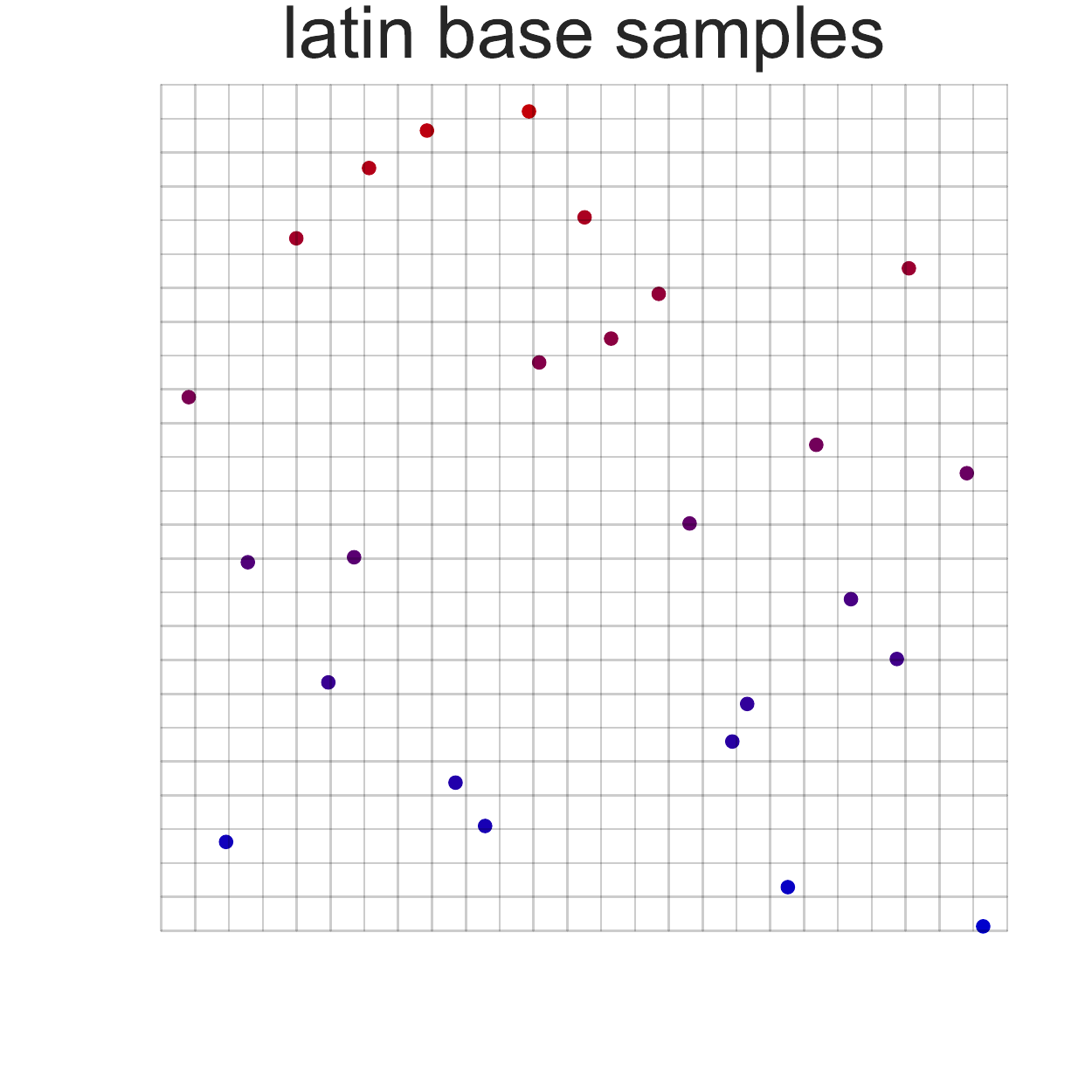}\vspace{-0.35cm}
\par\end{center}
\begin{center}
\includegraphics[width=0.22\textwidth]{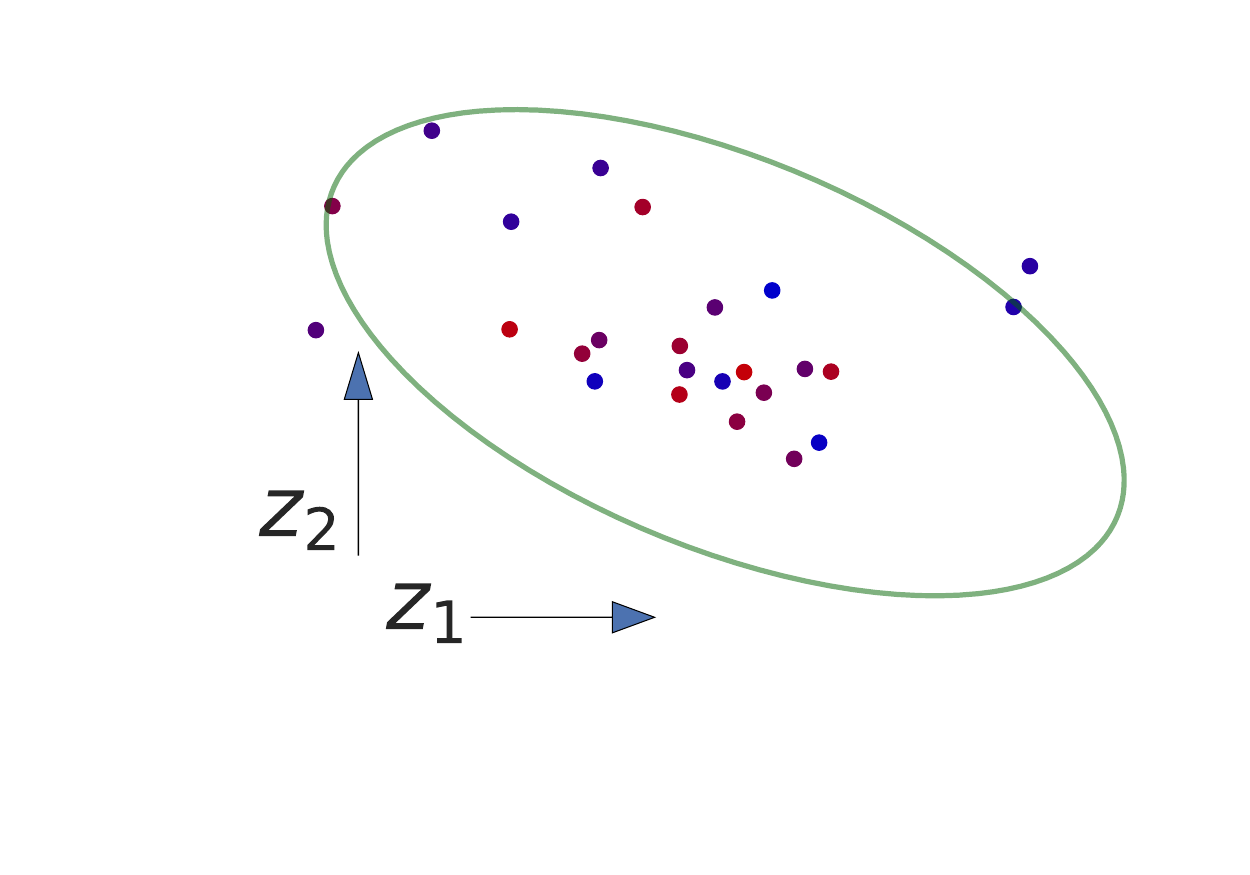}\hspace{-0.5cm}\includegraphics[width=0.22\textwidth]{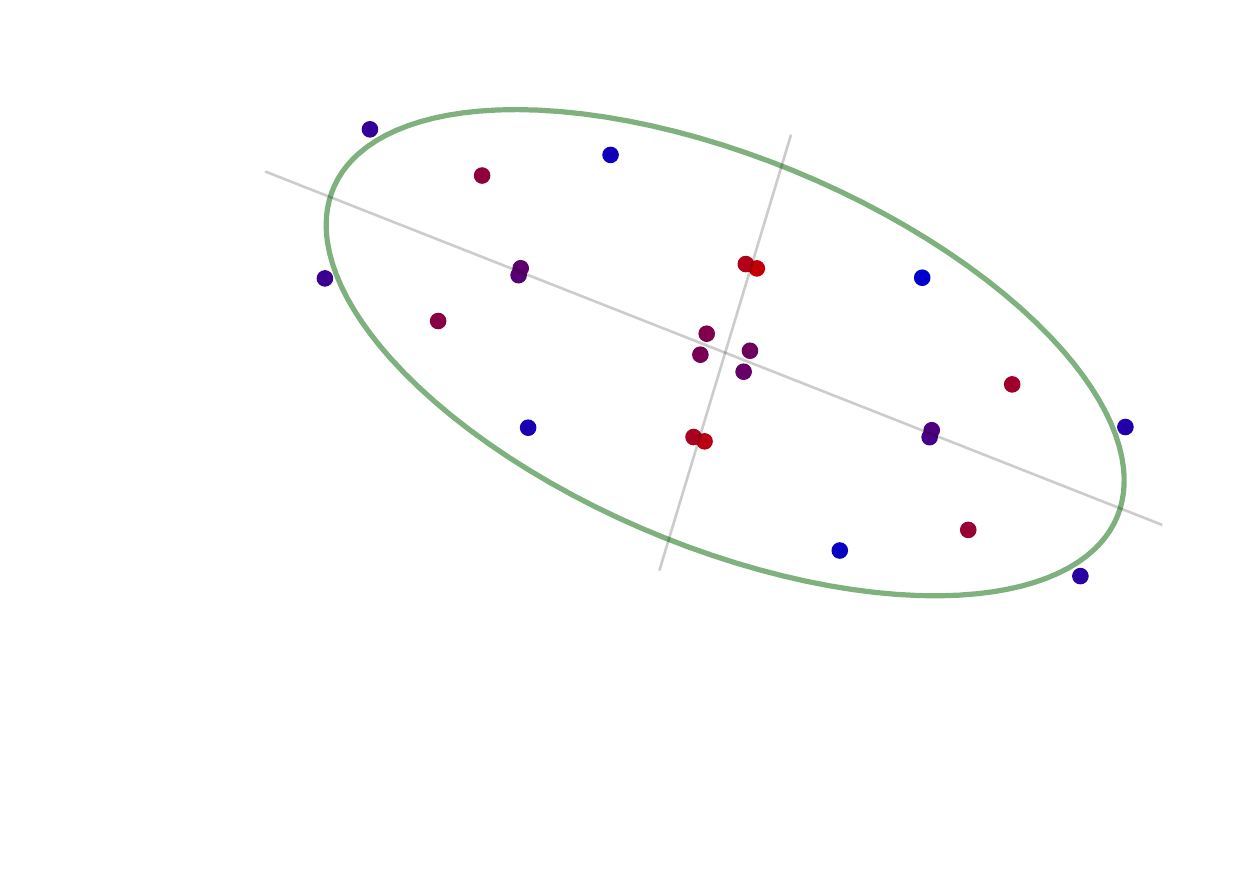}\hspace{-0.5cm}\includegraphics[width=0.22\textwidth]{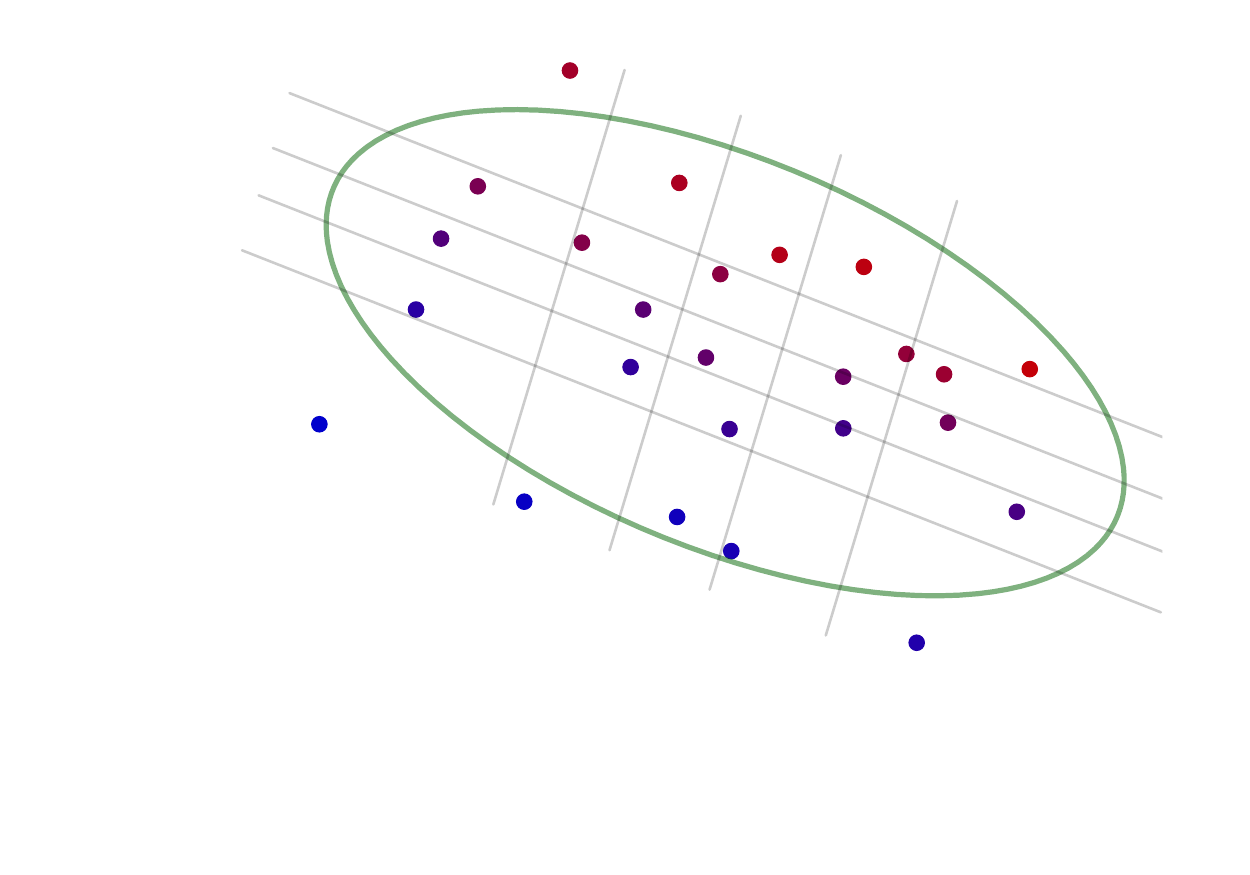}\hspace{-0.5cm}\includegraphics[width=0.22\textwidth]{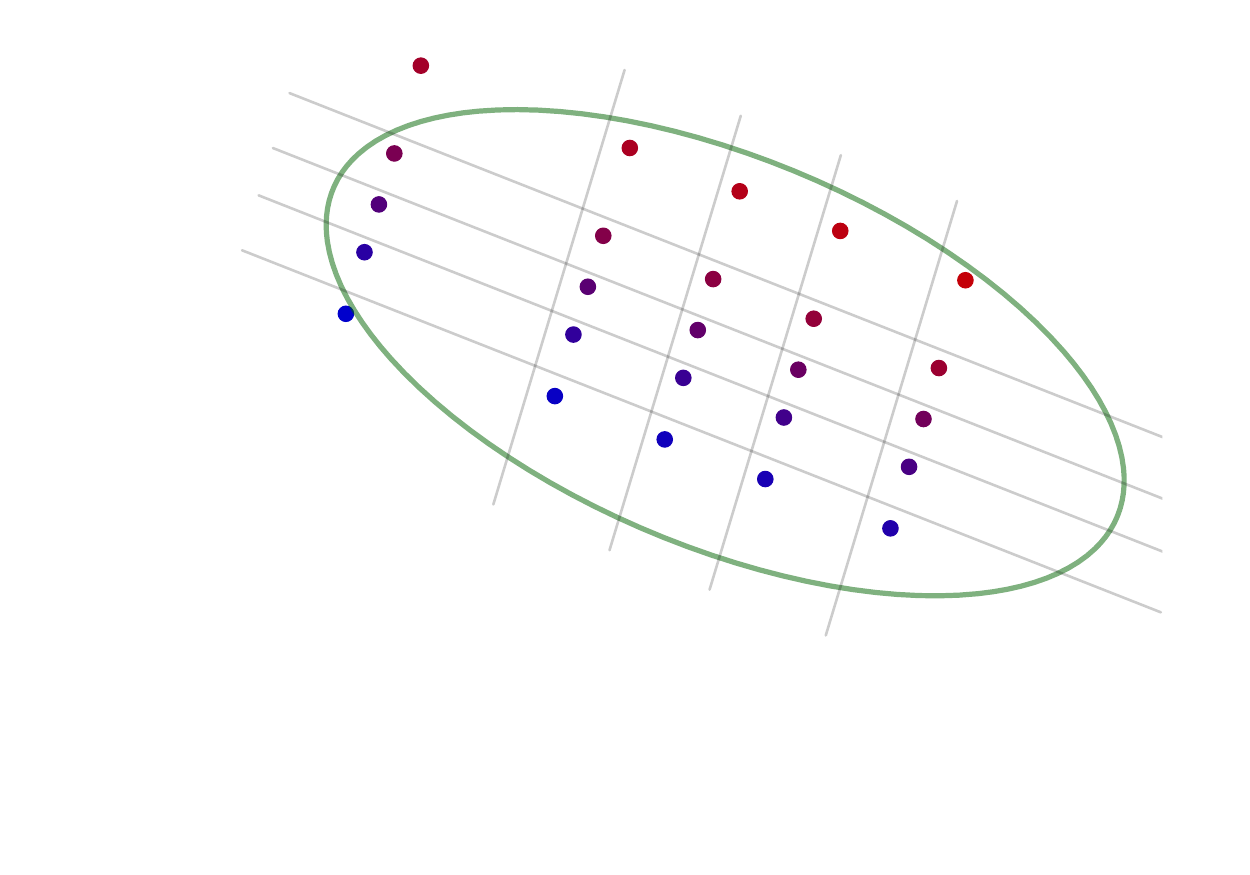}\hspace{-0.5cm}\includegraphics[width=0.22\textwidth]{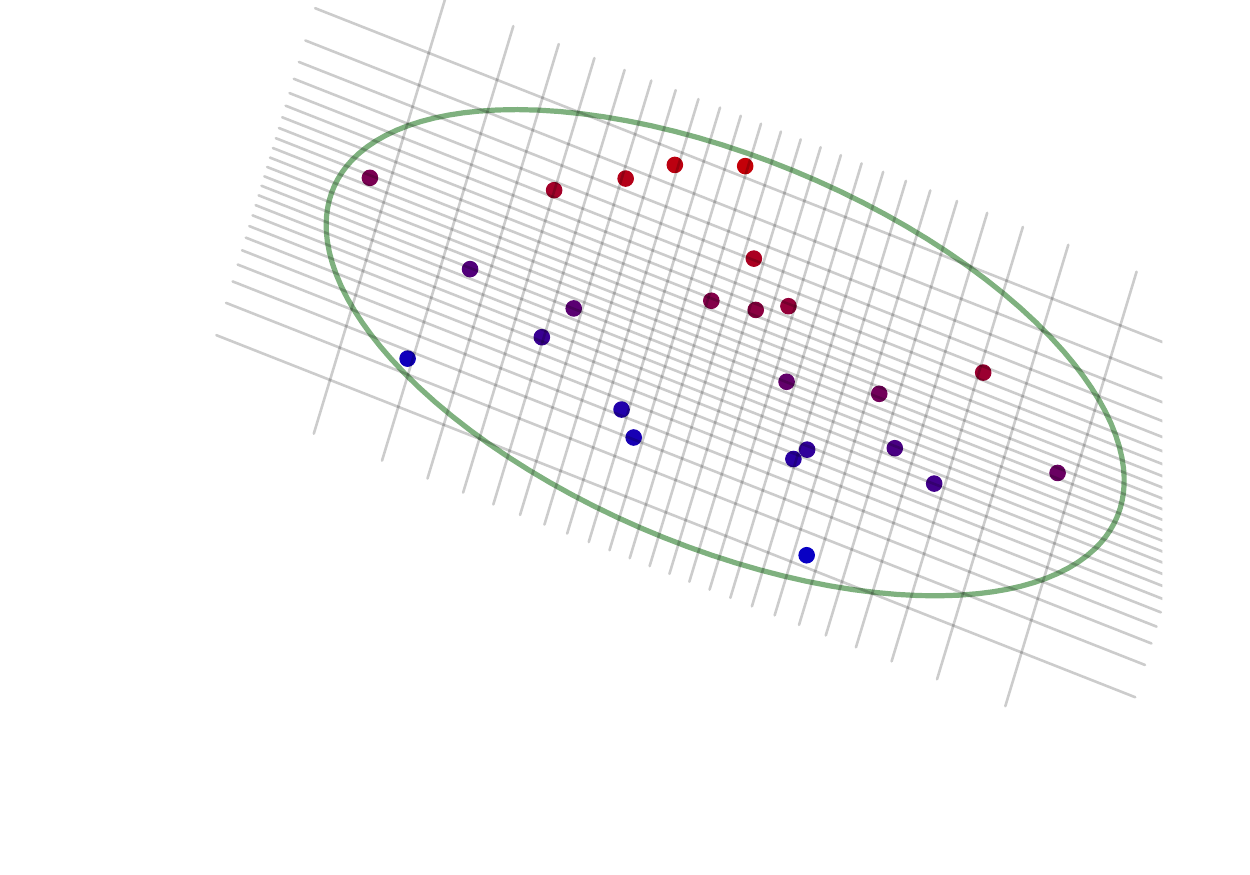}
\par\end{center}
\begin{center}
\vspace{-1.2cm}
\par\end{center}
\begin{center}
\includegraphics[width=0.22\textwidth]{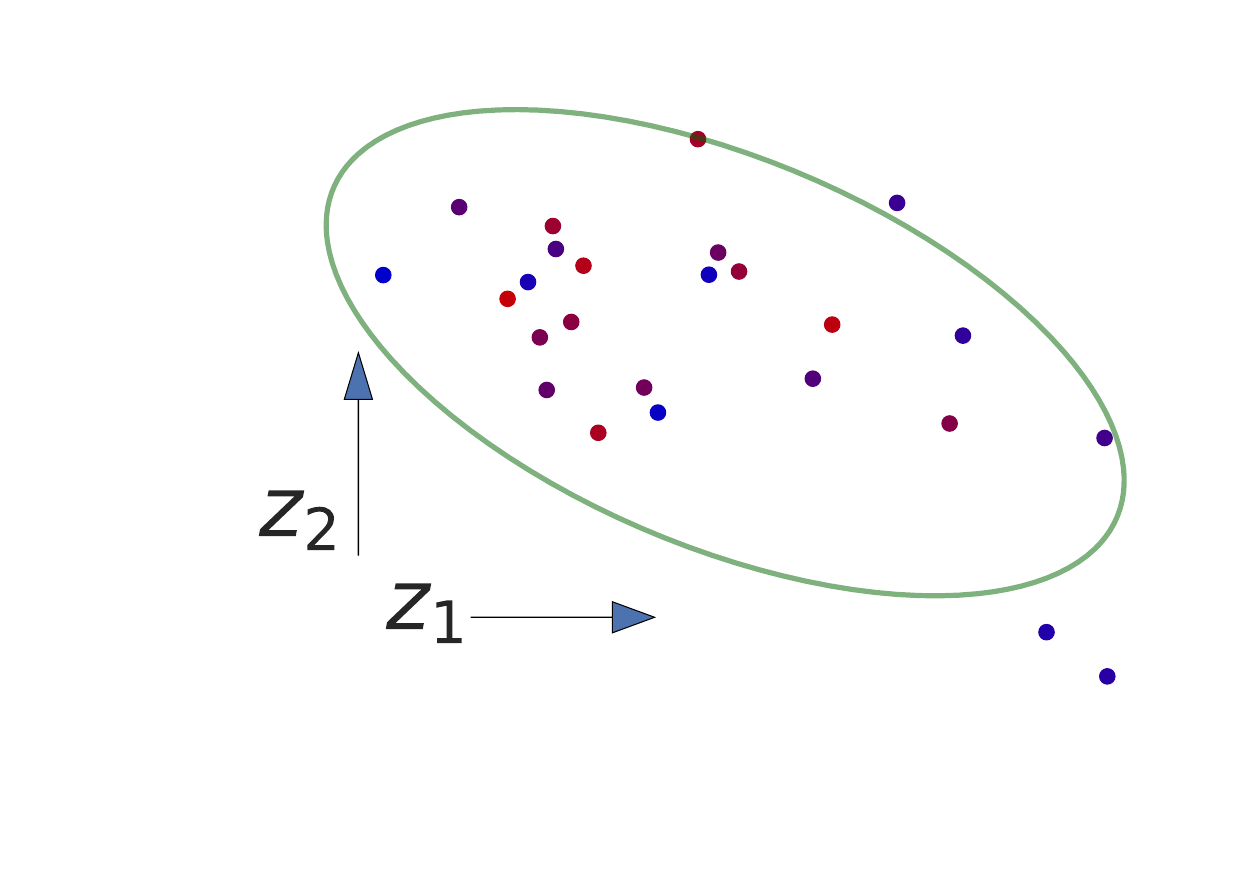}\hspace{-0.5cm}\includegraphics[width=0.22\textwidth]{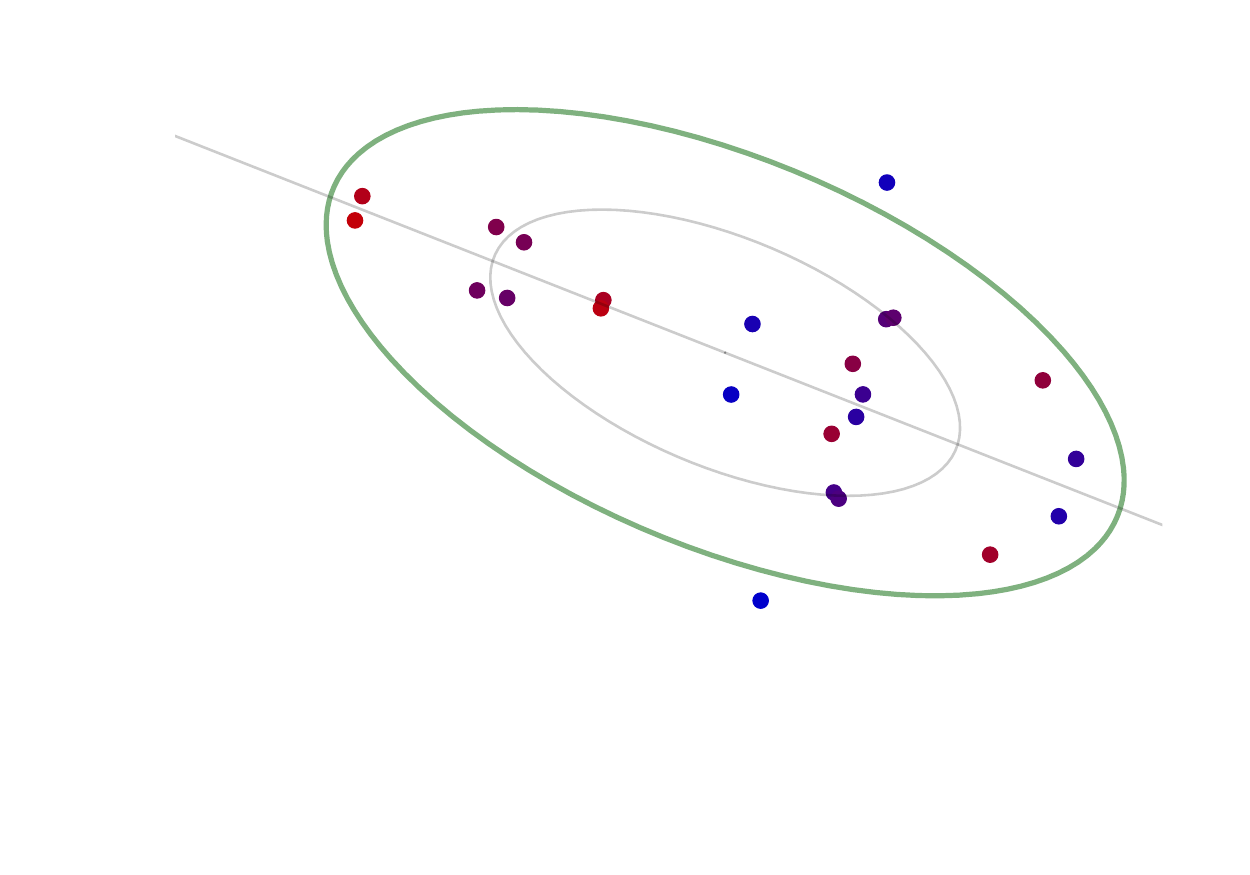}\hspace{-0.5cm}\includegraphics[width=0.22\textwidth]{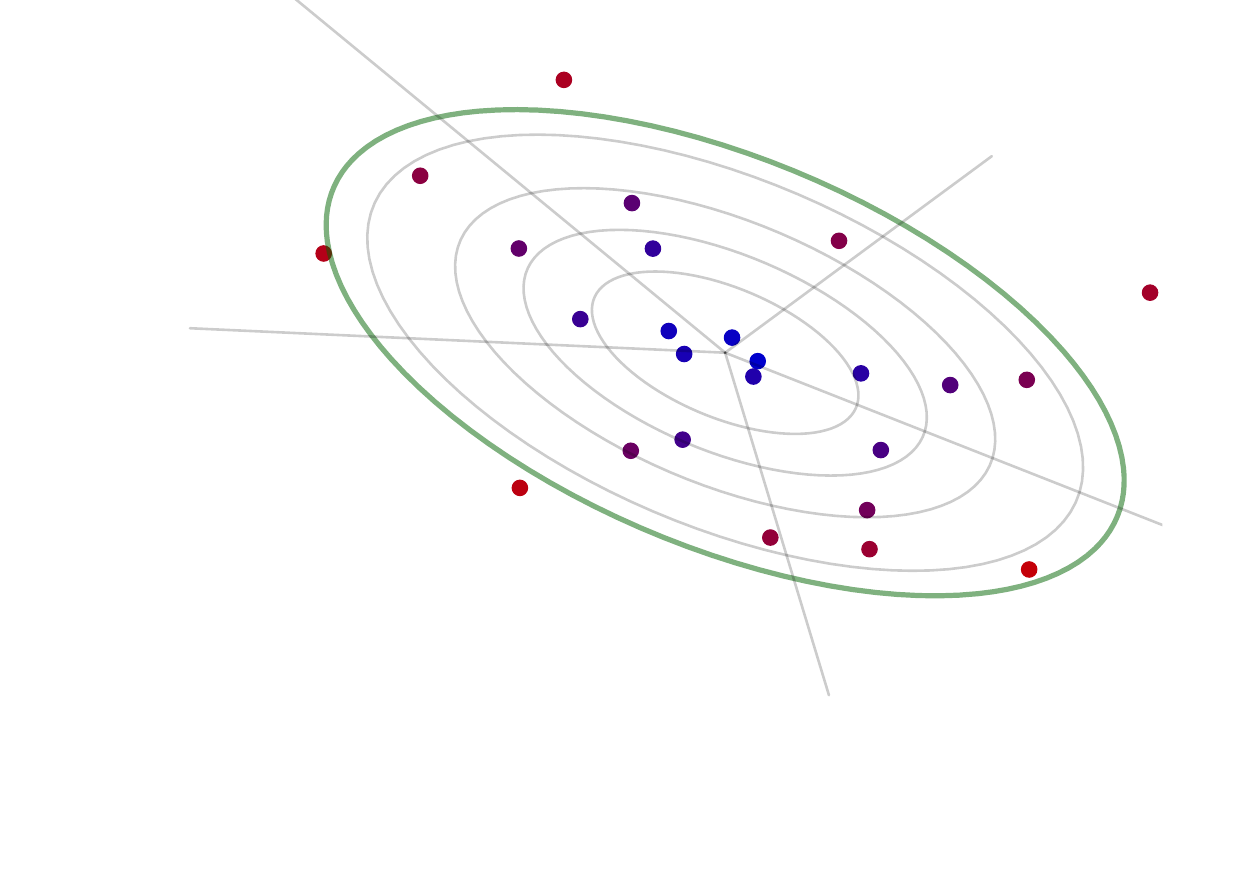}\hspace{-0.5cm}\includegraphics[width=0.22\textwidth]{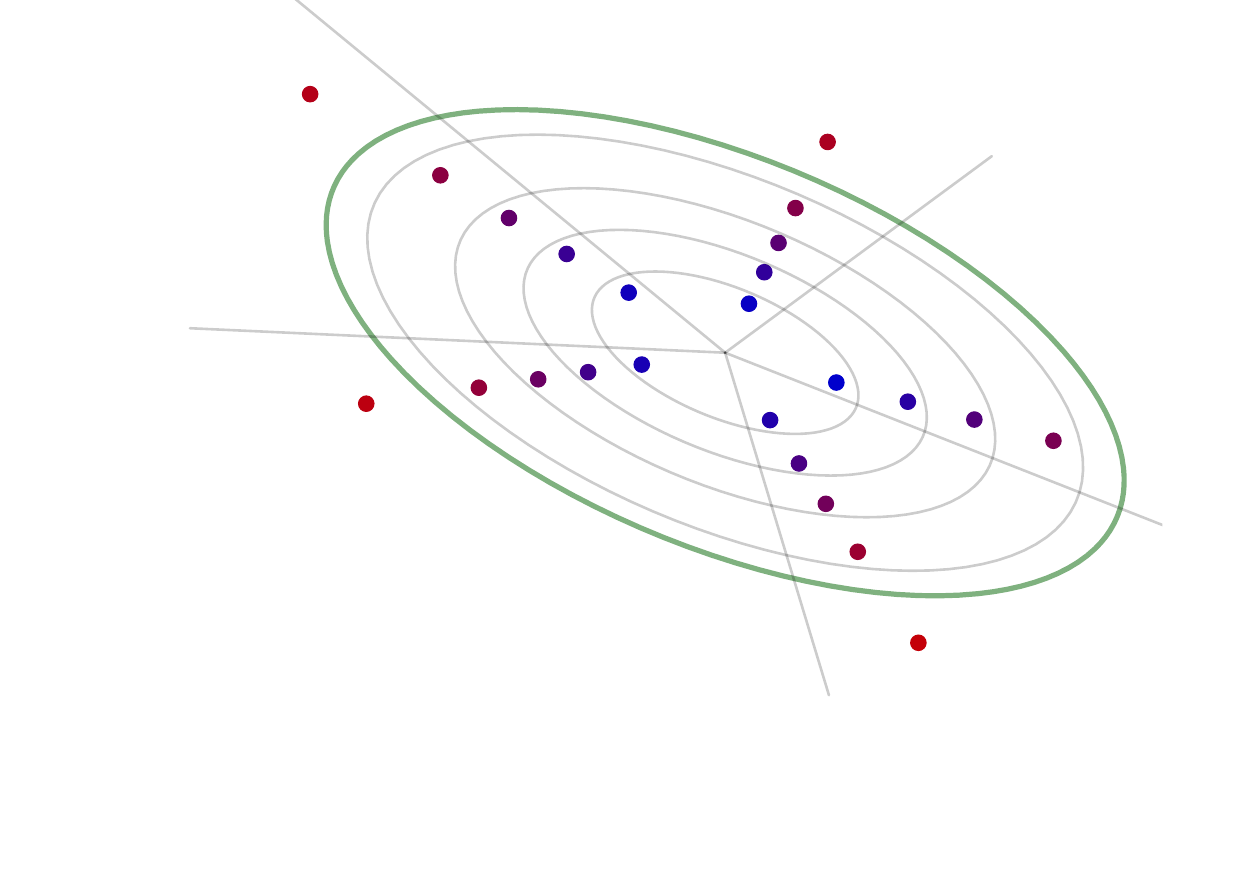}\hspace{-0.5cm}\includegraphics[width=0.22\textwidth]{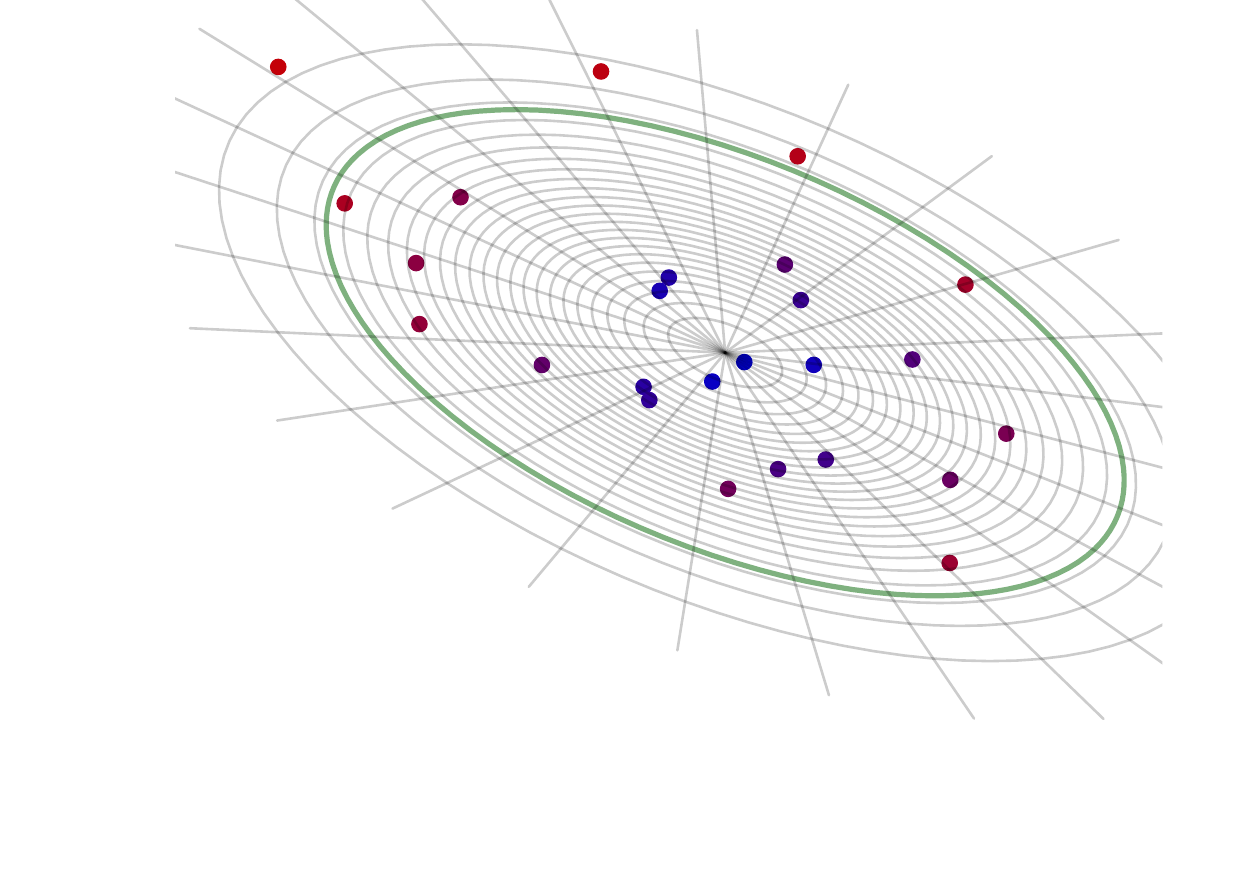}
\par\end{center}
\begin{center}
\vspace{-0.65cm}
\par\end{center}%
\end{minipage}
\par\end{centering}
\caption{Different sampling methods applied to Gaussian VI. Top row: Different
methods to sample from the unit cube. Middle row: these samples transformed
using the ``Cartesian'' mapping. Bottom row: Same samples transformed
using the ``Elliptical'' mapping.\label{fig:Different-sampling-methods}}
\end{figure}
Our results are easy to put into practice, e.g. for variational inference
with Gaussian approximating distributions and the reparameterization
trick to estimate gradients.. To illustrate this, we show a simple
but general approach. As shown in \ref{fig:Different-sampling-methods}
the idea is to start with a batch of samples $\omega_{1}\cdots\omega_{M}$
generated from the unit hypercube. Different sampling strategies can
give more uniform coverage of the cube than i.i.d. sampling. After
transformation, one obtains samples $z_{1}\cdots z_{M}$ that have
more uniform coverage of the Gaussian. This better coverage often
manifests as a lower-variance estimator $R$. Our coupling framework
gives a corresponding approximate posterior $Q\pp z$.

Formally, take any distribution $Q\pp{\omega_{1},\cdots,\omega_{M}}$
such that each marginal $Q\pp{\omega_{m}}$ is uniform over the unit
cube (but the different $\omega_{m}$ may be dependent). As shown
in \ref{fig:Different-sampling-methods}, there are various ways to
generate $\omega_{1}\cdots\omega_{M}$ and to map them to samples
$z_{1}\cdots z_{M}$ from a Gaussian $q\pp{z_{m}}$. Then, \ref{fig:Generic-methods}
gives algorithms to generate an estimator $R$ and to generate $z$
from a distribution $Q\pp z$ corresponding to a valid coupling. We
use mappings $\omega\stackrel{F^{-1}}{\to}u\stackrel{\mathcal{T_{\theta}}}{\to}z$
where $\T_{\theta}=\mathcal{T}_{\theta}\circ F^{-1}$ maps $\wr\sim\mathrm{Unif}([0,1]^{d})$
to $\T_{\theta}(\wr)\sim q_{\theta}$ for some density $q_{\theta}$.
The idea is to implement variance reduction to sample (batches of)
$\omega$, use $F^{-1}$ to map $\omega$ to a ``standard'' distribution
(typically in the same family as $q_{\theta}$), and then use $\mathcal{T}_{\theta}$
to map samples from the standard distribution to samples from $q_{\theta}$.\marginpar{}

The algorithms are again derived from \ref{thm:splitting-estimators}
and \ref{thm:rao-black-new}. Define $Q_{0}(\omega)$ uniform on $[0,1]^{d}$,
$R_{0}(\omega)=p(\T_{\theta}(\omega),x)/q_{\theta}(\T_{\theta}(\omega))$
and $a_{0}(z|\omega)=\delta(z-\T_{\theta}(\omega))$. These define
a valid estimator-coupling pair. Let $Q\pp{\omega_{1},\cdots,\omega_{M}}$
be as described (uniform marginals) and $\r m$ uniform on $\{1,\cdots,M\}$.
Then $Q\pp{\omega_{1},\cdots,\omega_{M},m}$ satisfies the assumptions
of \ref{thm:splitting-estimators}, so we can use that theorem then
\ref{thm:rao-black-new} to Rao-Blackwellize out $\r m$. This produces
the estimator-coupling pair in \ref{fig:Generic-methods}.

\begin{figure}[!b]
\begin{minipage}[t]{0.5\columnwidth}%
\textbf{Algorithm} (Generate $R$)
\begin{itemize}
\item Generate $\omega_{1},\cdots,\omega_{M}$ from any distribution where
$\omega_{m}$ is marginally uniform over $\left[0,1\right]^{d}$.
\item Map to a standard dist. as $u_{m}=F^{-1}\pp{\omega_{m}}$.
\item Map to $q_{\theta}$ as $z_{m}=\mathcal{T}_{\theta}\pp{u_{m}}.$
\item Return $R=\frac{1}{M}\sum_{m=1}^{M}\frac{p\pp{z_{m},x}}{q_{\theta}\pp{z_{m}}}$
\end{itemize}
\end{minipage}%
\begin{minipage}[t]{0.5\columnwidth}%
\textbf{Algorithm} (Sample from $Q\pp z$)
\begin{itemize}
\item Generate $z_{1},\cdots z_{M}$ as on the left.
\item For all $m$ compute weight $w_{m}=\frac{p\pp{z_{m},x}}{q_{\theta}\pp{z_{m}}}.$
\item Select $m$ with probability $\frac{w_{m}}{\sum_{m'=1}^{M}w_{m'}}.$
\item Return $z_{m}$
\end{itemize}
\end{minipage}

\caption{Generic methods to sample $R$ (left) and $Q\protect\pp z$ (right).
Here, $Q\protect\pp{\omega_{1},\cdots,\omega_{M}}$ is any distribution
where the marginals $Q\protect\pp{\omega_{m}}$ are uniform over the
unit hypercube.\label{fig:Generic-methods}}
\end{figure}
The value of this approach is the many off-the-shelf methods to generate
``batches'' of samples $\pp{\omega_{1},\cdots,\omega_{M}}$ that
have good ``coverage'' of the unit cube. This manifests as coverage
of $q_{\theta}$ after being mapped. \ref{fig:Different-sampling-methods}
shows examples of this with multivariate Gaussians. As shown, there
may be multiple mappings $F^{-1}.$ These manifest as different coverage
of $q_{\theta},$ so the choice of mapping influences the quality
of the estimator. We consider two examples\marginpar{}: The ``Cartesian''
mapping $F_{\mathcal{N}}^{-1}\pp{\omega}$ simply applies the inverse
CDF of the standard Gaussian. An ``elliptical'' mapping, meanwhile,
uses the ``elliptical'' reparameterization of the Gaussian \citep{Domke_2018_ImportanceWeightingVariational}:
If $\r r\sim\chi_{d}$ and $\r v$ is uniform over the unit sphere,
then $\r{r\,}\r v\sim\N(0,I).$ In \ref{fig:Different-sampling-methods}
we generate $r$ and $v$ from the uniform distribution as $r=F_{\chi_{d}}^{-1}\pp{\omega_{1}}$
and $v=\pp{\cos\pp{2\pi\omega_{2}},\sin\pp{2\pi\omega_{2}}}$, and
then set $F^{-1}\pp{\omega}=r\ v.$ In higher dimensions, it is easier
to generate samples from the unit sphere using redundant dimensions.
Thus, we use $\omega\in\R^{d+1}$ and map the first component to $r$
again using the inverse $\chi$ distribution CDF $F_{\chi_{d}}^{-1}$.
The other components are mapped to the unit sphere by first applying
the Gaussian inverse CDF in each component, then normalizing.

In the experiments, we use a multivariate Gaussian $q_{\theta}$ with
parameters $\theta=\pp{C,\mu}.$ The mapping is $\mathcal{T}_{\theta}\pp u=Cu+\mu$.
To ensure a diverse test, we downloaded the corpus of models from
the Stan \citep{Carpenter_2017_StanProbabilisticProgramming} model
library \citep{Standevelopers._2018_ExampleModels} (see also \citet{Regier_2017_FastBlackboxVariationala})
and created an interface for automatic differentiation in Stan to
interoperate with automatic differentiation code written in Python.
We compare VI in terms of the likelihood bound and in terms of the
(squared Frobenius norm) error in the estimated posterior variance.
As a surrogate for the true variance, we computed the empirical variance
of 100,000 samples generated via Stan's Hamiltonian Markov chain Monte
Carlo (MCMC) method. For tractability, we restrict to the 88 models
where profiling indicates MCMC would take at most 10 hours, and evaluating
the posterior for 10,000 settings of the latent variables would take
at most 2 seconds. It was infeasible to tune stochastic gradient methods
for all models. Instead we used a fixed batch of 50,000 batches $\omega_{1},\cdots,\omega_{M}$
and optimized the empirical ELBO using BFGS, initialized using Laplace's
method. A fresh batch of 500,000 samples was used to compute the final
likelihood bound and covariance estimator. \ref{fig:errors} shows
example errors for a few models. The supplement contains similar plots
for all models, as well as plots aggregating statistics, and a visualization
of how the posterior density approximation changes.

\begin{figure}
\includegraphics[viewport=0bp 33.7744bp 300bp 166.6204bp,clip,width=0.33\columnwidth]{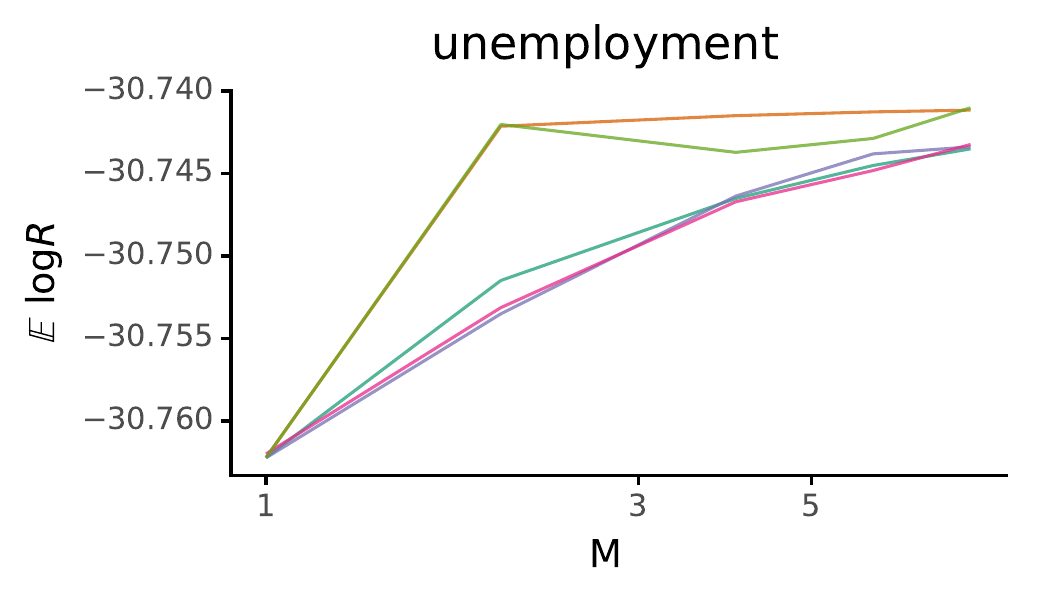}\includegraphics[viewport=0bp 33.5826bp 314bp 177bp,clip,width=0.33\columnwidth]{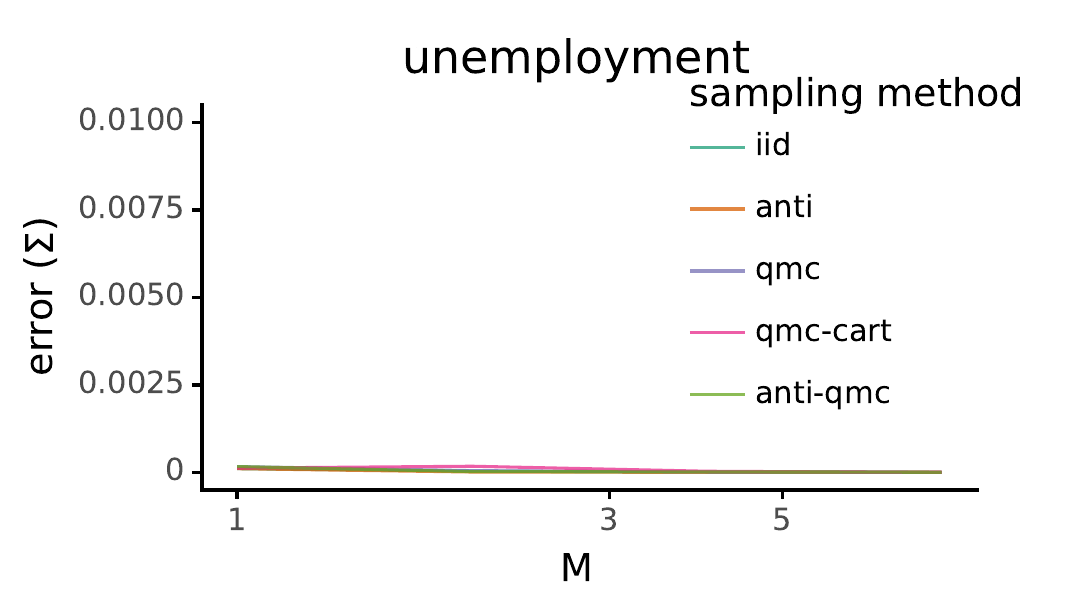}\includegraphics[viewport=0bp 33.6577bp 315bp 177bp,clip,width=0.33\columnwidth]{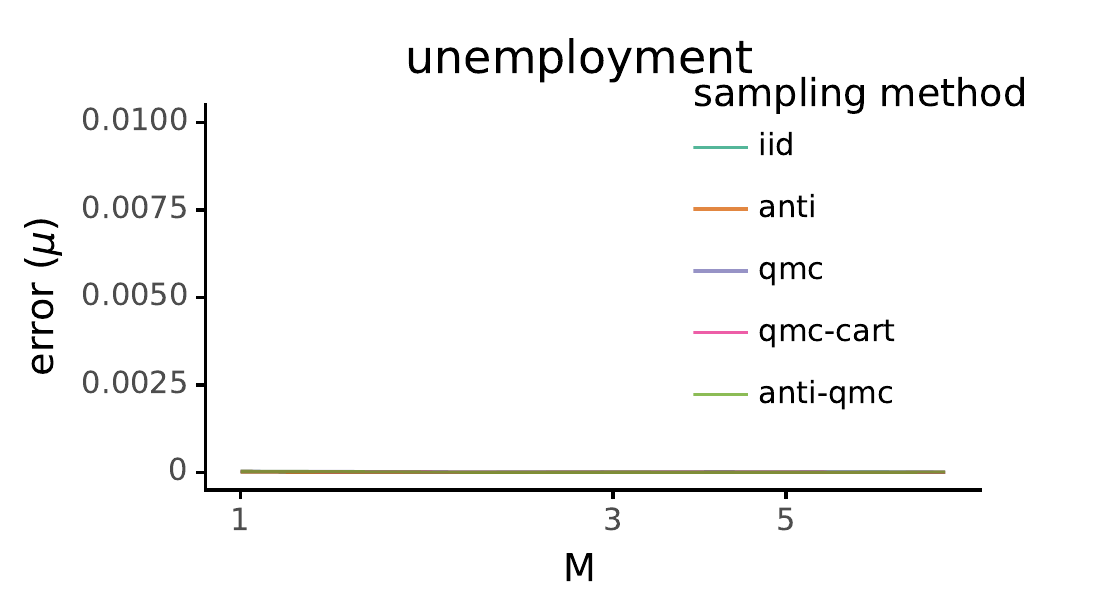}

\includegraphics[viewport=-3.675bp 33.2692bp 294bp 173bp,clip,width=0.33\columnwidth]{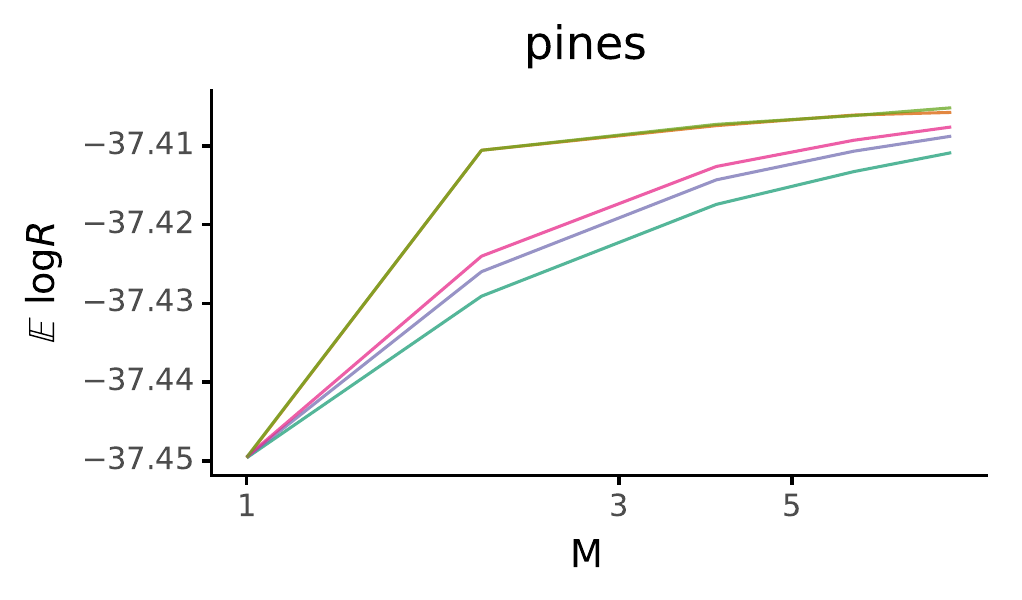}\includegraphics[viewport=-11.3625bp 33.569bp 303bp 177bp,clip,width=0.33\columnwidth]{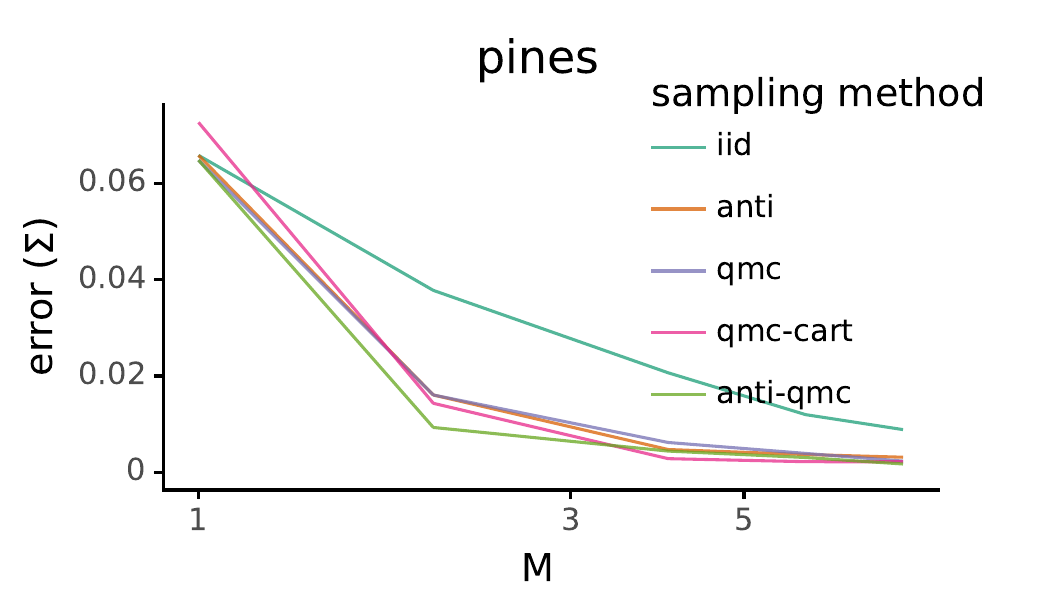}\includegraphics[viewport=0bp 33.6577bp 315bp 177bp,clip,width=0.33\columnwidth]{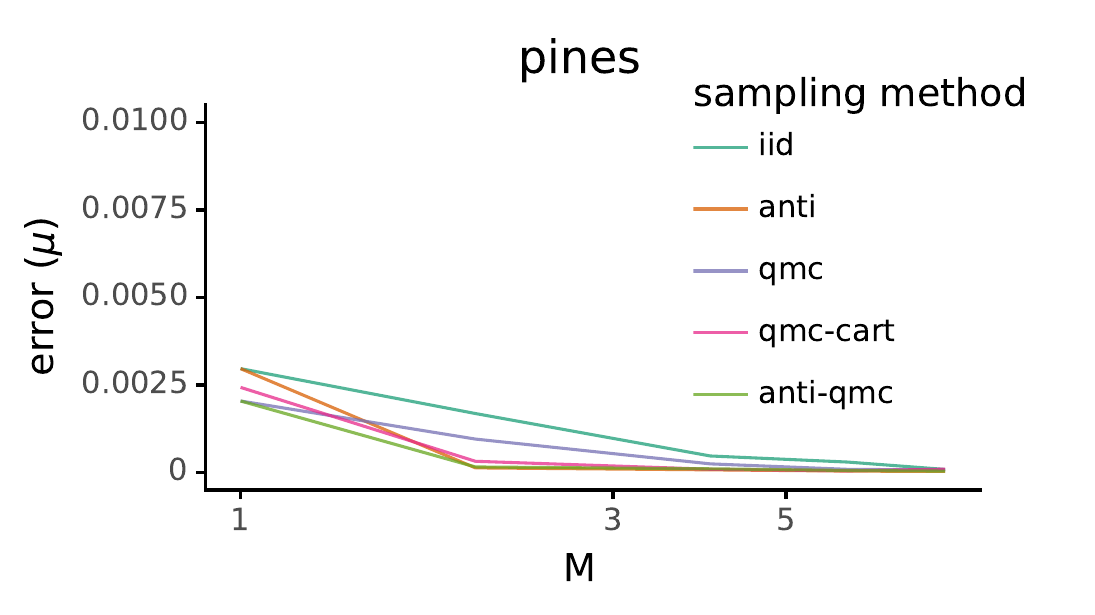}

\includegraphics[viewport=0bp 33.7744bp 300bp 173bp,clip,width=0.33\columnwidth]{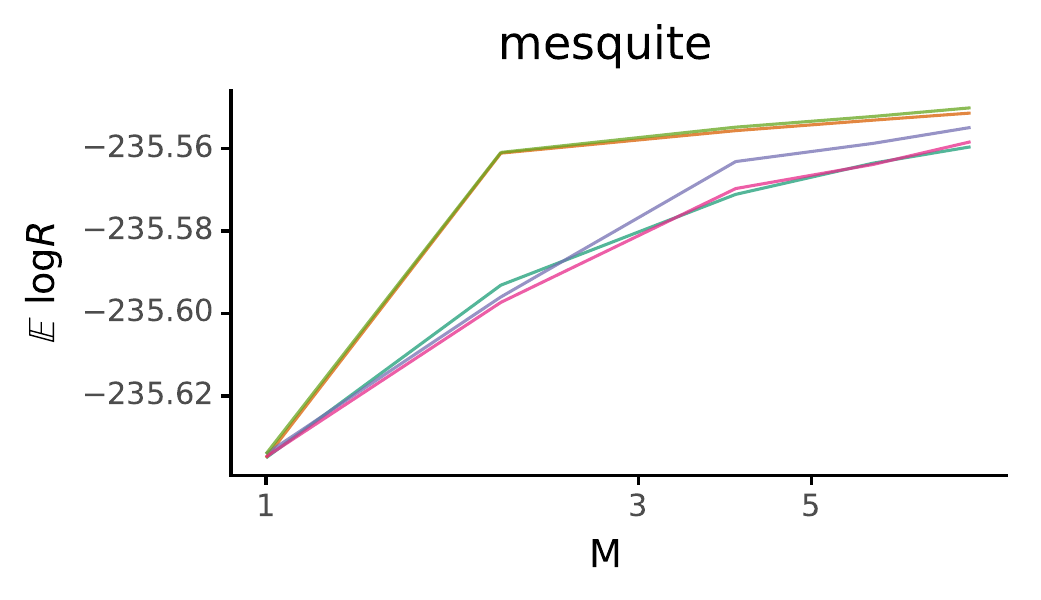}\includegraphics[viewport=-7.725bp 32.9212bp 309bp 177bp,clip,width=0.33\columnwidth]{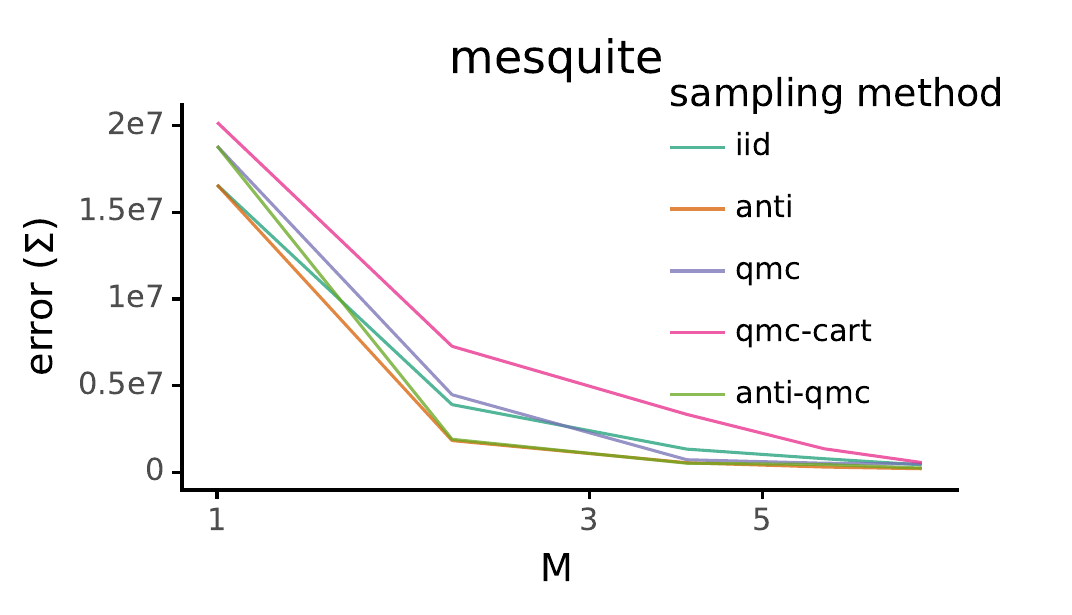}\includegraphics[viewport=-14.95bp 33.75bp 299bp 177bp,clip,width=0.33\columnwidth]{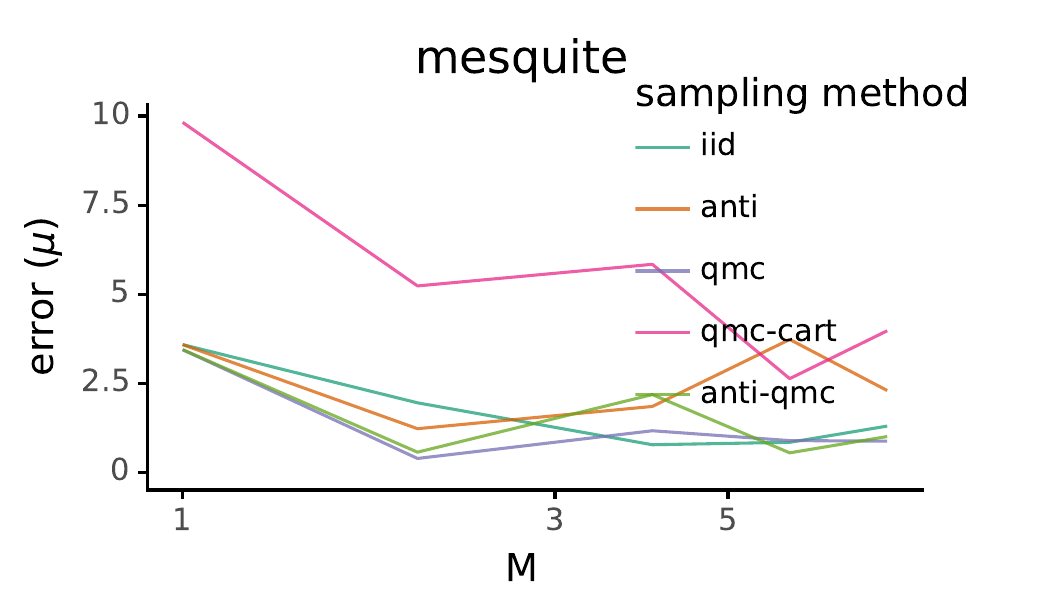}

\includegraphics[viewport=-3.675bp 0bp 294bp 173bp,clip,width=0.33\columnwidth]{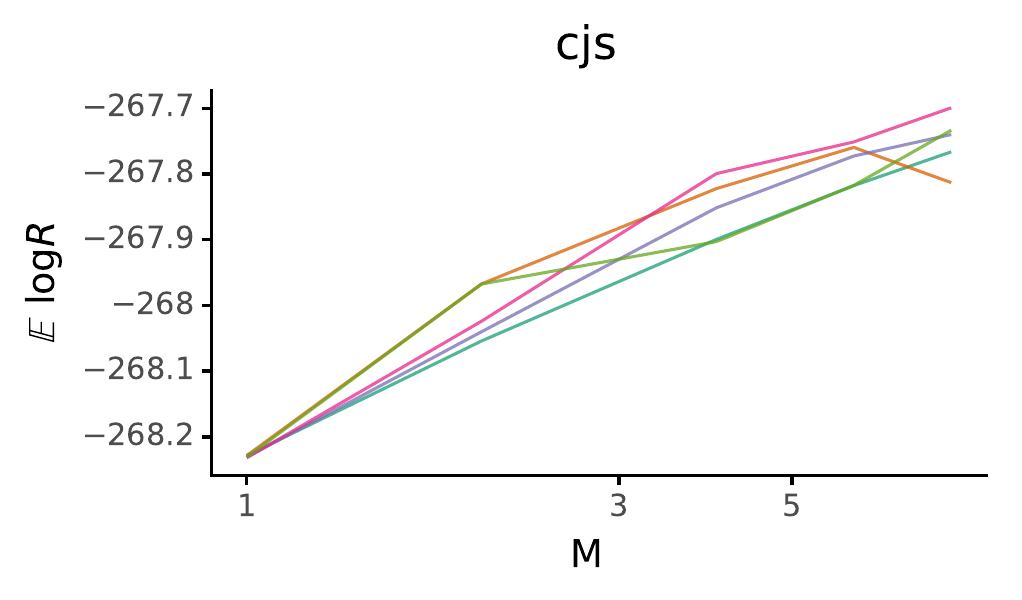}\includegraphics[viewport=-25.2875bp 0bp 289bp 177bp,clip,width=0.33\columnwidth]{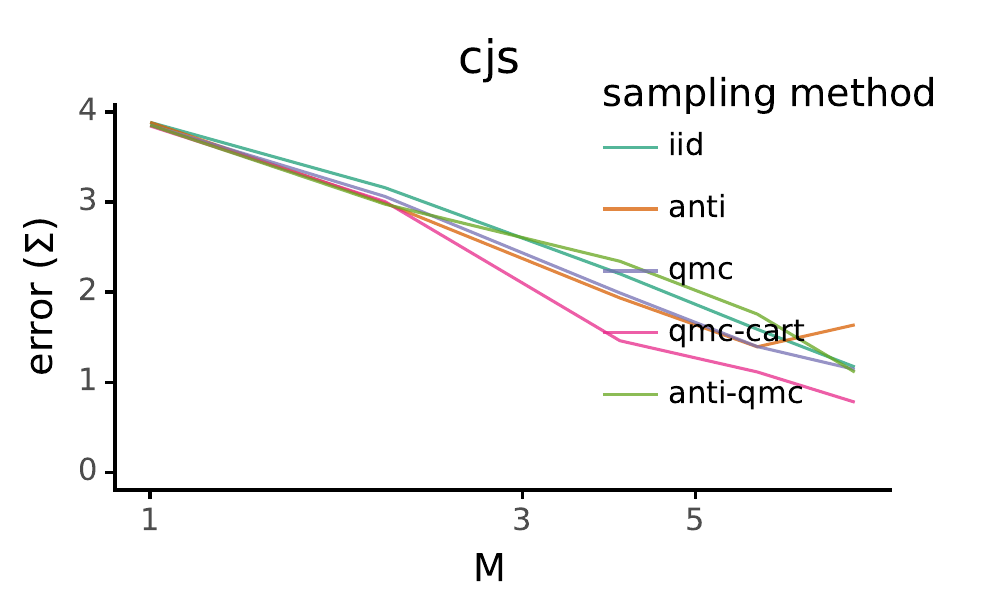}\includegraphics[viewport=0bp 0bp 310bp 183.2105bp,clip,width=0.33\columnwidth]{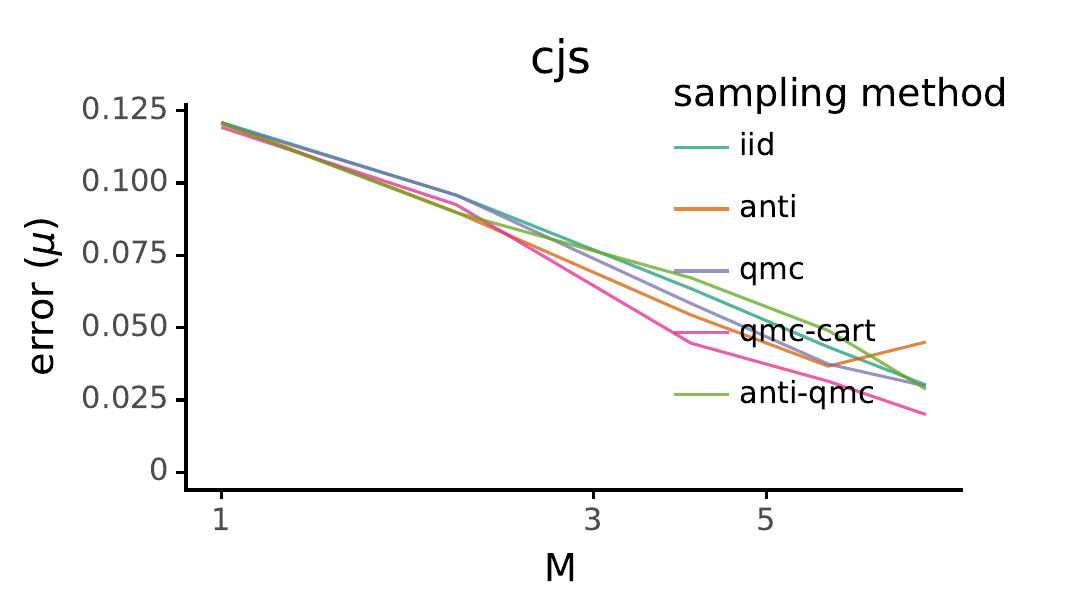}

\caption{\textbf{Across all models, improvements in likelihood bounds correlate
strongly with improvements in posterior accuracy}. Better sampling
methods can improve both. First row: the common case where a simple
Gaussian posterior is already very accurate. Here, only a tiny improvement
in the ELBO is possible, and improvement in the posterior is below
the level detectable when comparing to MCMC. The other rows show cases
where larger improvements are possible. \textbf{\label{fig:errors}}}
\end{figure}

\section{Conclusions\label{sec:Conclusions}}

Recent work has studied the use of improved Monte Carlo estimators
for better variational likelihood bounds. The central insight of this
paper is that an \emph{approximate posterior} can be constructed from
an estimator using a \emph{coupling}. This posterior's divergence
is bounded by the looseness of the likelihood bound. We suggest a
framework of ``estimator-coupling'' pairs to make this coupling
easy to construct for many estimators.

Several recent works have viewed Monte Carlo VI bounds through the
lens of augmented VI \citep{Bachman_2015_TrainingDeepGenerative,Cremer_2017_ReinterpretingImportanceWeightedAutoencoders,Naesseth_2018_VariationalSequentialMonte,Domke_2018_ImportanceWeightingVariational,Lawson_2019_EnergyInspiredModelsLearning}.
These establish connections between particular likelihood estimators
and approximate posteriors through extended distributions. They differ
from our work primarily in the \textquotedblleft direction\textquotedblright{}
of the construction, the generality, or both. Most of the work uses
the following reasoning, which starts with an approximate posterior
and arrives at a tractable likelihood estimator. Take a Monte Carlo
method (e.g. self-normalized importance sampling) to approximately
sample from $p\pp{z|x}$. Call the approximation $q\pp z,$ but suppose
it is not tractable to evaluate $q(z)$. A tractable likelihood estimator
can be obtained as $R(\omega,z)=p(z,x)p(\omega\vert z,x)/q(\omega,z)$,
where $q(\omega,z)$ is the (tractable) joint density over the \textquotedblleft internal
randomness\textquotedblright{} $\omega$ of the Monte Carlo procedure
and the final sample $z$, and $p(\omega|z,x)$ is a conditional distribution
used to extend the target to also contain these variables. Different
choices for the Monte Carlo procedure $q(\omega,z)$ and the target
extension $p(\omega|z,x)$ lead to different estimators. To arrive
at a particular existing likelihood estimator $R$ requires careful
estimator-specific choices and derivations. In contrast, our work
proceeds in the opposite direction: we start with an arbitrary estimator
$R$ and show (via coupling) how to find a corresponding Monte Carlo
procedure $q(\omega,z)$. We also provide a set of tools to \textquotedblleft automatically\textquotedblright{}
find couplings for many types of estimators.

The idea of using extended state-spaces is common in (Markov chain)
Monte Carlo inference methods \citep{Finke_2015_ExtendedStateSpaceConstructions,Andrieu_2010_ParticleMarkovchain,Neal_1998_AnnealedImportanceSampling,Neal_2005_HamiltonianImportanceSampling}.
These works also identify extended target distributions that admit
$p(z,x)$ as a marginal, i.e., a coupling in our terminology. By running
an Markov chain Monte Carlo (MCMC) sampler on the extended target
and dropping the auxiliary variables, they obtain an MCMC sampler
for $p(z|x)$. Our work can be seen as the VI analogue of these MCMC
methods. Other recent work \citep{Burda_2015_ImportanceWeightedAutoencoders,Maddison_2017_FilteringVariationalObjectives,Cremer_2017_ReinterpretingImportanceWeightedAutoencoders,Naesseth_2018_VariationalSequentialMonte,Le_2018_AutoEncodingSequentialMonte,Domke_2018_ImportanceWeightingVariational,ChristianAnderssonNaesseth_2018_Machinelearningusing,Ren_2019_AdaptiveAntitheticSampling}
that has explored the connection between using estimators in variational
bounds and auxiliary variational inference \citep{Agakov_2004_AuxiliaryVariationalMethod}.
To the best of our knowledge, all of these works consider situations
in which the relevant extended state space $\pp{z,\omega}$ is known.
Thus, in these works, the estimator essentially comes with an ``obvious''
coupling distribution $a\pp{z|\omega}$. In contrast, the goal of
this paper is to consider an arbitrary estimator $R\pp{\omega}$,
where it is not obvious that a tractable coupling distribution $a\pp{z|\omega}$
even exists. This is the situation in which our framework of estimator-coupling
pairs is likely to be useful. The alternative would be manual construction
of extended state-spaces for each individual estimator.

\bibliographystyle{plainnat}
\bibliography{justindomke_zotero_betterbibtex2}
\clearpage\newpage{}

\global\long\def\B{\mathrm{\mathcal{B}}}%

\global\long\def\A{\mathrm{\mathcal{A}}}%

\global\long\def\Pmc{P^{\mathrm{MC}}}%

\clearpage{}

\section{Proofs for Main Framework (\ref{sec:The-Divide-and-Couple-Framework})}

\trivdivide*
\begin{proof}
Since $\pmc\pp{\omega,x}\geq0$ and $\pmc\pp x=\int\pmc(\omega,x)d\omega=\E_{Q(\wr)}R(\wr)=p(x)$,
it is a valid distribution. Thus, one can apply the standard ELBO
decomposition to $\qmc\pp{\omega}$ and $\pmc\pp{\omega,x}$. But
since $R=\pmc/\qmc$, it follows that $\E_{\qmc\pp{\wr}}\log\pars{\pmc\pp{\wr,x}/Q\pp{\wr}}=\E_{Q\pp{\wr}}\log R\pp{\wr}$.
\end{proof}

\divandcouple*
\begin{proof}
First, note that
\begin{eqnarray*}
\pmc\pp{z,x} & = & \int\pmc\pp{z,\omega,x}d\omega\\
 & = & \int Q\pp{\omega}R\pp{\omega}a\pp{z\vert\omega}d\omega\\
 & = & \E_{Q\pp{\wr}}R\pp{\wr}a\pp{z\vert\wr}\\
 & = & p\pp{z,x},
\end{eqnarray*}
so $\pmc\pp{z,\omega,x}$ is a valid augmentation of $p\pp{z,x}.$

Next, observe for $\pmc$ and $Q$ as defined,
\[
\frac{\pmc\pp{z,\omega,x}}{Q\pp{z,\omega}}=R\pp{\omega}.
\]

Applying the ELBO decomposition from \ref{eq:ELBO-decomp} to $Q\pp{z,\omega}$
and $\pmc\pp{z,\omega,x}$ we get that
\[
\log\pmc\pp x=\E_{Q\pp{\zr,\wr}}\bracs{\log\frac{\pmc\pp{\zr,\wr,x}}{Q\pp{\zr,\wr}}}+\KL{Q\pp{\zr,\wr}}{\pmc\pp{\zr,\wr\vert x}}.
\]
Using the observations above and the chain rule of KL-divergence means
that
\begin{eqnarray*}
\log p\pp x & = & \E_{Q\pp{\wr}}\log R\pp{\wr}+\KL{Q\pp{\zr,\wr}}{\pmc\pp{\zr,\wr\vert x}}\\
 & = & \E_{Q\pp{\wr}}\log R\pp{\wr}+\KL{Q\pp{\zr}}{\pmc\pp{\zr\vert x}}+\KL{Q\pp{\wr\vert\zr}}{\pmc\pp{\wr\vert\zr,x}}\\
 & = & \E_{Q\pp{\wr}}\log R\pp{\wr}+\KL{Q\pp{\zr}}{p\pp{\zr\vert x}}+\KL{Q\pp{\wr\vert\zr}}{\pmc\pp{\wr\vert\zr,x}}.
\end{eqnarray*}

\end{proof}
\begin{restatable}{clm1}{antithetic-running-valid-pair}Suppose that
$Q\pp{T\pp{\omega}}=Q\pp{\omega}$. Then, the antithetic estimator\label{claim:antitheticrunningvalidpair}
\[
R\pp{\omega}=\frac{p\pp{\omega,x}+p\pp{T\pp{\omega},x}}{2Q\pp{\omega}}
\]
and the coupling distribution

\begin{eqnarray*}
a\pp{z|\omega} & = & \pi\pp{\omega}\ \delta\pp{z-\omega}+\pp{1-\pi\pp{\omega}}\ \delta\pp{z-T\pp{\omega}},\\
\pi\pp{\omega} & = & \frac{p(\omega,x)}{p(\omega,x)+p(T\pp{\omega},x)}.
\end{eqnarray*}
form a valid estimator / coupling pair under $Q\pp{\omega}.$\end{restatable}
\begin{proof}
\begin{alignat}{1}
 & \E_{Q\pp{\wr}}R\pp{\wr}a\pp{z|\wr}\nonumber \\
 & =\E_{Q\pp{\wr}}\frac{p\pp{\wr,x}+p\pp{T\pp{\wr},x}}{2Q\pp{\wr}}\pars{\pi\pp{\wr}\ \delta\pp{z-\wr}+\pp{1-\pi\pp{\wr}}\ \delta\pp{z-T\pp{\wr}}}\nonumber \\
 & =\E_{Q\pp{\wr}}\frac{p\pp{\wr,x}+p\pp{T\pp{\wr},x}}{2Q\pp{\wr}}\Bigl(\frac{p(\wr,x)}{p(\wr,x)+p(T\pp{\wr},x)}\ \delta\pp{z-\wr}\nonumber \\
 & \ \ \ \ \ \ \ \ \ \ \ \ \ \ \ \ \ \ \ \ \ \ \ \ \ \ \ \ \ \ \ \ \ \ \ \ \ \ \ \ \ \ \ \ \ \ \ \ \ \ +\frac{p\pp{T\pp{\wr},x}}{p(\wr,x)+p(T\pp{\wr},x)}\ \delta\pp{z-T\pp{\wr}}\Bigr)\nonumber \\
 & =\E_{Q\pp{\wr}}\frac{1}{2Q\pp{\wr}}\pars{p(\wr,x)\ \delta\pp{z-\wr}+p\pp{T\pp{\wr},x}\ \delta\pp{z-T\pp{\wr}}}\nonumber \\
 & =\E_{Q\pp{\wr}}\frac{1}{2}\pars{\frac{1}{Q\pp{\wr}}p(\wr,x)\ \delta\pp{z-\wr}+\frac{1}{Q\pp{\wr}}p\pp{T\pp{\wr},x}\ \delta\pp{z-T\pp{\wr}}}\nonumber \\
 & =\E_{Q\pp{\wr}}\frac{1}{2}\pars{\frac{1}{Q\pp{\wr}}p(\wr,x)\ \delta\pp{z-\wr}+\frac{1}{Q\pp{T\pp{\wr}}}p\pp{T\pp{\wr},x}\ \delta\pp{z-T\pp{\wr}}}\label{eq:antithetic-proof-eq-1}\\
 & =\E_{Q\pp{\wr}}\frac{1}{2}\pars{\frac{1}{Q\pp{\wr}}p(\wr,x)\ \delta\pp{z-\wr}+\frac{1}{Q\pp{\wr}}p\pp{\wr,x}\ \delta\pp{z-\wr}}\label{eq:antithetic-proof-eq-2}\\
 & =\E_{Q\pp{\wr}}\pars{\frac{1}{Q\pp{\wr}}p(\wr,x)\ \delta\pp{z-\wr}}\nonumber \\
 & =\int\pars{p(\omega,x)\ \delta\pp{z-\omega}}d\omega\nonumber \\
 & =p\pp{z,x}\nonumber 
\end{alignat}
Here, \ref{eq:antithetic-proof-eq-1} follows from the fact that $Q\pp{T\pp{\omega}}=Q\pp{\omega}$
while \ref{eq:antithetic-proof-eq-2} follows from the fact that $T\pp{\wr}$
is equal in distribution to $\wr$ when $\wr\sim Q$.
\end{proof}
\clearpage{}

\section{Measure-Theoretic Details}

The content of this section draws from \citep{Gray_2011_Entropyinformationtheory,Klenke_2014_ProbabilityTheoryComprehensive}.
We do not use sans-serif font in this section.

\subsection{Measures, KL, ELBO}

Let $\pp{\Omega,\mathcal{A}}$ be a measurable space and $Q$ and
$P$ be two measures over it. Write $Q\ll P$ when $Q$ is absolutely
continuous with respect to $P$, i.e. when $P(A)=0\Rightarrow Q(A)=0$.
Whenever $Q\ll P$, there exists measurable $f:\Omega\rightarrow\R$
such that 
\[
Q\pp A=\int_{A}f\ dP.
\]
The function $f$ is the\emph{ Radon-Nikodym derivative,} denoted
as $f=\frac{dQ}{dP}$. Write $Q\sim P$ when $Q\ll P$ and $P\ll Q$;
in this case $\frac{dQ}{dP}=\pars{\frac{dP}{dQ}}^{-1}$ $Q$-a.e.

For two probability measures $Q\ll P$, the KL-divergence is
\[
\KL QP=\int\log\pars{\frac{dQ}{dP}}Q\pp{d\omega}=\E_{Q\pp{\omega}}\log\frac{dQ}{dP}.
\]

For a probability measure $Q$ and measure $\hat{P}$ (not necessarily
a probability measure) with $Q\ll\hat{P}$, the evidence lower bound
or ``ELBO'' is
\[
\elbo Q{\hat{P}}=-\E_{Q}\log\frac{dQ}{d\hat{P}}.
\]

When $Q\sim\hat{P},$ we can equivalently write $\elbo Q{\hat{P}}=\E_{Q}\log\frac{d\hat{P}}{dQ}$.

Let $(Z,\B)$ be a measurable space. Let $P_{z,x}$ be an unnormalized
distribution over $z$ representing the joint distribution over $\pp{z,x}$
for a fixed $x.$ Write either $P_{z,x}(B)$ or $P_{z,x}(z\in B)$
for the measure of $B\in\B$. Define 
\[
p(x)=P_{z,x}(Z)
\]
 to be the total measure or the normalization constant of $P_{z,x}$,\marginpar{}
and write $P_{z|x}(z\in B):=P_{z,x}(z\in B)/p(x)$ for the corresponding
normalized measure, which represents the conditional distribution
of $z$ given $x$. Henceforth, $x$ will \emph{always} denote a fixed
constant, and, for any $u$, the measure $P_{u,x}$ is unnormalized
with total measure $p(x)$.

The following gives a measure-theoretic version of the ``ELBO decomposition''
from \ref{eq:ELBO-decomp}.

\begin{restatable}{lem1}{decomp-measures}

\label{lem:ELBO-decomposition-measures}Given a probability measure
$Q$ and a measure $P_{z,x}$ on $(Z,\B)$, whenever $Q\ll P_{z,x}$
we have the following \textquotedbl ELBO decomposition\textquotedbl :
\[
\log p(x)=\elbo Q{P_{z,x}}+\KL Q{P_{z|x}}.
\]

\end{restatable}
\begin{proof}
It is easy to check that $\frac{dQ}{dP_{z|x}}=p(x)\frac{dQ}{dP_{z,x}}$.\footnote{$\int_{A}p(x)\frac{dQ}{dP_{z,x}}dP_{z|x}=\int_{A}\frac{dQ}{dP_{z,x}}dP_{z,x}=Q(z\in A).$}
Then

\begin{multline*}
\KL Q{P_{z|x}}=\E_{Q}\log\frac{dQ}{dP_{z|x}}=\E_{Q}\log\pars{p(x)\frac{dQ}{dP_{z,x}}}\\
=\log p(x)+\E_{Q}\log\frac{dQ}{dP_{z,x}}=\log p(x)-\elbo Q{P_{z,x}}.
\end{multline*}

Rearranging, we see the ELBO decomposition.
\end{proof}

\subsection{Conditional, Marginal, and Joint Distributions}

\paragraph{Standard Borel and product spaces}

We will assume that each relevant measure space is a \emph{standard
Borel space}, that is, isomorphic to a Polish space (a separable complete
metric space) with the Borel \emph{$\sigma$}-algebra. Standard Borel
spaces capture essentially all spaces that arise in practice in probability
theory \citep{Klenke_2014_ProbabilityTheoryComprehensive}. Let $(\Omega,\A)$
and $(Z,\B)$ be standard Borel spaces. The \emph{product space} \emph{$(\Omega\times Z,\A\otimes\B)$
}is the measurable space on $\Omega\times Z$ with $\A\otimes\B=\{A\times B:A\in\A,B\in\B\}$,
and is also a standard Borel space.

\paragraph{Conditional distributions}

We require tools to augment a distribution with a new random variable
and define the conditional distribution of one random variable with
respect to another. We begin with a Markov kernel, which we will use
to augment a distribution $P_{\omega}$ with a new random variable
to obtain a joint distribution $P_{\omega,z}$.

Formally, a \emph{Markov kernel \citep[Def. 8.24]{Klenke_2014_ProbabilityTheoryComprehensive}}
from $(\Omega,\A)$ to $(Z,\B)$ is a mapping $a(B\mid\omega)$ that
satisfies:
\begin{enumerate}
\item For fixed $\omega$, $a(B\mid\omega)$ is a probability measure on
$(Z,\B)$.
\item For fixed $B$, $a(B\mid\omega)$ is an $\A$-measurable function
of $\omega$.
\end{enumerate}
Let $P_{\omega}$ be a measure on $(\Omega,A)$ and $a(B\mid\omega)$
a Markov kernel from $(\Omega,A)$ to $(Z,\B)$. These define a unique
measure $P_{\omega,z}$ over the product space defined as

\[
P_{\omega,z}(\omega\in A,z\in B)=\int_{A}a(z\in B\mid\omega)P_{\omega}(d\omega),
\]
such that if $P_{\omega}$ is a probability measure, then $P_{\omega,z}$
is also a probability measure \citep[Cor. 14.23]{Klenke_2014_ProbabilityTheoryComprehensive}.

Alternately, we may have a joint distribution $P_{\omega,z}$ (a
measure on the product space $(\Omega\times Z,\A\otimes\B)$) and
want to define the marginals and conditionals. The \emph{marginal
distribution $P_{z}$} is the measure on $(Z,\B)$ with $P_{z}(z\in B)=P_{\omega,z}(\omega\in\Omega,z\in B)$,
and the marginal $P_{\omega}$ is defined analogously. Since the product
space is standard Borel \citep[Thm 14.8]{Klenke_2014_ProbabilityTheoryComprehensive},
there exists a \emph{regular conditional distribution} $P_{\omega|z}(\omega\in A\mid z)$
\citep[Def. 8.27, Thm. 8.36]{Klenke_2014_ProbabilityTheoryComprehensive},
which is a Markov kernel (as above) and satisfies the following for
all $A\in\A$, $B\in\B$:

\[
P_{\omega,z}(\omega\in A,z\in B)=\int_{B}P_{\omega|z}(\omega\in A\mid z)P_{z}(dz).
\]
The regular conditional distribution is unique up to null sets of
$P_{z}$.

The conditional distribution $P_{z|\omega}$ is defined analogously.

\subsection{KL Chain Rule}

Let $P_{\omega,z}$ and $Q_{\omega,z}$ be two probability measures
on the standard Borel product space $(\Omega\times Z,\A\otimes\B)$
with $Q_{\omega,z}\ll P_{\omega,z}$. The \emph{conditional} \emph{KL-divergence
}$\KL{Q_{\omega|z}}{P_{\omega|z}}$ is defined\footnote{While this (standard) notation for the divergence refers to ``$Q_{\omega|z}$''
it is a function of the joint $Q_{\omega,z}$ and similiarly for $P_{\omega,z}$.} as \citep[Ch. 5.3]{Gray_2011_Entropyinformationtheory}

\[
\KL{Q_{\omega|z}}{P_{\omega|z}}=\E_{Q_{\omega,z}}\pars{\frac{dQ_{\omega|z}}{dP_{\omega|z}}},
\]

where $\frac{dQ_{\omega|z}}{dP_{\omega|z}}(\omega|z)=\pars{\frac{dQ_{\omega,z}}{dP_{\omega,z}}(\omega,z)}\pars{\frac{dQ_{z}}{dP_{z}}(z)}^{-1}$
when $\frac{dQ_{z}}{dP_{z}}(z)>0$ and $1$ otherwise\marginpar{}.
When all densities exist, $\frac{dQ_{\omega|z}}{dP_{\omega|z}}(\omega|z)=\frac{q(\omega|z)}{p(\omega|z)}.$
Under the same conditions as above, we have the \emph{chain rule for
KL-divergence \citep[Lem. 5.3.1]{Gray_2011_Entropyinformationtheory}}

\[
\KL{Q_{\omega,z}}{P_{\omega,z}}=\KL{Q_{\omega}}{P_{\omega}}+\KL{Q_{\omega|z}}{P_{\omega|z}}=\KL{Q_{z}}{P_{z}}+\KL{Q_{z|\omega}}{P_{z|\omega}}.
\]

\subsection{Our Results}

Now consider a strictly positive estimator $R(\omega)$ over probability
space $(\Omega,\A,Q_{\omega})$ such that $\E_{Q_{\omega}}R=\int R\,dQ_{\omega}=p(x).$
We wish to define $\Pmc_{\omega,x}$ so that $\frac{d\Pmc_{\omega,x}}{dQ_{\omega}}=R$,
to justify interpreting $\E_{Q_{\omega}}\log R$ as an ELBO. This
is true when $R=\frac{d\Pmc_{\omega,x}}{dQ_{\omega}}$ is the Radon-Nikodym
derivative, i.e., a change of measure from $Q_{\omega}$ to $\Pmc_{\omega,x}$,
and is strictly positive. This leads to the definition
\[
\Pmc_{\omega,x}(\omega\in A)=\int_{A}R\,dQ_{\omega}.
\]

\begin{restatable}{lem1}{le-measures}

\label{lem:le-measures}Let $R(\omega)$ be an almost-everywhere positive
random variable on $(\Omega,\A,Q_{\omega})$ with $\E_{Q_{\omega}}R=p(x)$
and define $\Pmc_{\omega,x}(\omega\in A)=\int_{A}R\,dQ_{\omega}$.
The ELBO decomposition applied to $Q_{\omega}$ and $\Pmc_{\omega,x}$
gives: 
\[
\log p(x)=\E_{Q_{\omega}}\log R+\KL{Q_{\omega}}{\Pmc_{\omega|x}}.
\]

\end{restatable}
\begin{proof}
By construction, $R=\frac{d\Pmc_{\omega,x}}{dQ_{\omega}}$ and $\Pmc_{\omega,x}\sim Q_{\omega}$,
since $R$ is positive $Q$-a.e. Therefore $\E_{Q_{\omega}}\log R=\E_{Q_{\omega}}\log\frac{d\Pmc_{\omega,x}}{dQ}=\elbo Q{\Pmc_{\omega,x}}$,
where the final equality uses the definition of the ELBO for the case
when $\Pmc_{\omega,x}\sim Q_{\omega}$. Now apply \ref{lem:ELBO-decomposition-measures}
and the fact that $\elbo Q{\Pmc_{\omega,x}}=\E_{Q_{\omega}}\log R.$
\end{proof}
\ref{lem:le-measures} provides distributions $Q_{\omega}$ and $\Pmc_{\omega,x}$
so that $\E_{Q_{\omega}}\log R=\elbo{Q_{\omega}}{\Pmc_{\omega,x}}$,
which justifies maximizing the likelihood bound $\E_{Q_{\omega}}\log R$
as minimzing the KL-divergence from $Q_{\omega}$ to the ``target''
$\Pmc_{\omega|x}$. However, neither distribution contains the random
variable $z$ from the original target distribution $P_{z|x}$, so
the significance of \ref{lem:le-measures} on its own is unclear.
We now describe a way to couple $\Pmc_{\omega,x}$ to the original
target distribution using a Markov kernel $a(z\in B\mid\omega)$.
\begin{defn}
A valid \emph{estimator-coupling pair} with respect to target distribution
$P_{z,x}$ is an estimator $R(\omega)$ on probability space $(\Omega,\A,Q_{\omega})$
and Markov kernel $a(z\in B\mid\omega)$ from $(\Omega,\A)$ to $(Z,\B)$
such that:
\end{defn}

\[
\E_{Q_{\omega}}R(\omega)a(z\in B\mid\omega)=P_{z,x}(z\in B).
\]

\begin{restatable}{lem1}{pmc-omega-z-x-measures}

\label{lem:pmc-omega-z-x-measures}Assume $R(\omega)$ and $a(z\in B\mid\omega)$
are a valid estimator-coupling pair with respect to target $P_{z,x}$,
and define 
\[
\Pmc_{\omega,z,x}(\omega\in A,z\in B)=\int_{A}a(z\in B\mid\omega)\,R(\omega)\,Q_{\omega}(d\omega).
\]
 Then $\Pmc_{\omega,z,x}$ admits $P_{z,x}$ as a marginal, i.e.,
$\Pmc_{z,x}(z\in B)=P_{z,x}(z\in B)$.

\end{restatable}
\begin{proof}
We have

\[
\begin{aligned}\Pmc_{z,x}(z\in B) & =\Pmc_{\omega,z,x}(\omega\in\Omega,z\in B)\\
 & =\int_{\Omega}a(z\in B\mid\omega)R(\omega)Q_{\omega}(d\omega).\\
 & =\E_{Q_{\omega}}R(\omega)a(z\in B\mid\omega)\\
 & =P_{z,x}(z\in B).
\end{aligned}
\]

The second line uses the definition of $\Pmc_{\omega,z,x}$ . The
last line uses the definition of a valid estimator-coupling pair.
\end{proof}
\begin{thm}
Let $P_{z,x}$ be an unnormalized distribution with normalization
constant $p(x)$. Assume $R(\omega)$ and $a(z\in B\mid\omega)$ are
a valid estimator-coupling pair with respect to $P_{z,x}$. Define
$\Pmc_{\omega,z,x}$ as in \ref{lem:pmc-omega-z-x-measures} and define
\textup{$Q_{\omega,z}(\omega\in A,z\in B)=\int_{A}a(z\in B\mid\omega)\,Q_{\omega}(d\omega)$.
}Then
\end{thm}

\[
\log p(x)=\E_{Q_{\omega}}\log R+\KL{Q_{z}}{P_{z\mid x}}+\KL{Q_{\omega|z}}{\Pmc_{\omega|z,x}}.
\]

\begin{proof}
From \ref{lem:le-measures}, we have
\[
\log p(x)=\E_{Q_{\omega}}\log R+\KL{Q_{\omega}}{\Pmc_{\omega|x}}.
\]

We will show by two applications of the KL chain rule that the second
term can be expanded as

\begin{equation}
\KL{Q_{\omega}}{\Pmc_{\omega|x}}=\KL{Q_{z}}{P_{z\mid x}}+\KL{Q_{\omega|z}}{\Pmc_{\omega|z,x}},\label{eq:KL-chain-combined}
\end{equation}

which will complete the proof.

We first apply the KL chain rule as follows:

\begin{equation}
\KL{Q_{\omega,z}}{\Pmc_{\omega,z|x}}=\KL{Q_{\omega}}{\Pmc_{\omega|x}}+\underbrace{\KL{Q_{z|\omega}}{\Pmc_{z|\omega,x}}}_{=0}.\label{eq:KL-chain-1}
\end{equation}

We now argue that the second term is zero, as indicated in the equation.
Note from above that $\frac{d\Pmc_{\omega,x}}{dQ_{\omega}}=R$. It
is also true that $\frac{d\Pmc_{\omega,z,x}}{dQ_{\omega,z}}=R$. To
see this, observe that

\[
\begin{aligned}\int_{A\times B}R(\omega)\,Q_{\omega,z}(d\omega,dz) & =\int_{A}\Big(\int_{B}R(\omega)a(z\in dz\mid\omega)\Big)\,Q_{\omega}(d\omega)\\
 & =\int_{A}R(\omega)\Big(\int_{B}a(z\in dz\mid\omega)\Big)\,Q_{\omega}(d\omega)\\
 & =\int_{A}R(\omega)\,a(z\in B\mid\omega)\,Q_{\omega}(d\omega)\\
 & =\Pmc_{\omega,z,x}(\omega\in A,z\in B).
\end{aligned}
\]

The first equality above uses a version of Fubini's theorem for Markov
kernels \citep[Thm. 14.29]{klenke2013probability}. Because $\Pmc_{\omega,x}\sim Q_{\omega}$
it also follows that $\frac{dQ_{\omega}}{d\Pmc_{\omega,x}}=\frac{dQ_{\omega,z}}{d\Pmc_{\omega,z,x}}=\frac{1}{R}$.
Since the normalized distributions $\Pmc_{\omega|x}$ and $\Pmc_{\omega,z|x}$
differ from the unnormalized counterparts by the constant factor $p(x)$,
it is straightforward to see that $\frac{dQ_{\omega}}{d\Pmc_{\omega|x}}=\frac{dQ_{\omega,z}}{d\Pmc_{\omega,z|x}}=\frac{p(x)}{R}$.\footnote{$\int_{A}\frac{p(x)}{R}d\Pmc_{\omega|x}=\int_{A}\frac{1}{R}d\Pmc_{\omega,x}=Q_{\omega}(\omega\in A)$,
and similarly for $Q_{\omega,z}$ and $\Pmc_{\omega,z,x}$.} This implies that $\frac{d\Pmc_{z|\omega,x}}{dQ_{z|\omega}}=1$ a.e.,
which in turn implies that the conditional divergence $\KL{Q_{z|\omega}}{\Pmc_{z|\omega,x}}$
is equal to zero.

We next apply the chain rule the other way and use the fact that $\Pmc_{z|x}=P_{z|x}$
(\ref{lem:pmc-omega-z-x-measures}) to see that:

\begin{equation}
\KL{Q_{\omega,z}}{\Pmc_{\omega,z|x}}=\KL{Q_{z}}{\Pmc_{z|x}}+\KL{Q_{\omega|z}}{\Pmc_{\omega|z,x}}=\KL{Q_{z}}{P_{z|x}}+\KL{Q_{\omega|z}}{\Pmc_{\omega|z,x}}.\label{eq:KL-chain-2}
\end{equation}
Combining \ref{eq:KL-chain-1} and \ref{eq:KL-chain-2} we get \ref{eq:KL-chain-combined},
as desired.

\end{proof}
\clearpage{}

\section{Specific Variance Reduction Techniques\label{subsec:Specific-Variance-Reduction}}

\subsection{IID Mean}

As a simple example, consider the IID mean. Suppose $R_{0}\pp{\omega}$
and $a_{0}\pp{z|\omega}$ are valid under $Q_{0}$. If we define
\[
Q\pp{\omega_{1},\cdots,\omega_{M},m}=\frac{1}{M}\prod_{m=1}^{M}Q_{0}\pp{\omega_{m}}
\]
(with $\wr_{1},\cdots,\wr_{M}\sim Q_{0}$ i.i.d. and $\r m$ uniform
on $\left\{ 1,\cdots,M\right\} $) then this satisfies the condition
of \ref{thm:splitting-estimators} that $\wr_{\r m}\sim Q_{0}.$ Thus
we can define $R$ and $a$ as in \ref{eq:splitting-new-R} and \ref{eq:splitting-new-a},
to get that
\begin{eqnarray*}
R\pp{\omega_{1},\cdots,\omega_{M},m} & = & R_{0}\pp{\omega_{m}}\\
a(z|\omega_{1},\cdots,\omega_{M},m) & = & a_{0}(z|\omega_{m})
\end{eqnarray*}
are a valid estimator-coupling pair under $Q.$ Note that $Q\pp{m|\omega_{1},\cdots,\omega_{M}}=\frac{1}{M}$,
so if we apply \ref{thm:rao-black-new} to marginalize out $m$, we
get that
\begin{eqnarray*}
R\pp{\omega_{1},\cdots,\omega_{M}} & = & \E_{Q\pp{\r m|\wr_{1},\cdots,\wr_{M}}}R\pars{\omega_{1},\cdots,\omega_{M},\r m}\\
 & = & \frac{1}{M}\sum_{m=1}^{M}R\pars{\omega_{1},\cdots,\omega_{M},m}\\
 & = & \frac{1}{M}\sum_{m=1}^{M}R_{0}\pp{\omega_{m}}\\
a\pp{z|\omega_{1},\cdots,\omega_{M}} & = & \frac{1}{R\pp{\omega_{1},\cdots,\omega_{M}}}\E_{Q\pp{\r m|\omega_{1},\cdots,\omega_{M}}}\bracs{R\pars{\omega_{1},\cdots,\omega_{M},\r m}a(z|\omega_{1},\cdots,\omega_{M},\r m)}\\
 & = & \frac{1}{R\pp{\omega_{1},\cdots,\omega_{M}}}\frac{1}{M}\sum_{m=1}^{M}\bracs{R\pars{\omega_{1},\cdots,\omega_{M},m}a(z|\omega_{1},\cdots,\omega_{M},m)}\\
 & = & \frac{1}{\frac{1}{M}\sum_{m=1}^{M}R_{0}\pp{\omega_{m}}}\frac{1}{M}\sum_{m=1}^{M}\bracs{R_{0}\pp{\omega_{m}}a_{0}(z|\omega_{m})}\\
 & = & \frac{\sum_{m=1}^{M}\bracs{R_{0}\pp{\omega_{m}}a_{0}(z|\omega_{m})}}{\sum_{m=1}^{M}R_{0}\pp{\omega_{m}}}.
\end{eqnarray*}

These are exactly the forms for $R\pp{\cdot}$ and $a\pp{z|\cdot}$
shown in the table.

\subsection{Stratified Sampling}

As another example, take stratified sampling. The estimator-coupling
pair can be derived similiarly to with the i.i.d. mean. For simplicity,
we assume here one sample in each strata ($N_{m}=1$). Suppose $\Omega_{1}\cdots\Omega_{M}$
partition the state-space and define
\[
Q\pp{\omega_{1},\cdots,\omega_{M},m}=\frac{1}{M}\prod_{k=1}^{M}\frac{Q_{0}\pp{\omega_{k}}I\pp{\omega_{k}\in\Omega_{m}}}{\mu\pp k}\times\mu\pp m,\ \ \ \mu\pp m=\E_{Q_{0}\pp{\wr}}I\pp{\wr\in\Omega_{m}}.
\]
This again satisfies the condition of \ref{thm:splitting-estimators},
so \ref{eq:splitting-new-R} and \ref{eq:splitting-new-a} give that
\begin{eqnarray*}
R\pp{\omega_{1},\cdots,\omega_{M},m} & = & R_{0}\pp{\omega_{m}}\\
a(z|\omega_{1},\cdots,\omega_{M},m) & = & a_{0}(z|\omega_{m})
\end{eqnarray*}
is a valid estimator-coupling pair with respect to $Q$. Note that
$Q\pp{m\vert\omega_{1},\cdots,\omega_{M}}=\mu\pp m$, so if we apply
\ref{thm:rao-black-new} to marginalize out $m$, we get that
\begin{eqnarray*}
R\pp{\omega_{1},\cdots,\omega_{M}} & = & \E_{Q\pp{\r m|\omega_{1},\cdots,\omega_{M}}}R\pars{\omega_{1},\cdots,\omega_{M},\r m}\\
 & = & \sum_{m=1}^{M}\mu\pp m\ R\pars{\omega_{1},\cdots,\omega_{M},m}\\
 & = & \sum_{m=1}^{M}\mu\pp mR_{0}\pp{\omega_{m}}\\
a\pp{z|\omega_{1},\cdots,\omega_{M}} & = & \frac{1}{R\pp{\omega_{1},\cdots,\omega_{M}}}\E_{Q\pp{\r m|\omega_{1},\cdots,\omega_{M}}}\bracs{R\pars{\omega_{1},\cdots,\omega_{M},\r m}a(z|\omega_{1},\cdots,\omega_{M},\r m)}\\
 & = & \frac{1}{R\pp{\omega_{1},\cdots,\omega_{M}}}\sum_{m=1}^{M}\mu\pp mR\pars{\omega_{1},\cdots,\omega_{M},m}a(z|\omega_{1},\cdots,\omega_{M},m)\\
 & = & \frac{\sum_{m=1}^{M}\mu\pp mR_{0}\pp{\omega_{m}}a_{0}(z|\omega_{m})}{\sum_{m=1}^{M}\mu\pp mR_{0}\pp{\omega_{m}}}.
\end{eqnarray*}

Again, this is the form shown in the table.

\clearpage{}

\section{Proofs for Deriving Couplings (\ref{sec:Deriving-Couplings})}

\splitting*
\begin{proof}
Substitute the definitions of $R$ and $a$ to get that

\begin{eqnarray*}
\E_{Q\pp{\wr_{1},\cdots,\wr_{M},\r m}}R\pp{\omega_{1},\cdots,\omega_{M},m}a(z|\omega_{1},\cdots,\omega_{M},m) & = & \E_{Q\pp{\wr_{1},\cdots,\wr_{M},\r m}}R_{0}\pp{\wr_{\r m}}a_{0}(z|\wr_{\r m})\\
 & = & \E_{Q_{0}\pp{\wr}}R_{0}\pp{\wr}a_{0}(z|\wr)\\
 & = & p\pp{z,x},
\end{eqnarray*}
which is equivalent to the definition of $R$ and $a$ being a valid
estimator-coupling pair. The second line follows from the assumption
on $Q\pp{\omega_{1},\cdots,\omega_{M},m}.$
\end{proof}
\raoblack*
\begin{proof}
Substitute the definition of $a$ to get that
\begin{eqnarray*}
\E_{Q\pp{\wr}}R\pp{\wr}a\pp{z|\wr} & = & \E_{Q\pp{\wr}}R\pp{\wr}\frac{1}{R\pp{\wr}}\E_{Q_{0}\pp{\vr|\wr}}\bb{R_{0}\pp{\wr,\vr}a_{0}\pp{z|\wr,\vr}}\\
 & = & \E_{Q_{0}\pp{\wr,\vr}}\bb{R_{0}\pp{\wr,\vr}a_{0}\pp{z|\wr,\vr}}\\
 & = & p\pp{z,x},
\end{eqnarray*}
which is equivalent to $R$ and $a$ being a valid estimator-coupling
pair under $R\pp{\omega}.$ The last line follows from the fact that
$R_{0}$ and $a_{0}$ are a valid estimator-coupling pair under $Q_{0}$.
\end{proof}

\clearpage{}

\section{Additional Experimental Results}

\begin{figure}
\noindent\begin{minipage}[t]{1\columnwidth}%
\includegraphics[viewport=0bp 245.0309bp 506bp 593bp,clip,width=1\columnwidth]{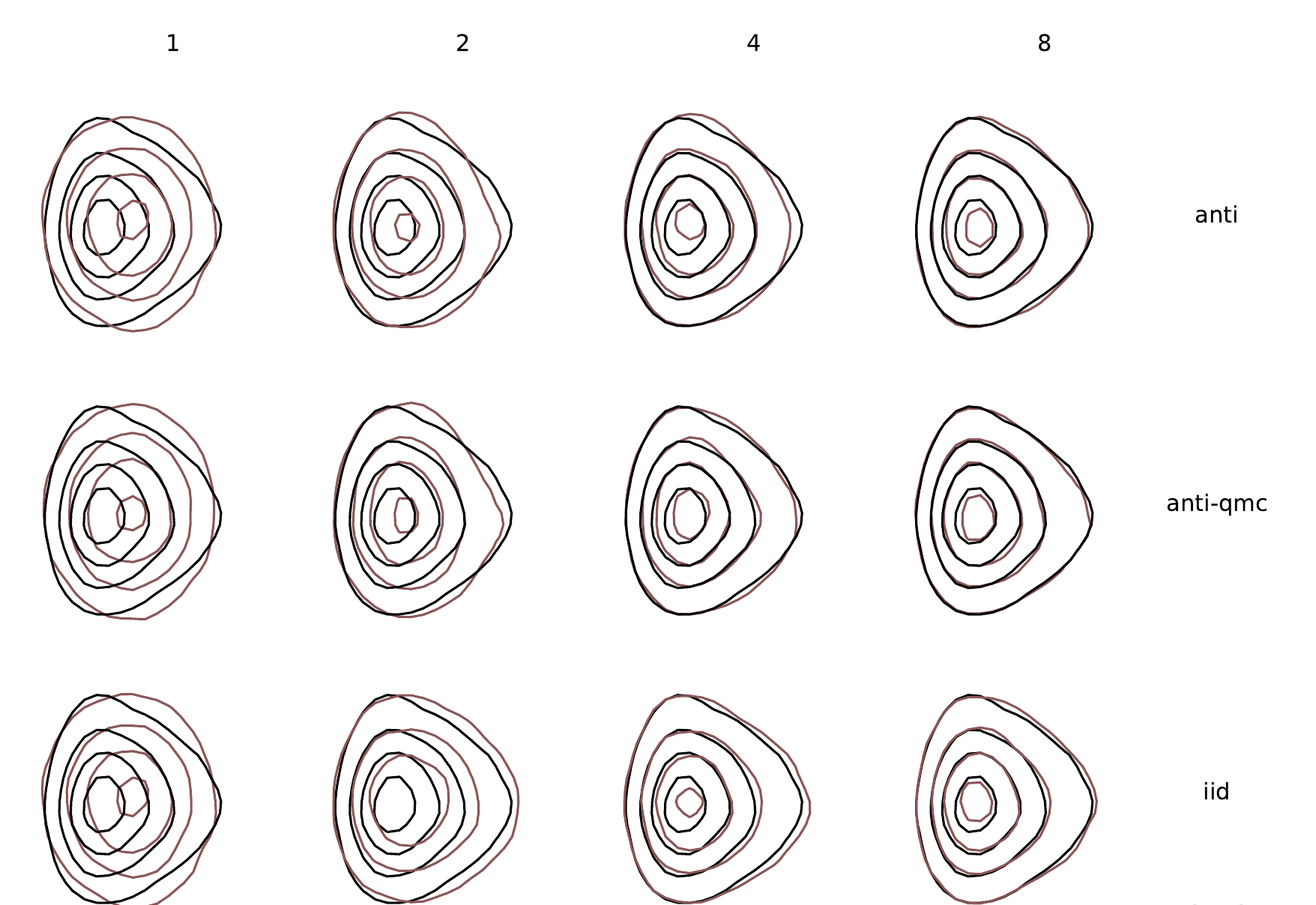}%
\end{minipage}

\caption{The target density $p\protect\pp{z|x}$ and the approximation $\protect\qmc\protect\pp{z|x}$
produced by various sampling methods (row) with various $M$ (columns).
The dark curves show isocontours of kernel density estimate for samples
generated using Stan and projected to the first two principal components.
The darker curves show isocontours for the process applied to samples
from $\protect\qmc\protect\pp{z|x}$. Antithetic sampling is visibly
(but subtly) better than iid for $M=2$ while the combination of quasi-Monte
Carlo and antithetic sampling is (still more subtly) best for $M=8$.
\label{fig:2d-approx}}
\end{figure}

\ref{fig:aggregate-statistics} shows additional aggregate statistics
of ELBO and posterior variance error for different methods across
model from the Stan library.

\begin{figure*}
\begin{centering}
\noindent\begin{minipage}[t]{1\textwidth}%
\includegraphics[viewport=0bp 36.5625bp 560bp 218.6719bp,clip,width=1\textwidth]{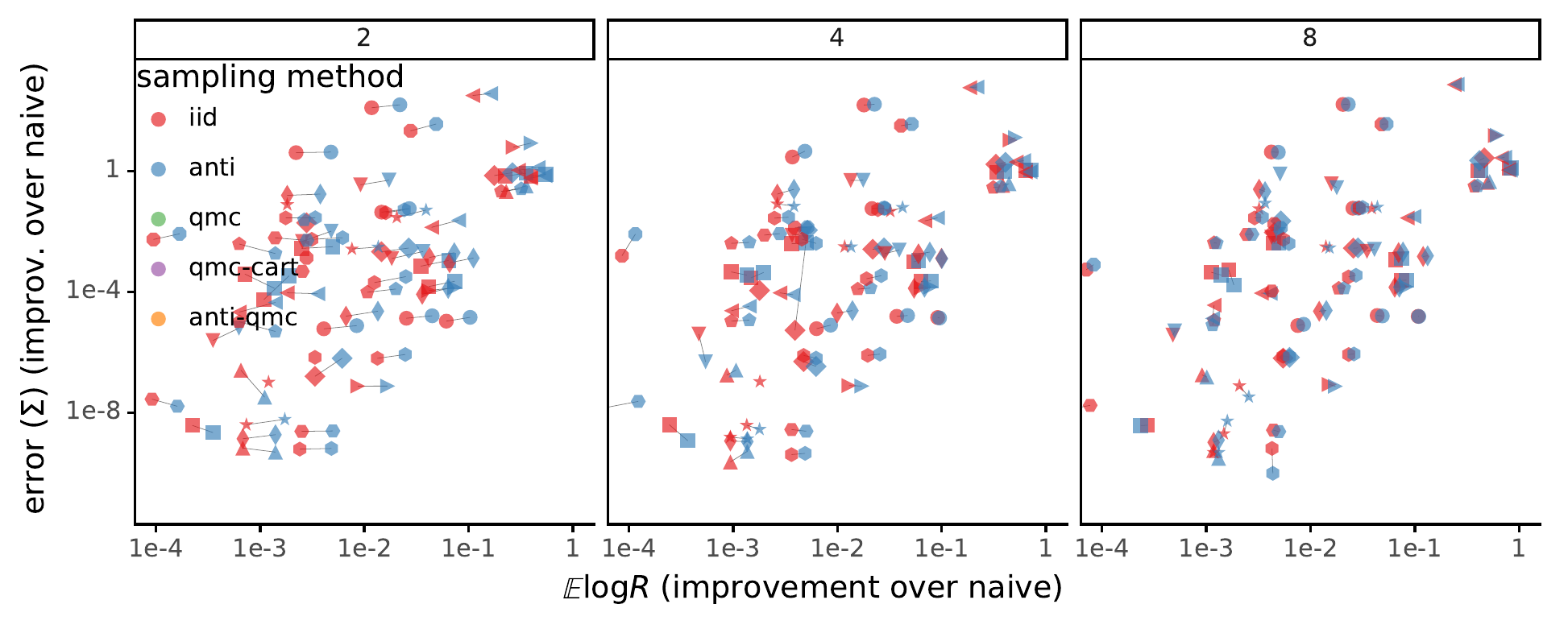}

\includegraphics[viewport=0bp 36.5625bp 560bp 203.2031bp,clip,width=1\textwidth]{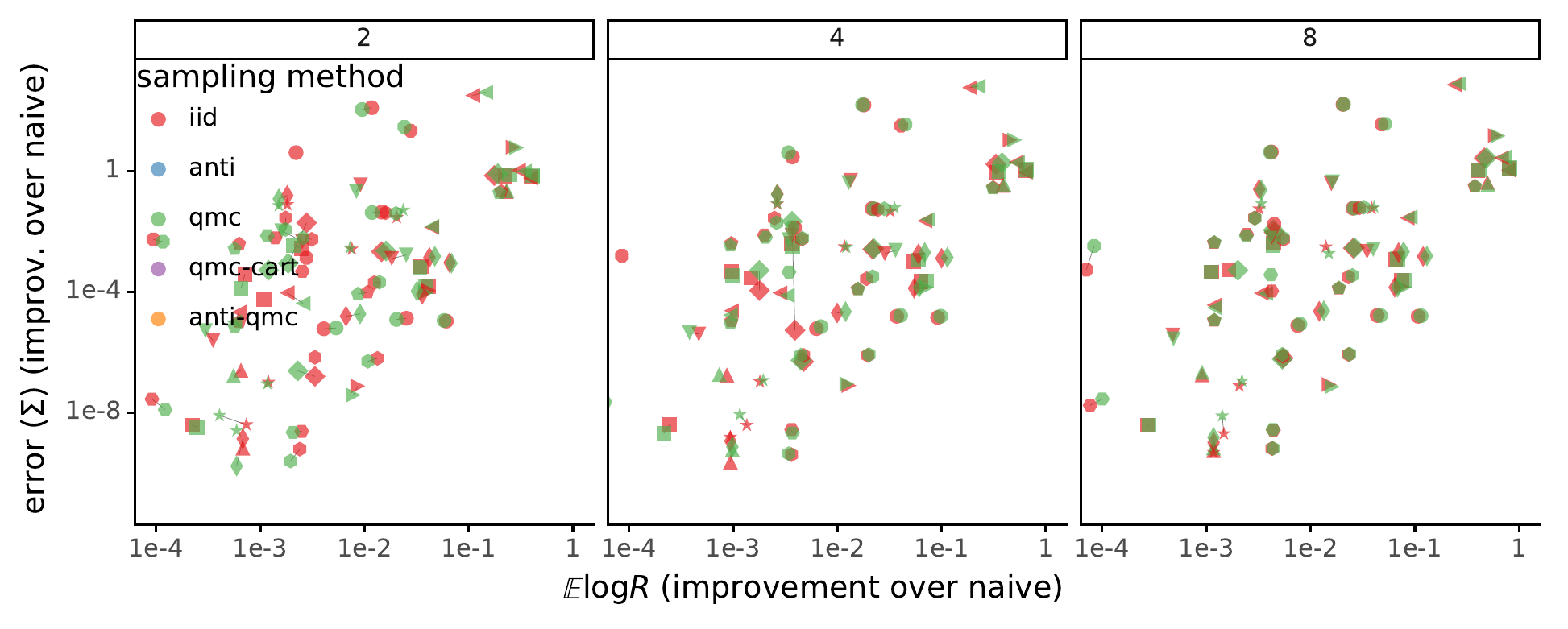}

\includegraphics[viewport=0bp 36.5625bp 560bp 203.2031bp,clip,width=1\textwidth]{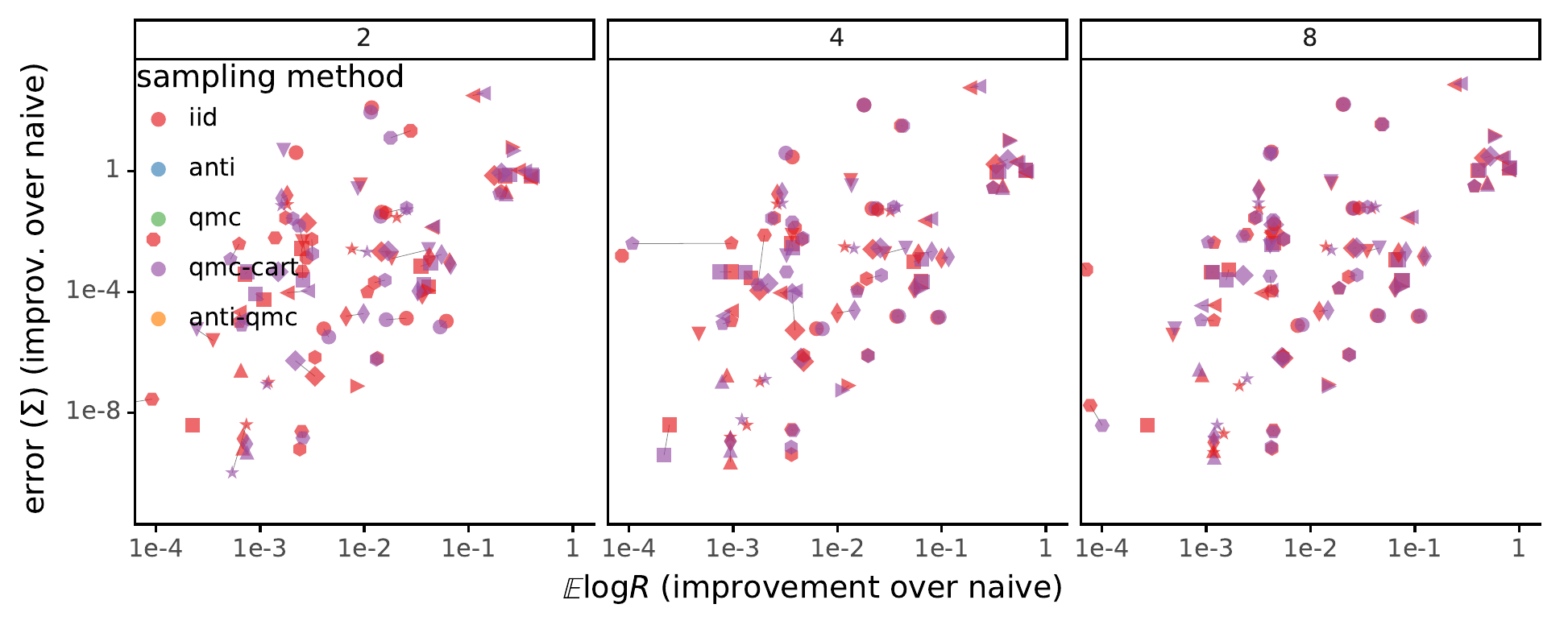}

\includegraphics[viewport=0bp 0bp 560bp 203.2031bp,clip,width=1\textwidth]{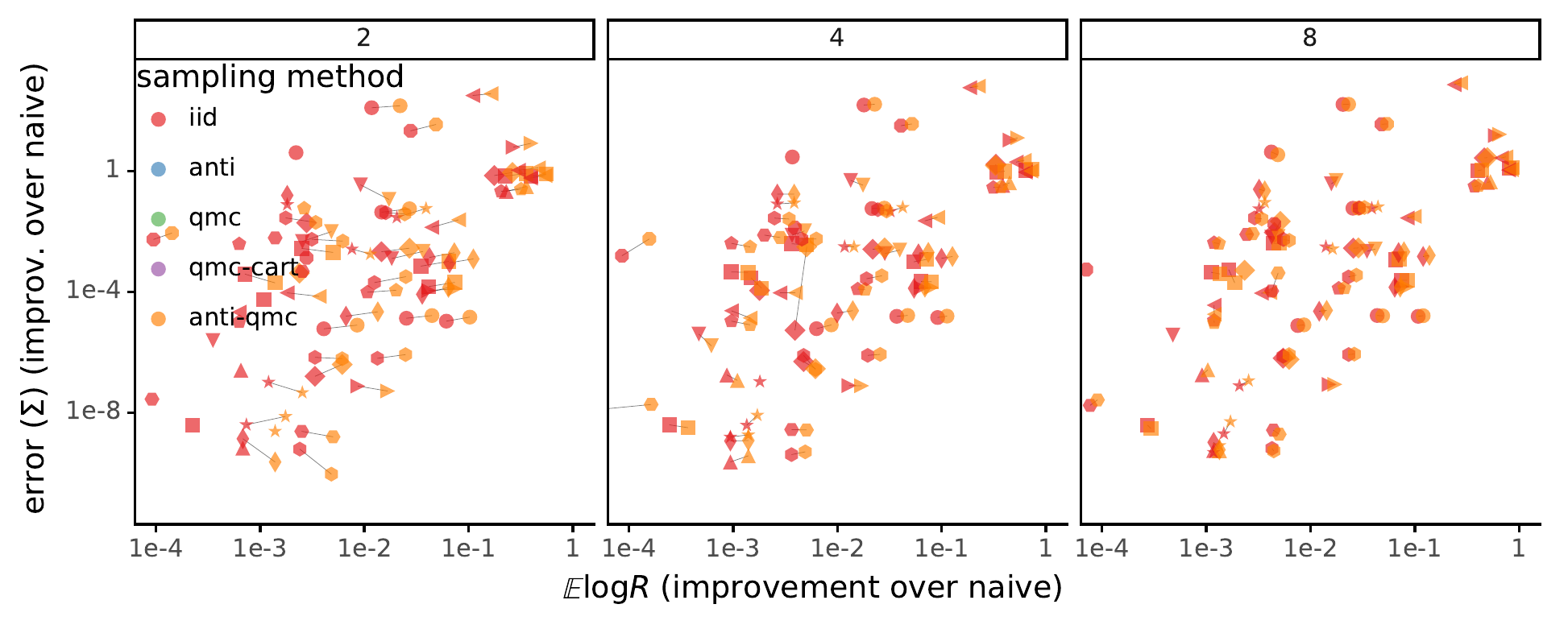}%
\end{minipage}
\par\end{centering}
\caption{\textbf{How much do methods improve over naive VI in likelihood bound
(x-axis) and in estimating posterior variance (y-axis)}? Each point
corresponds to a model from the Stan library, with a random shape.
Each plot compares \texttt{iid} sampling against some other strategy.
From top, these are antithetic sampling (\texttt{anti}), Quasi-Monte
Carlo, either using an elliptical mapping (\texttt{qmc}) or a Cartesian
mapping (\texttt{qmc-cart}), and antithetic sampling after an elliptical
mapping (\texttt{anti-qmc}). The columns correspond to using $M=2,4$
and $8$ samples for each estimate. \protect \linebreak{}
\textbf{Conclusions}: Improvements in ELBO and error are correlated.
Improvements converge to those of \texttt{iid} for larger $M$, as
all errors decay towards zero. Different sampling methods are best
on different datasets. A few cases are not plotted where the measured
``improvement'' was negative (if naive VI has near-zero error, or
due to local optima).\label{fig:aggregate-statistics}}
\end{figure*}

\clearpage{}

\section{Full Results For All Models}

\begin{figure}
\includegraphics[width=0.33\columnwidth]{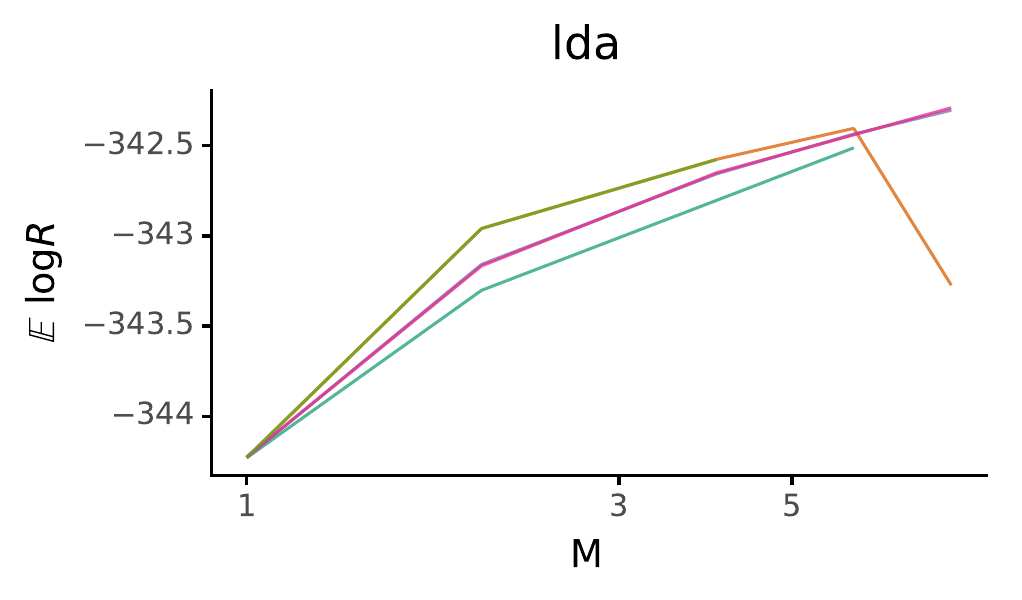}\includegraphics[width=0.33\columnwidth]{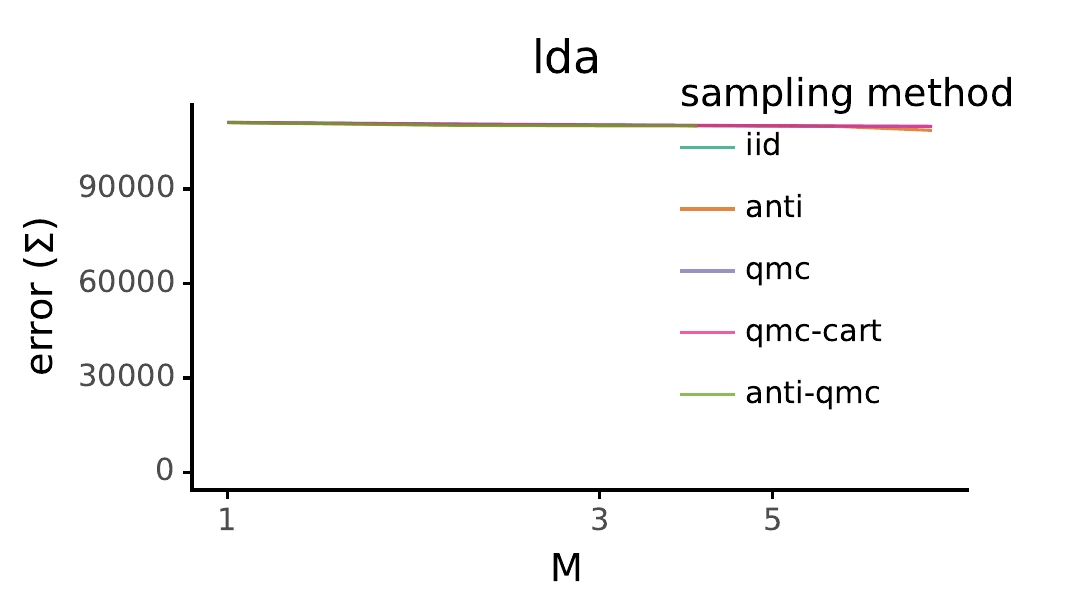}\includegraphics[width=0.33\columnwidth]{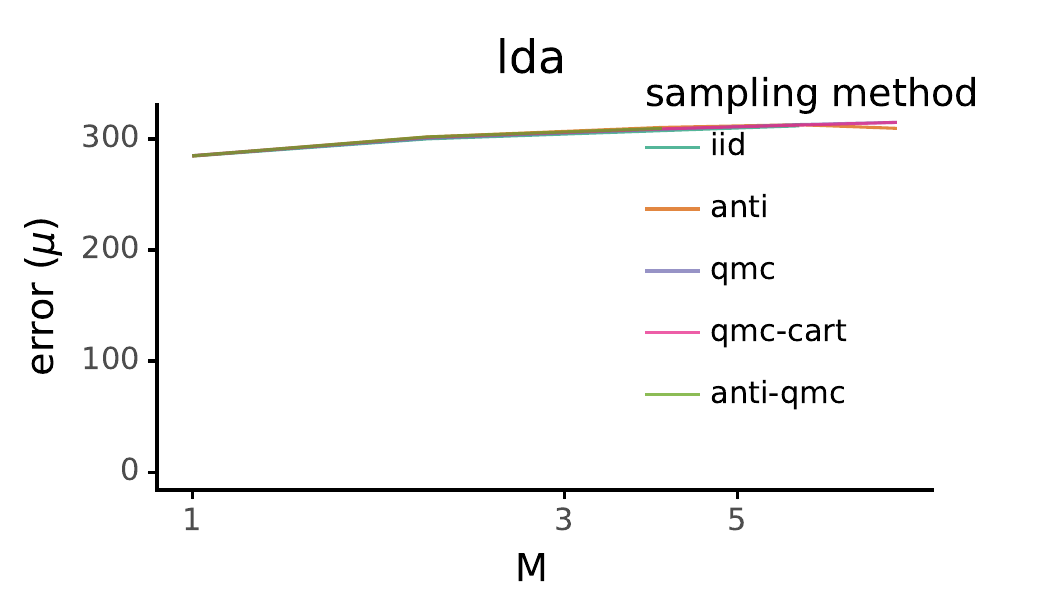}\linebreak{}

\includegraphics[width=0.33\columnwidth]{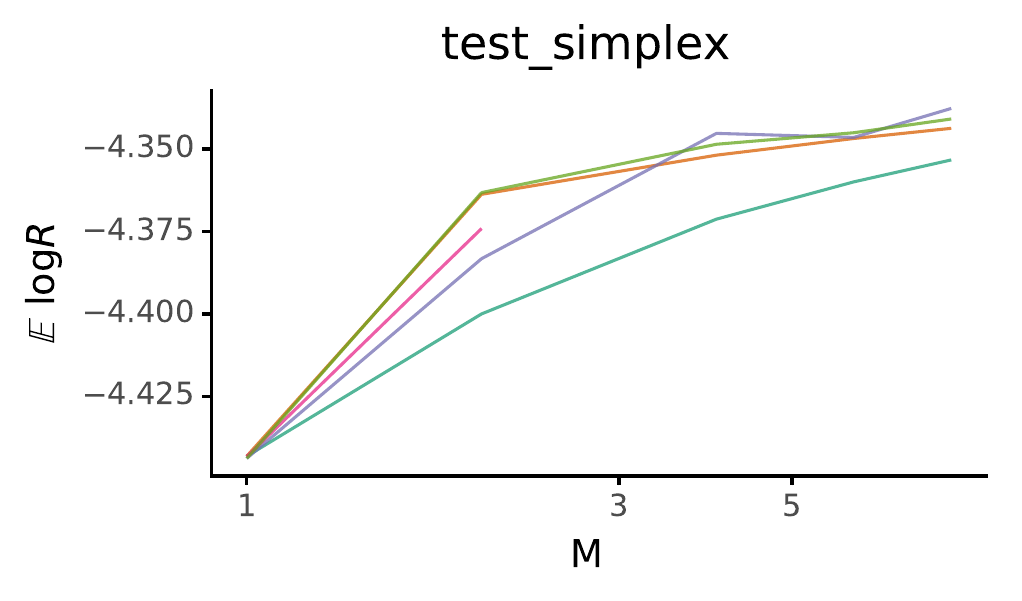}\includegraphics[width=0.33\columnwidth]{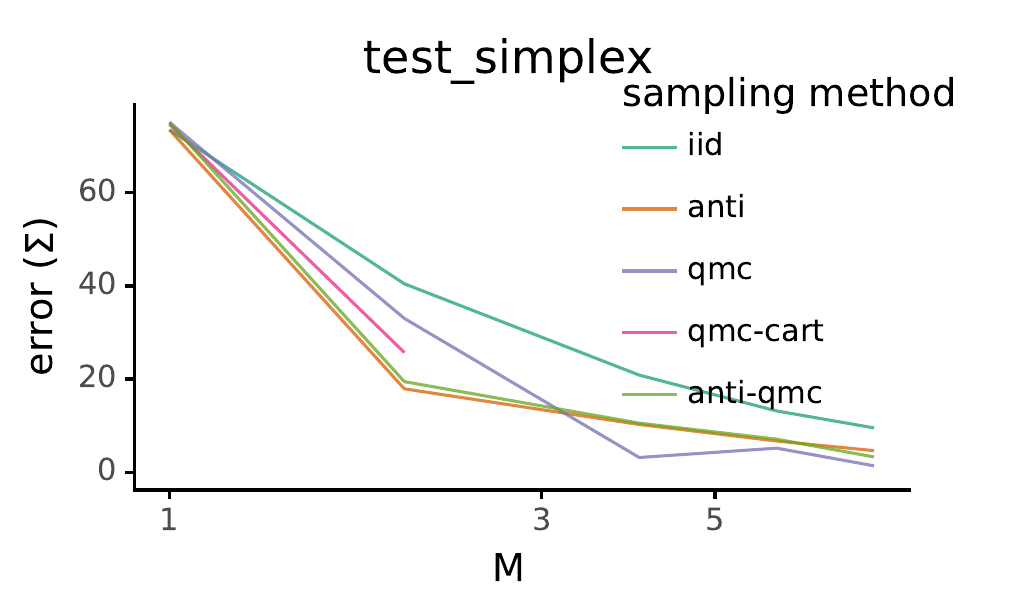}\includegraphics[width=0.33\columnwidth]{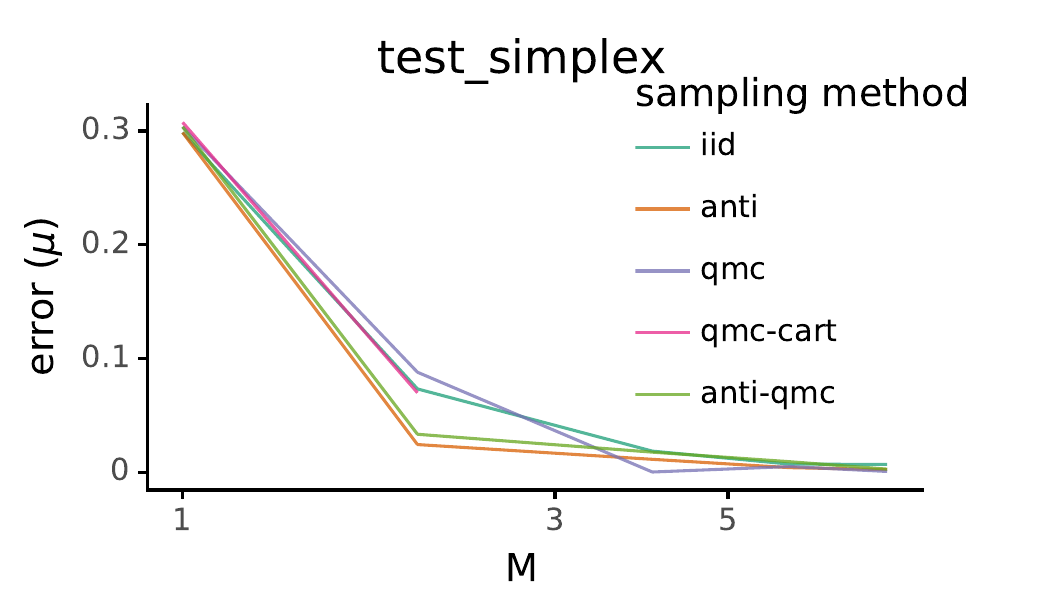}\linebreak{}

\includegraphics[width=0.33\columnwidth]{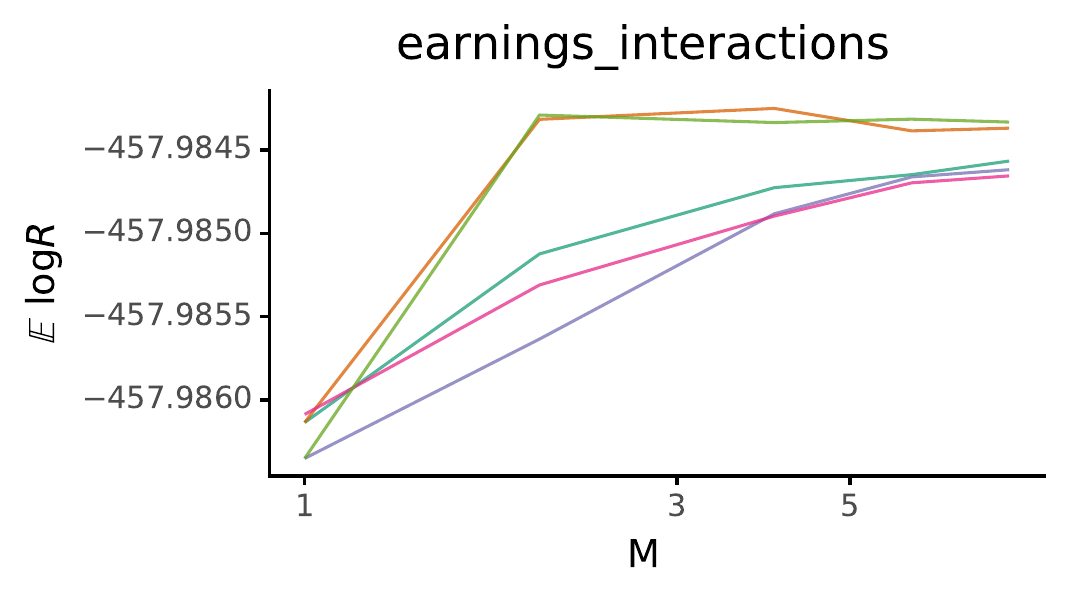}\includegraphics[width=0.33\columnwidth]{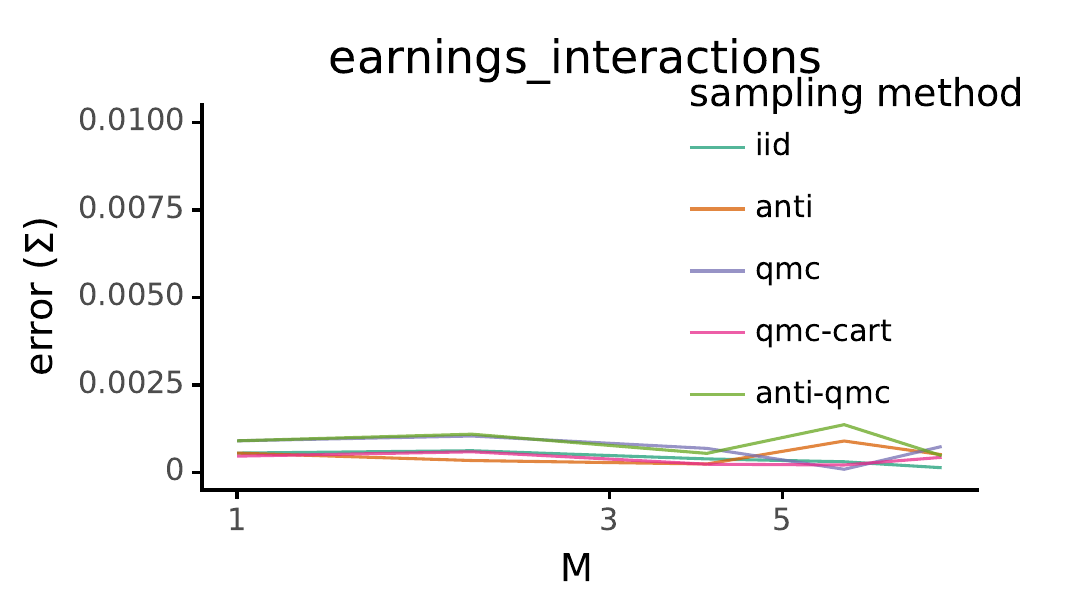}\includegraphics[width=0.33\columnwidth]{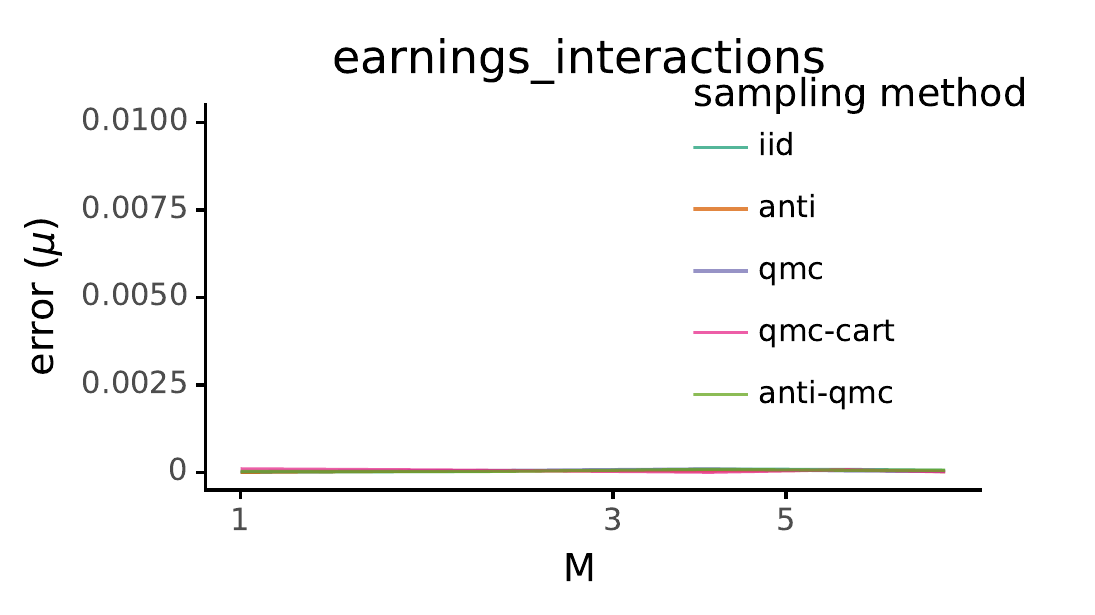}\linebreak{}

\includegraphics[width=0.33\columnwidth]{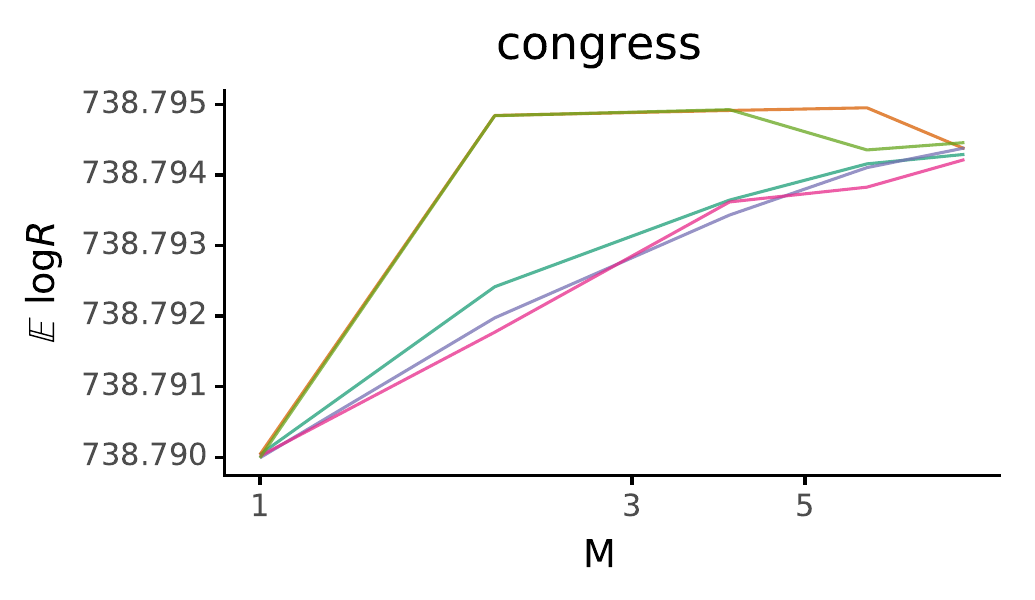}\includegraphics[width=0.33\columnwidth]{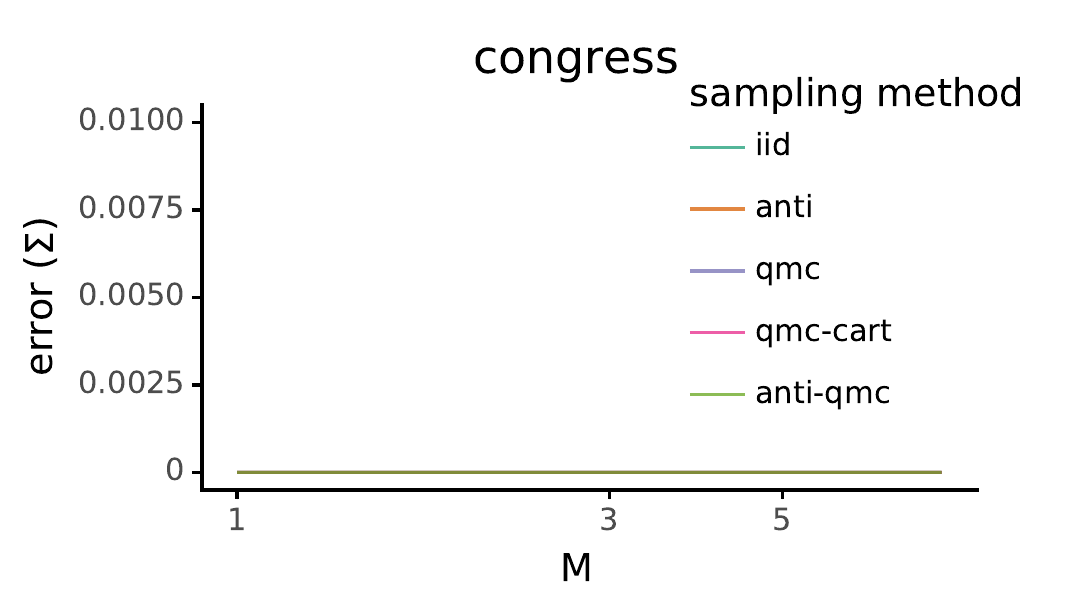}\includegraphics[width=0.33\columnwidth]{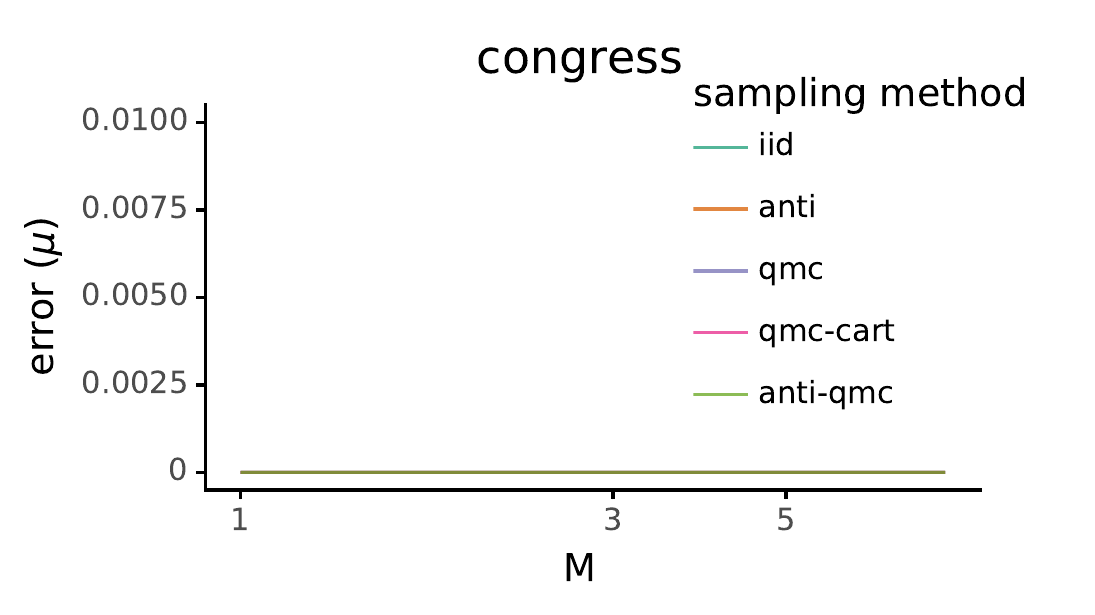}\linebreak{}

\includegraphics[width=0.33\columnwidth]{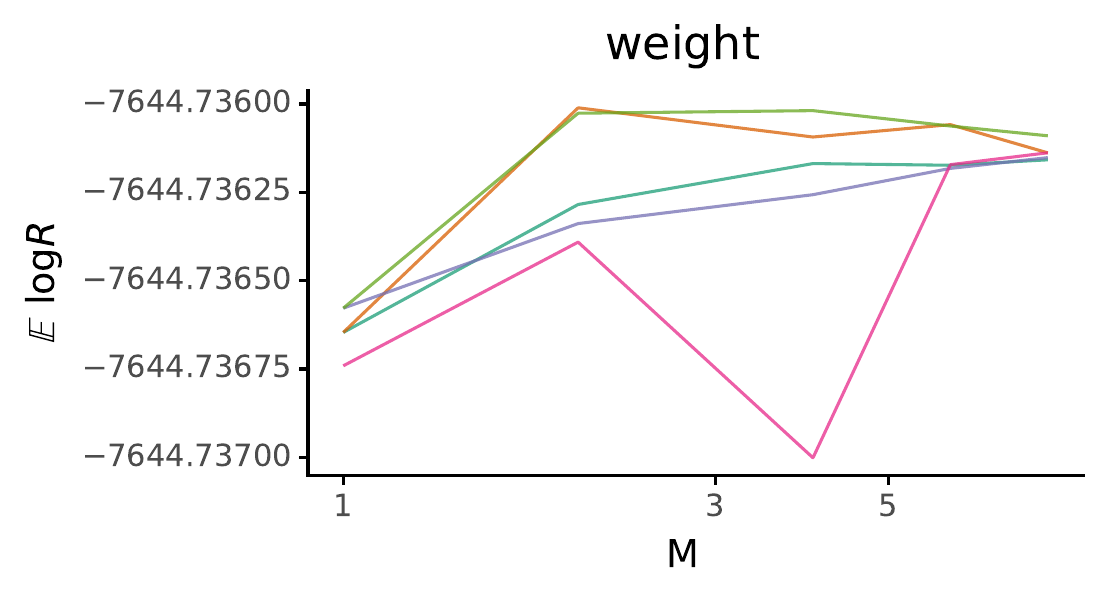}\includegraphics[width=0.33\columnwidth]{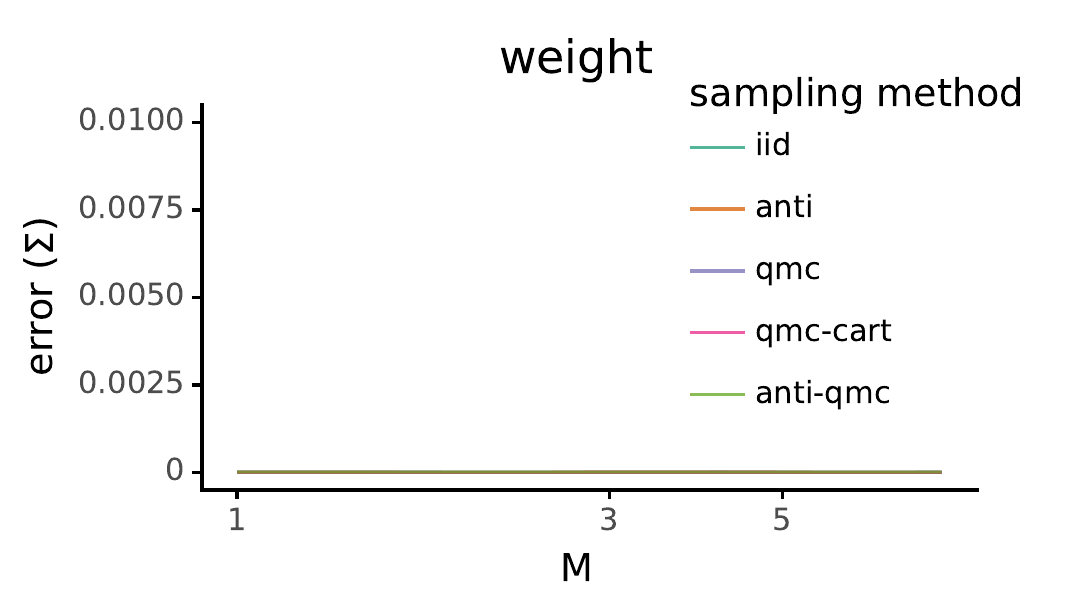}\includegraphics[width=0.33\columnwidth]{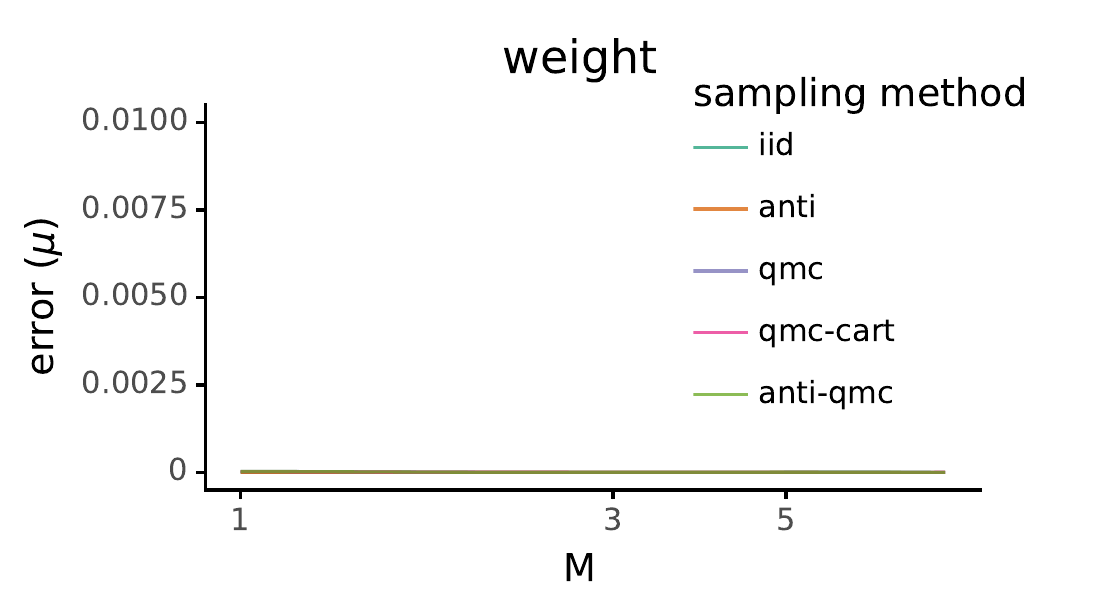}\linebreak{}

\includegraphics[width=0.33\columnwidth]{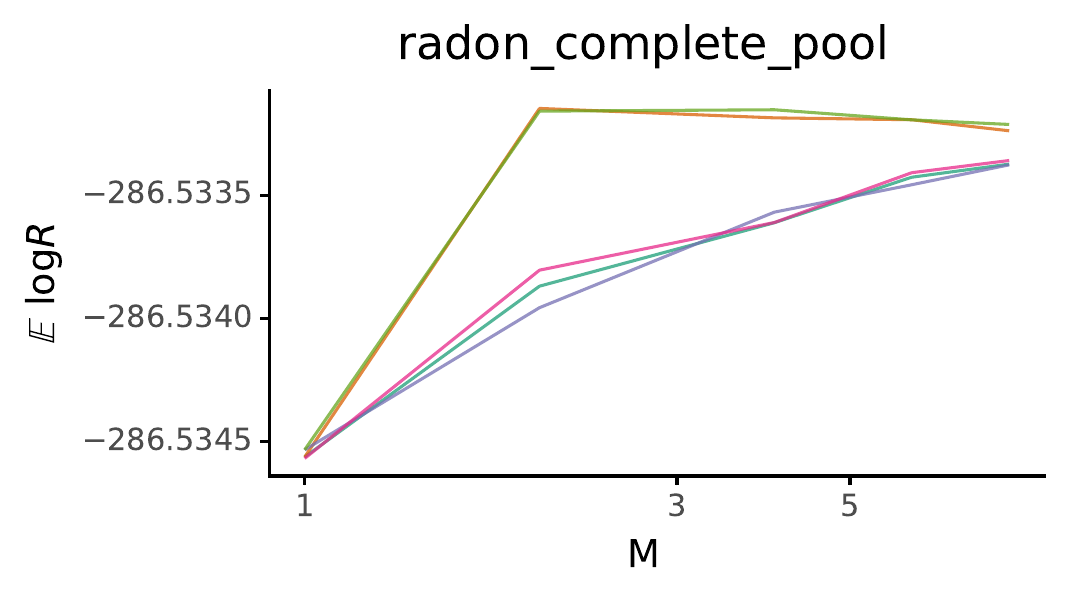}\includegraphics[width=0.33\columnwidth]{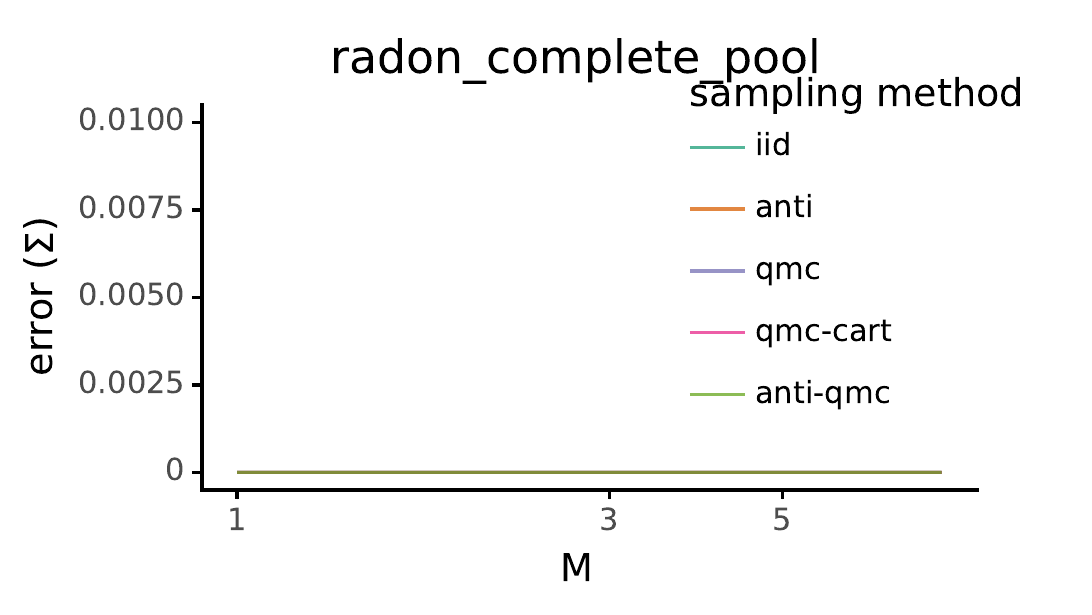}\includegraphics[width=0.33\columnwidth]{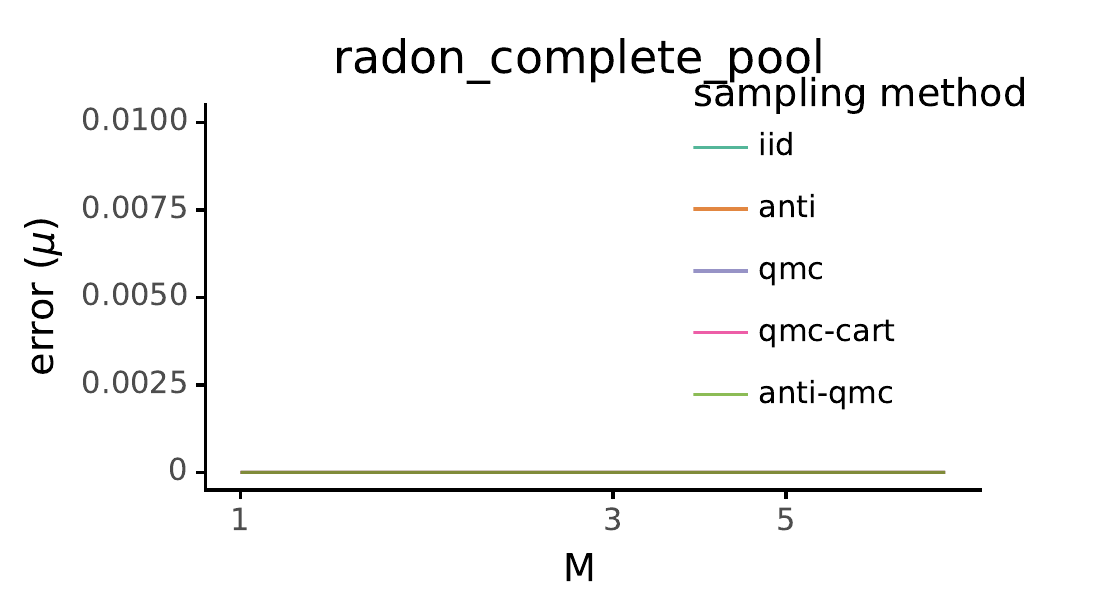}\linebreak{}

\includegraphics[width=0.33\columnwidth]{final_swarm_figs/stan12_1_elbos}\includegraphics[width=0.33\columnwidth]{final_swarm_figs/stan12_1_err_Sigma}\includegraphics[width=0.33\columnwidth]{final_swarm_figs/stan12_1_err_mu}\linebreak{}

\caption{\textbf{Across all models, improvements in likelihood bounds correlate
strongly with improvements in posterior accuracy. Better sampling
methods can improve both.}}
\end{figure}

\begin{figure}
\includegraphics[width=0.33\columnwidth]{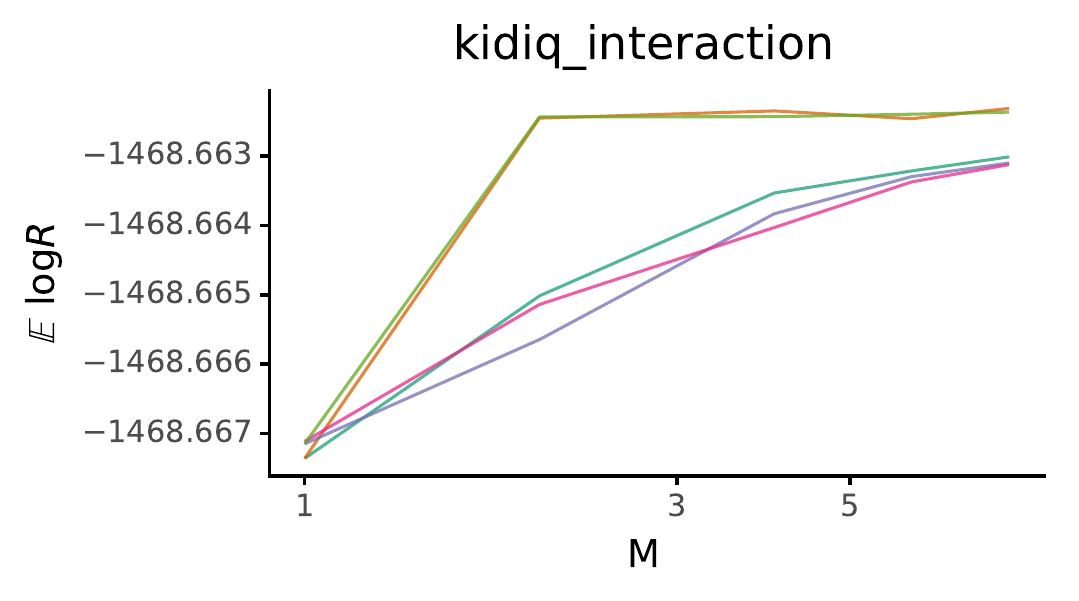}\includegraphics[width=0.33\columnwidth]{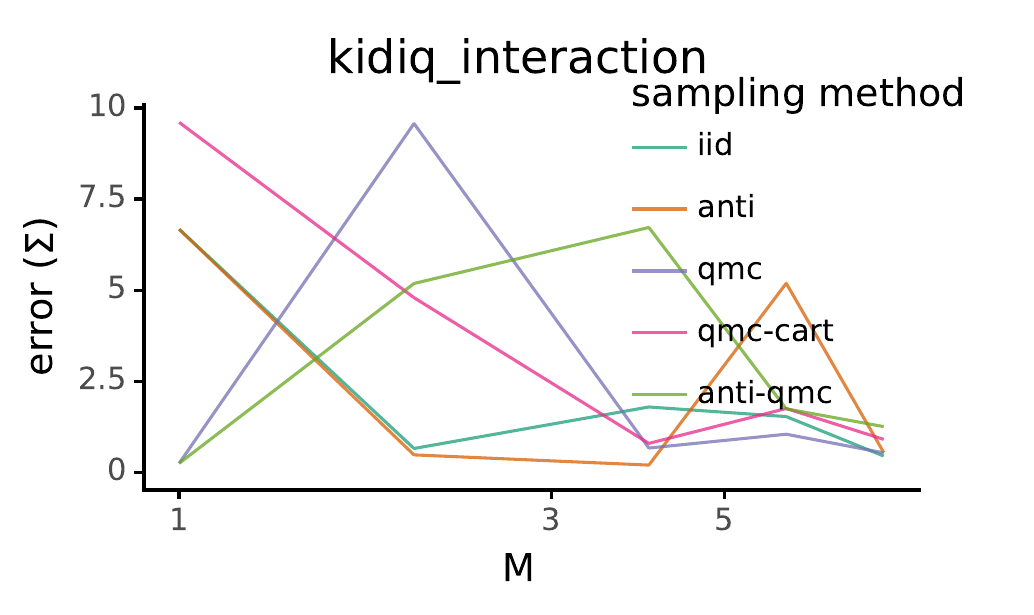}\includegraphics[width=0.33\columnwidth]{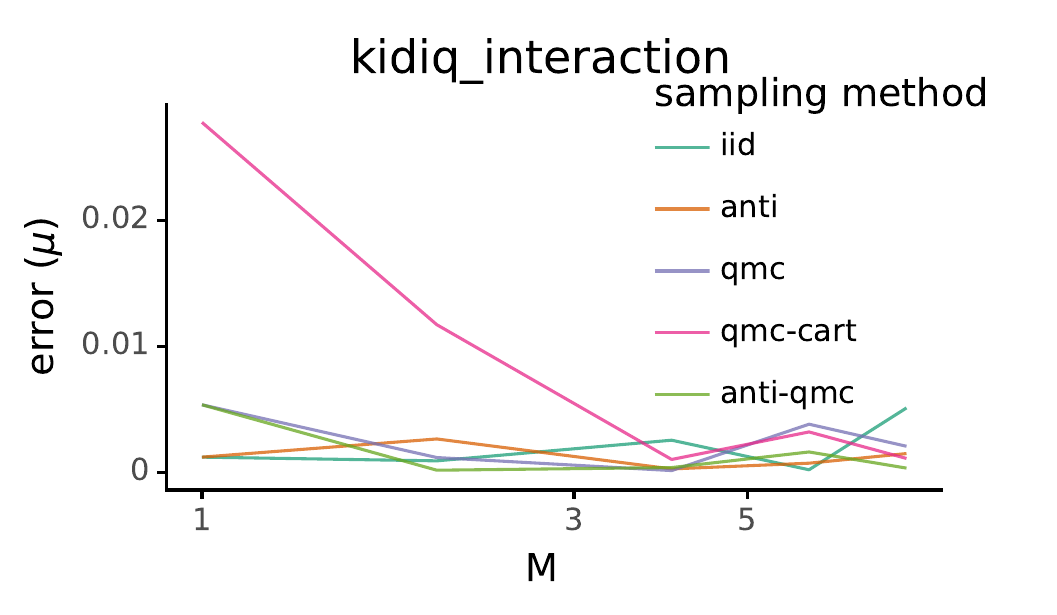}\linebreak{}

\includegraphics[width=0.33\columnwidth]{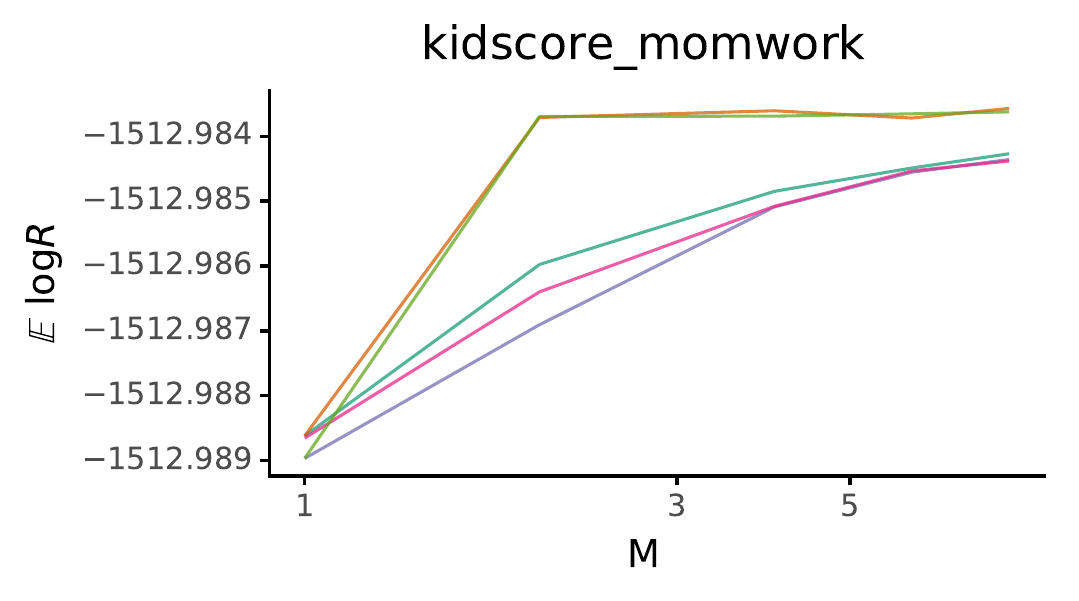}\includegraphics[width=0.33\columnwidth]{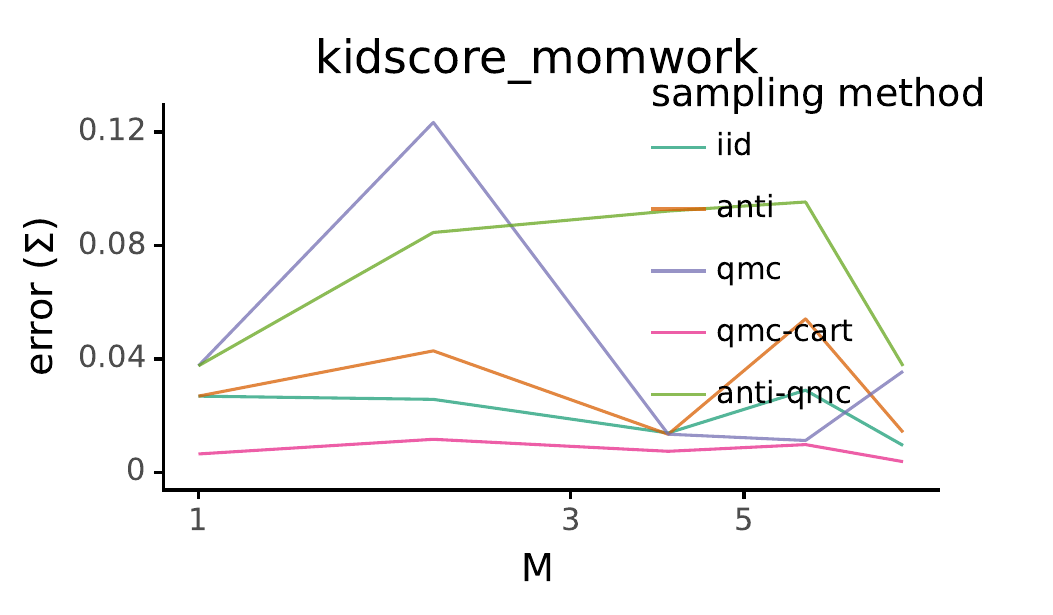}\includegraphics[width=0.33\columnwidth]{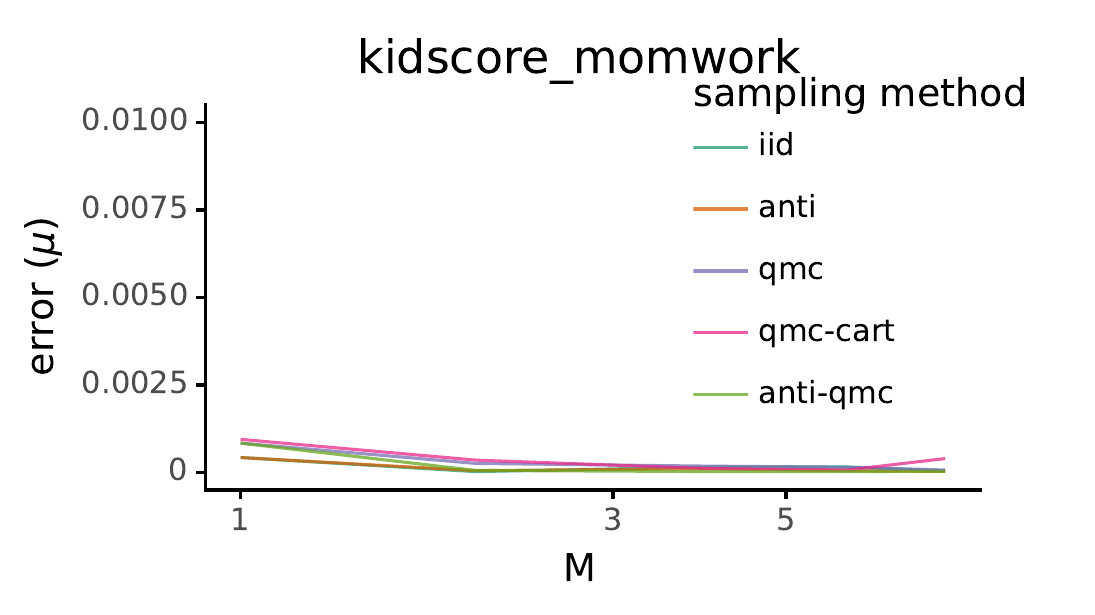}\linebreak{}

\includegraphics[width=0.33\columnwidth]{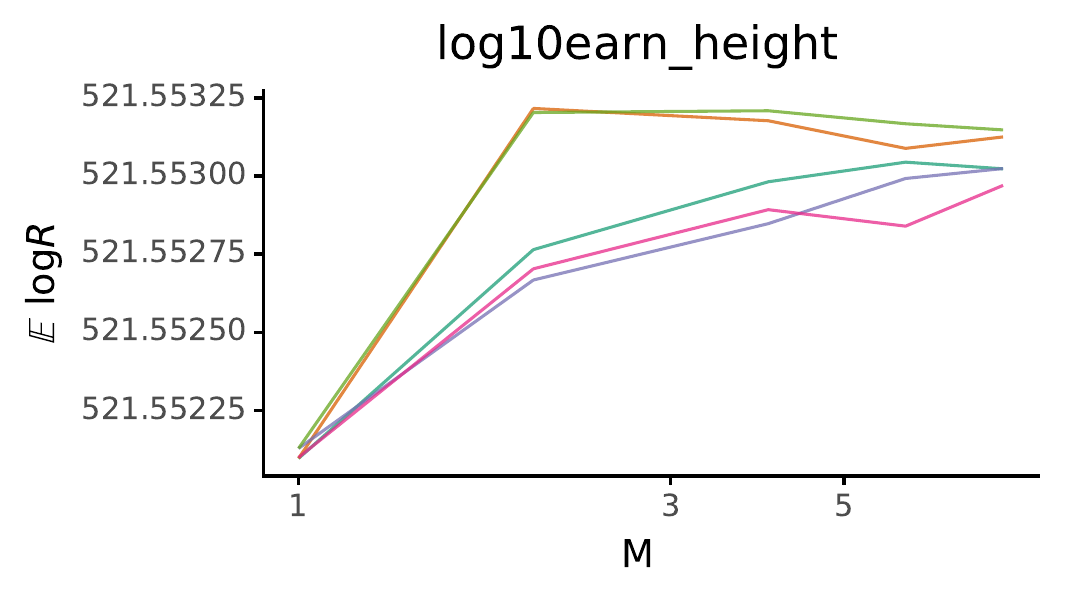}\includegraphics[width=0.33\columnwidth]{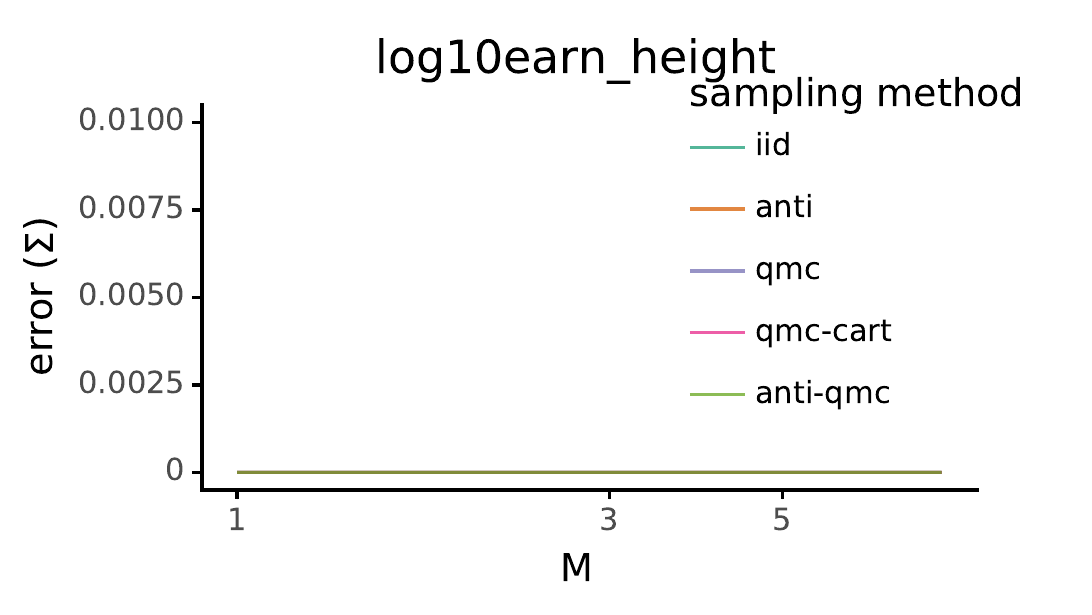}\includegraphics[width=0.33\columnwidth]{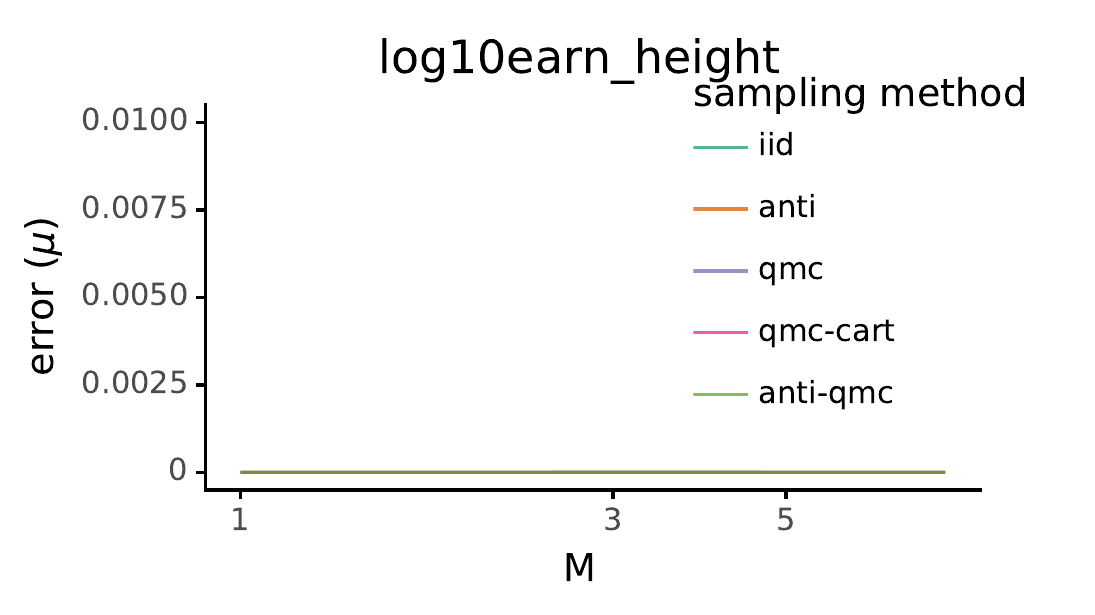}\linebreak{}

\includegraphics[width=0.33\columnwidth]{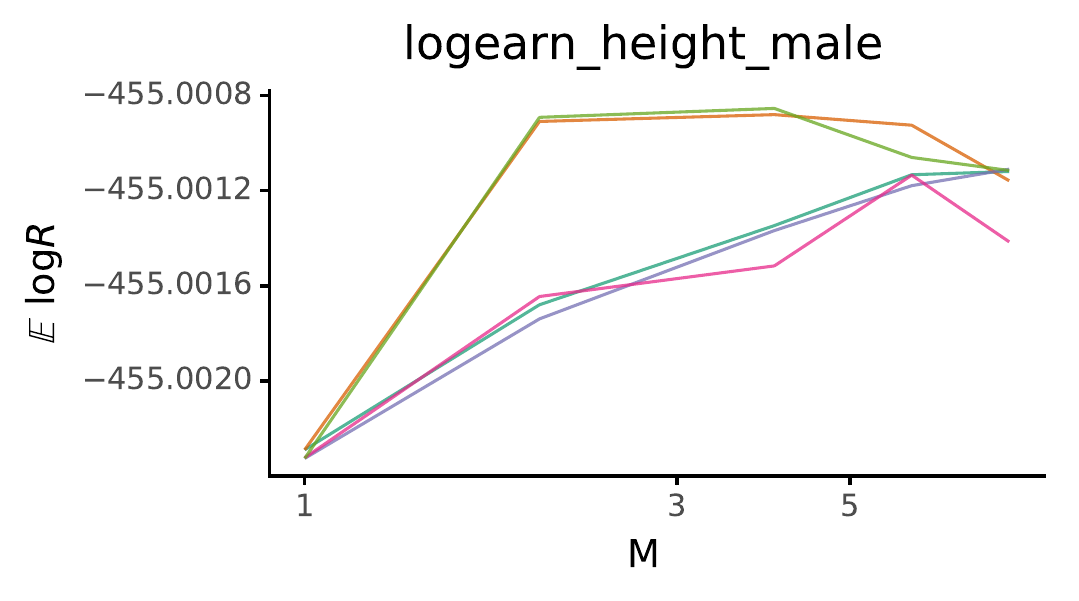}\includegraphics[width=0.33\columnwidth]{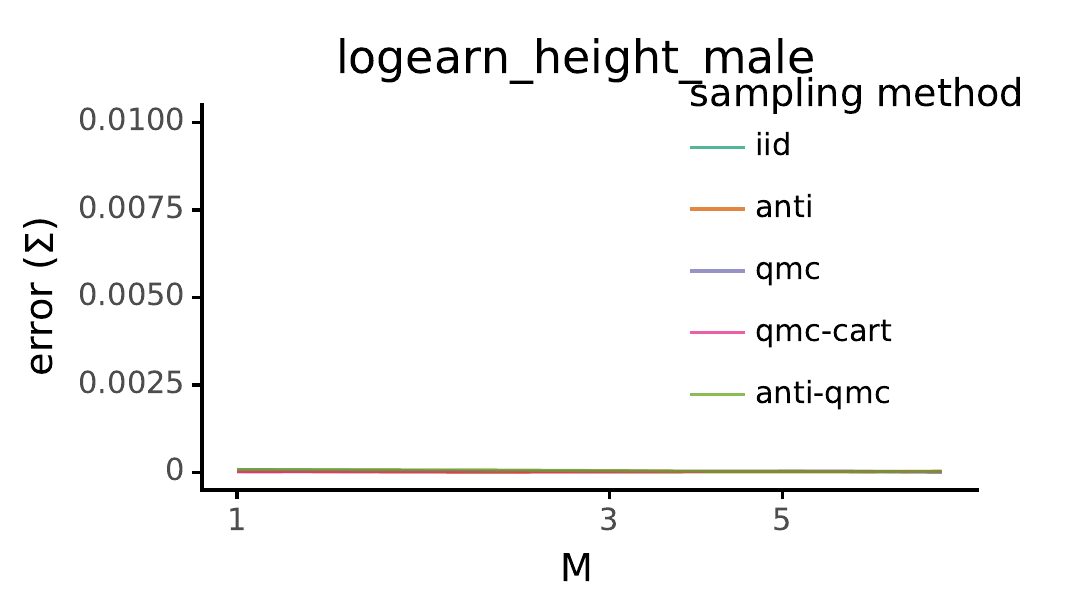}\includegraphics[width=0.33\columnwidth]{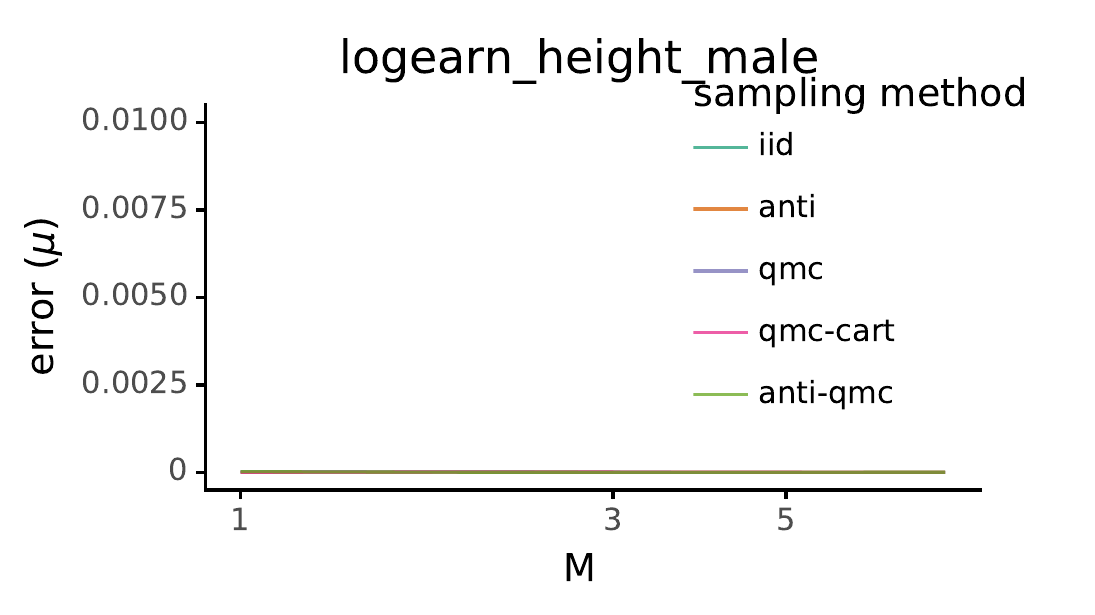}\linebreak{}

\includegraphics[width=0.33\columnwidth]{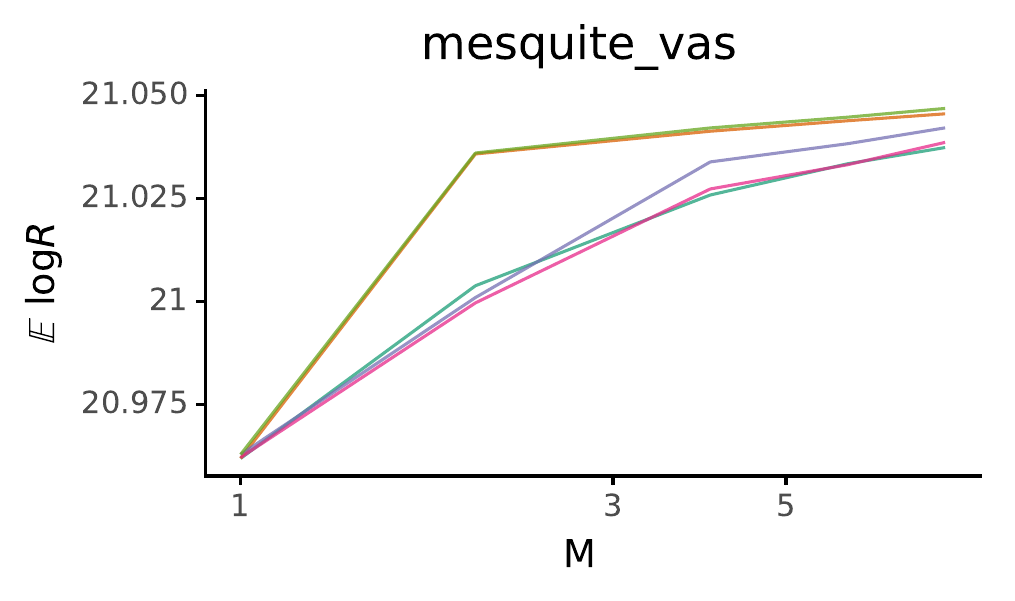}\includegraphics[width=0.33\columnwidth]{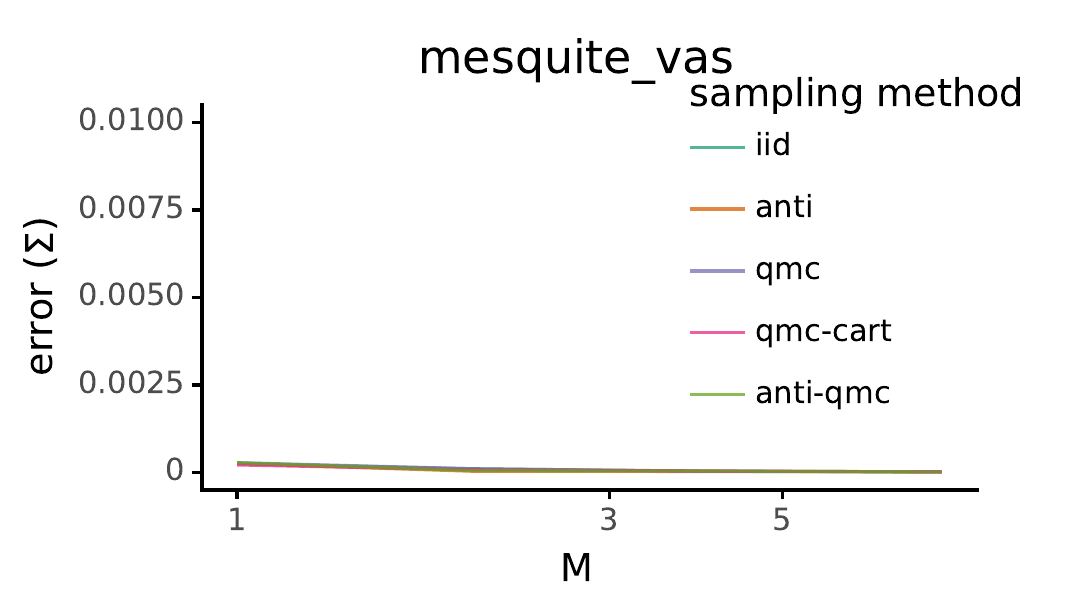}\includegraphics[width=0.33\columnwidth]{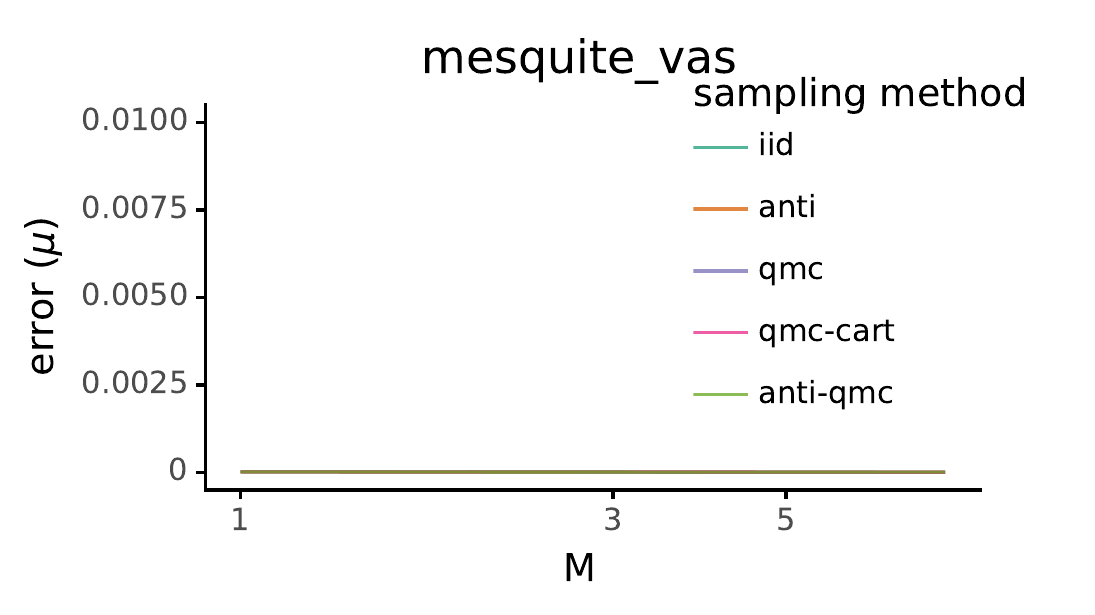}\linebreak{}

\includegraphics[width=0.33\columnwidth]{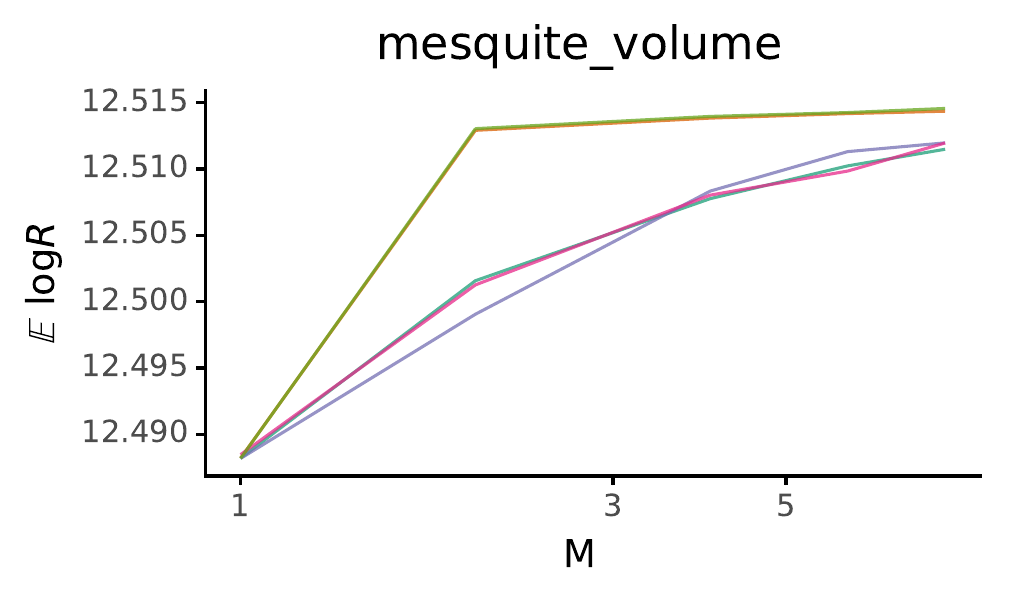}\includegraphics[width=0.33\columnwidth]{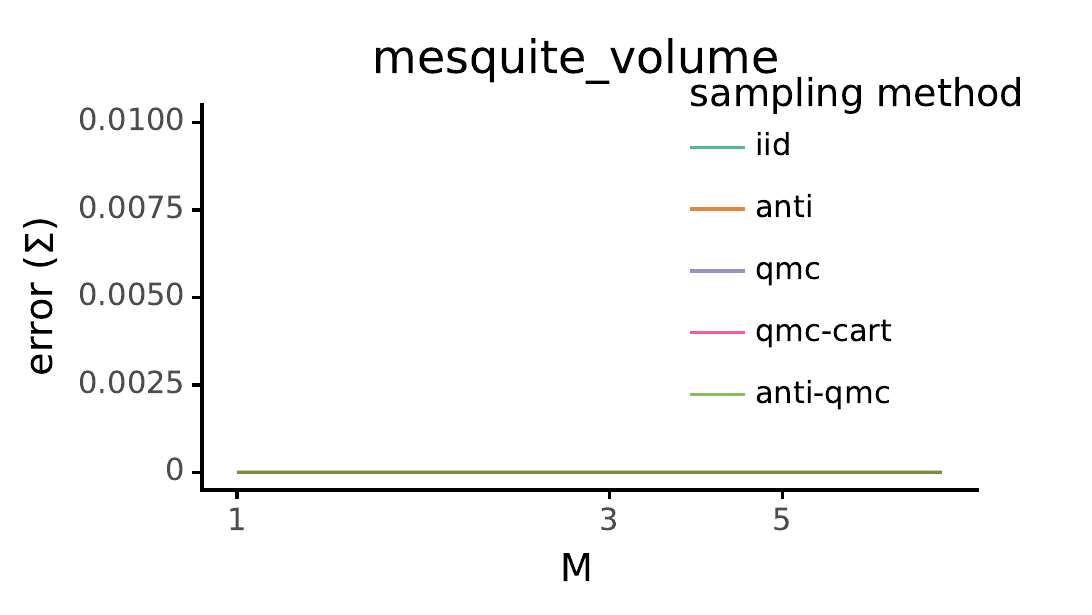}\includegraphics[width=0.33\columnwidth]{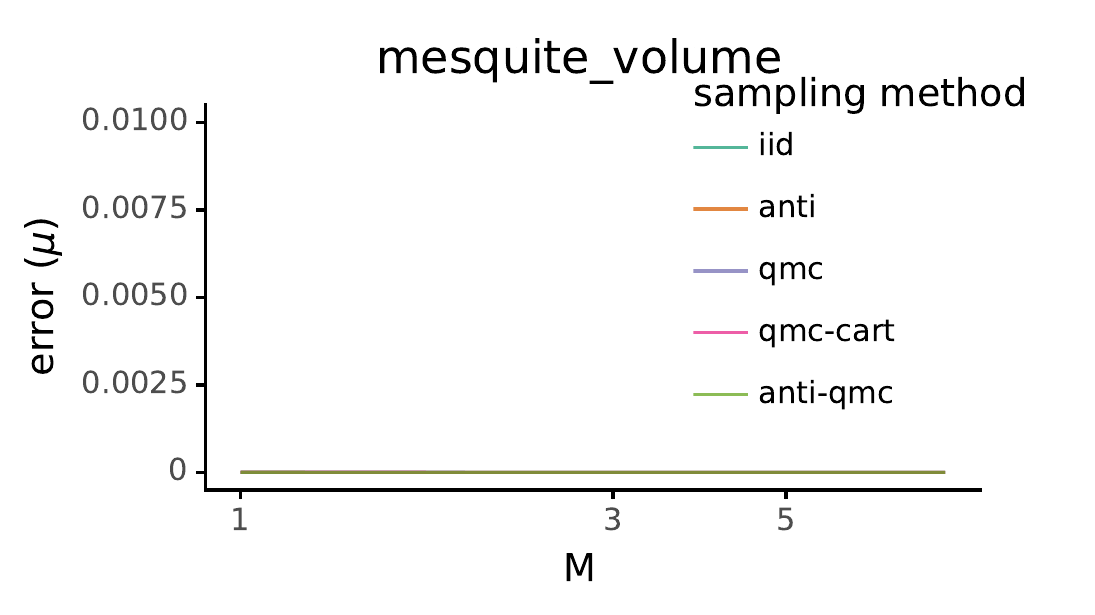}\linebreak{}

\includegraphics[width=0.33\columnwidth]{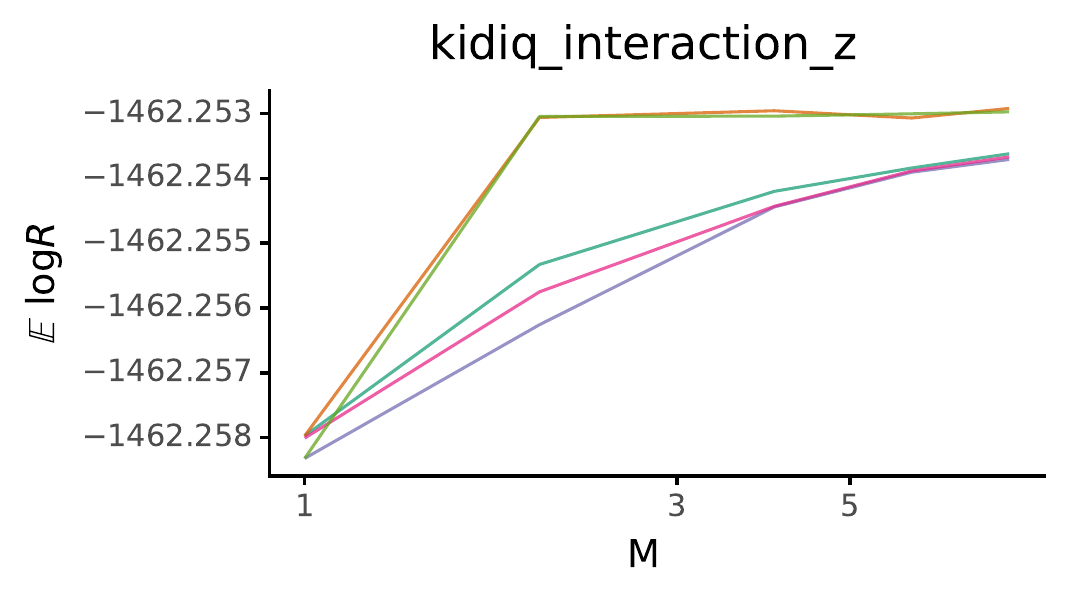}\includegraphics[width=0.33\columnwidth]{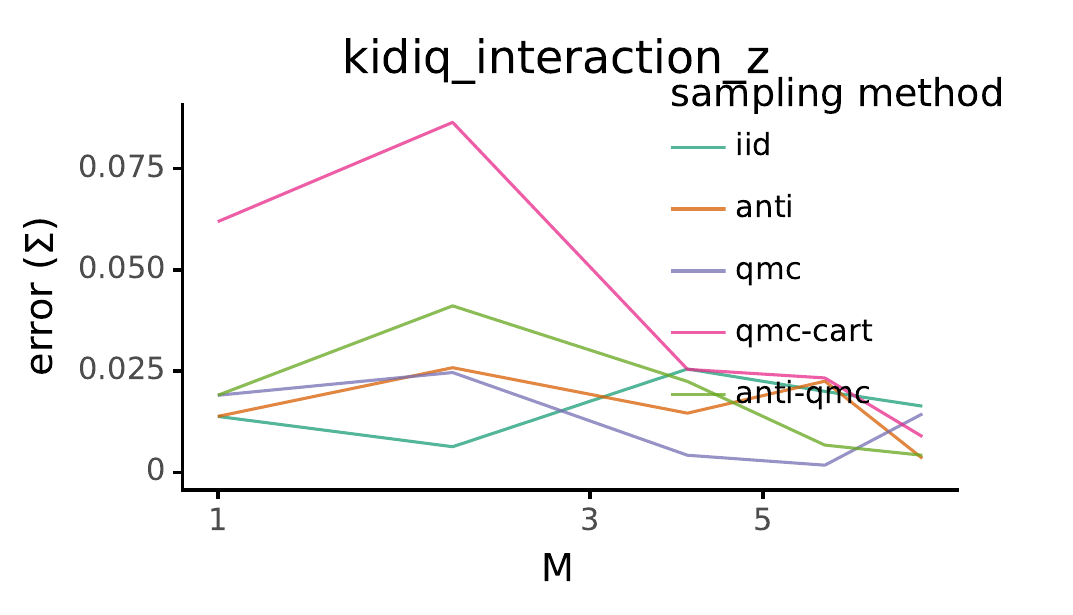}\includegraphics[width=0.33\columnwidth]{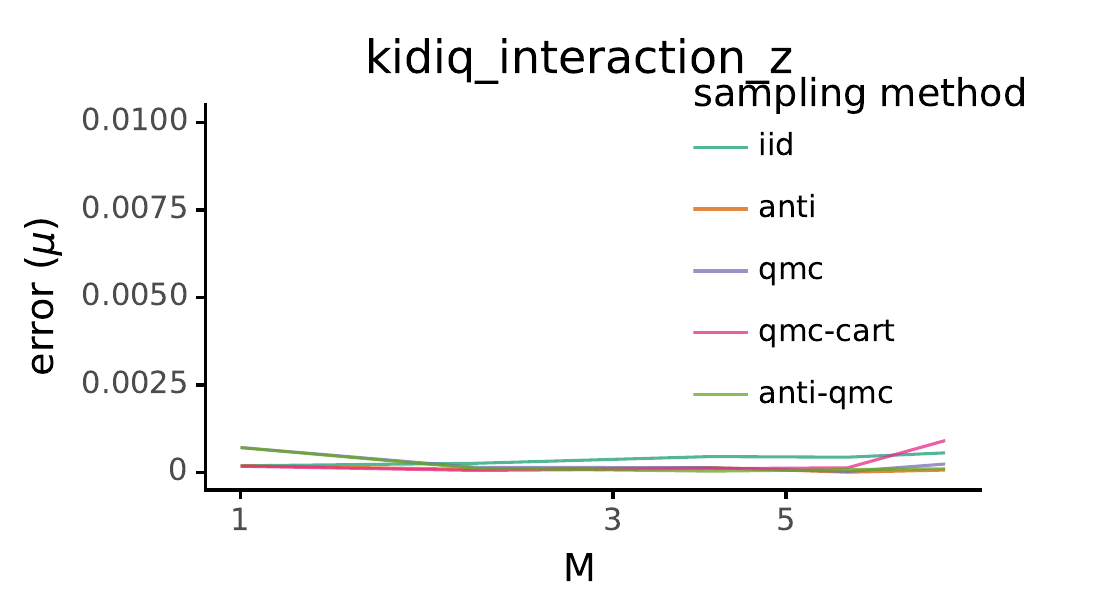}

\caption{\textbf{Across all models, improvements in likelihood bounds correlate
strongly with improvements in posterior accuracy. Better sampling
methods can improve both.}}
\end{figure}

\begin{figure}
\includegraphics[width=0.33\columnwidth]{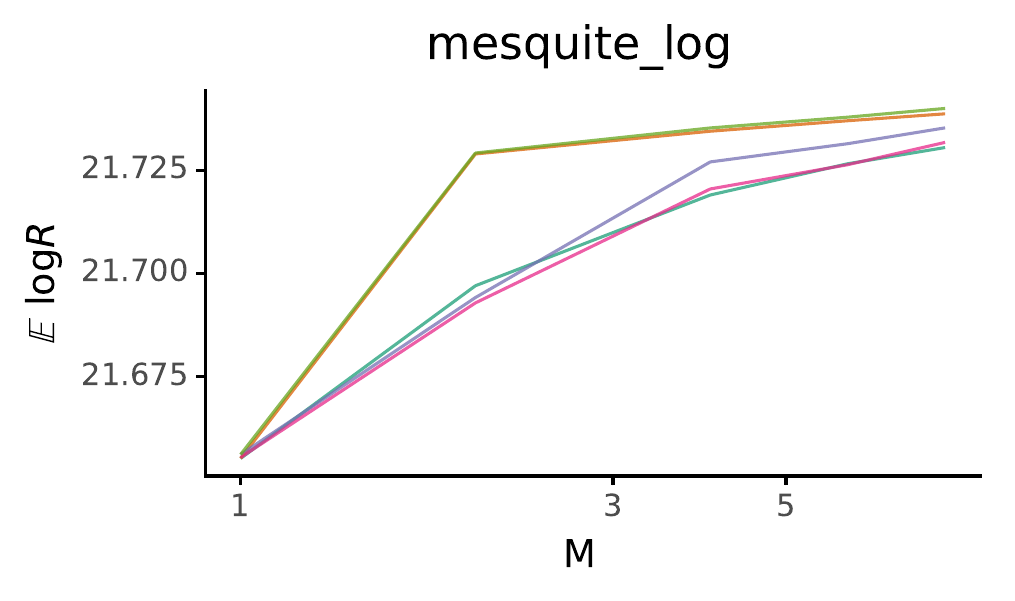}\includegraphics[width=0.33\columnwidth]{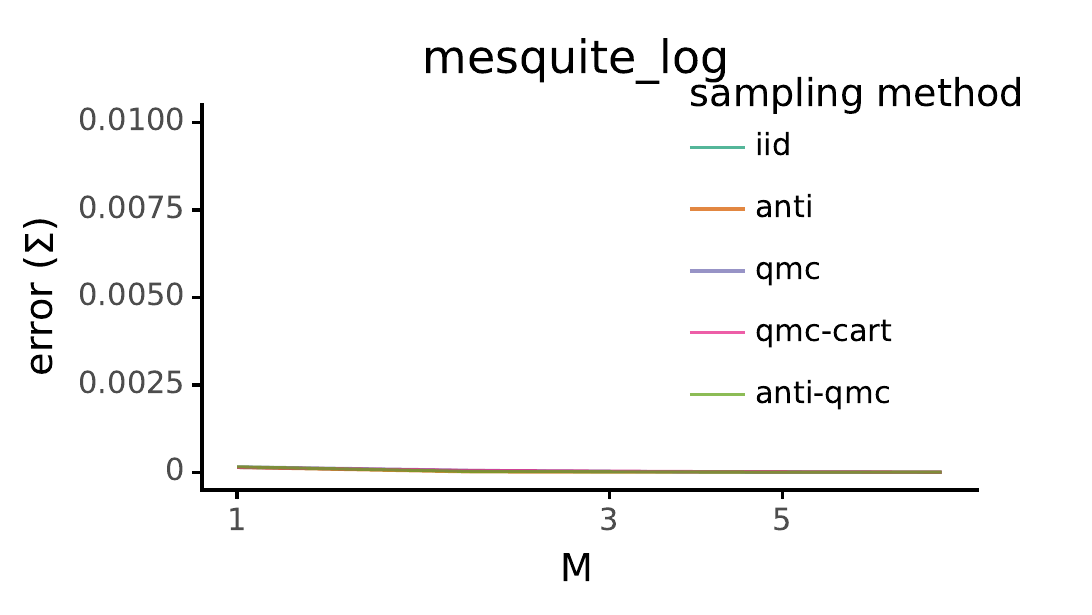}\includegraphics[width=0.33\columnwidth]{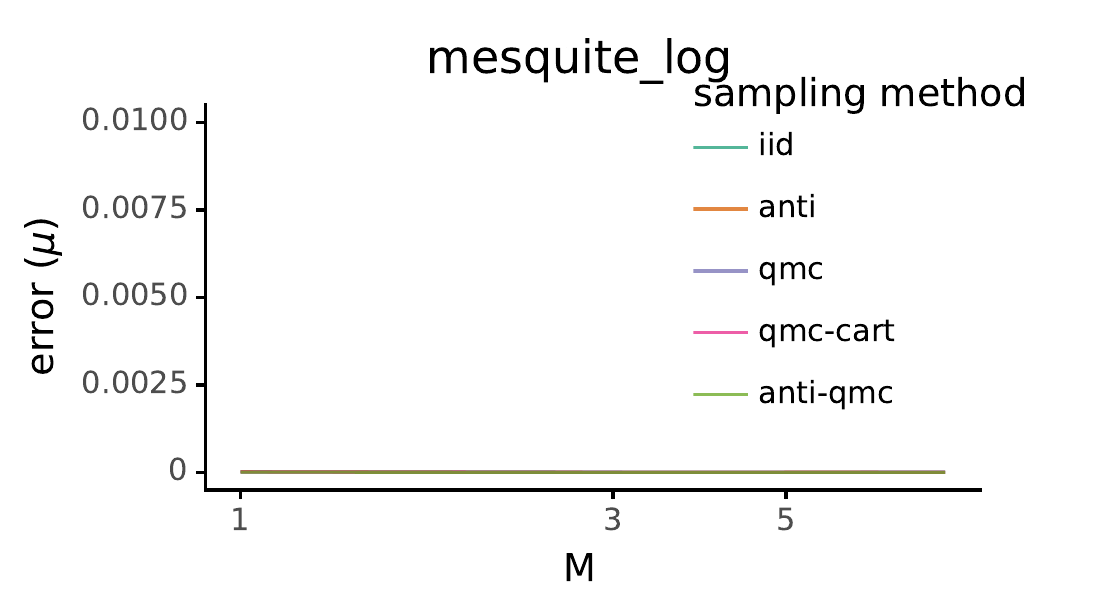}\linebreak{}

\includegraphics[width=0.33\columnwidth]{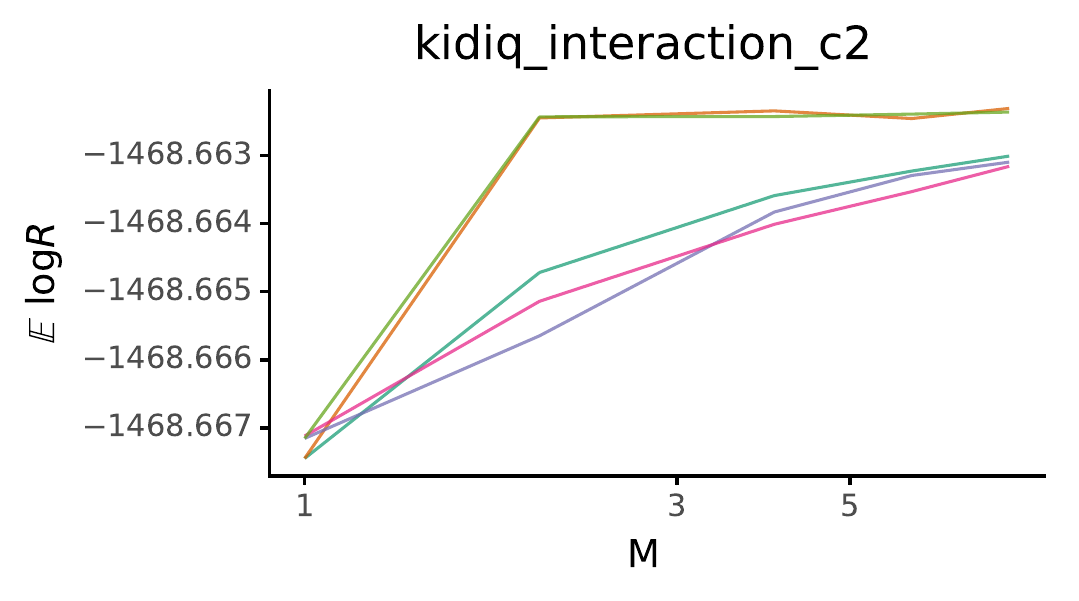}\includegraphics[width=0.33\columnwidth]{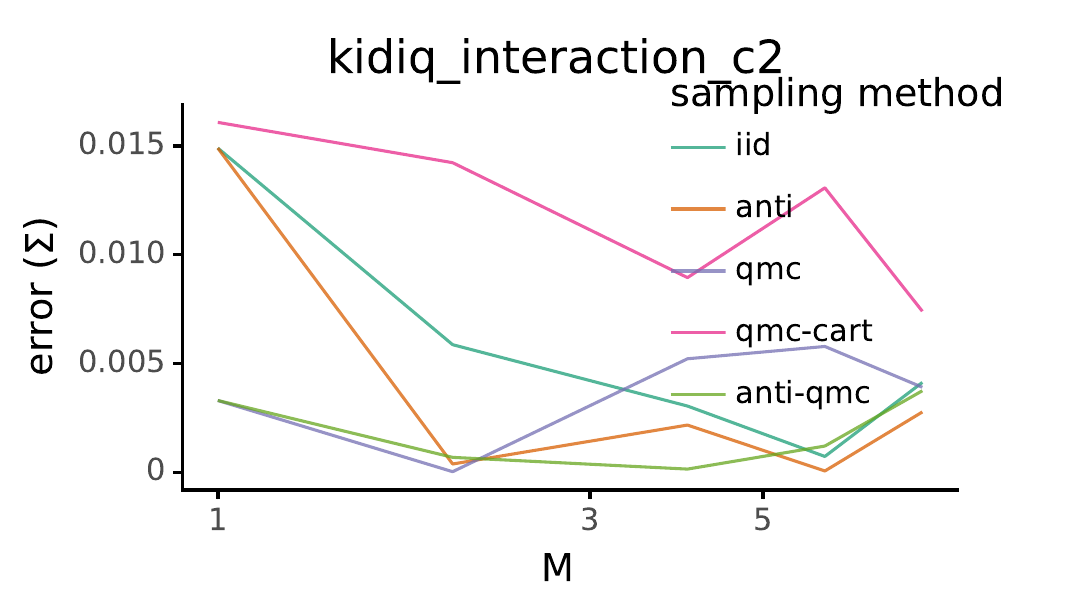}\includegraphics[width=0.33\columnwidth]{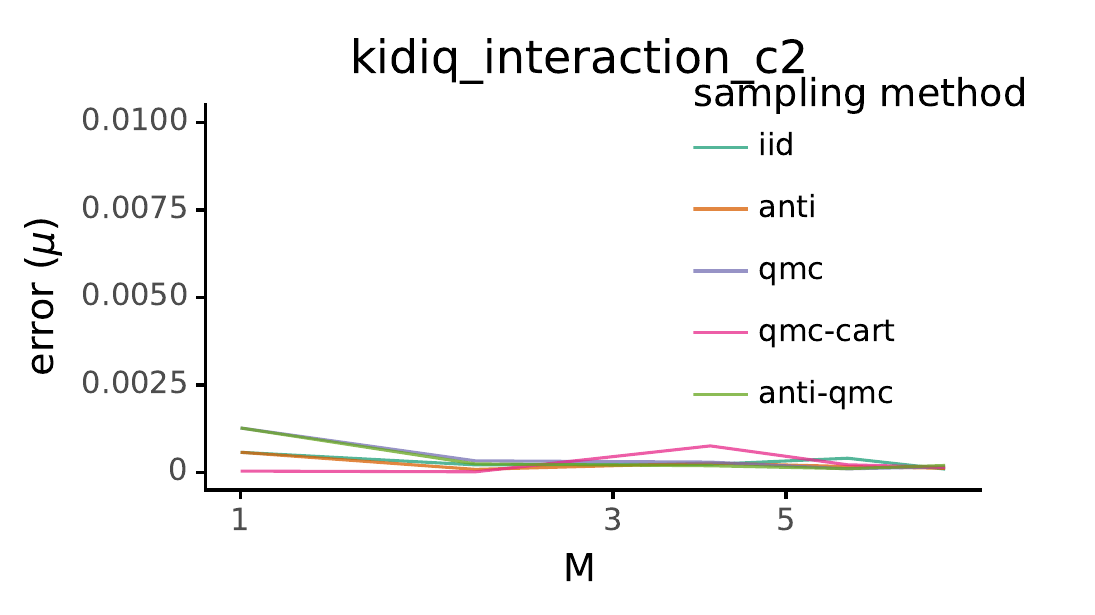}\linebreak{}

\includegraphics[width=0.33\columnwidth]{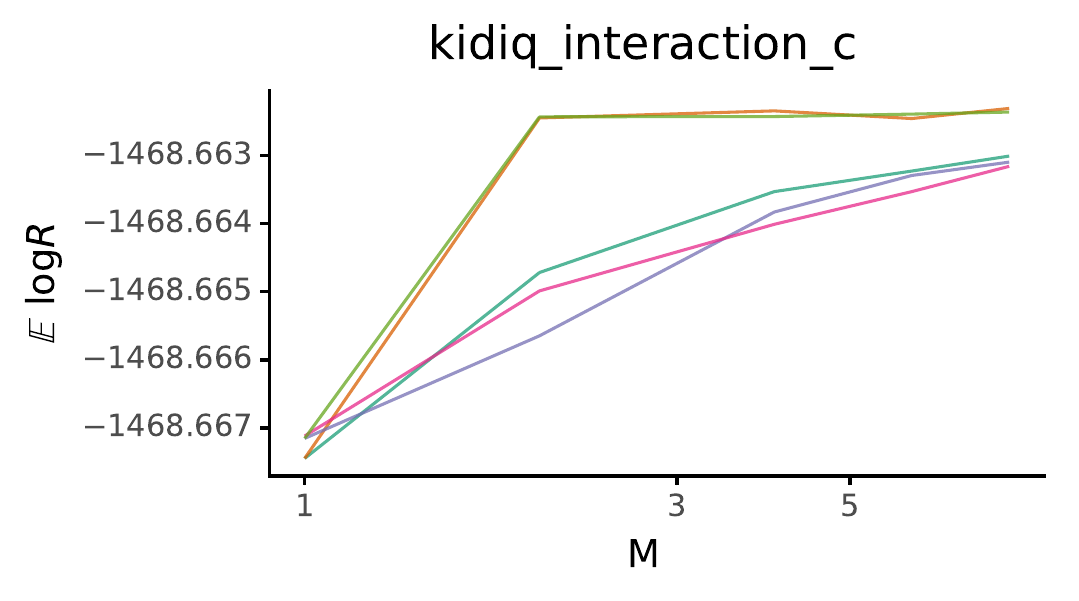}\includegraphics[width=0.33\columnwidth]{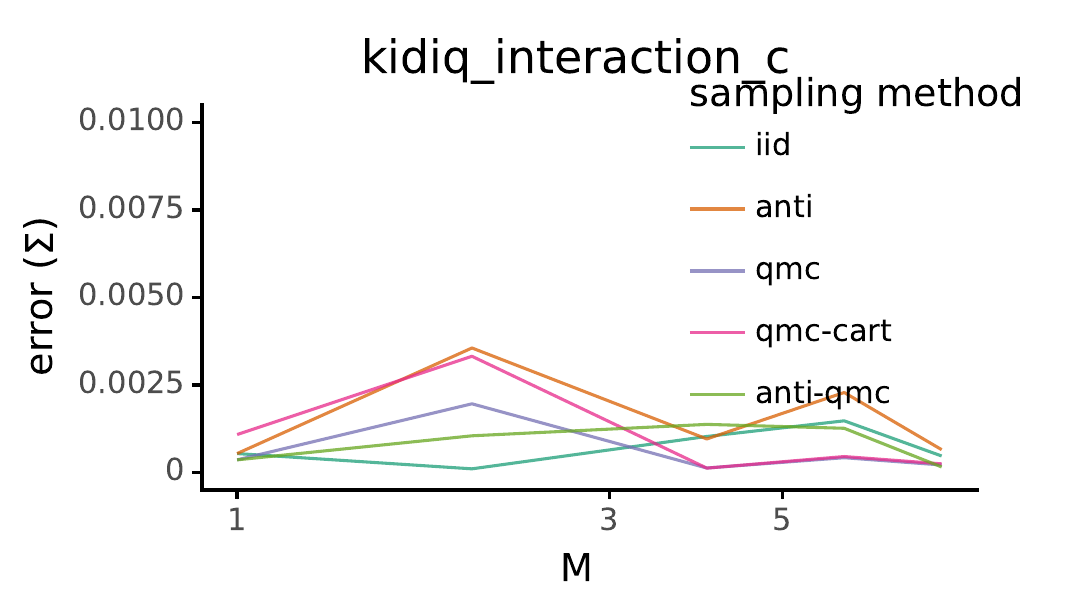}\includegraphics[width=0.33\columnwidth]{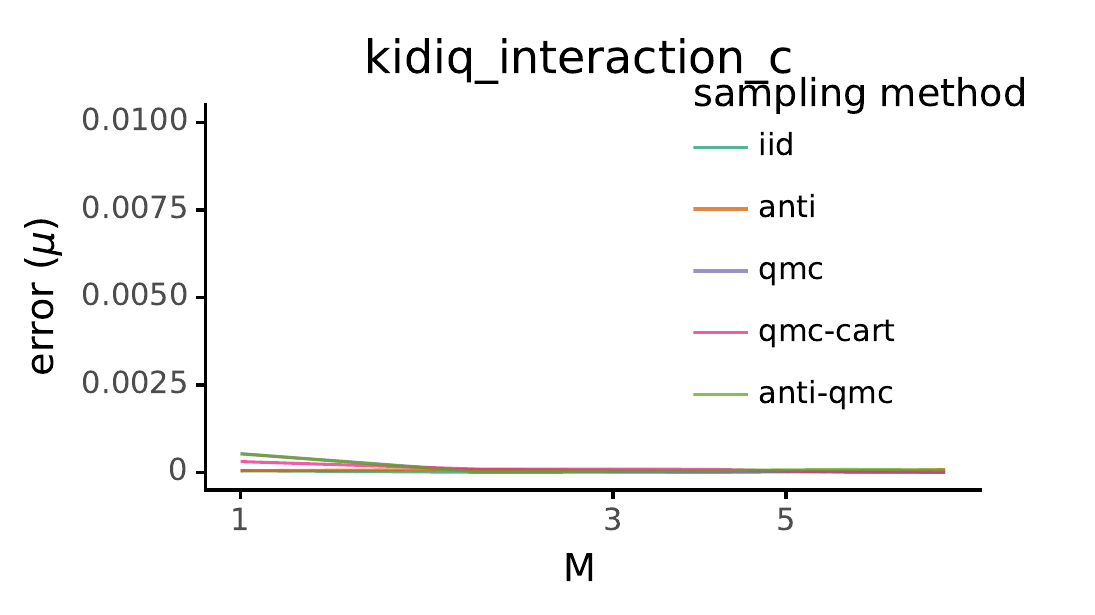}\linebreak{}

\includegraphics[width=0.33\columnwidth]{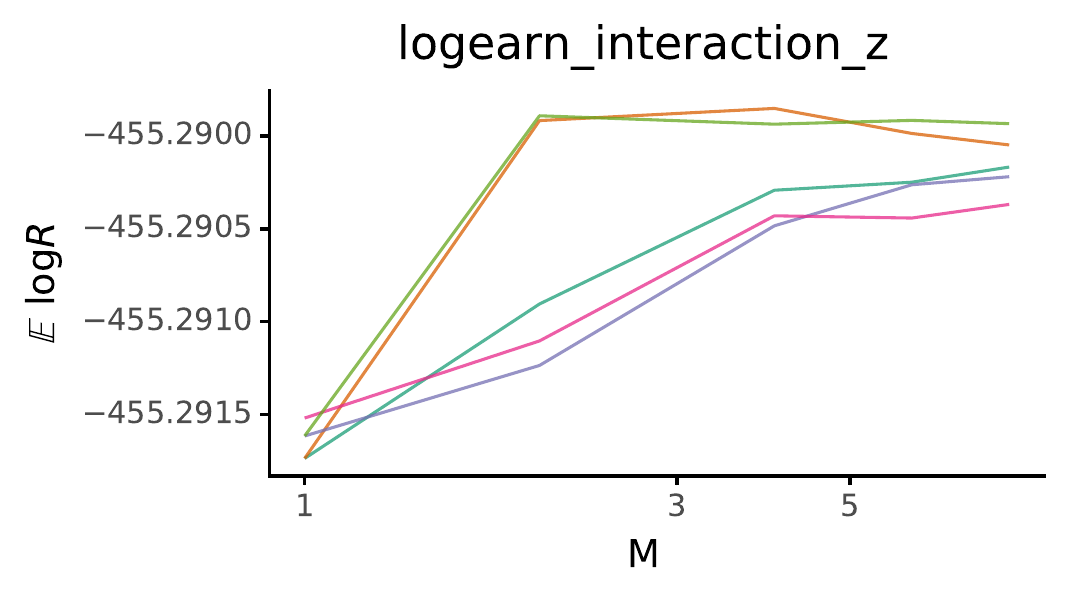}\includegraphics[width=0.33\columnwidth]{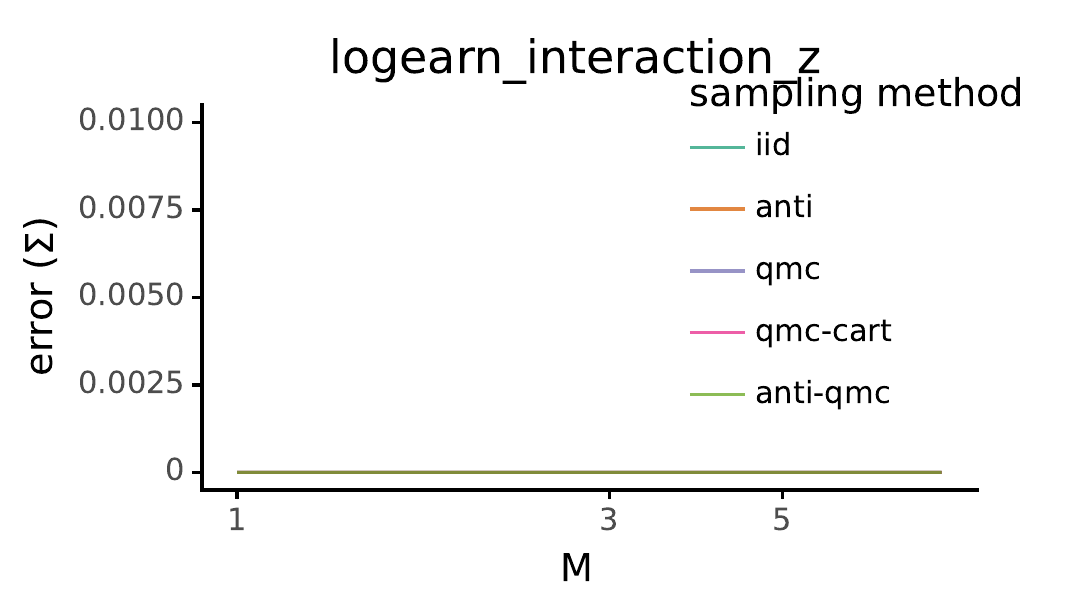}\includegraphics[width=0.33\columnwidth]{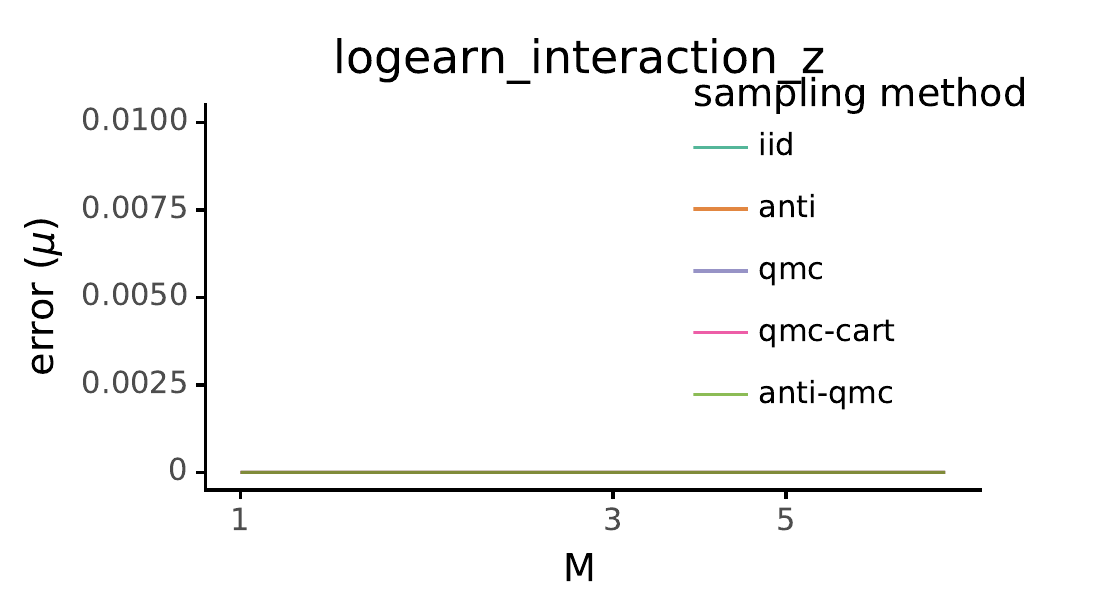}\linebreak{}

\includegraphics[width=0.33\columnwidth]{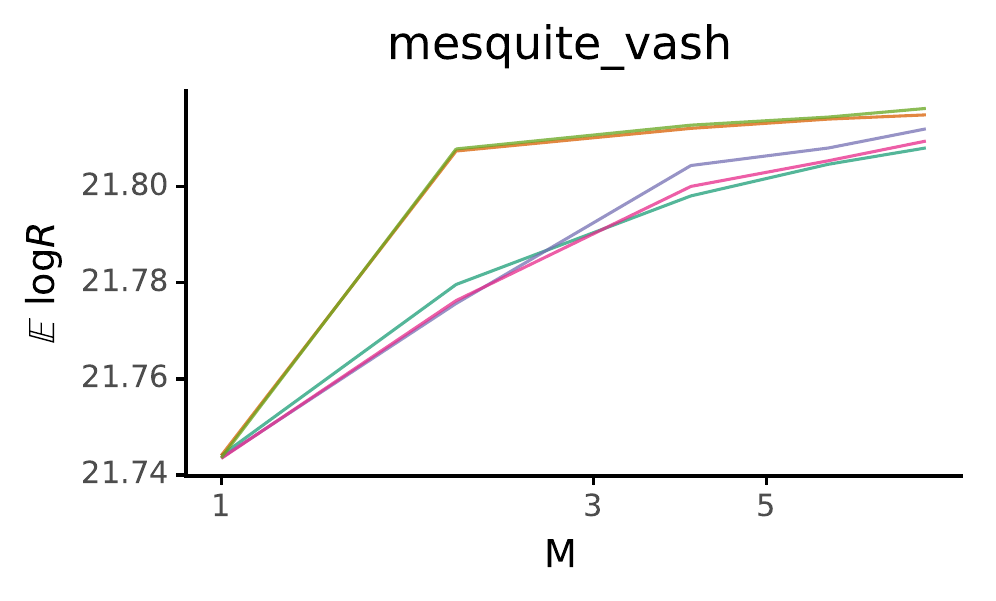}\includegraphics[width=0.33\columnwidth]{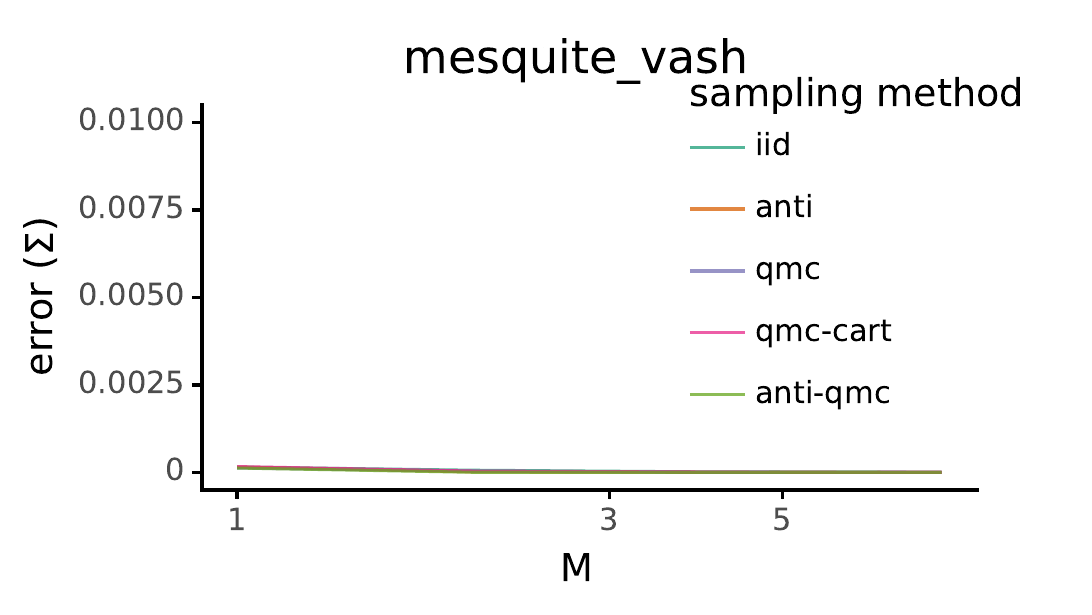}\includegraphics[width=0.33\columnwidth]{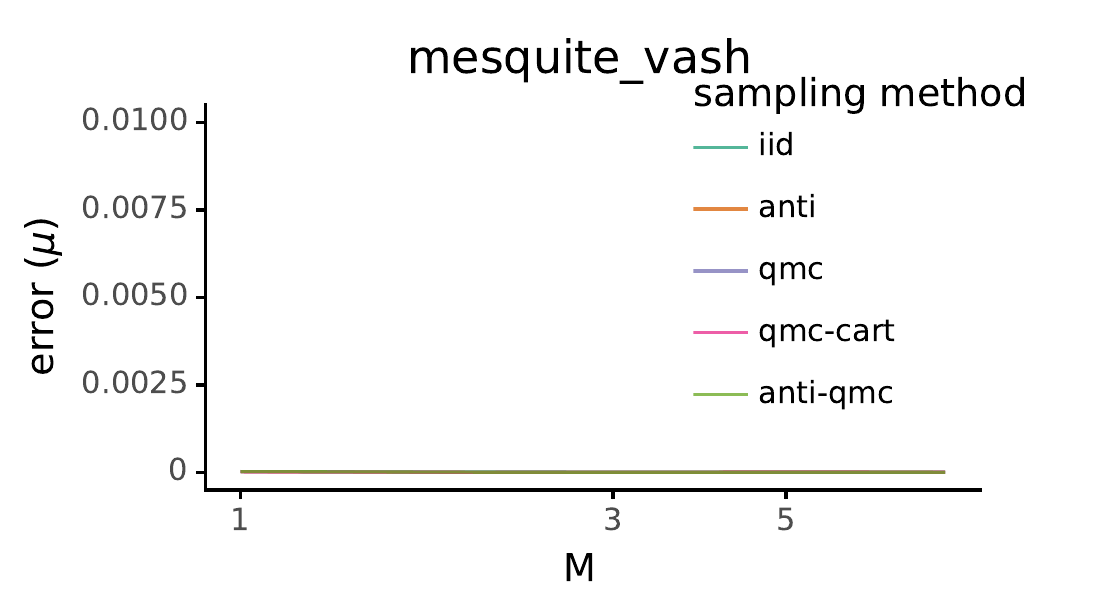}\linebreak{}

\includegraphics[width=0.33\columnwidth]{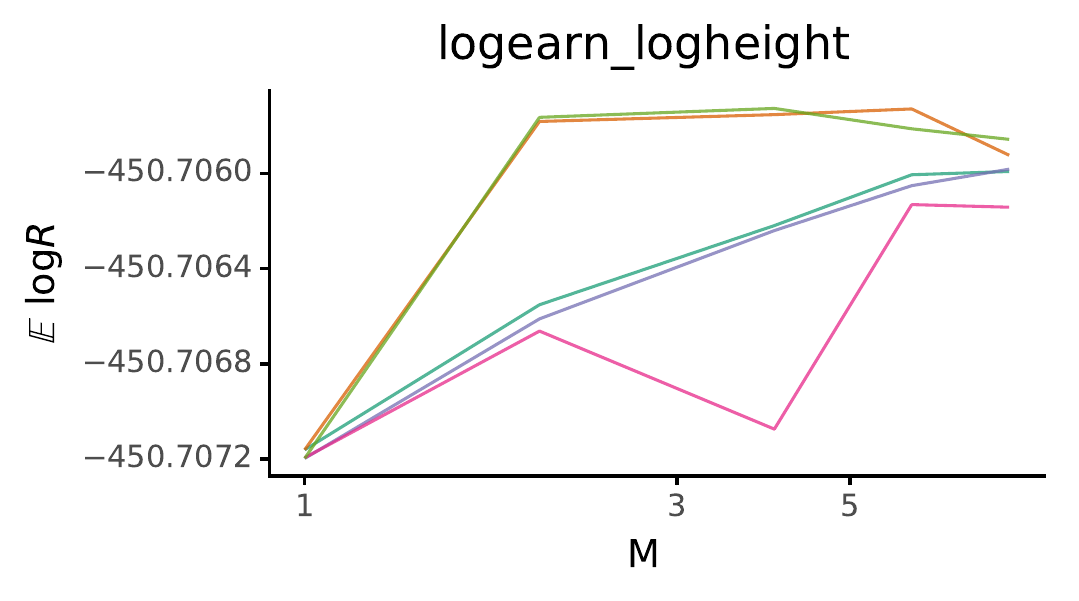}\includegraphics[width=0.33\columnwidth]{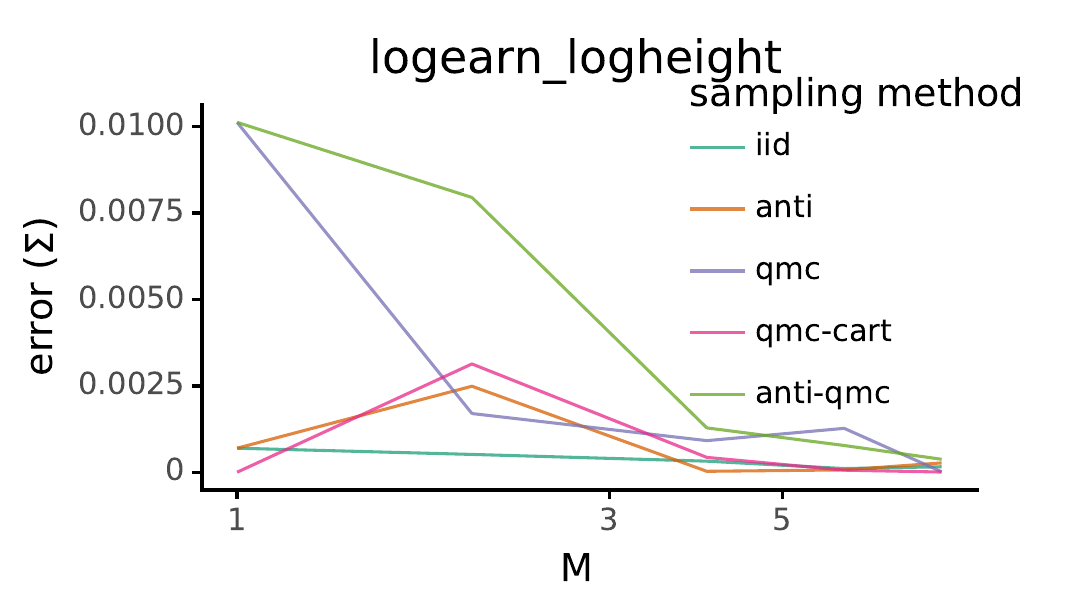}\includegraphics[width=0.33\columnwidth]{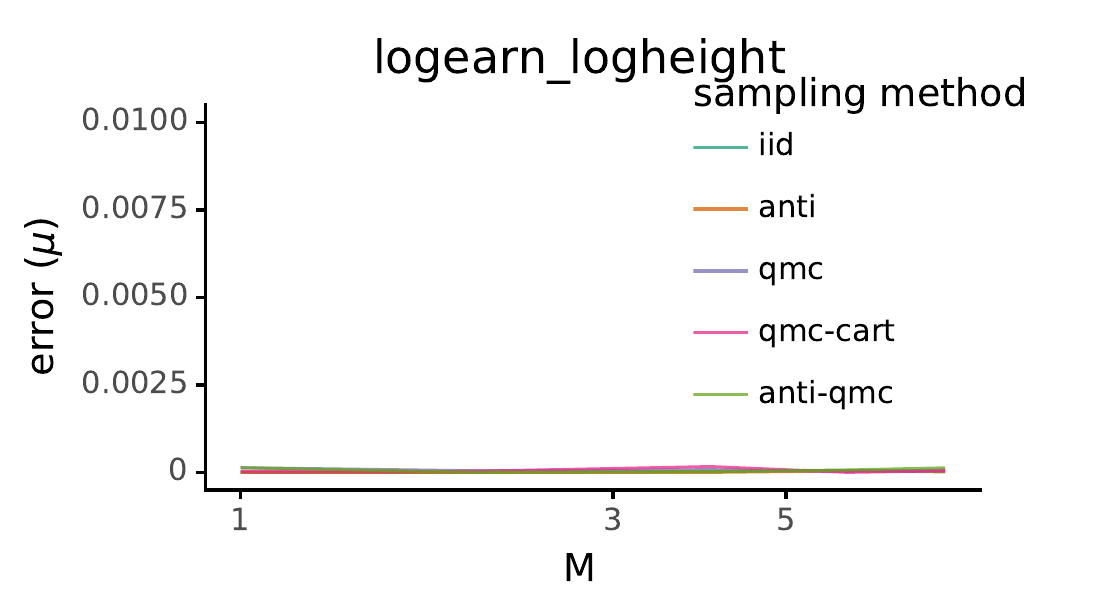}\linebreak{}

\includegraphics[width=0.33\columnwidth]{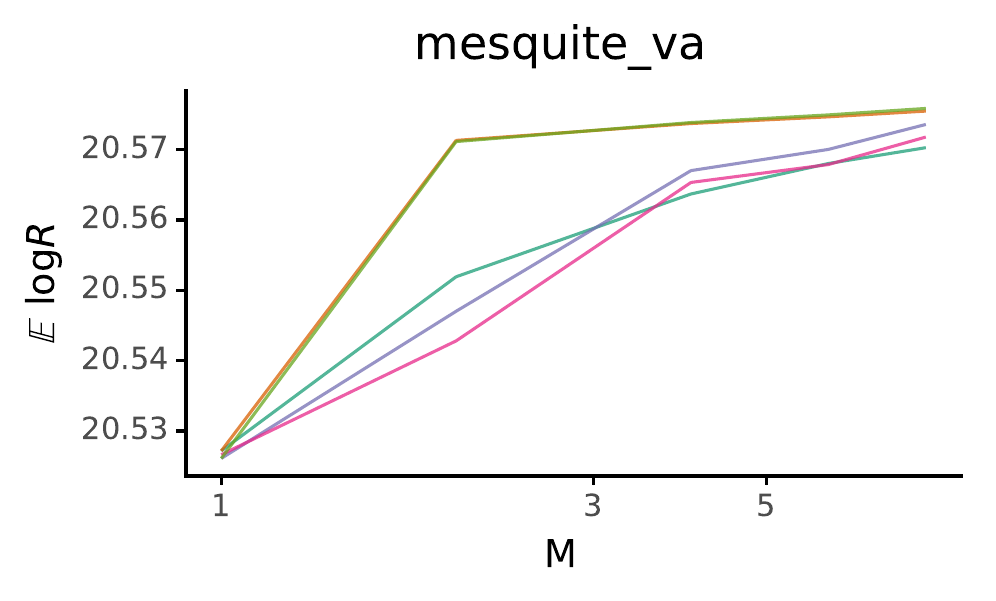}\includegraphics[width=0.33\columnwidth]{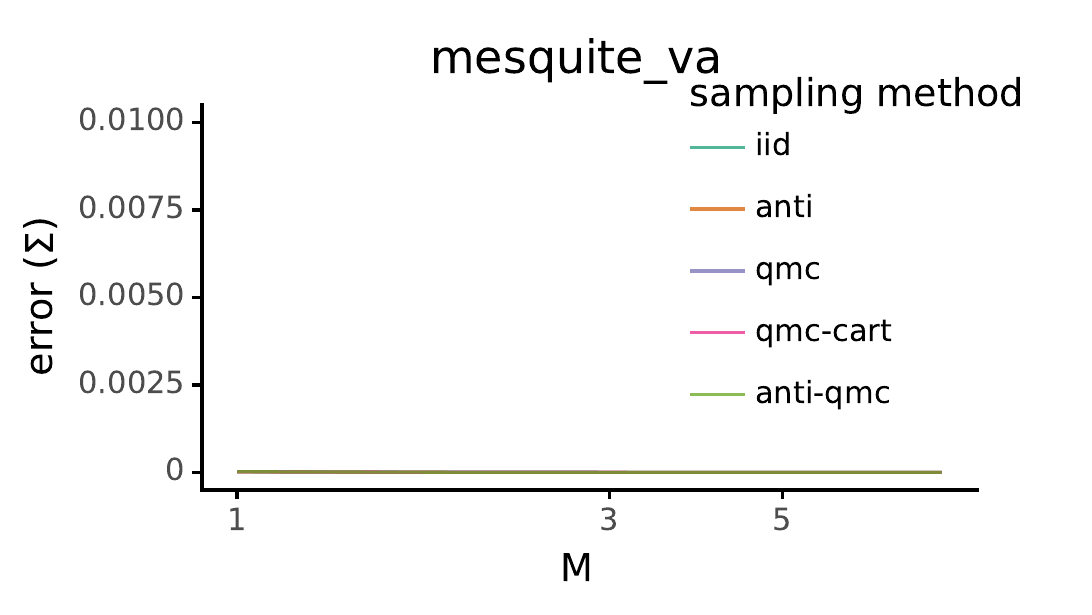}\includegraphics[width=0.33\columnwidth]{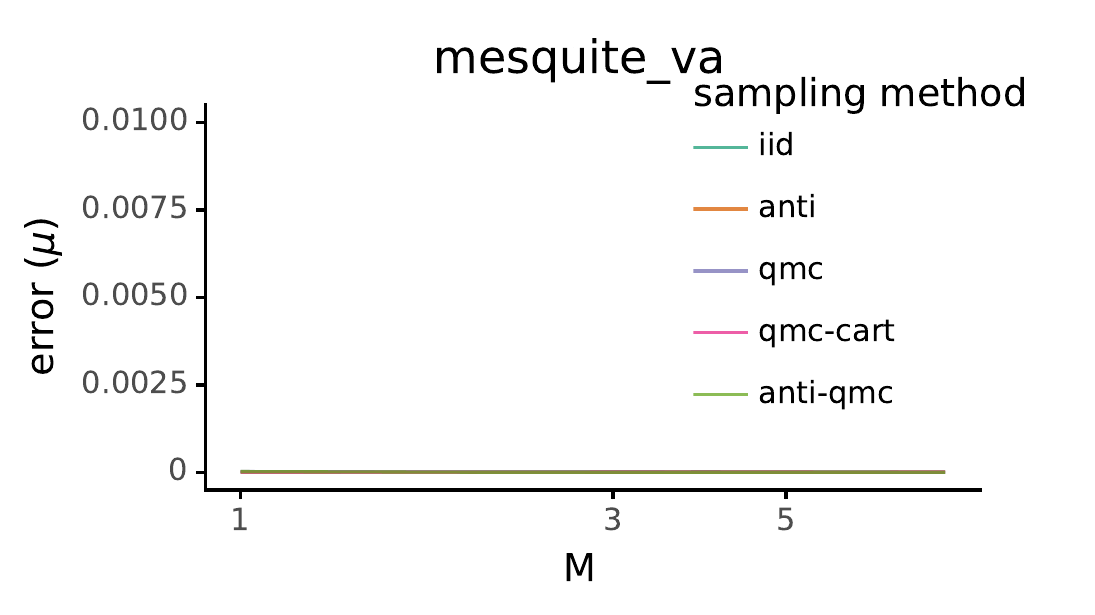}\linebreak{}

\caption{\textbf{Across all models, improvements in likelihood bounds correlate
strongly with improvements in posterior accuracy. Better sampling
methods can improve both.}}
\end{figure}

\begin{figure}
\includegraphics[width=0.33\columnwidth]{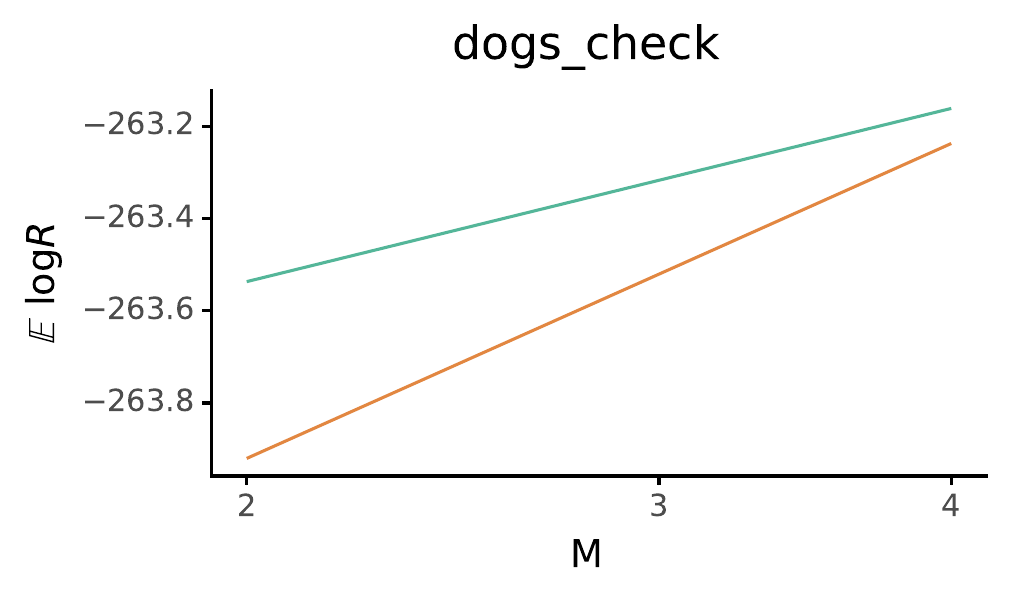}\includegraphics[width=0.33\columnwidth]{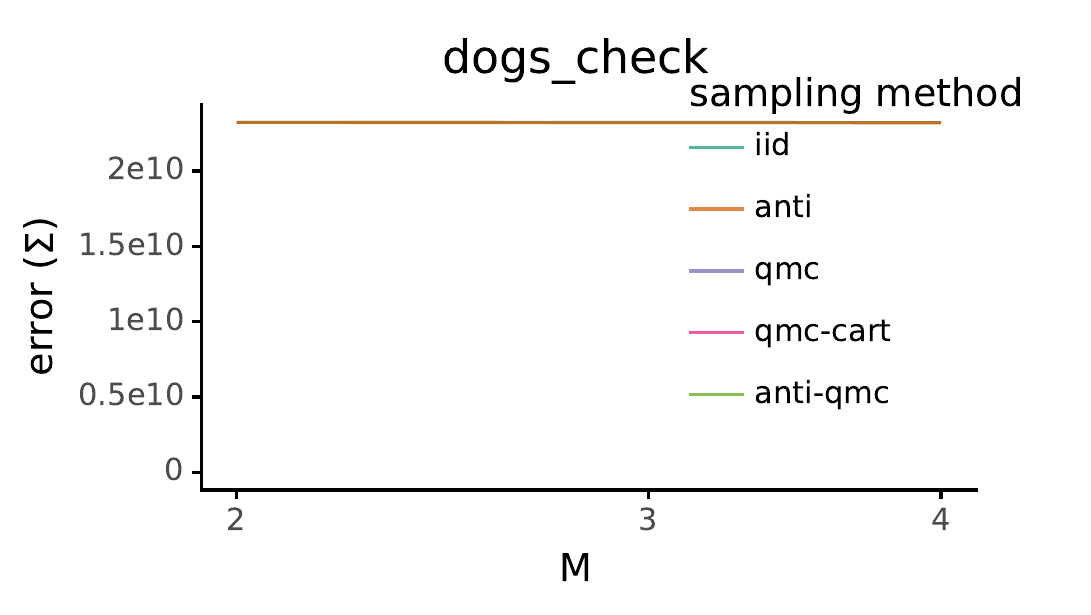}\includegraphics[width=0.33\columnwidth]{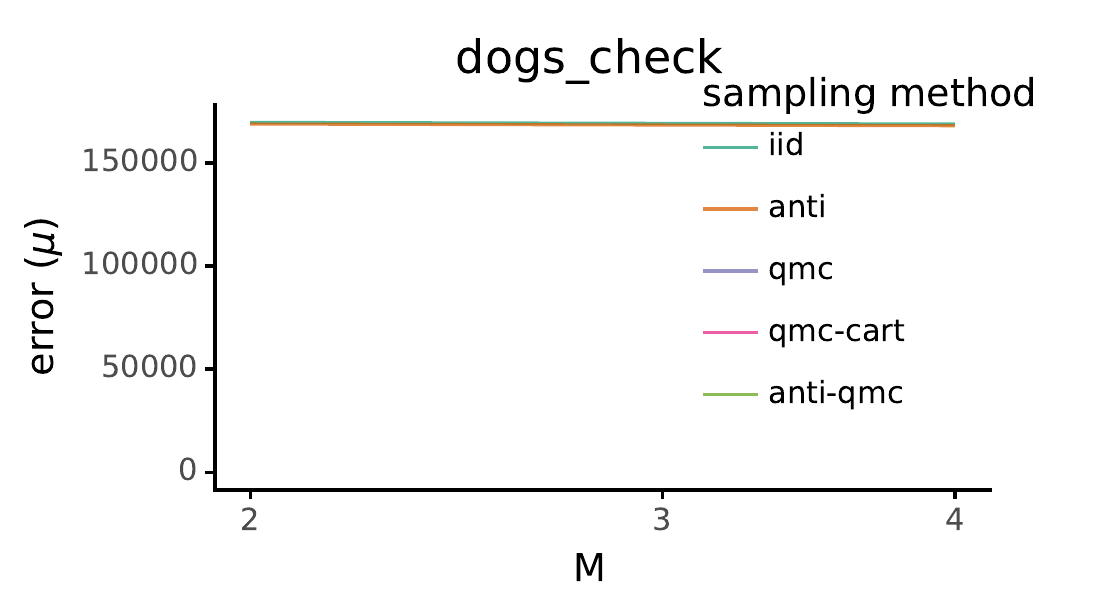}\linebreak{}

\includegraphics[width=0.33\columnwidth]{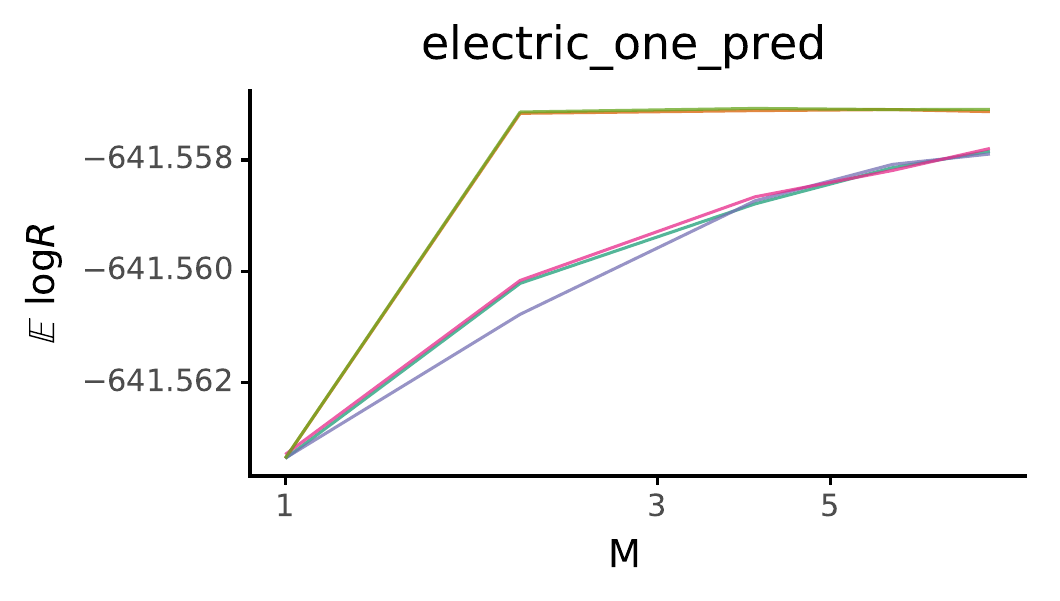}\includegraphics[width=0.33\columnwidth]{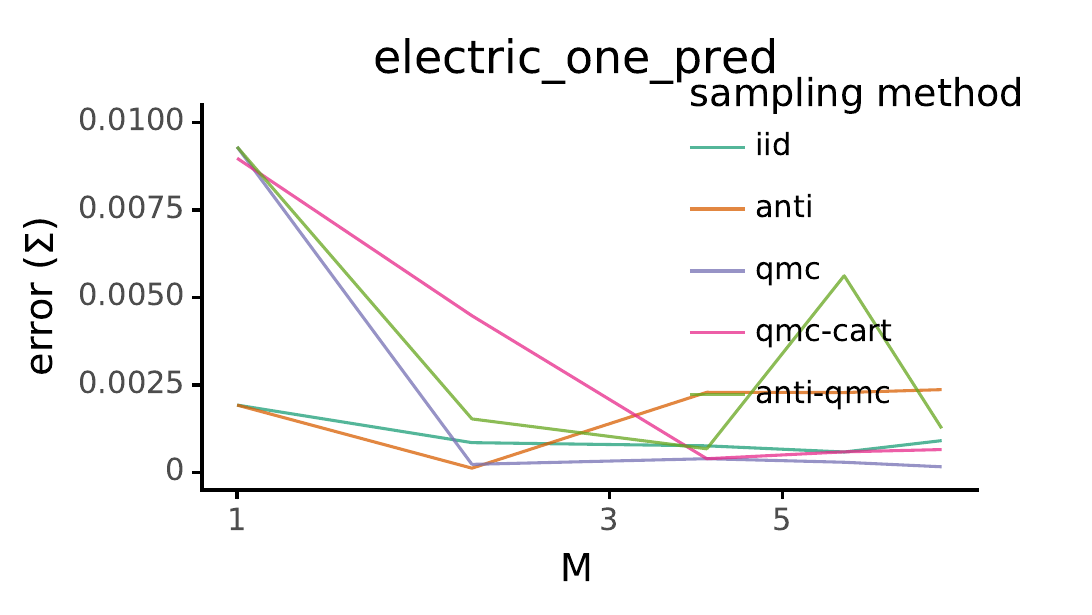}\includegraphics[width=0.33\columnwidth]{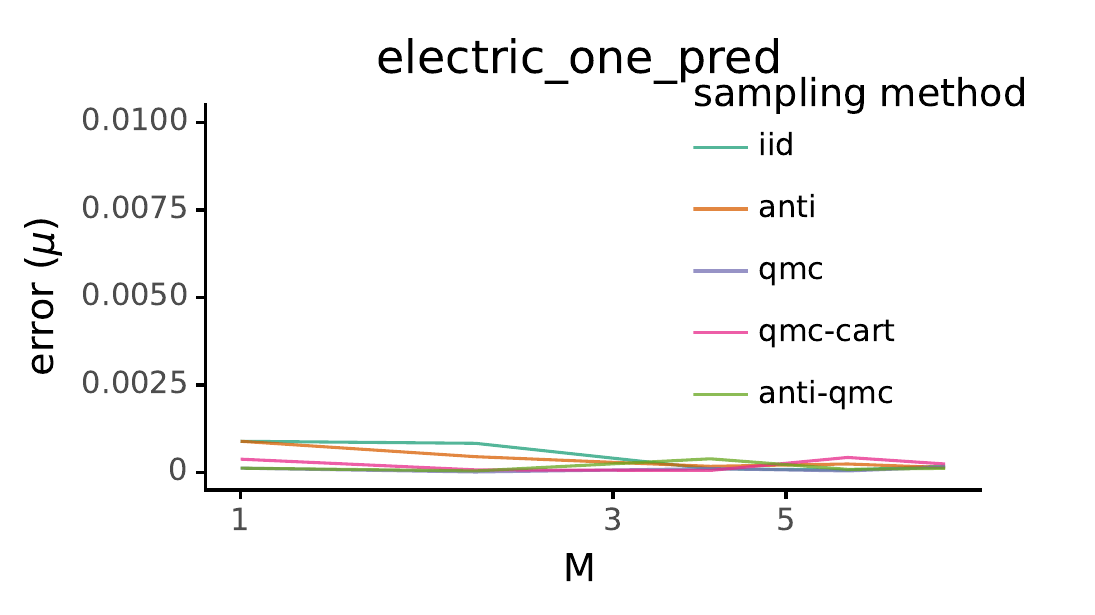}\linebreak{}

\includegraphics[width=0.33\columnwidth]{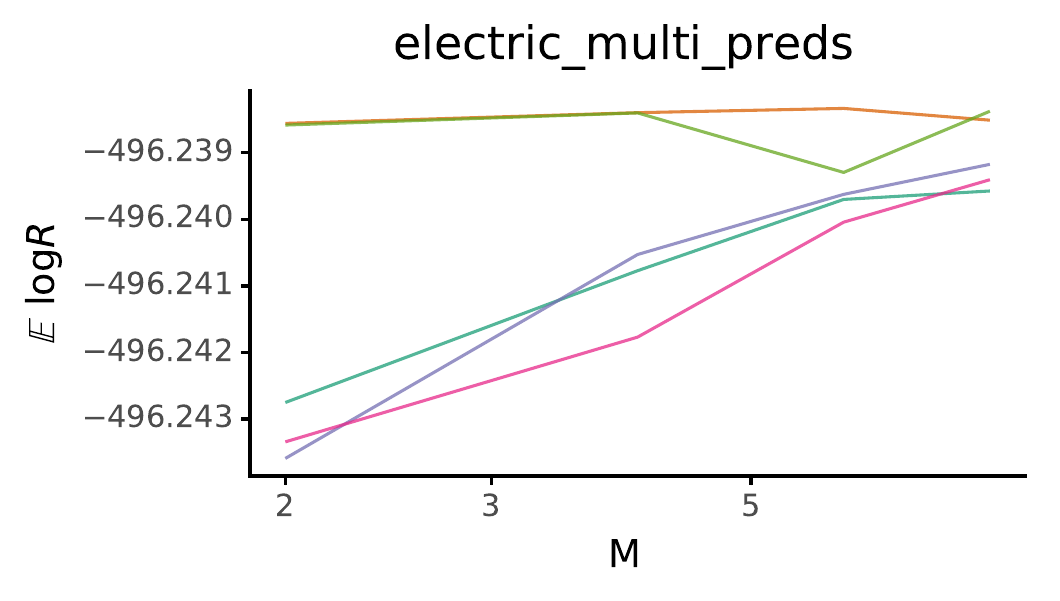}\includegraphics[width=0.33\columnwidth]{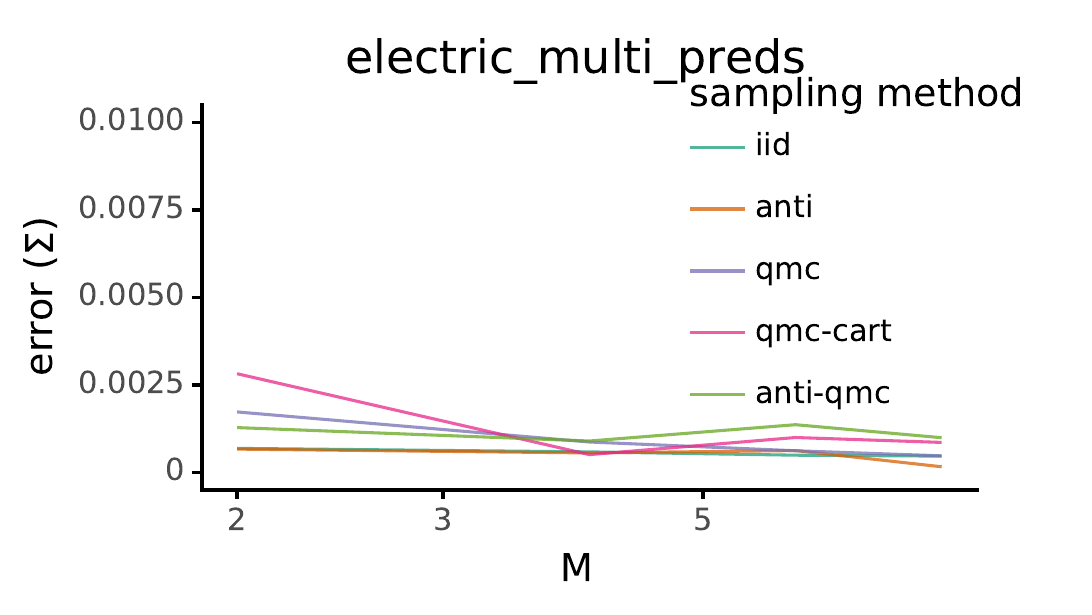}\includegraphics[width=0.33\columnwidth]{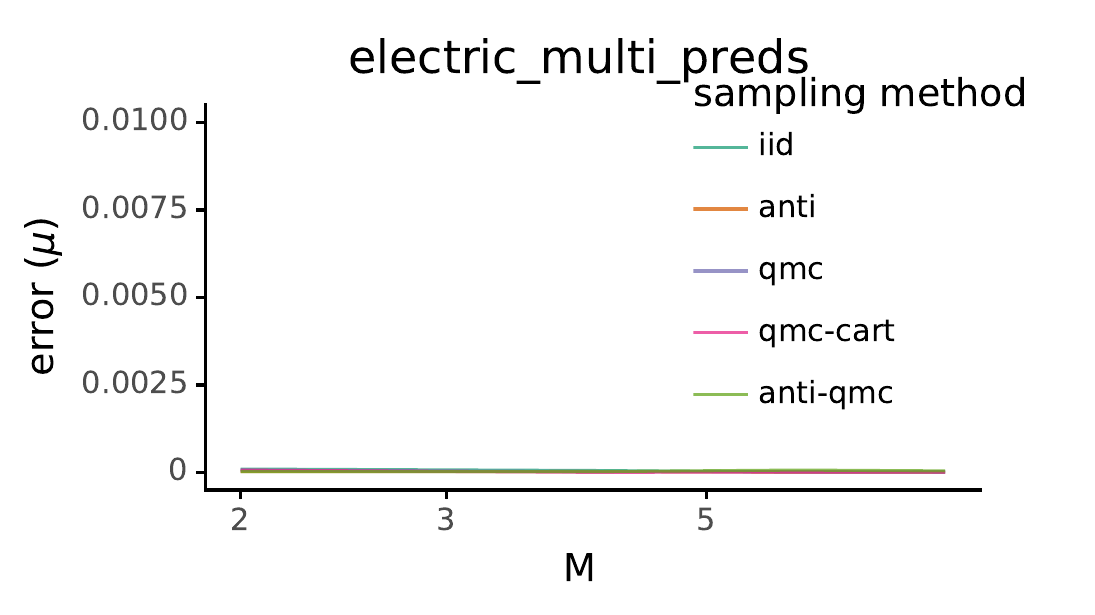}\linebreak{}

\includegraphics[width=0.33\columnwidth]{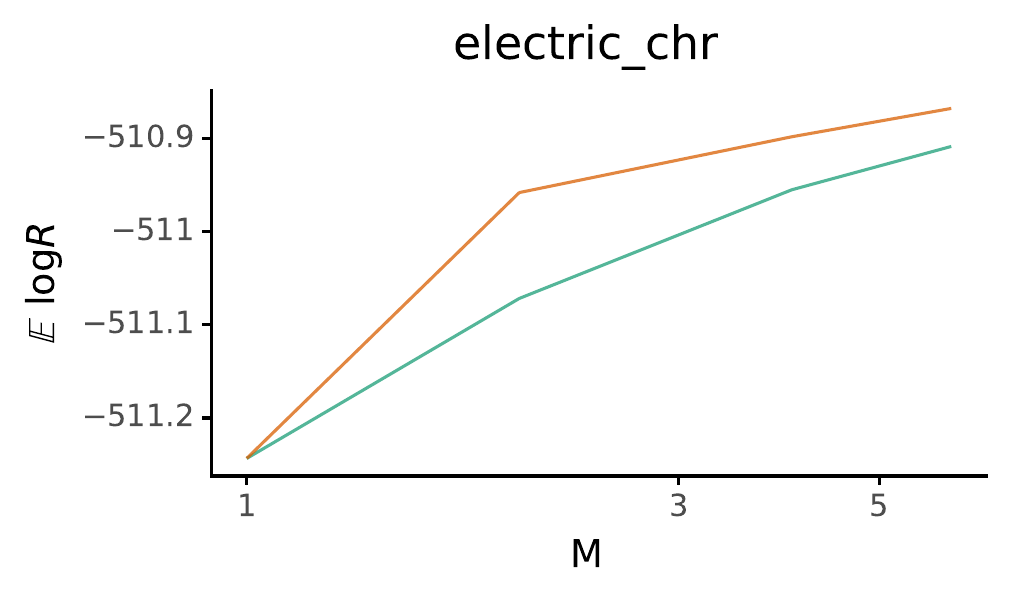}\includegraphics[width=0.33\columnwidth]{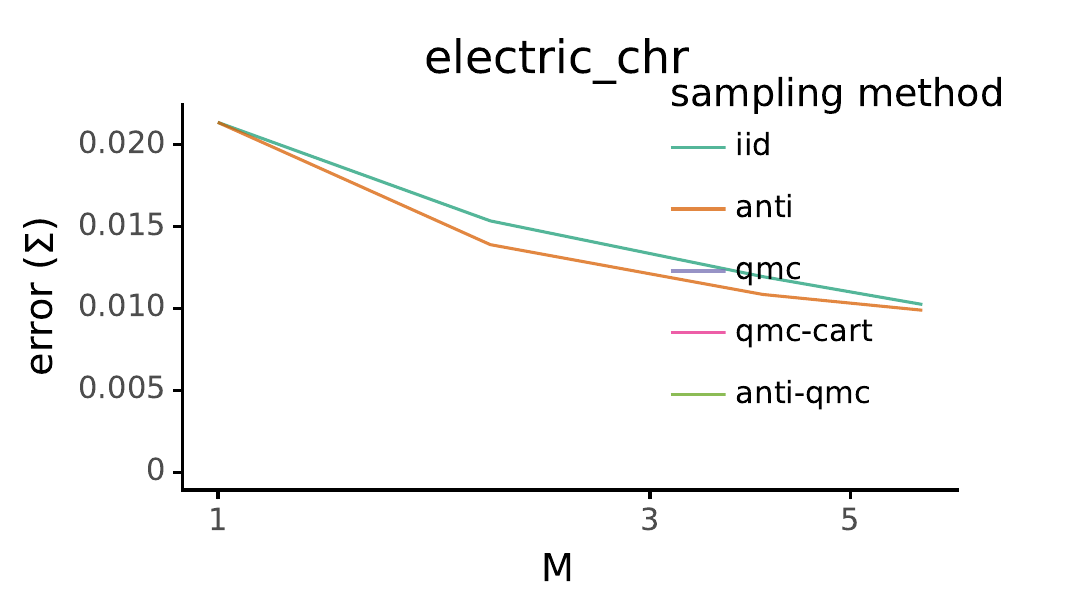}\includegraphics[width=0.33\columnwidth]{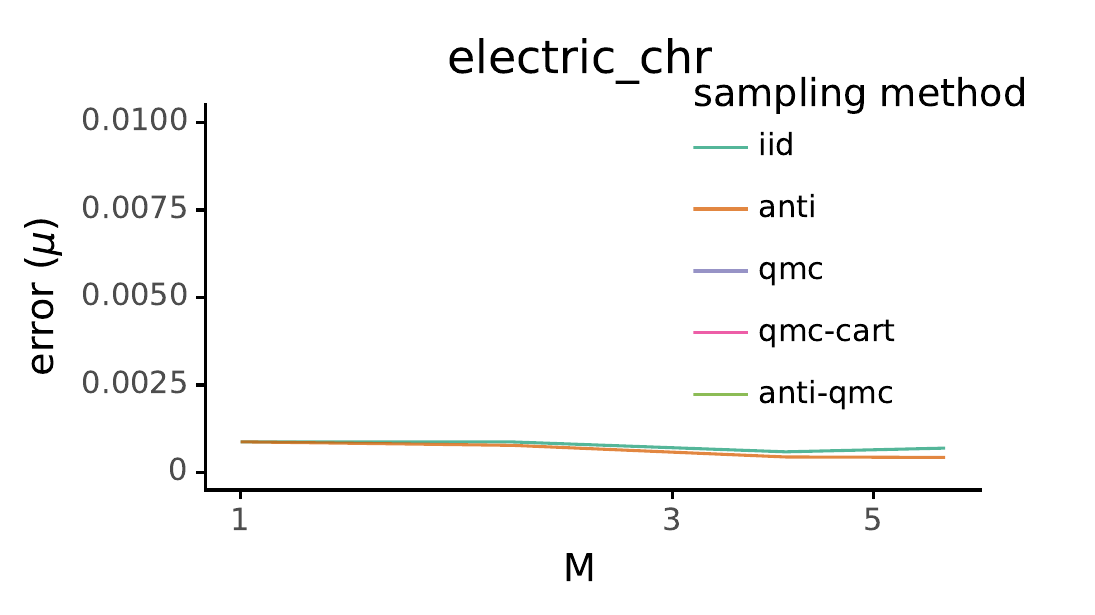}\linebreak{}

\includegraphics[width=0.33\columnwidth]{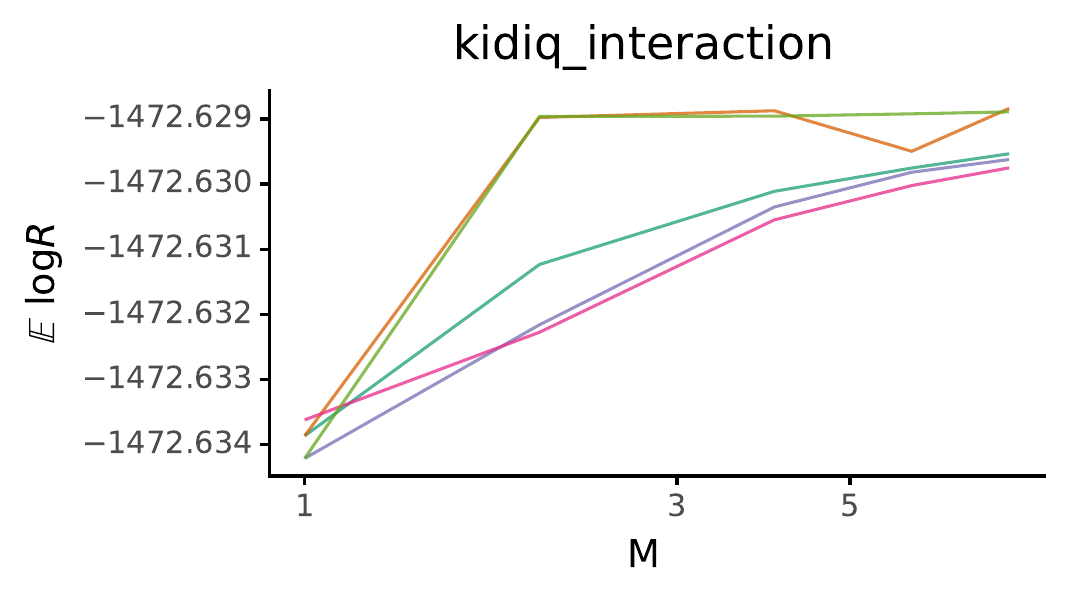}\includegraphics[width=0.33\columnwidth]{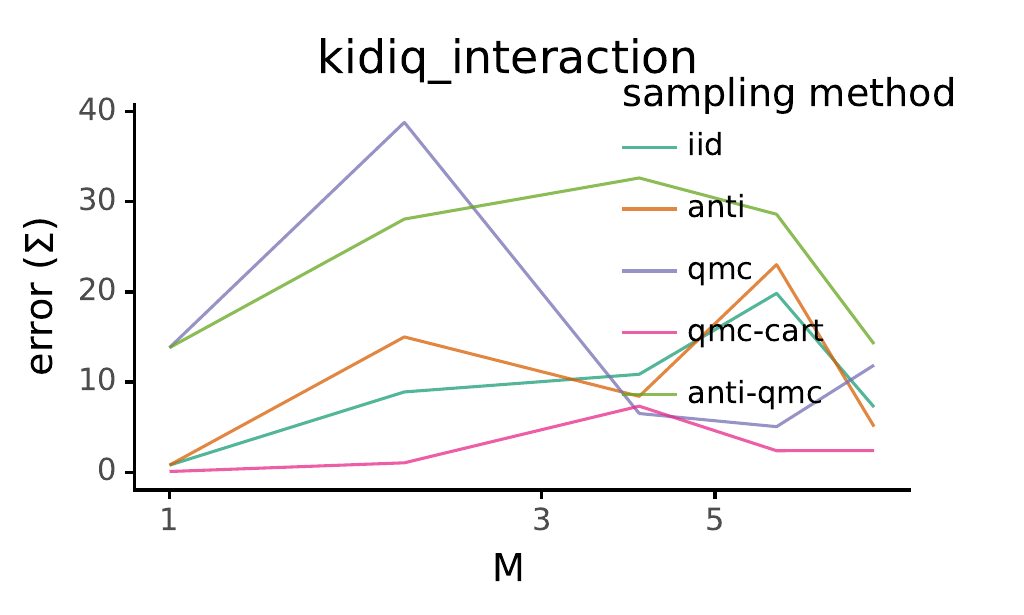}\includegraphics[width=0.33\columnwidth]{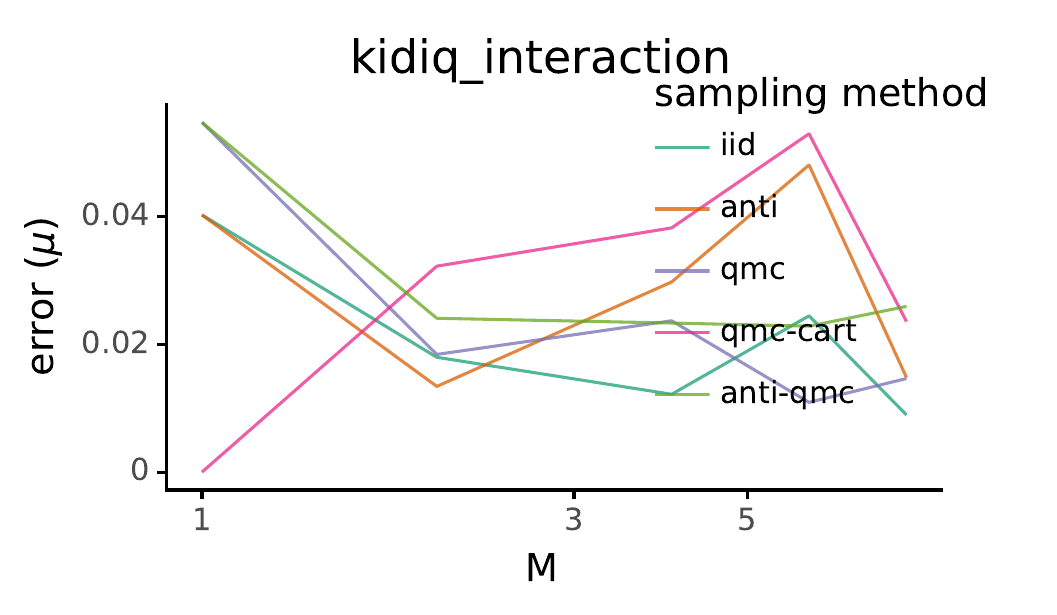}\linebreak{}

\includegraphics[width=0.33\columnwidth]{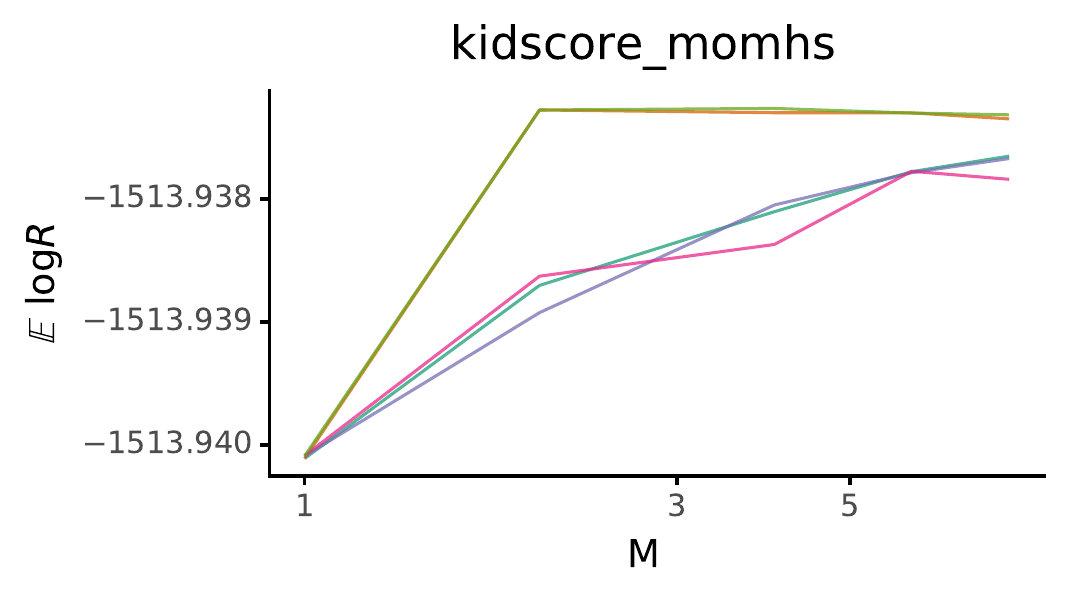}\includegraphics[width=0.33\columnwidth]{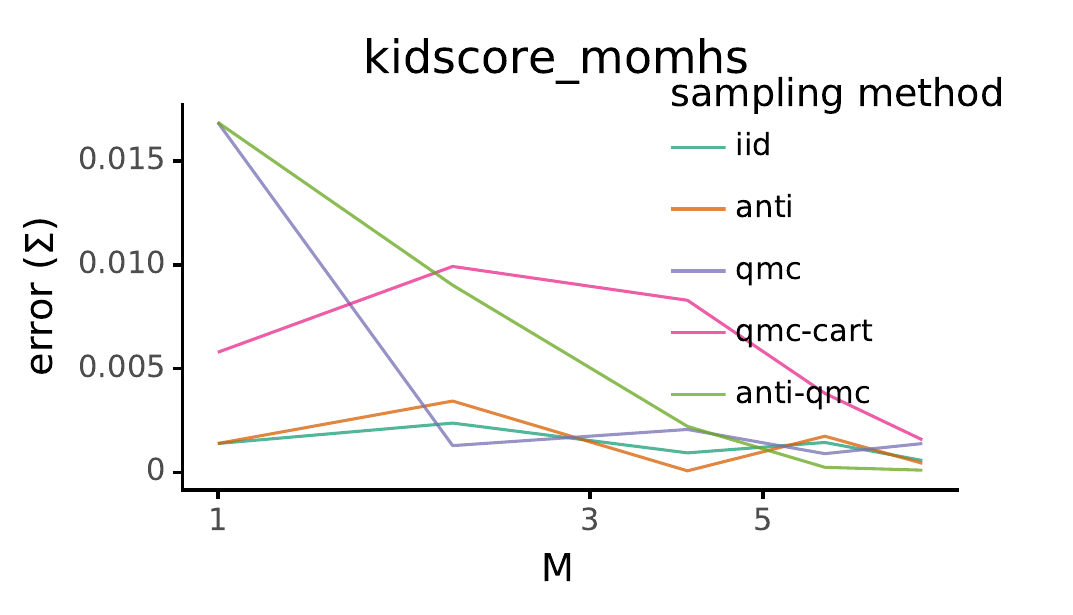}\includegraphics[width=0.33\columnwidth]{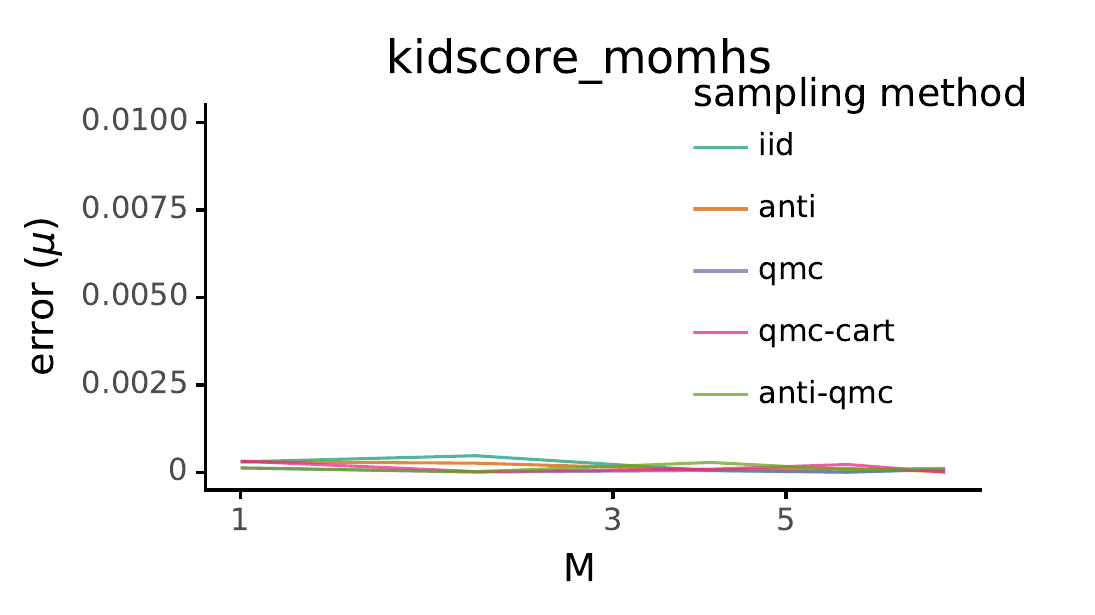}\linebreak{}

\includegraphics[width=0.33\columnwidth]{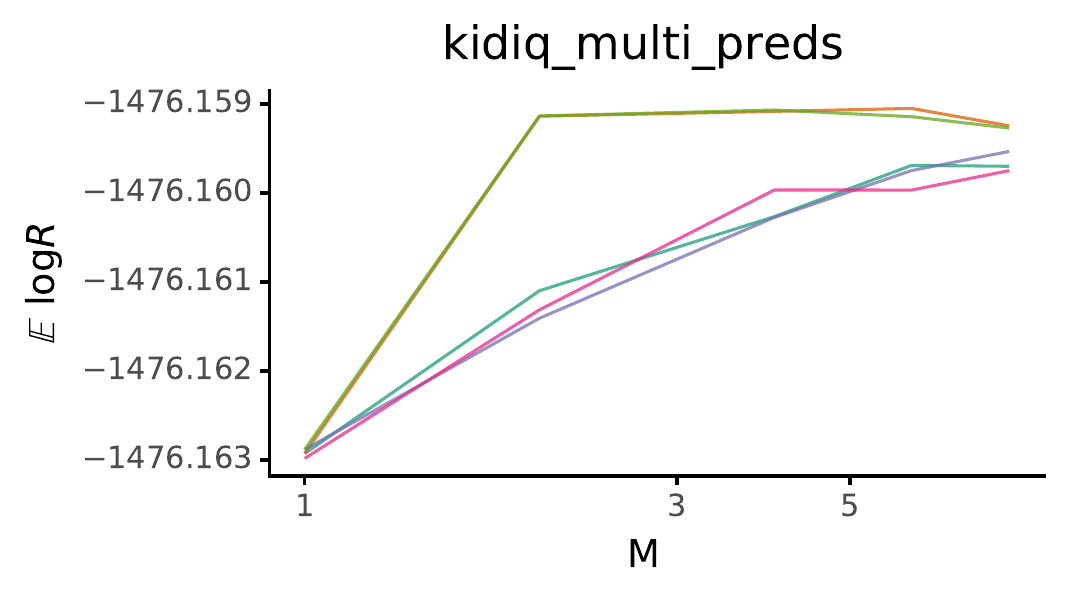}\includegraphics[width=0.33\columnwidth]{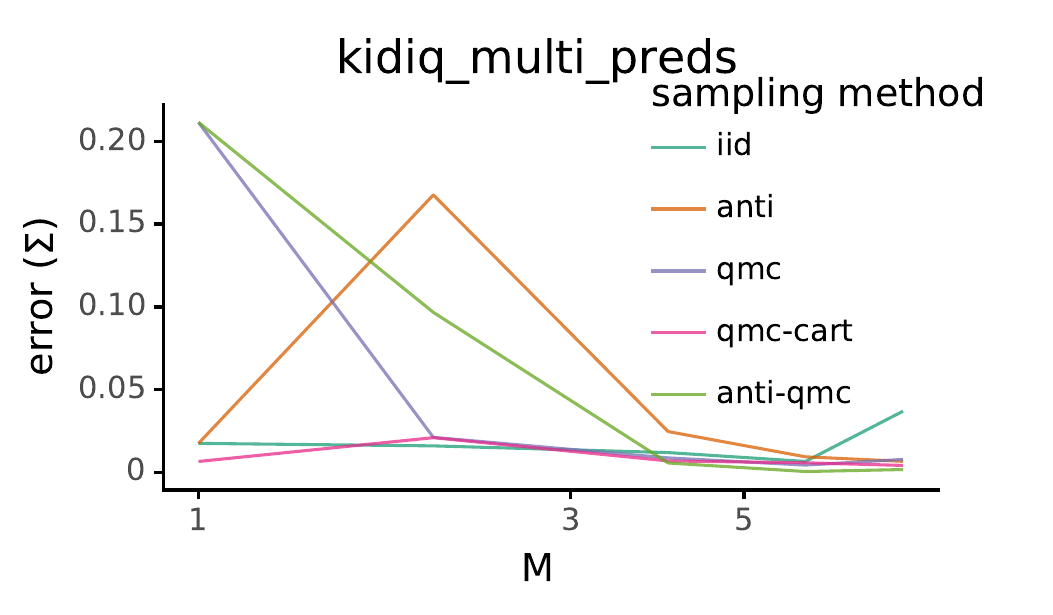}\includegraphics[width=0.33\columnwidth]{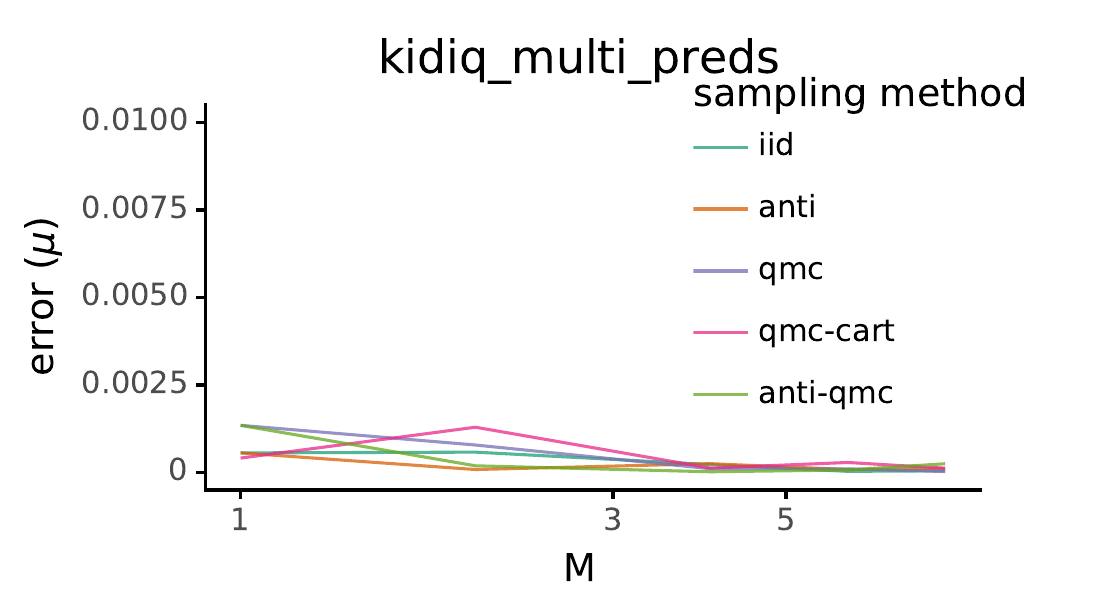}\linebreak{}

\caption{\textbf{Across all models, improvements in likelihood bounds correlate
strongly with improvements in posterior accuracy. Better sampling
methods can improve both.}}
\end{figure}

\begin{figure}
\includegraphics[width=0.33\columnwidth]{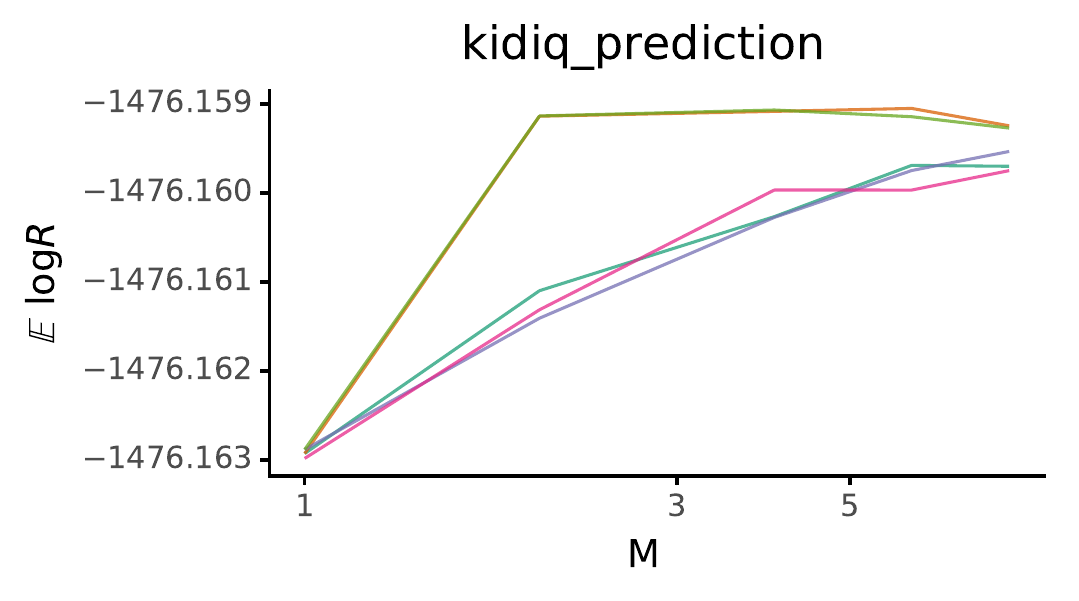}\includegraphics[width=0.33\columnwidth]{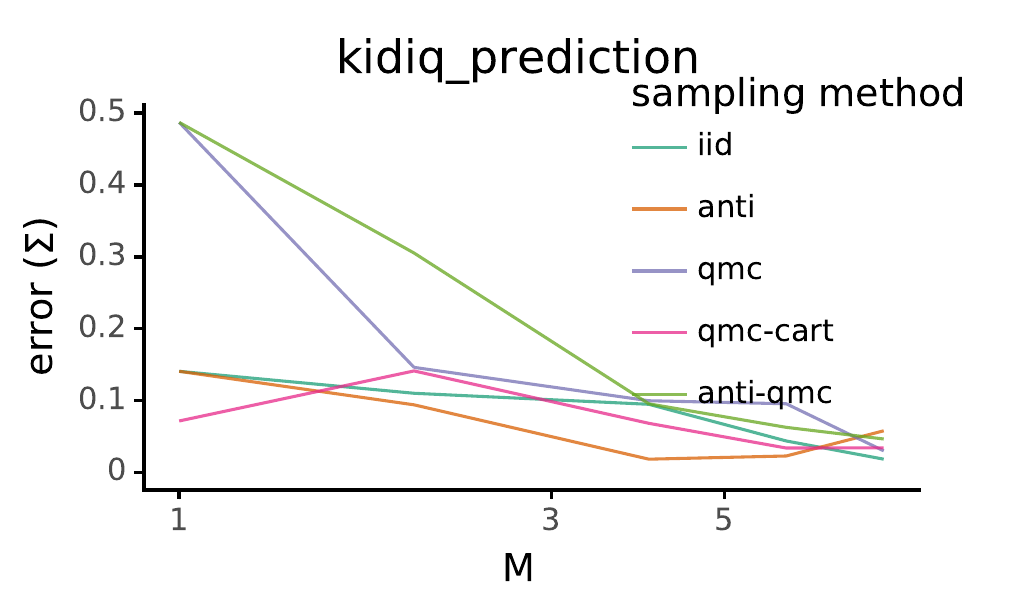}\includegraphics[width=0.33\columnwidth]{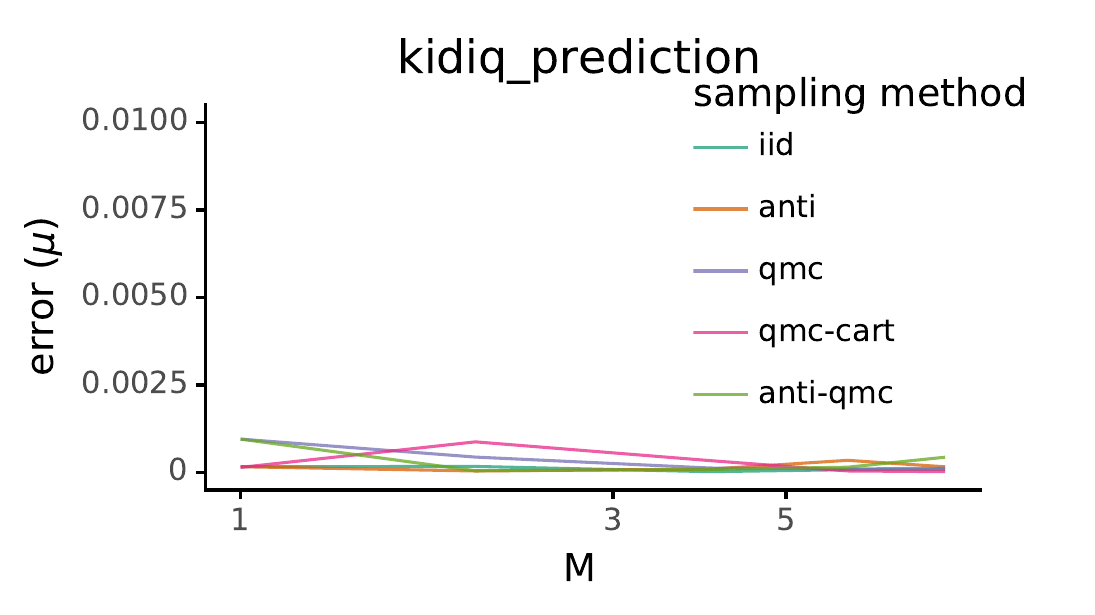}\linebreak{}

\includegraphics[width=0.33\columnwidth]{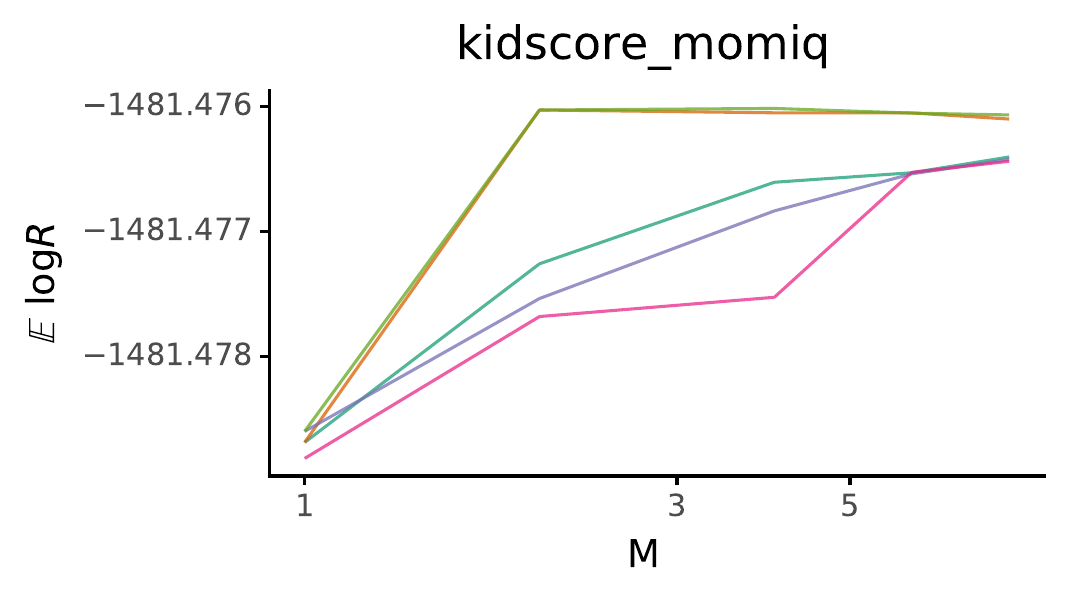}\includegraphics[width=0.33\columnwidth]{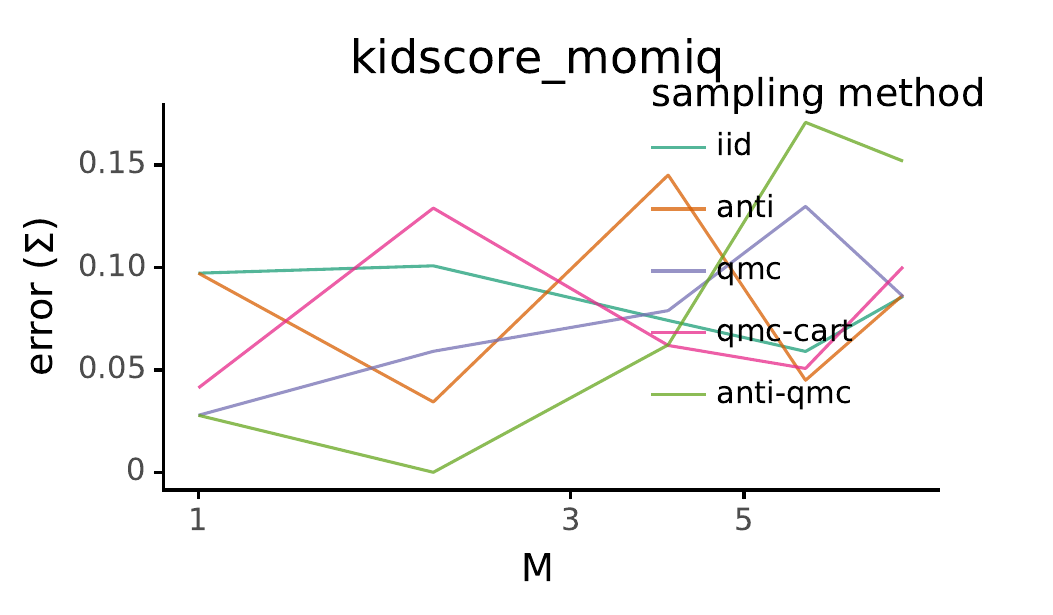}\includegraphics[width=0.33\columnwidth]{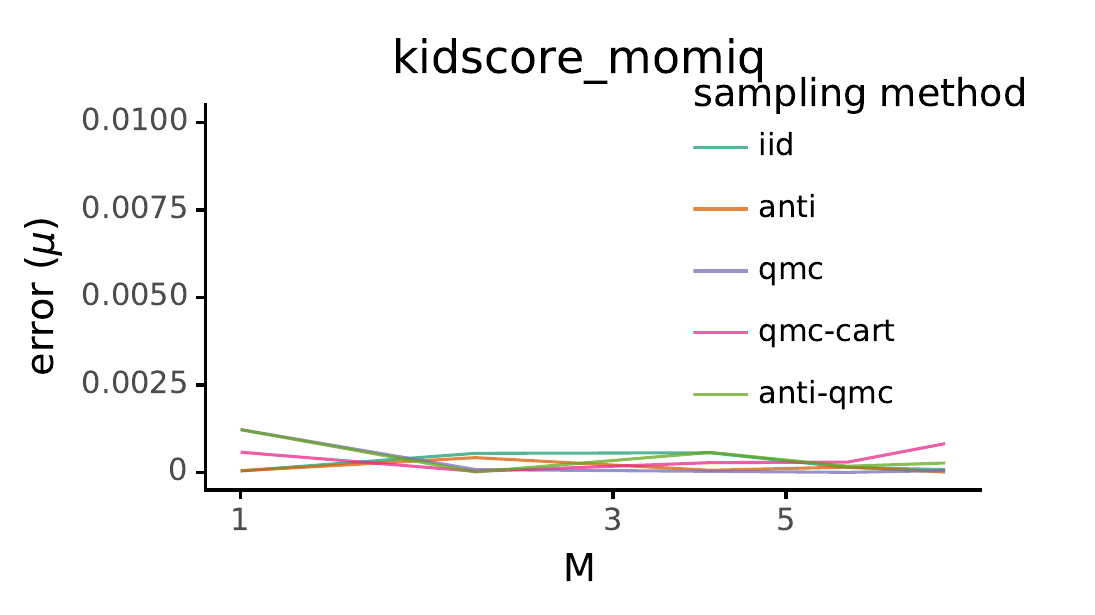}\linebreak{}

\includegraphics[width=0.33\columnwidth]{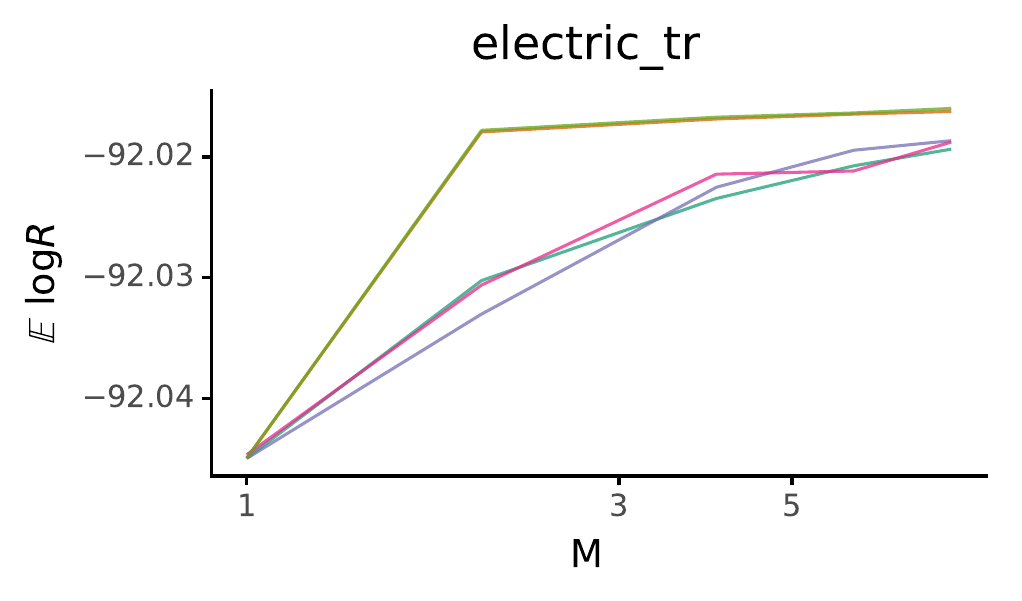}\includegraphics[width=0.33\columnwidth]{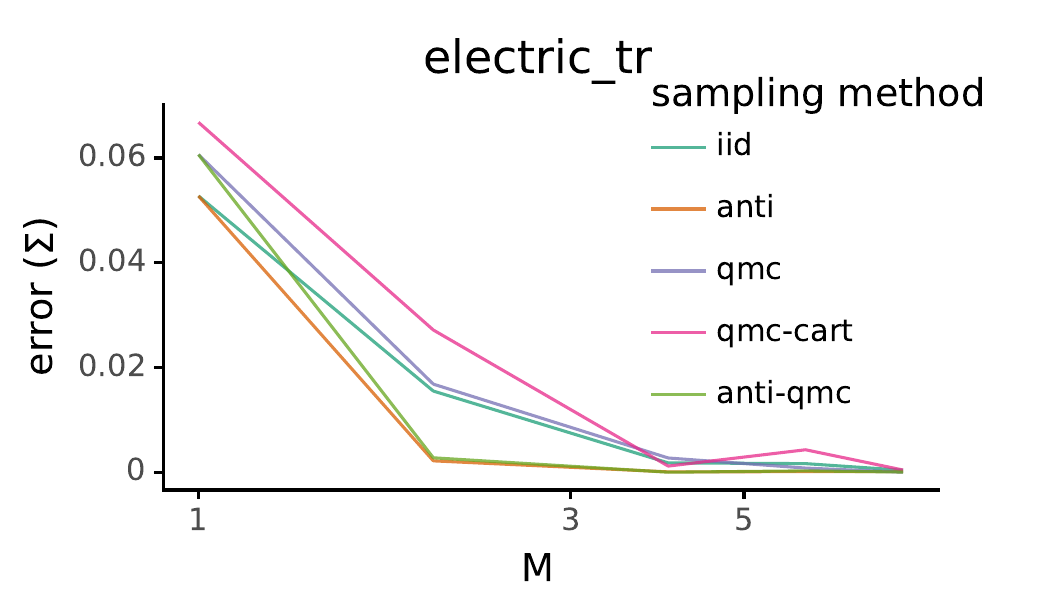}\includegraphics[width=0.33\columnwidth]{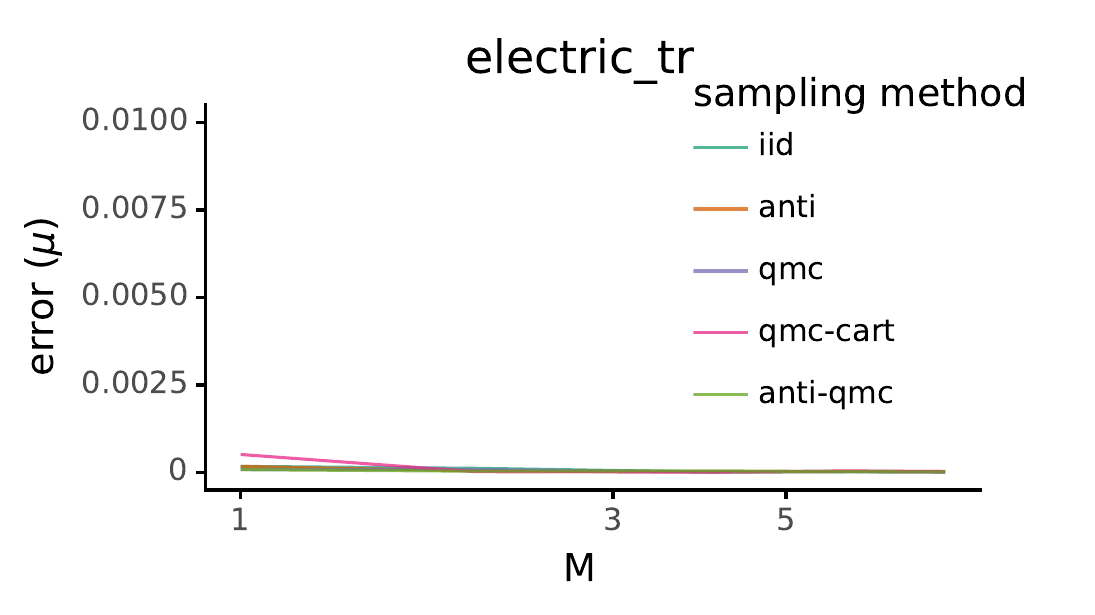}\linebreak{}

\includegraphics[width=0.33\columnwidth]{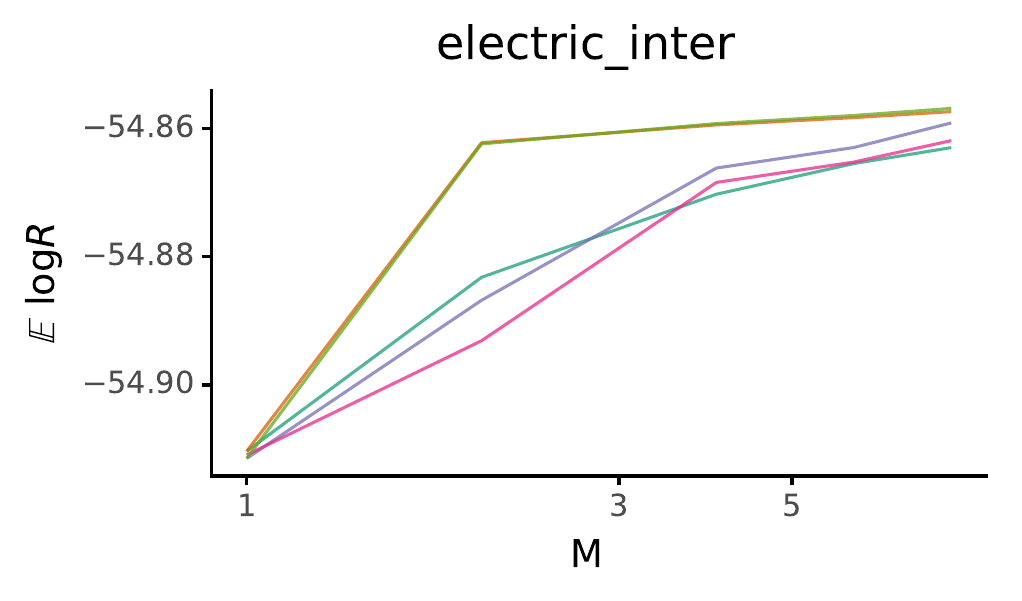}\includegraphics[width=0.33\columnwidth]{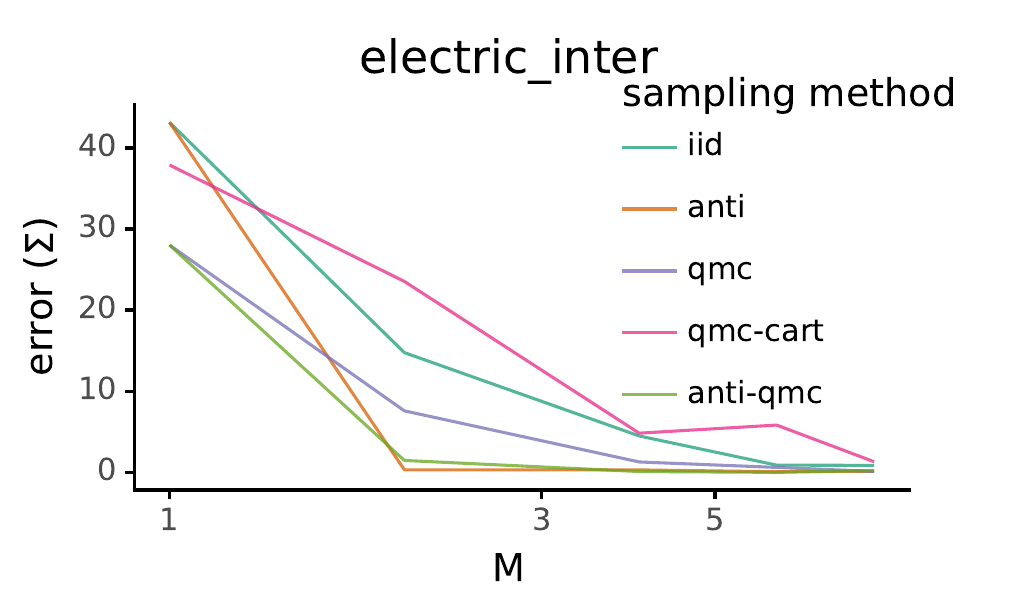}\includegraphics[width=0.33\columnwidth]{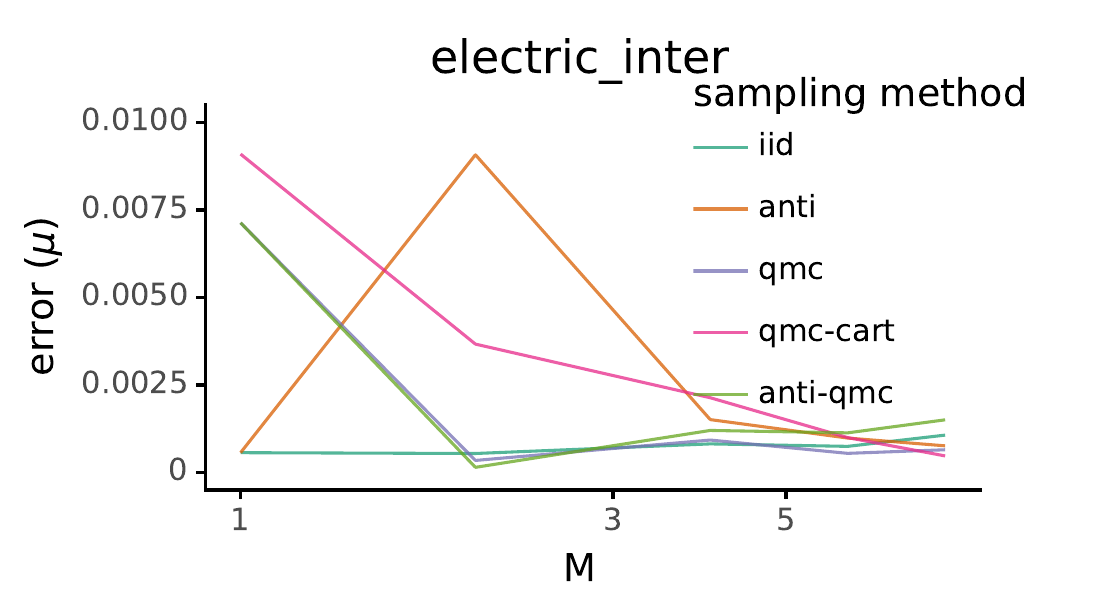}\linebreak{}

\includegraphics[width=0.33\columnwidth]{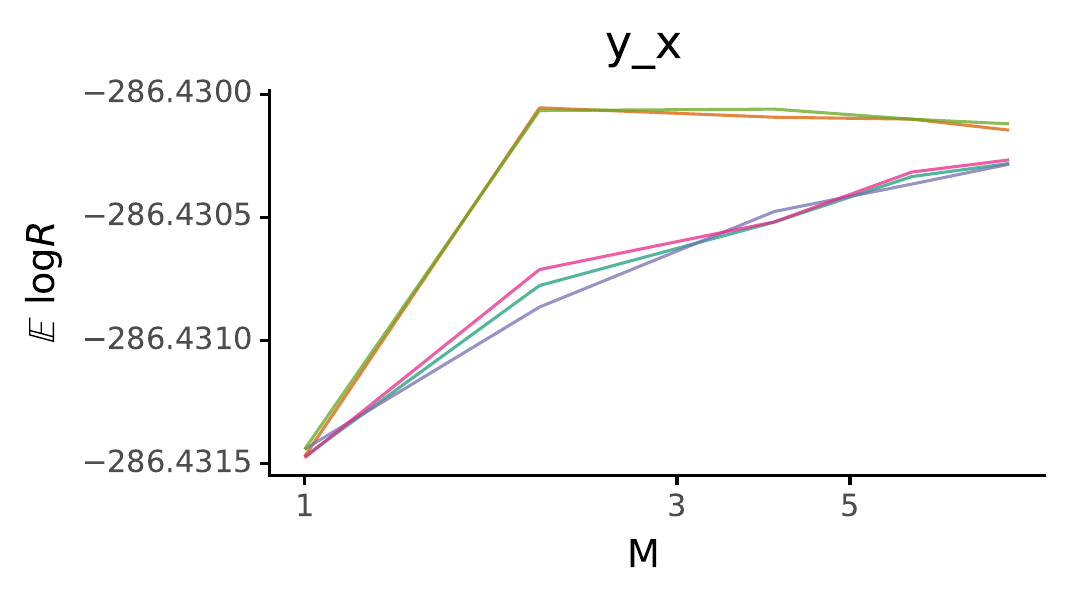}\includegraphics[width=0.33\columnwidth]{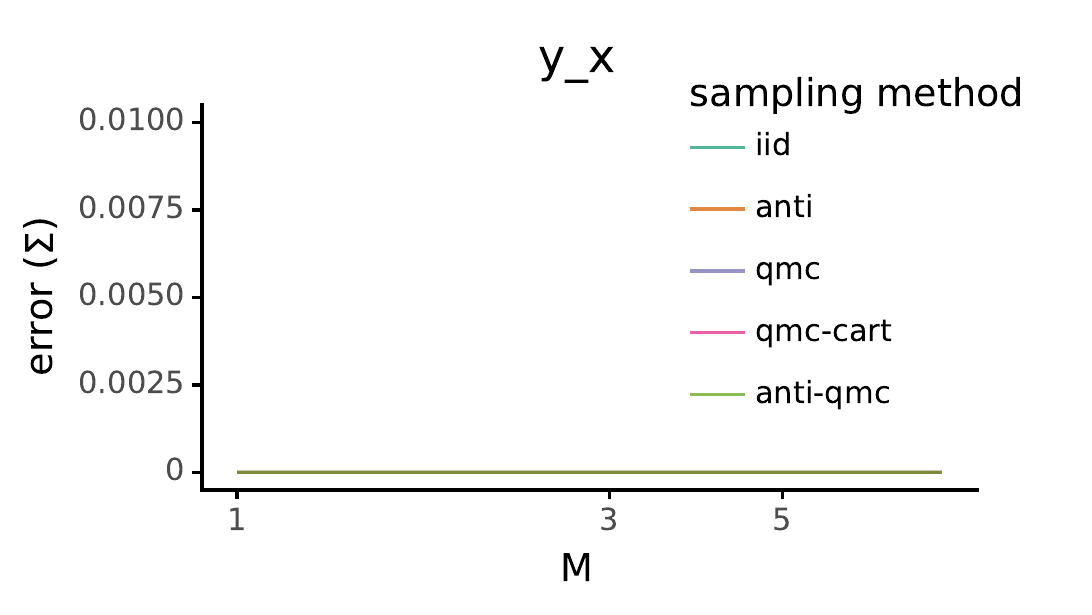}\includegraphics[width=0.33\columnwidth]{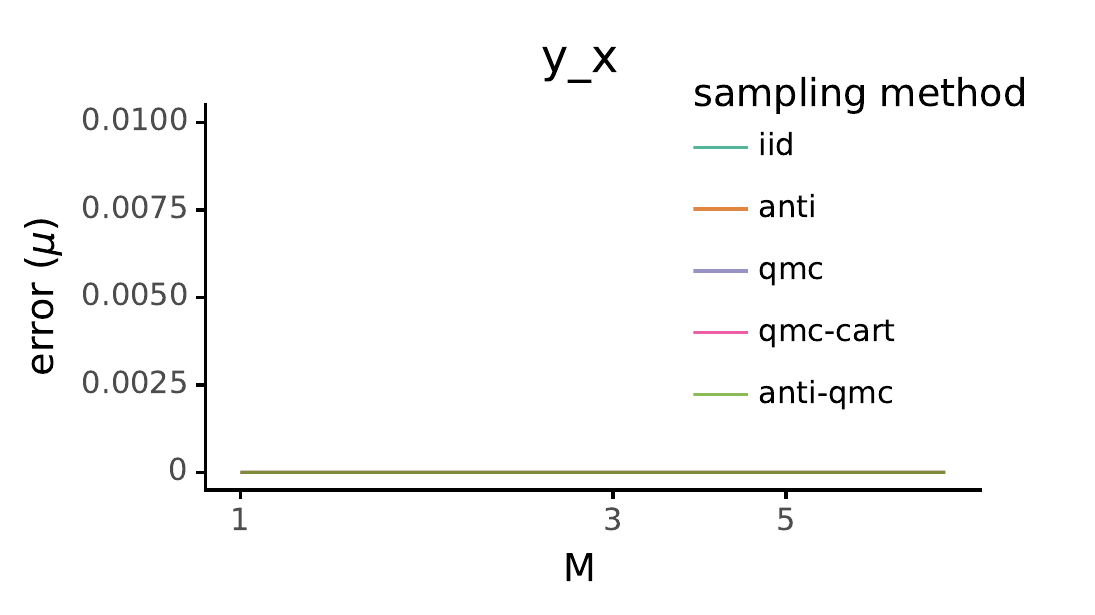}\linebreak{}

\includegraphics[width=0.33\columnwidth]{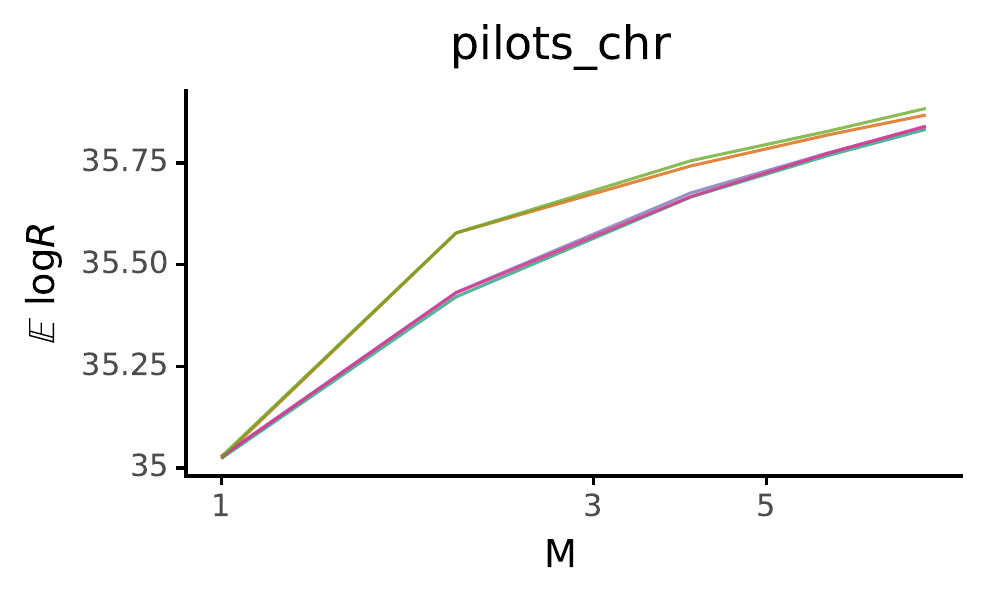}\includegraphics[width=0.33\columnwidth]{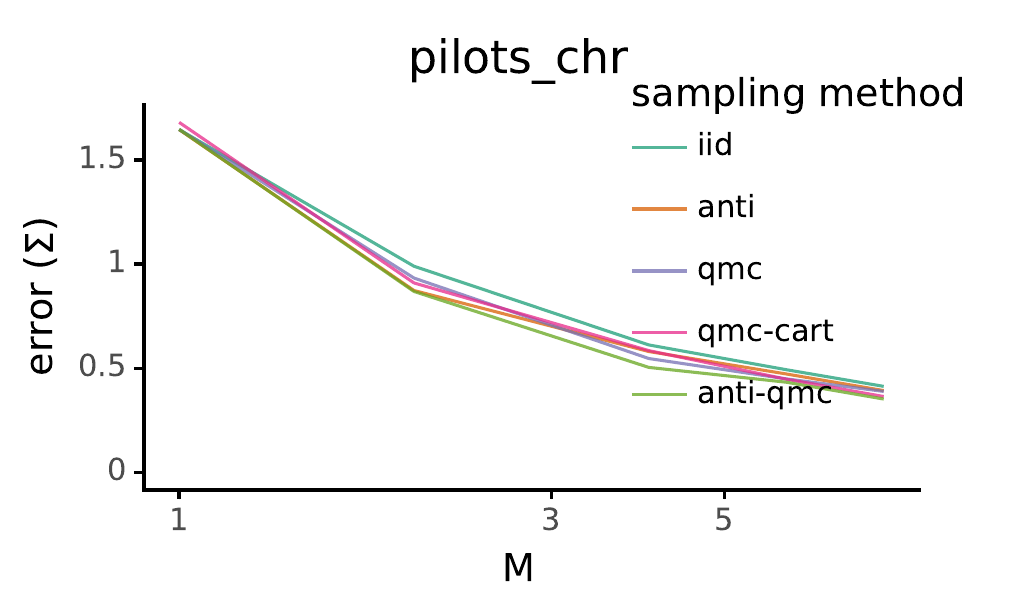}\includegraphics[width=0.33\columnwidth]{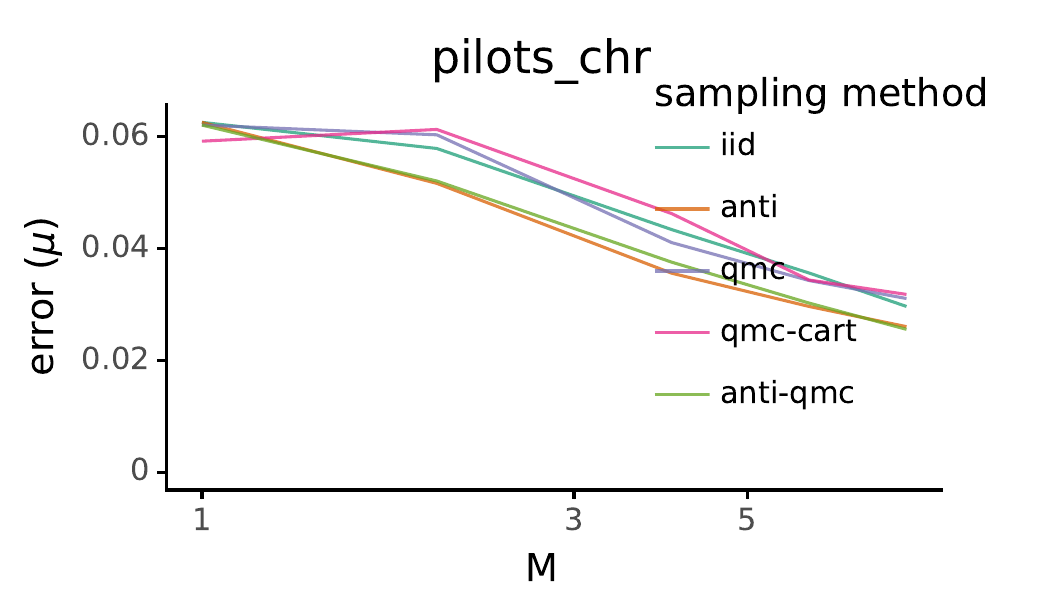}\linebreak{}

\includegraphics[width=0.33\columnwidth]{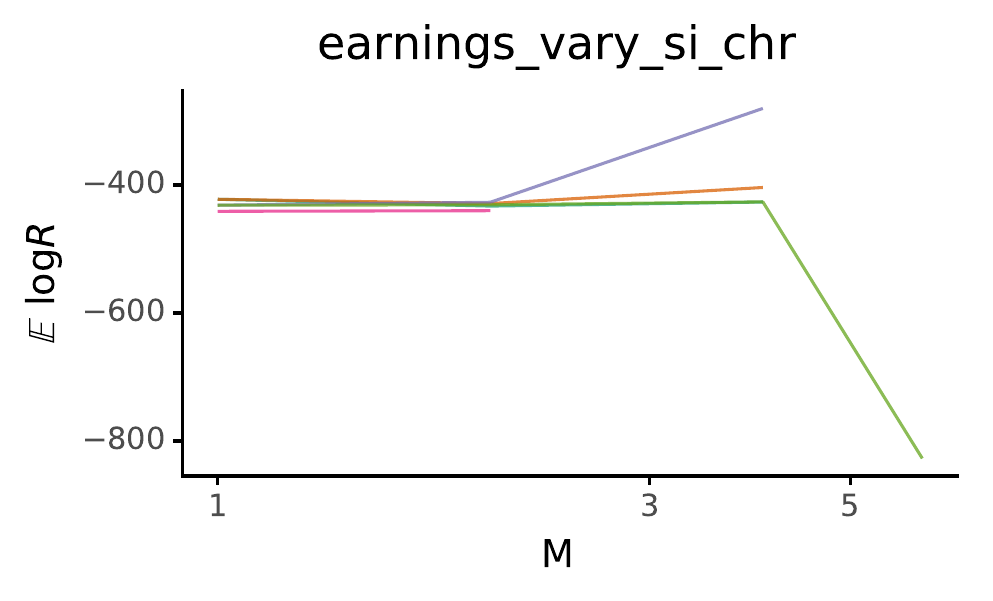}\includegraphics[width=0.33\columnwidth]{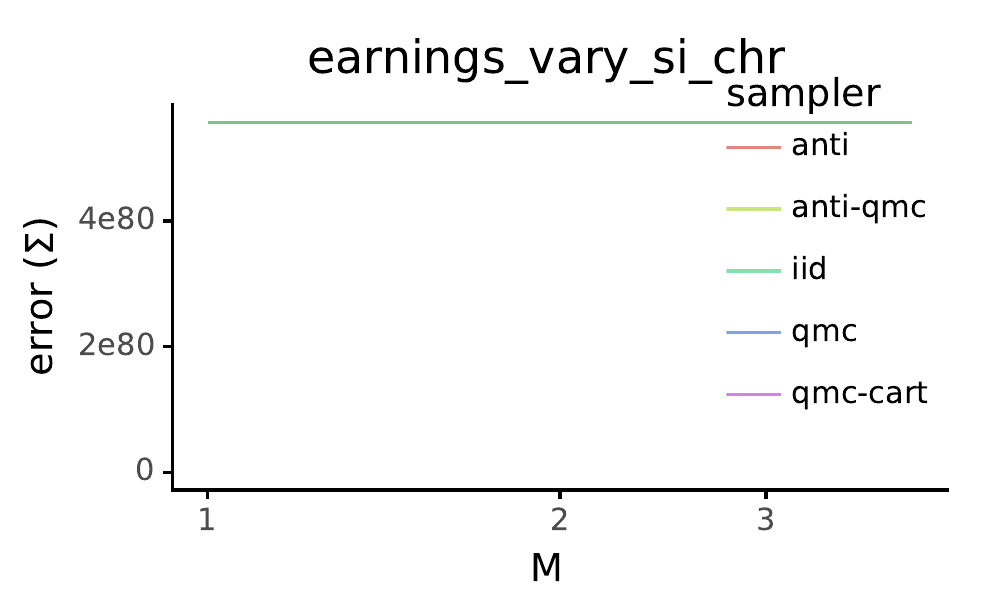}\includegraphics[width=0.33\columnwidth]{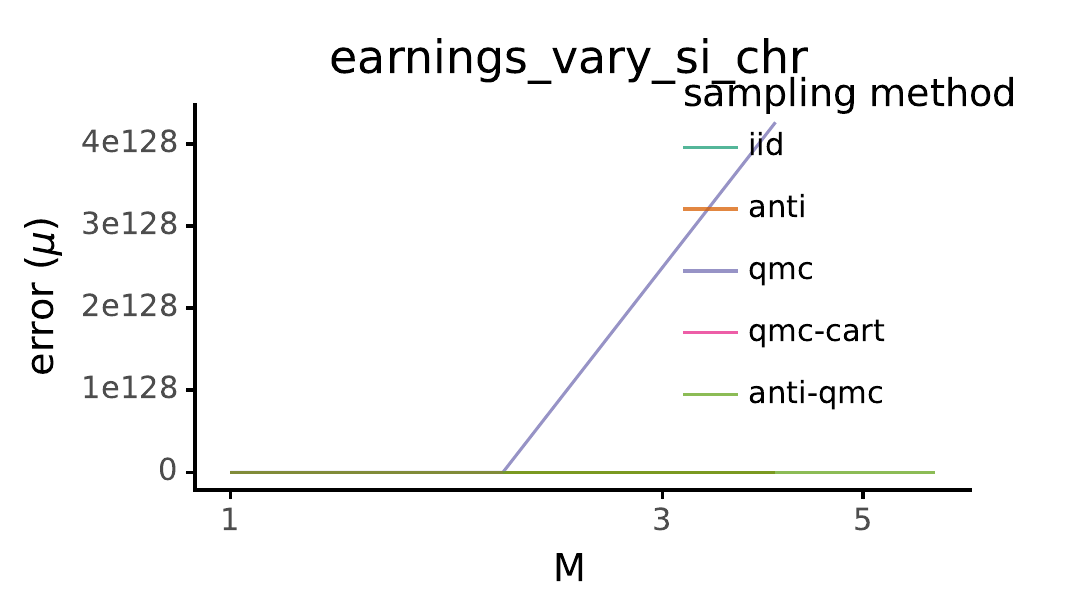}\linebreak{}

\caption{\textbf{Across all models, improvements in likelihood bounds correlate
strongly with improvements in posterior accuracy. Better sampling
methods can improve both.}}
\end{figure}

\begin{figure}
\includegraphics[width=0.33\columnwidth]{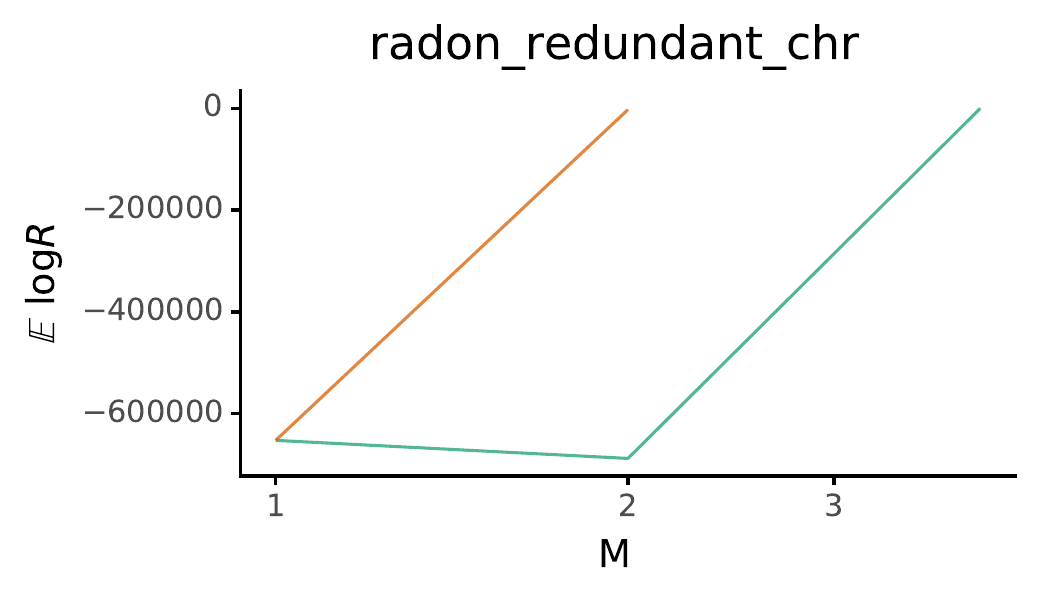}\includegraphics[width=0.33\columnwidth]{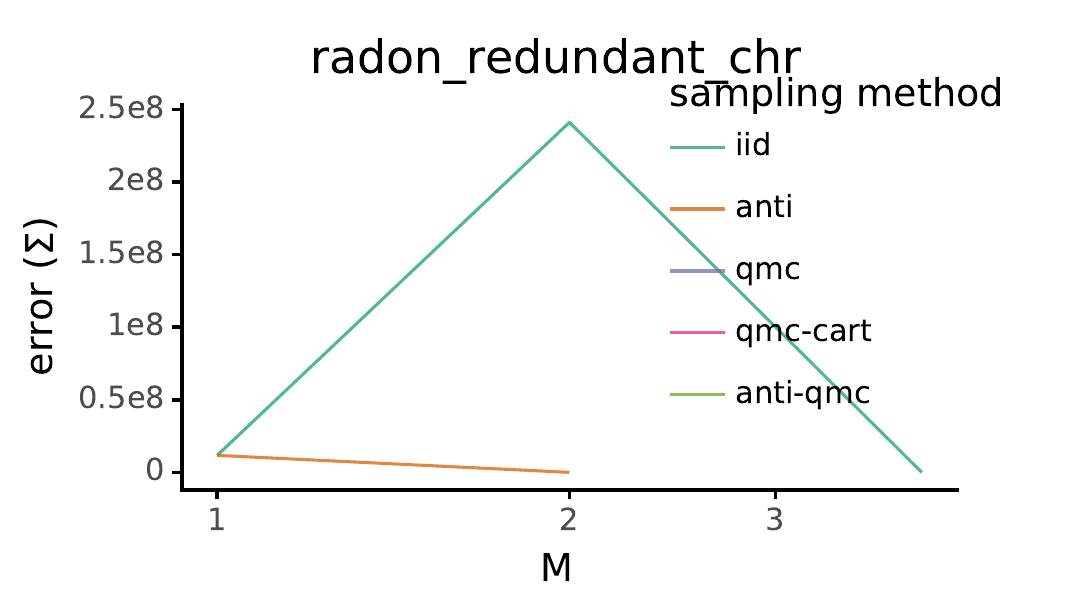}\includegraphics[width=0.33\columnwidth]{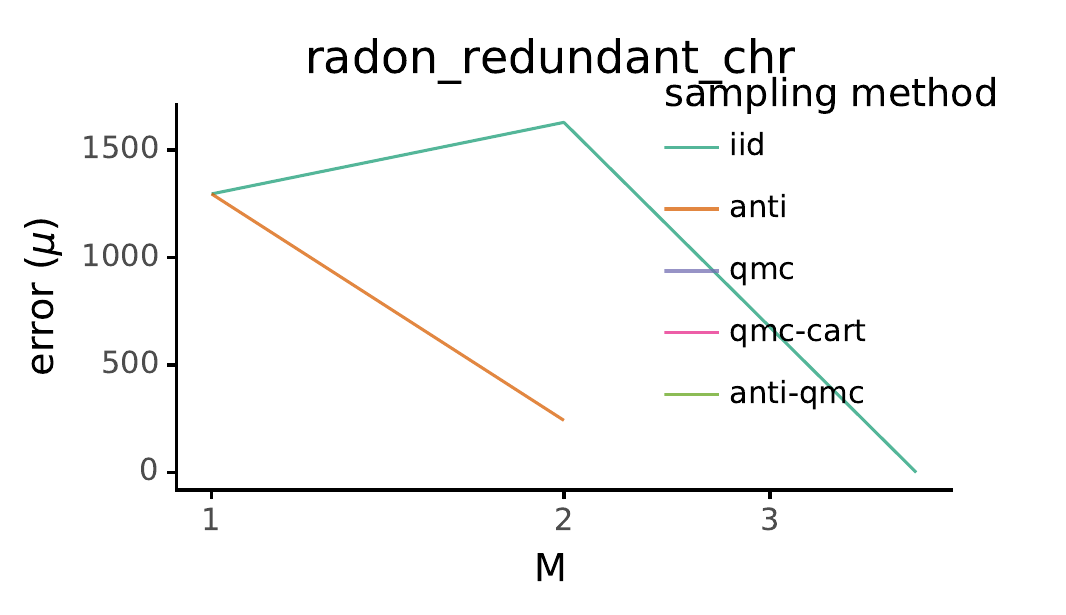}\linebreak{}

\includegraphics[width=0.33\columnwidth]{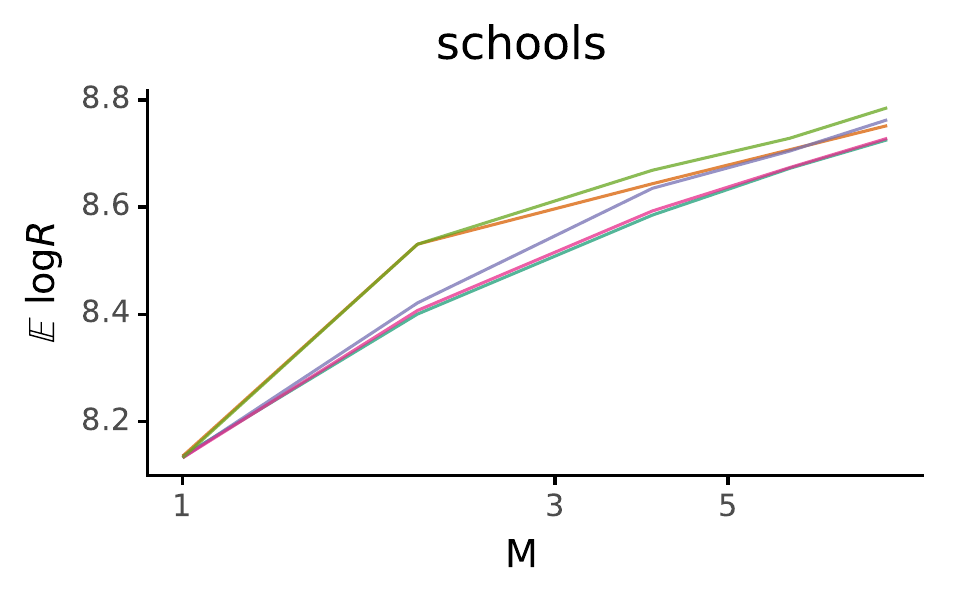}\includegraphics[width=0.33\columnwidth]{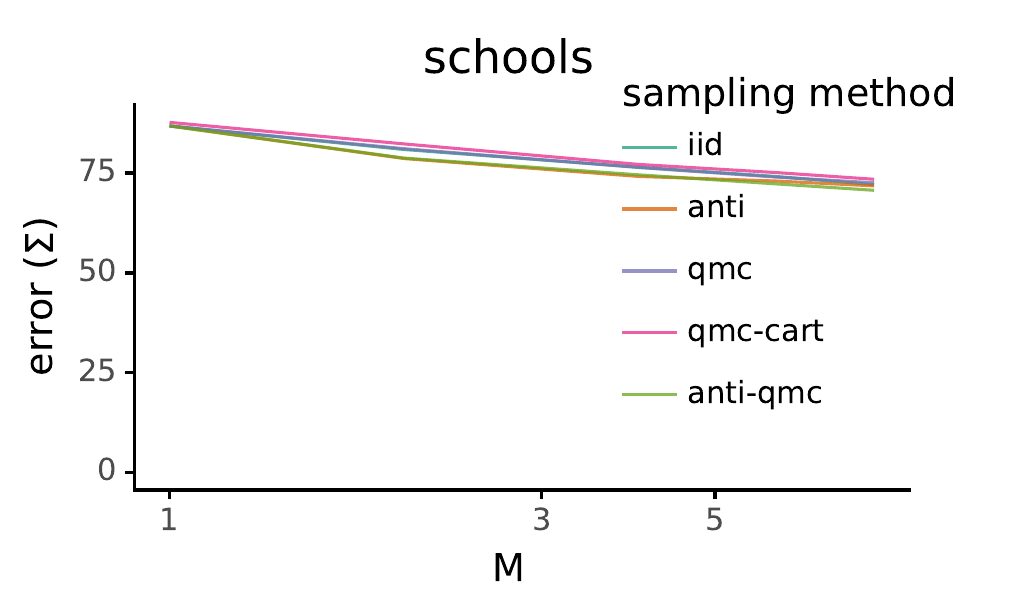}\includegraphics[width=0.33\columnwidth]{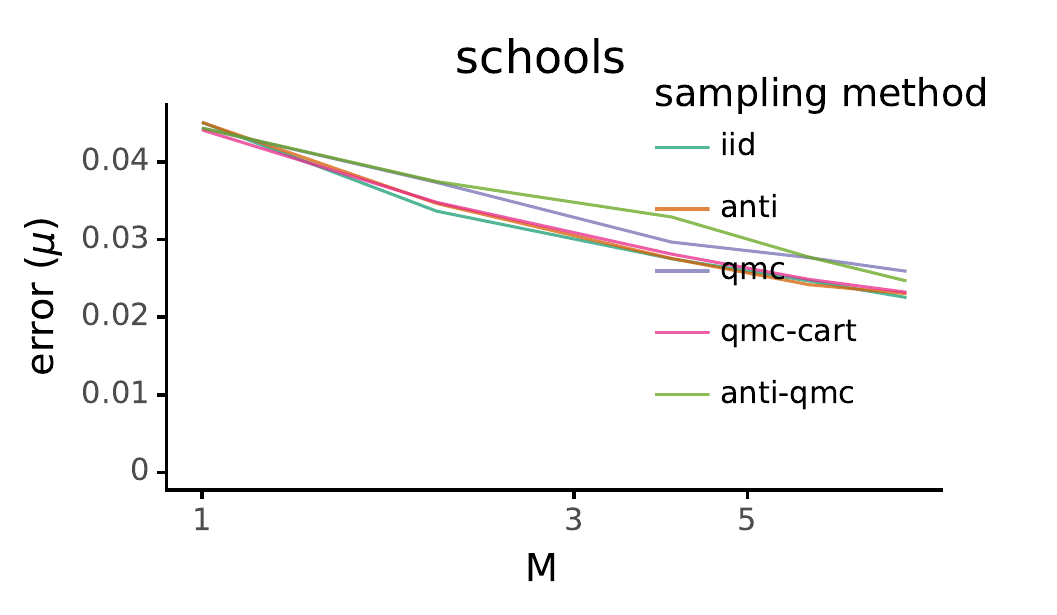}\linebreak{}

\includegraphics[width=0.33\columnwidth]{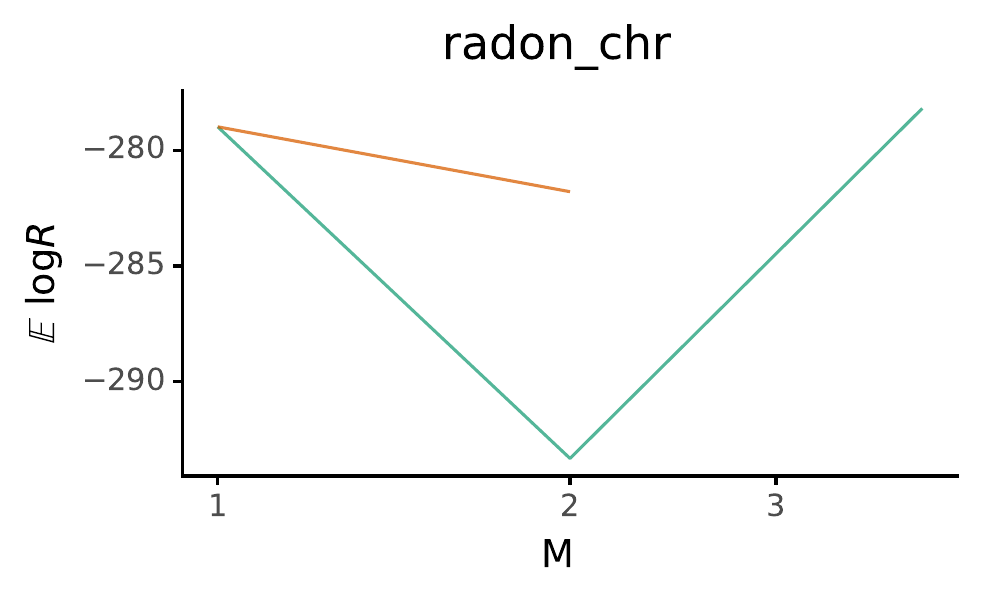}\includegraphics[width=0.33\columnwidth]{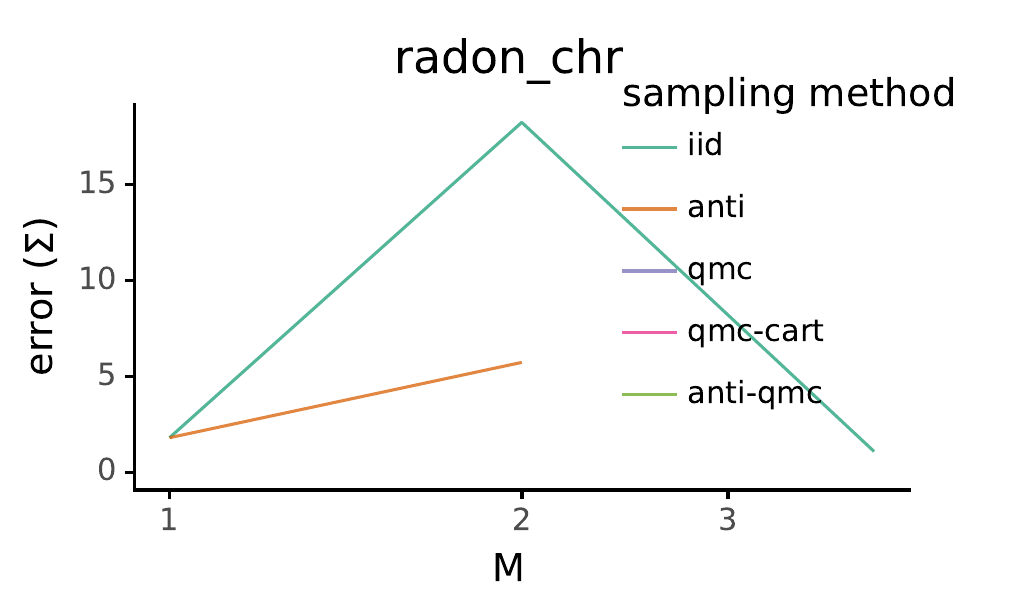}\includegraphics[width=0.33\columnwidth]{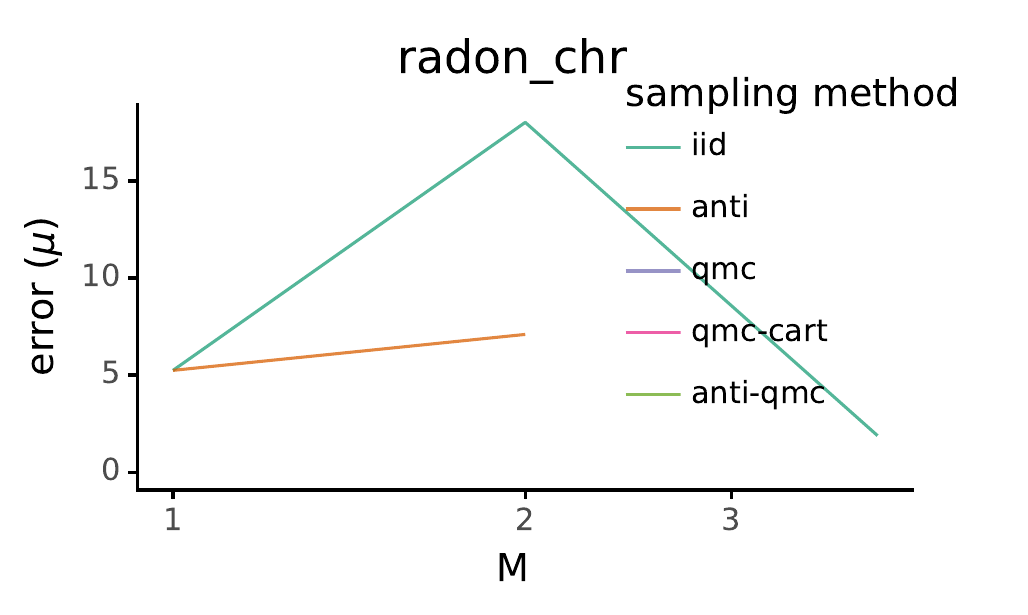}\linebreak{}

\includegraphics[width=0.33\columnwidth]{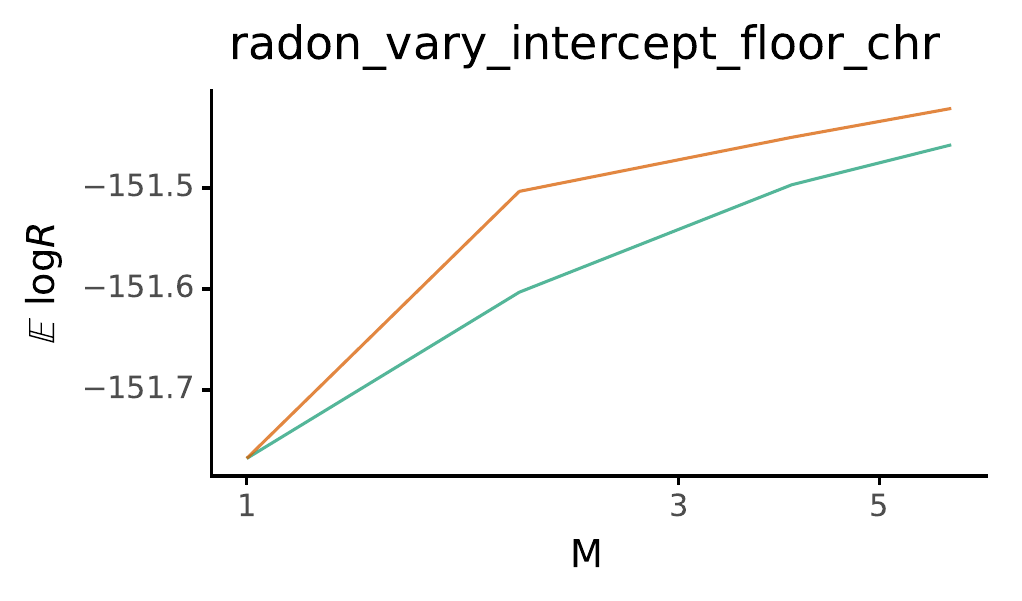}\includegraphics[width=0.33\columnwidth]{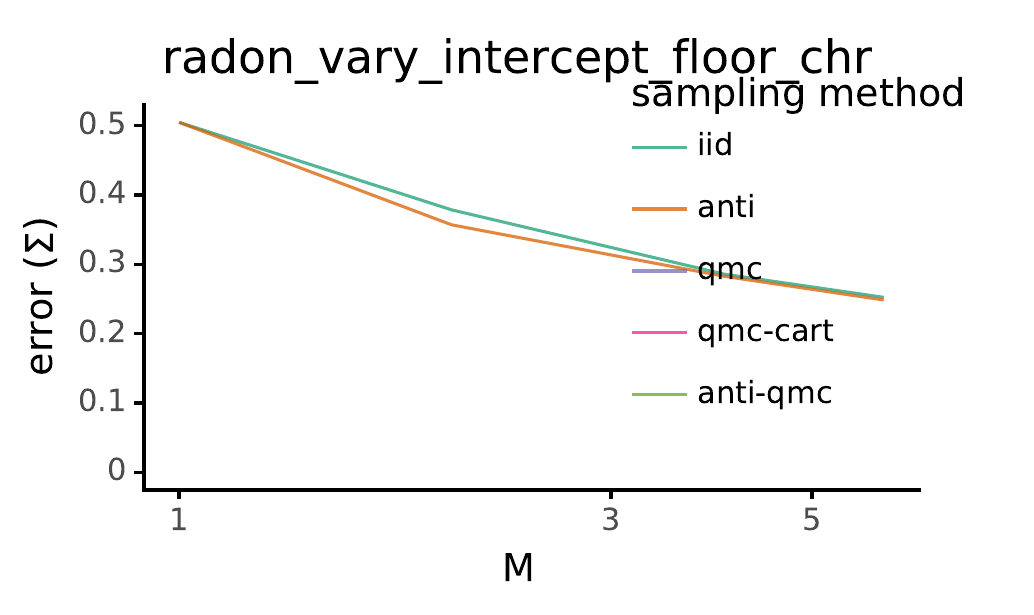}\includegraphics[width=0.33\columnwidth]{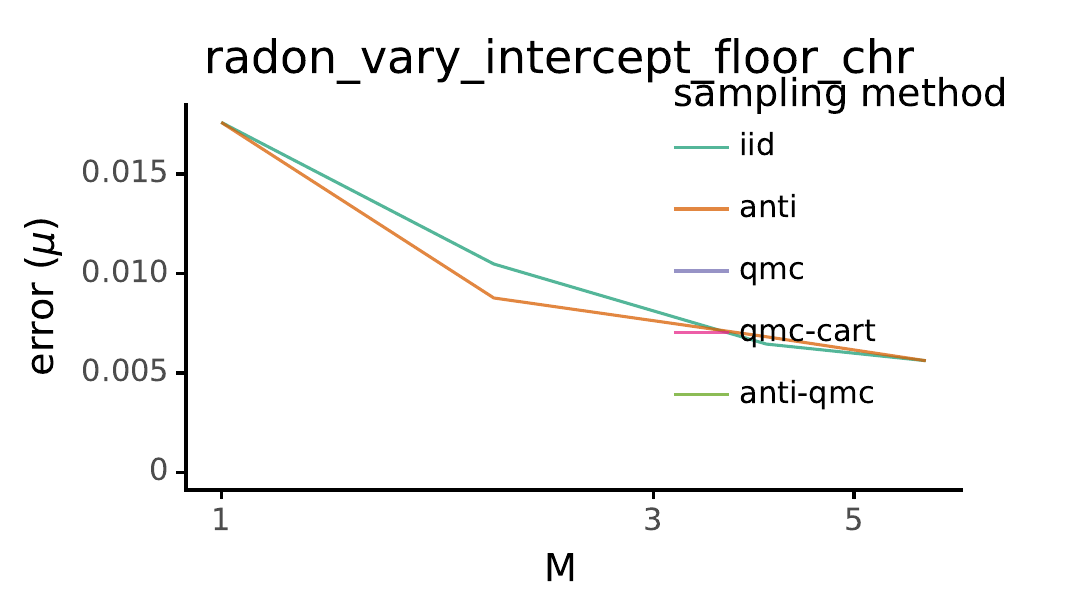}\linebreak{}

\includegraphics[width=0.33\columnwidth]{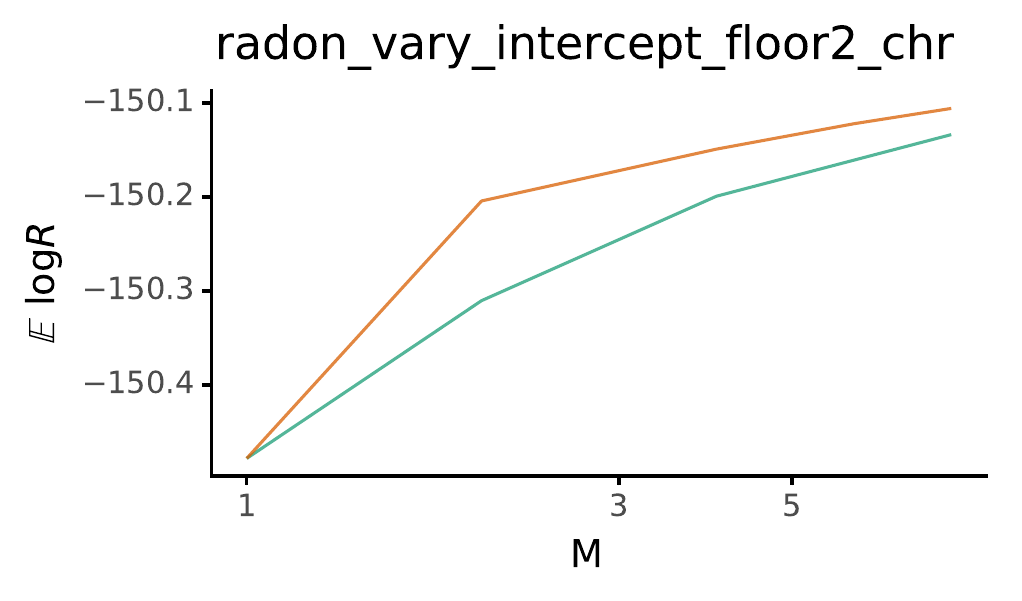}\includegraphics[width=0.33\columnwidth]{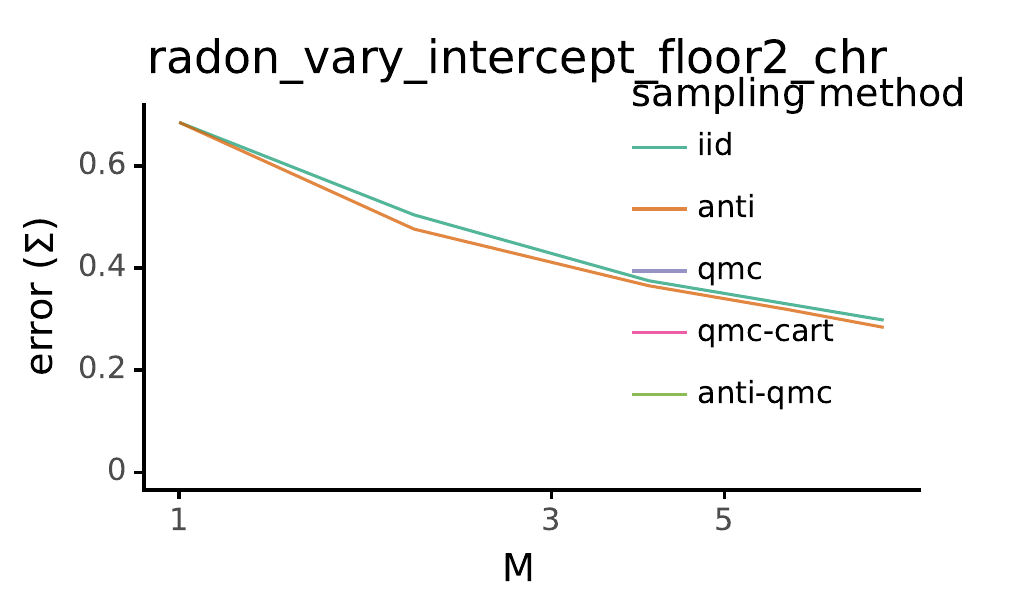}\includegraphics[width=0.33\columnwidth]{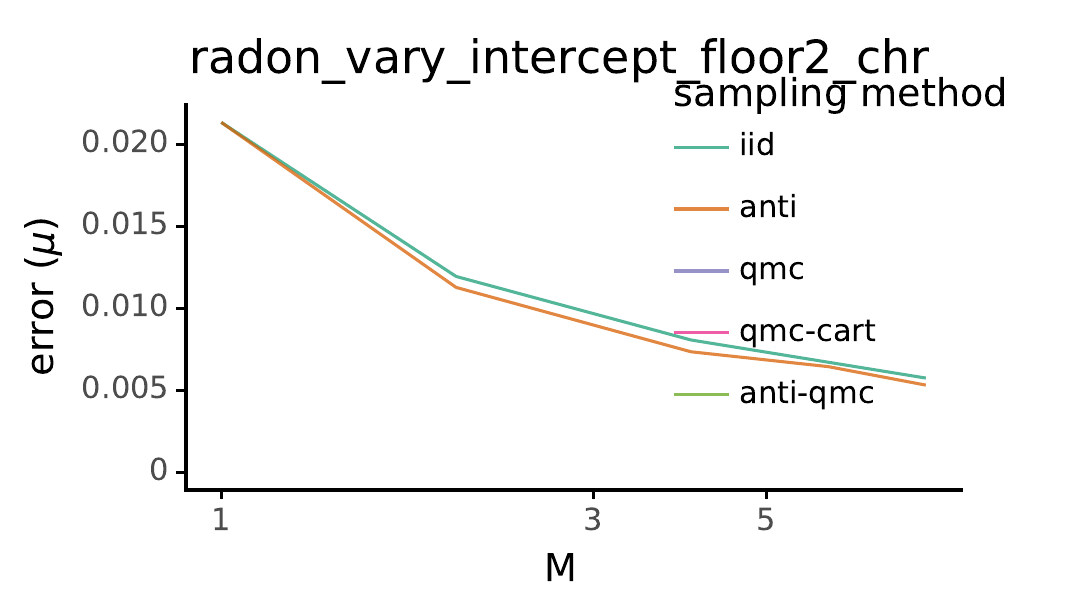}\linebreak{}

\includegraphics[width=0.33\columnwidth]{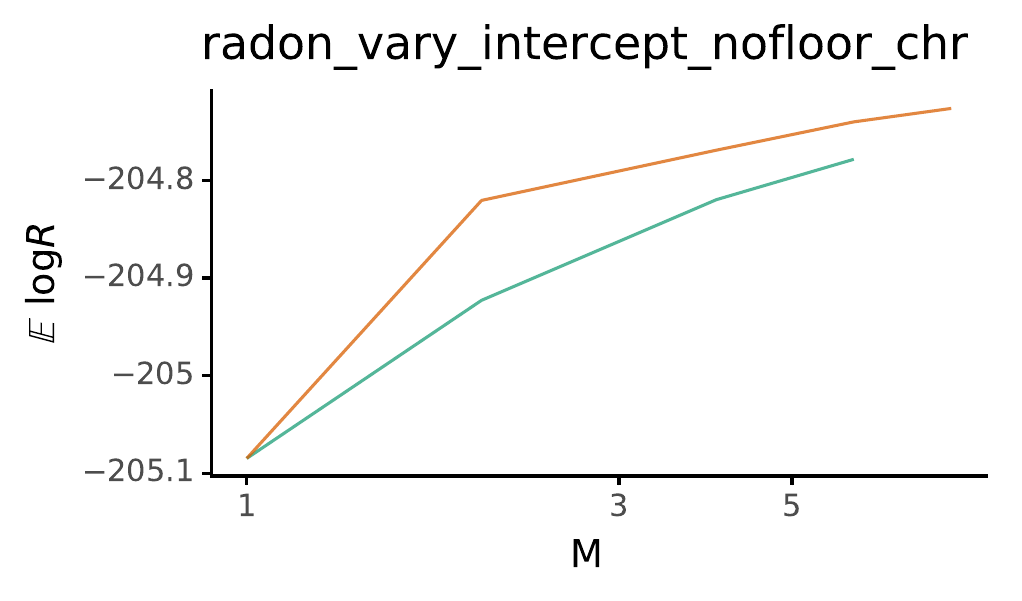}\includegraphics[width=0.33\columnwidth]{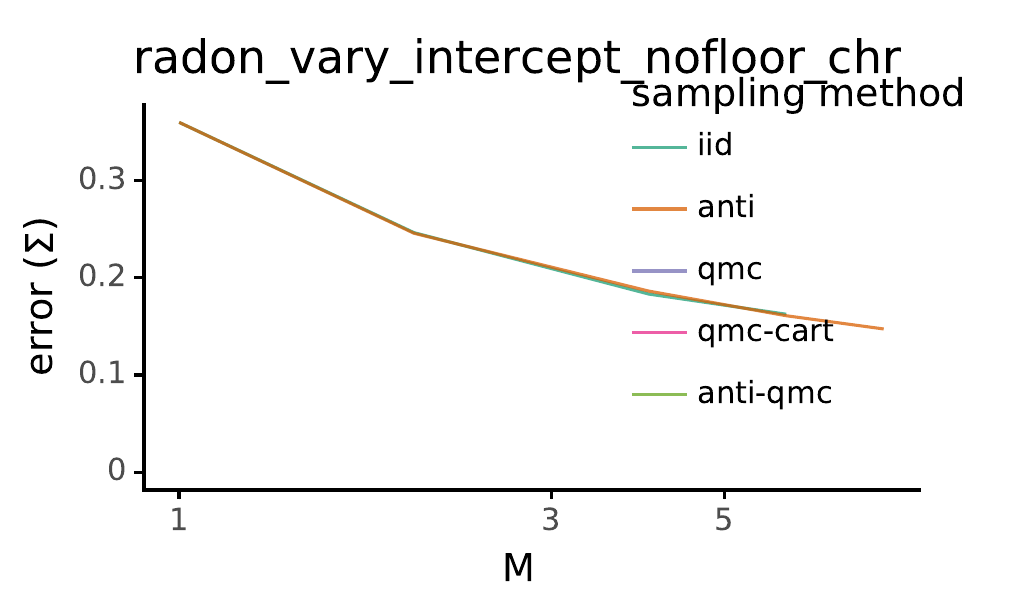}\includegraphics[width=0.33\columnwidth]{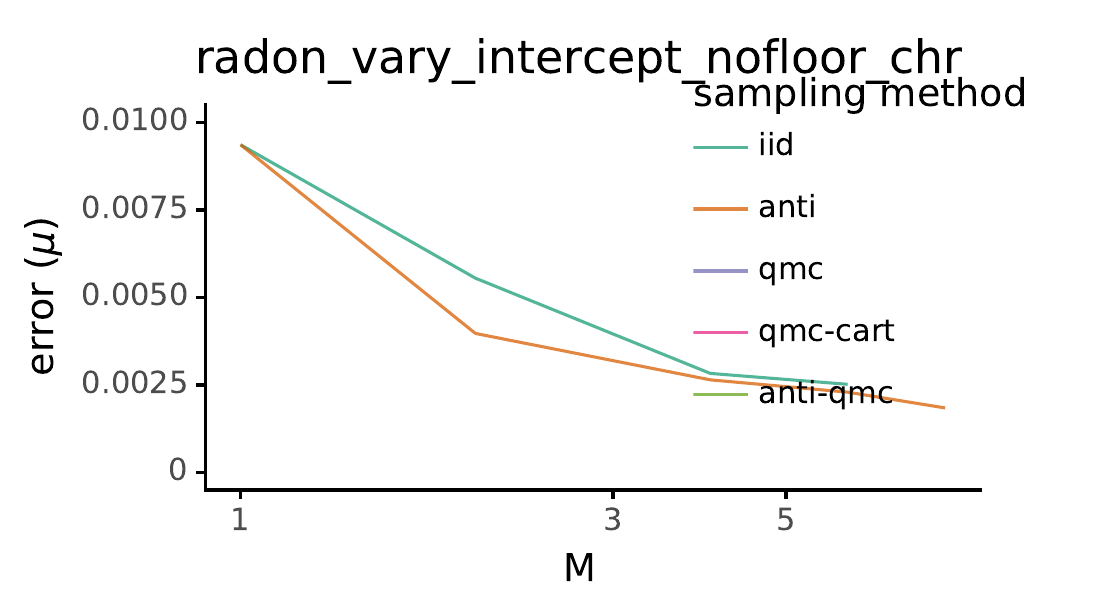}\linebreak{}

\includegraphics[width=0.33\columnwidth]{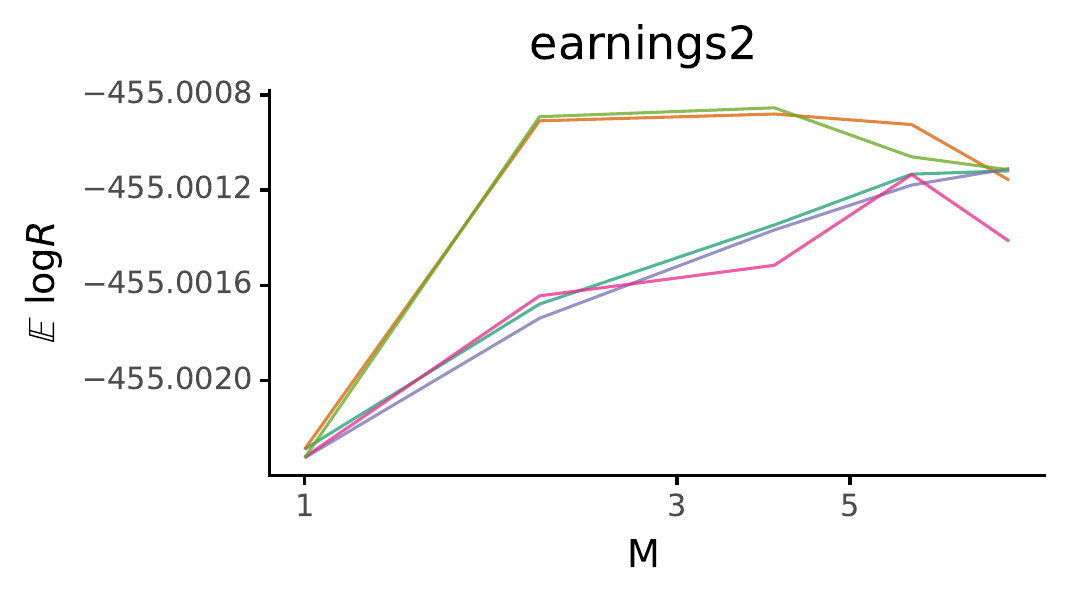}\includegraphics[width=0.33\columnwidth]{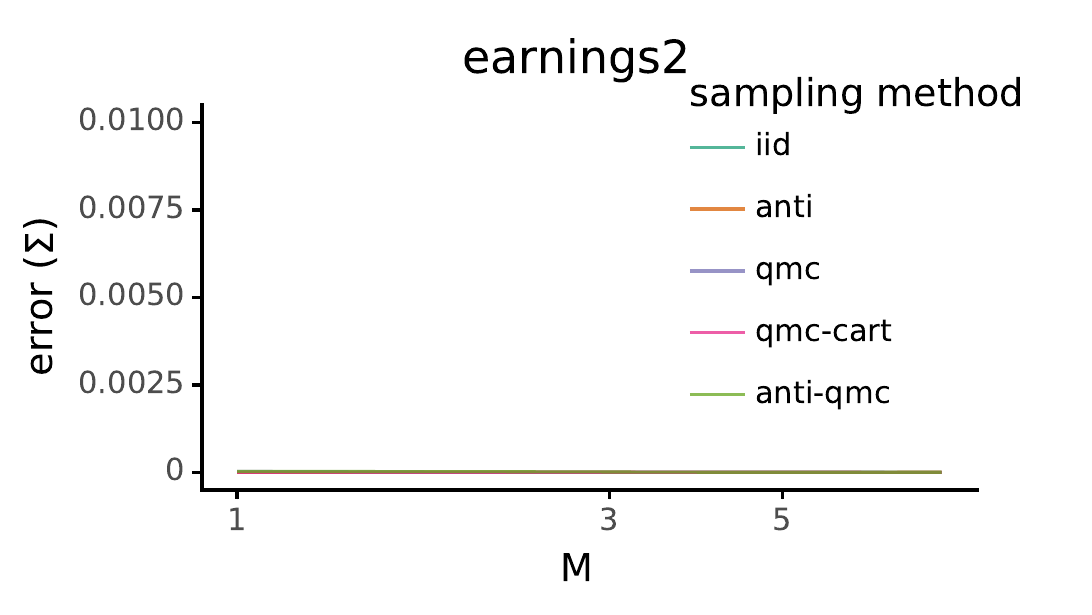}\includegraphics[width=0.33\columnwidth]{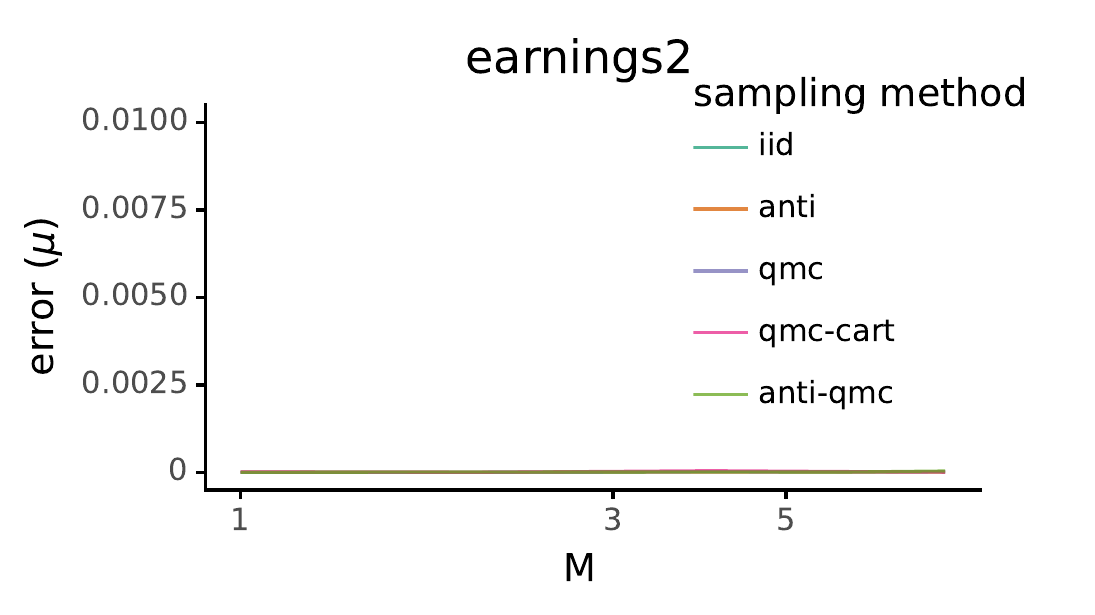}\linebreak{}

\caption{\textbf{Across all models, improvements in likelihood bounds correlate
strongly with improvements in posterior accuracy. Better sampling
methods can improve both.}}
\end{figure}

\begin{figure}
\includegraphics[width=0.33\columnwidth]{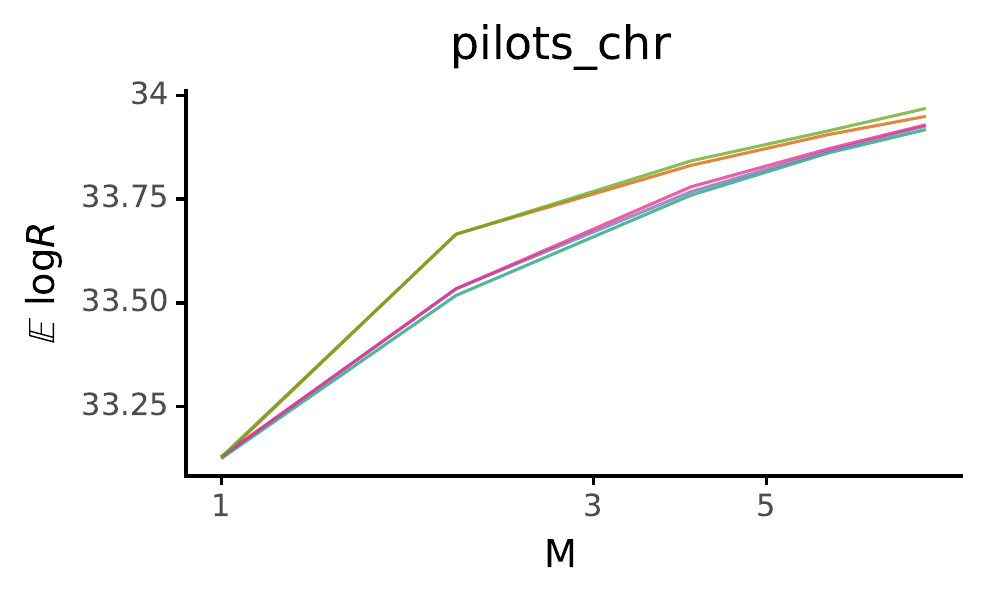}\includegraphics[width=0.33\columnwidth]{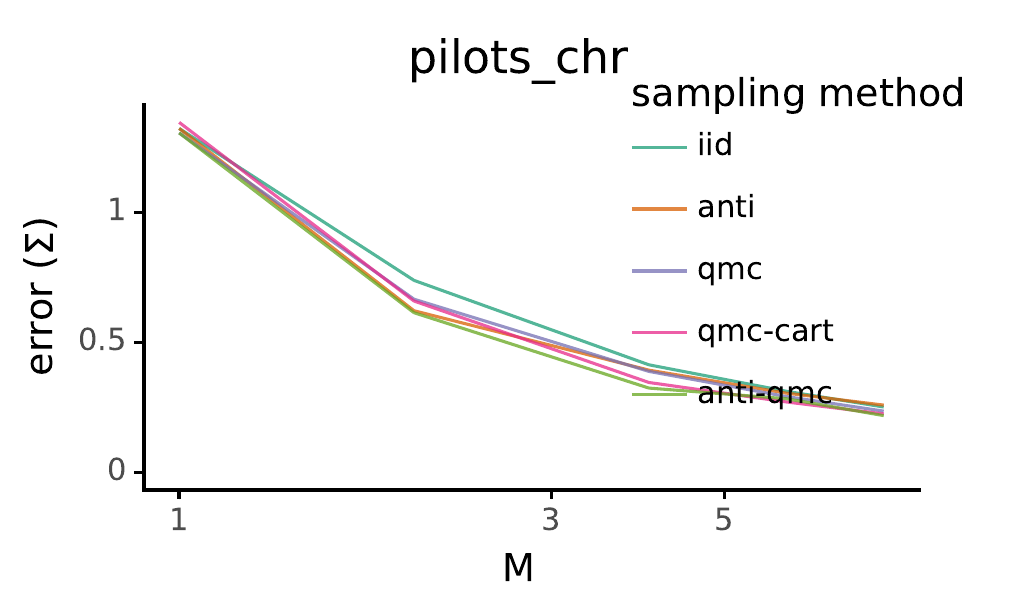}\includegraphics[width=0.33\columnwidth]{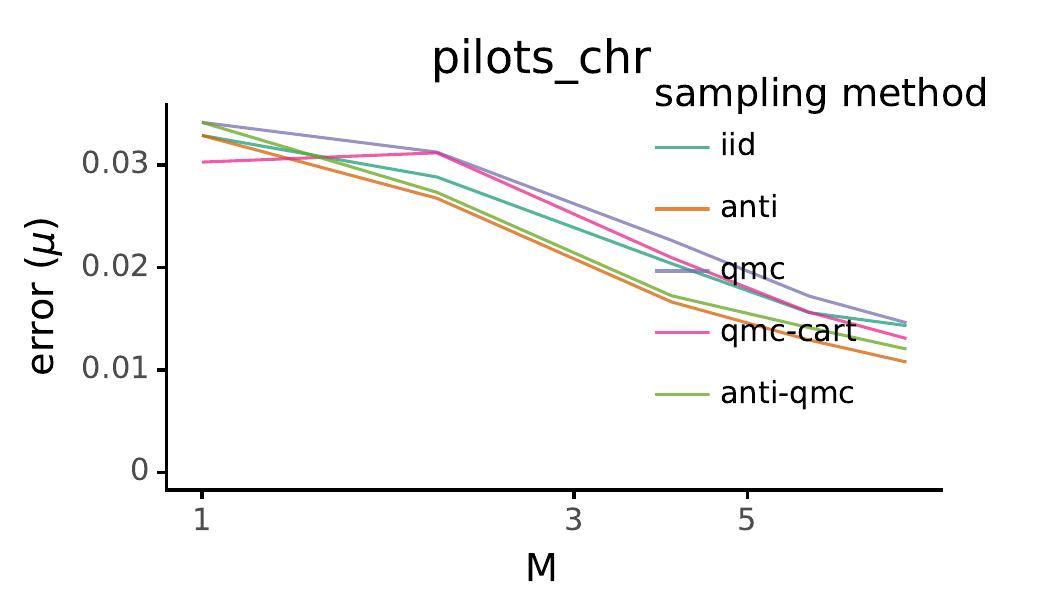}\linebreak{}

\includegraphics[width=0.33\columnwidth]{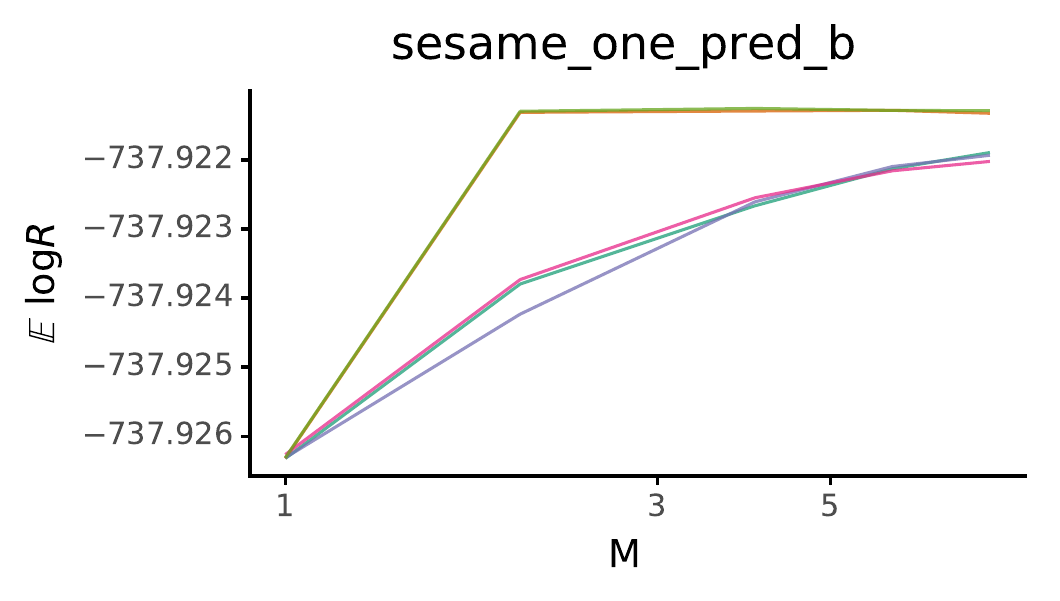}\includegraphics[width=0.33\columnwidth]{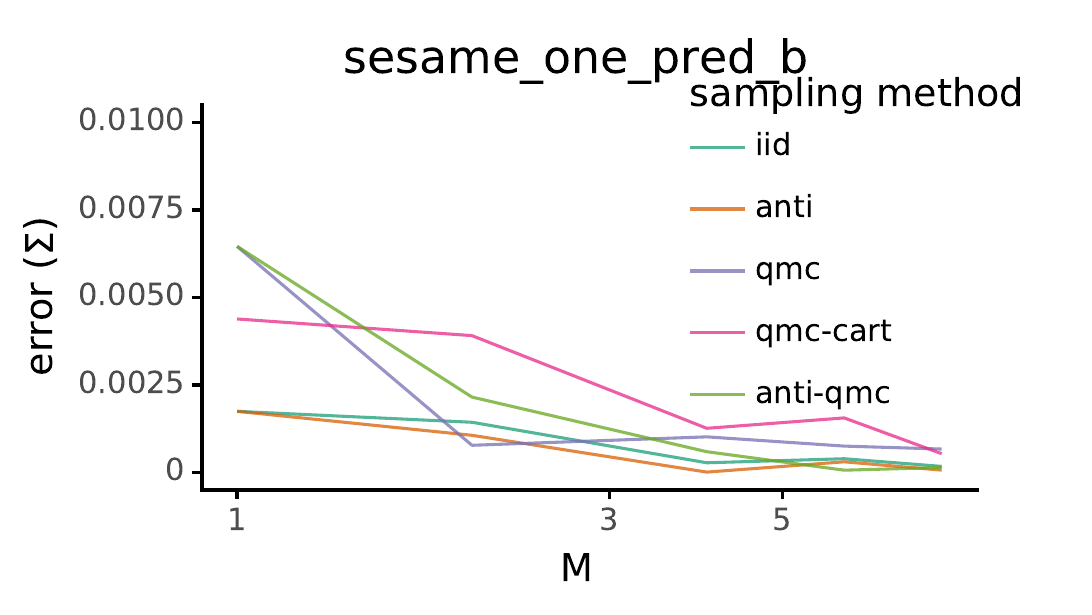}\includegraphics[width=0.33\columnwidth]{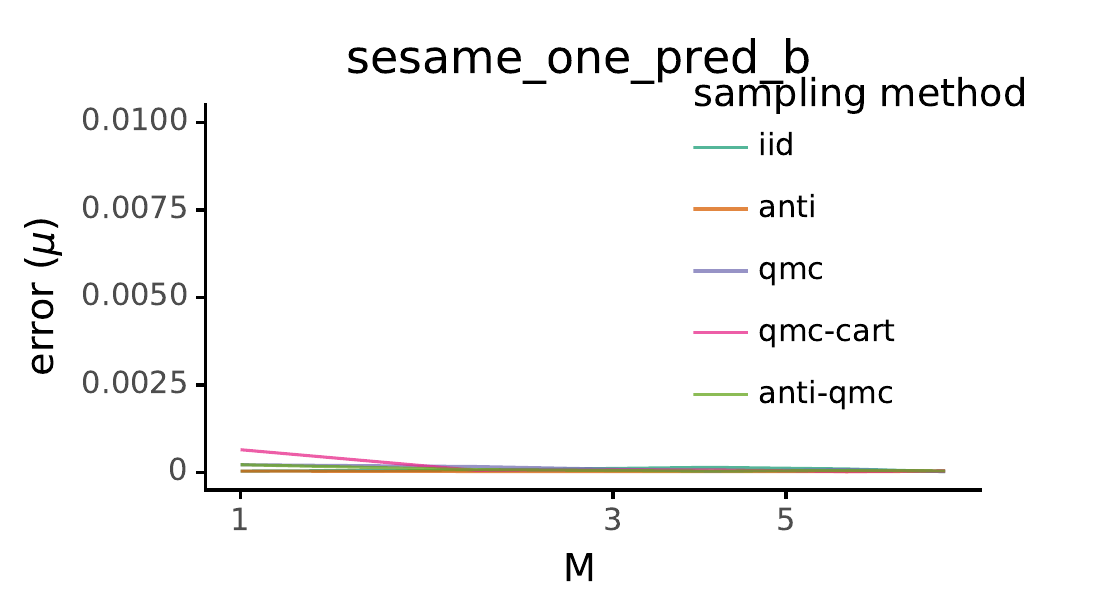}\linebreak{}

\includegraphics[width=0.33\columnwidth]{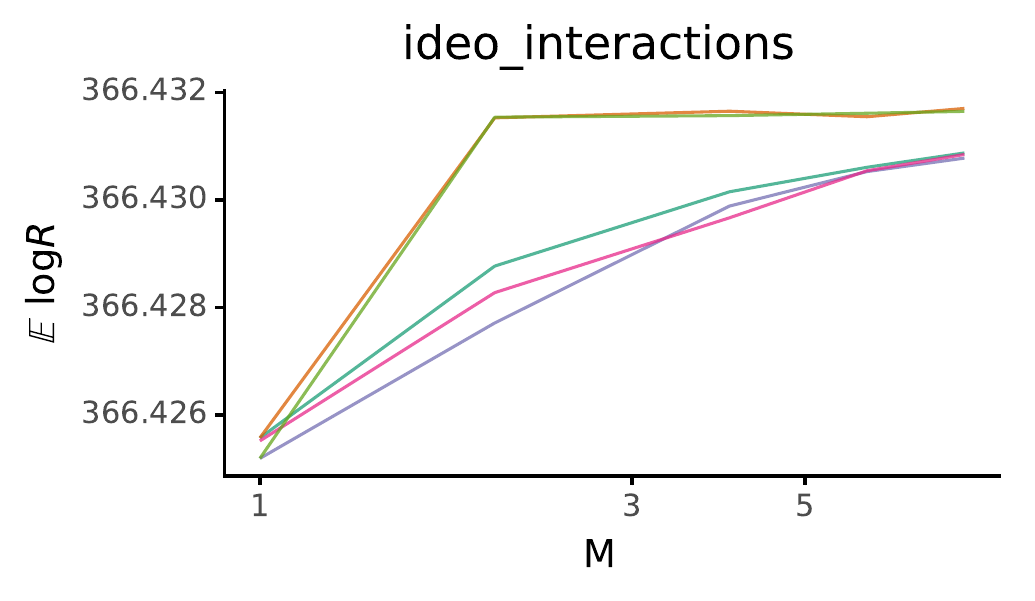}\includegraphics[width=0.33\columnwidth]{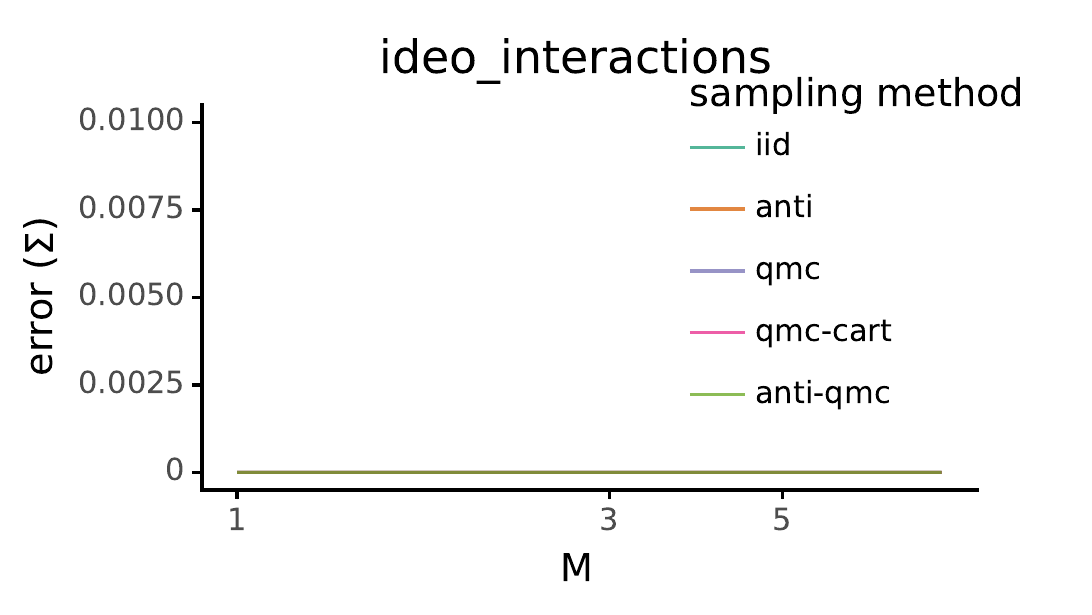}\includegraphics[width=0.33\columnwidth]{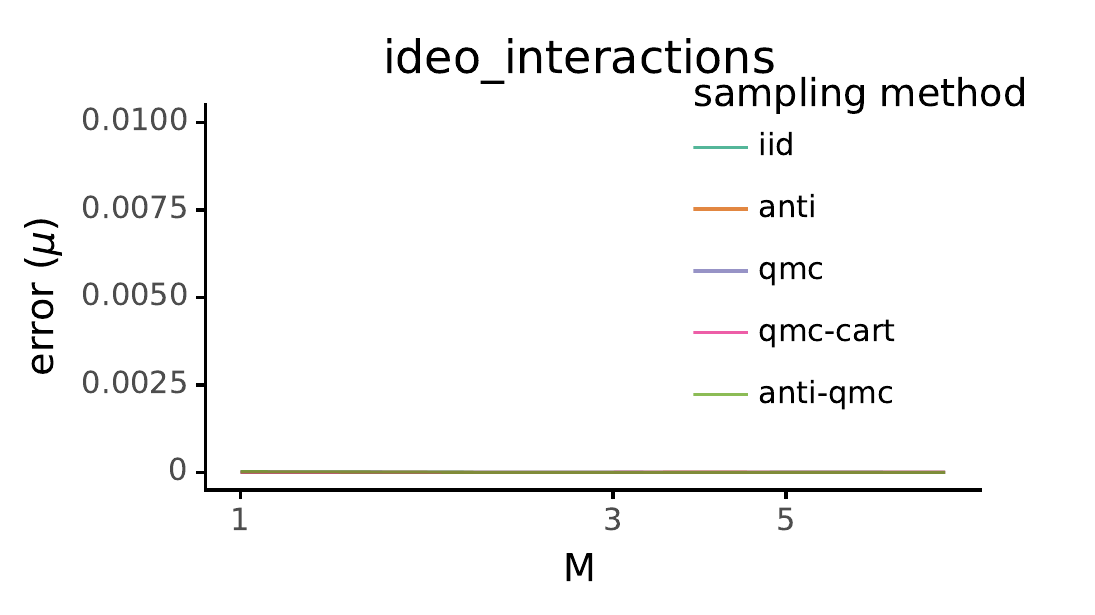}\linebreak{}

\includegraphics[width=0.33\columnwidth]{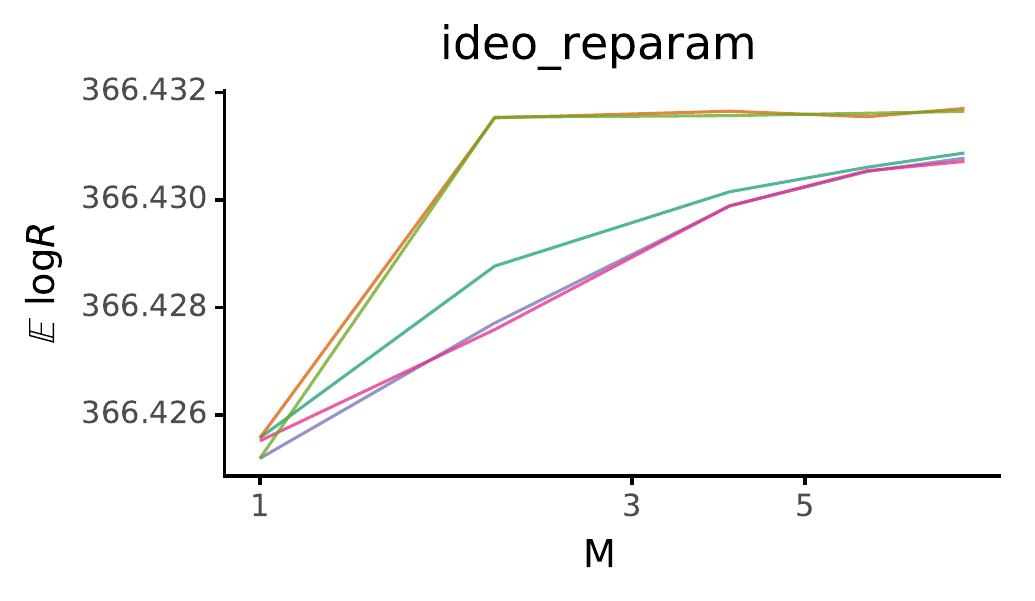}\includegraphics[width=0.33\columnwidth]{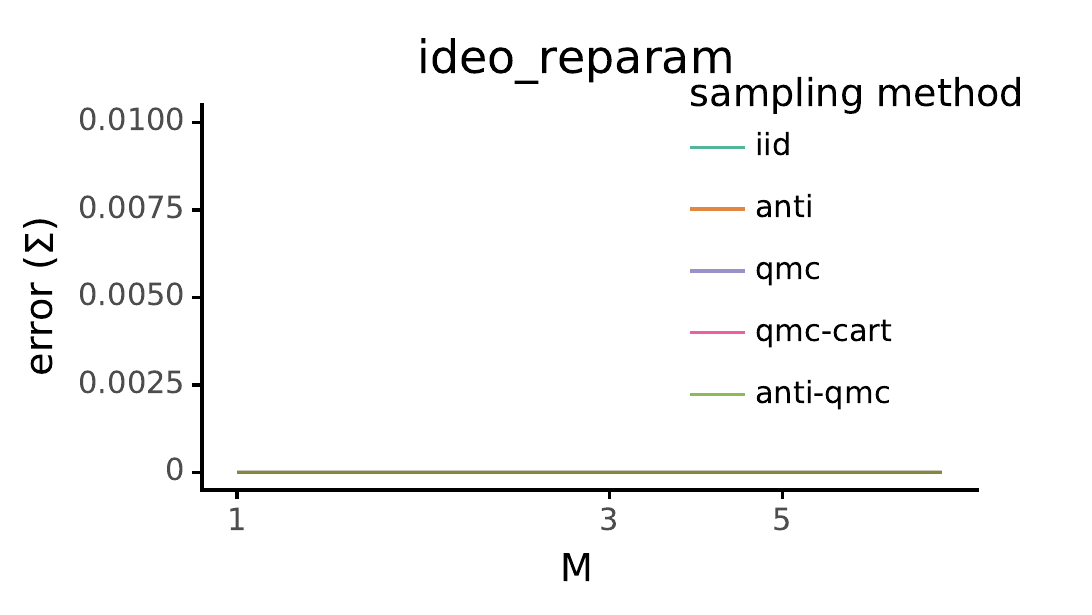}\includegraphics[width=0.33\columnwidth]{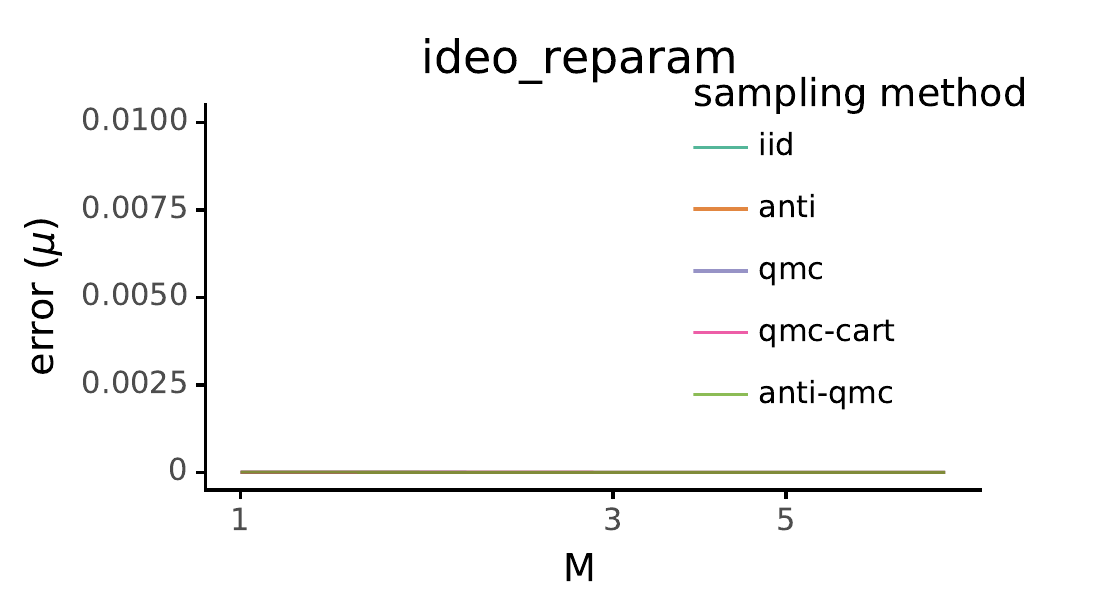}\linebreak{}

\includegraphics[width=0.33\columnwidth]{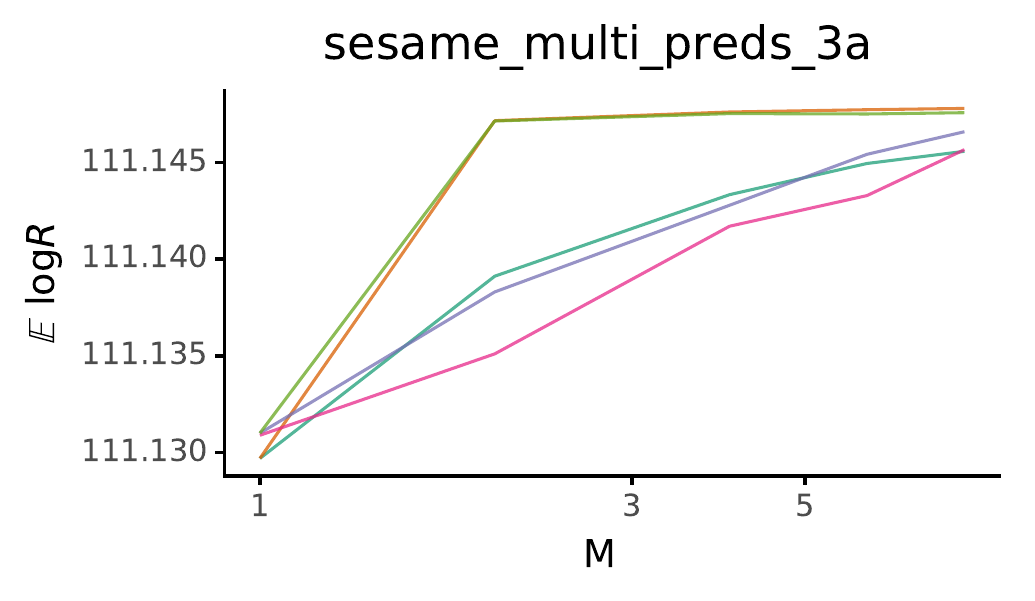}\includegraphics[width=0.33\columnwidth]{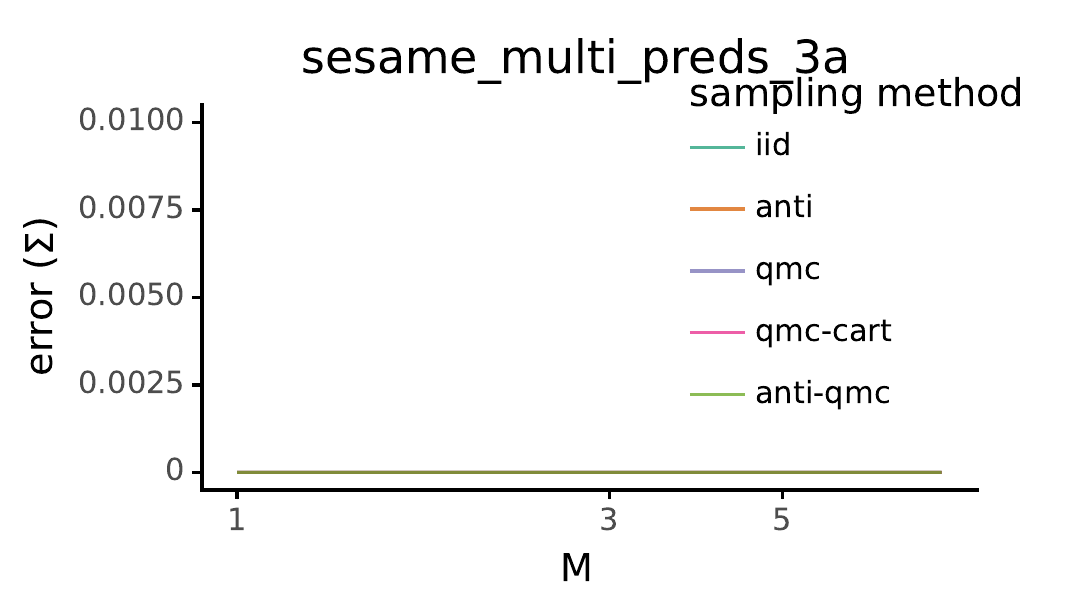}\includegraphics[width=0.33\columnwidth]{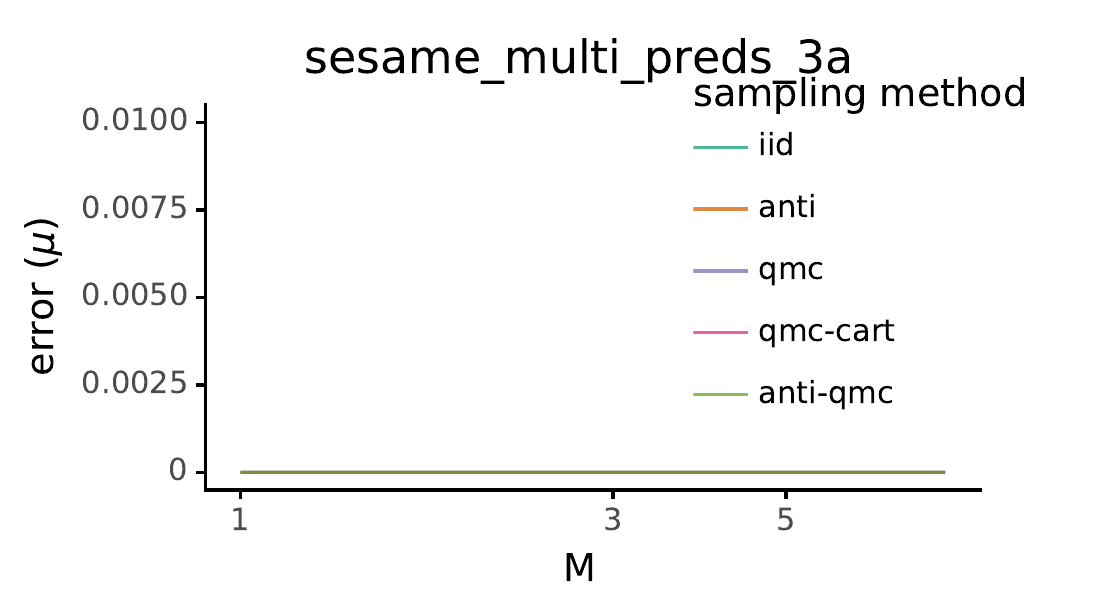}\linebreak{}

\includegraphics[width=0.33\columnwidth]{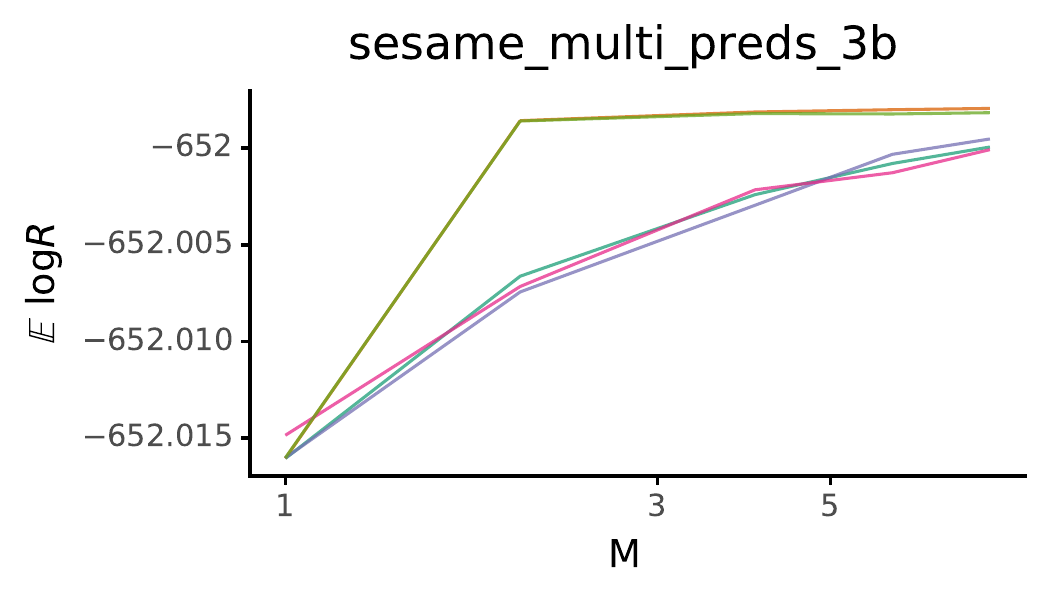}\includegraphics[width=0.33\columnwidth]{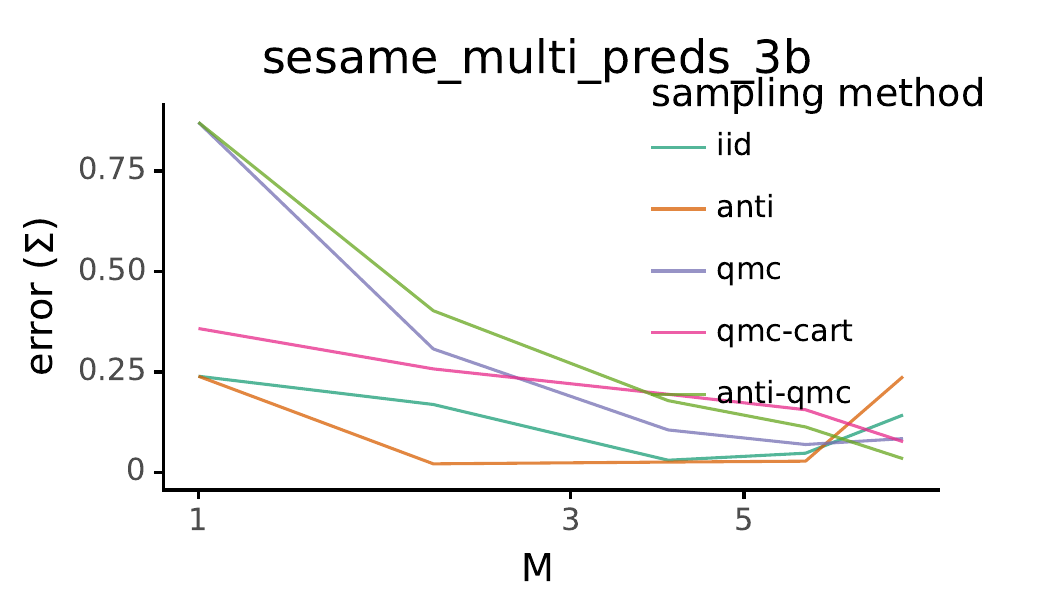}\includegraphics[width=0.33\columnwidth]{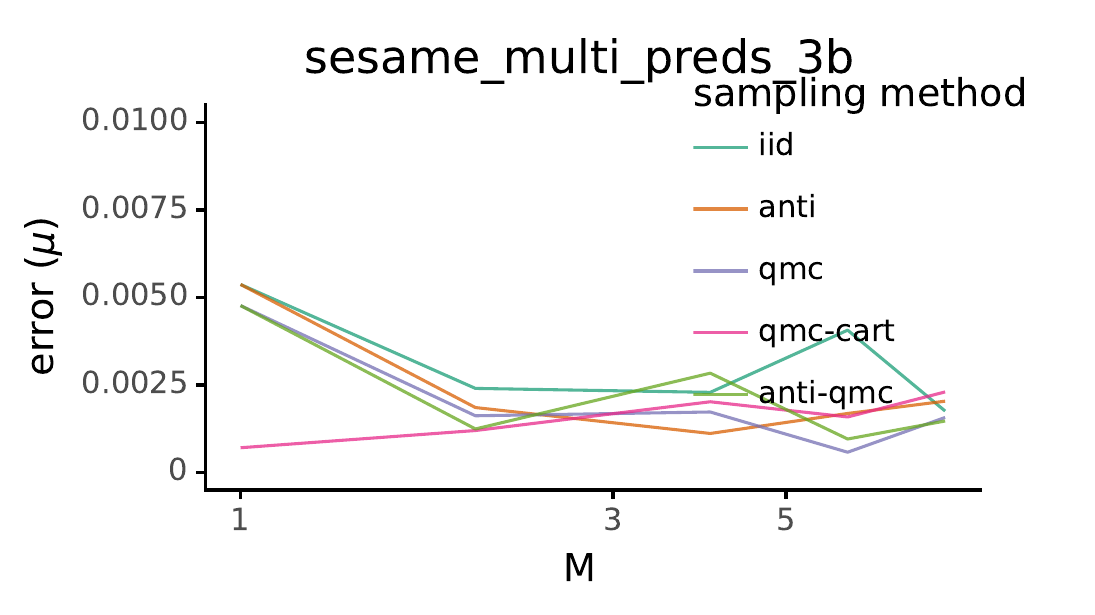}\linebreak{}

\includegraphics[width=0.33\columnwidth]{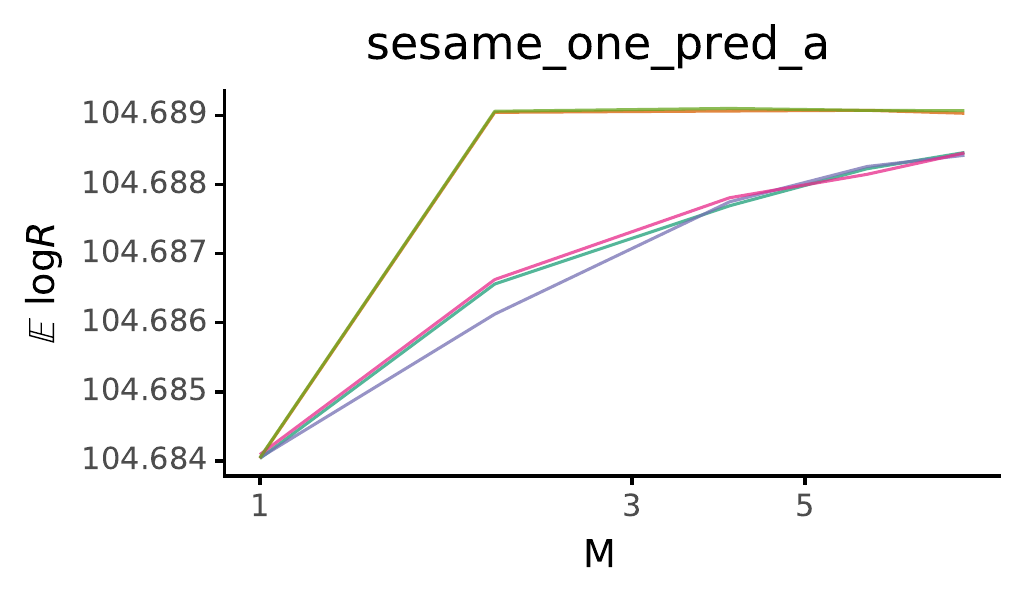}\includegraphics[width=0.33\columnwidth]{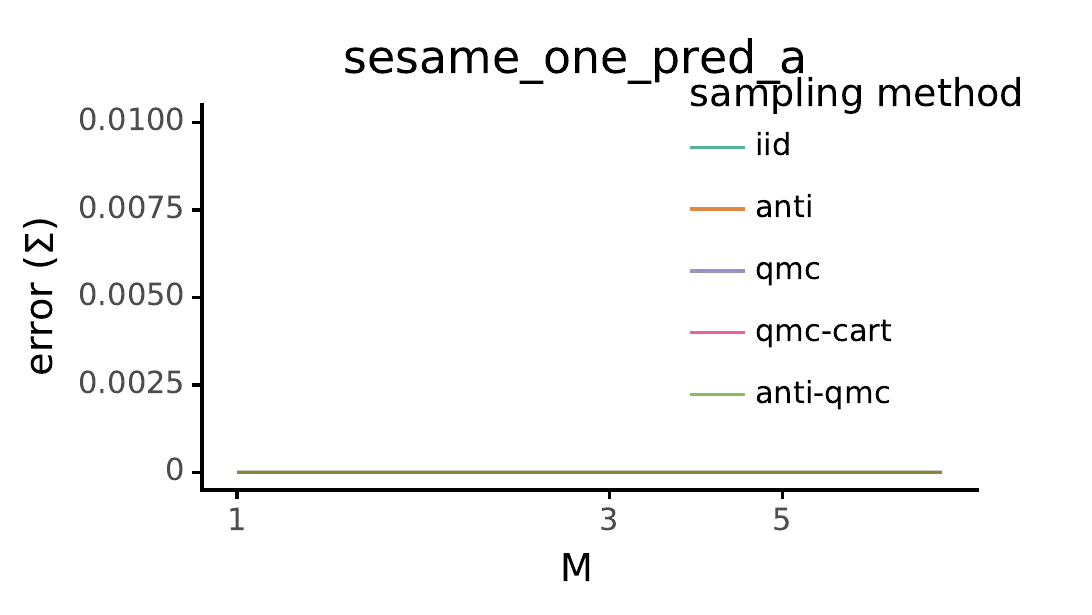}\includegraphics[width=0.33\columnwidth]{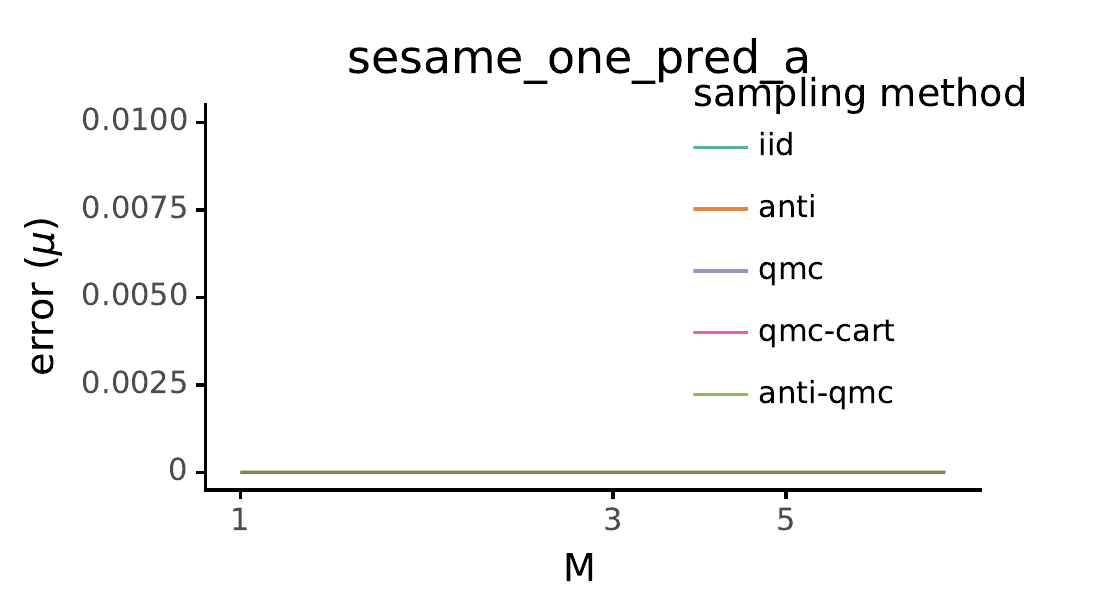}\linebreak{}

\caption{\textbf{Across all models, improvements in likelihood bounds correlate
strongly with improvements in posterior accuracy. Better sampling
methods can improve both.}}
\end{figure}

\begin{figure}
\includegraphics[width=0.33\columnwidth]{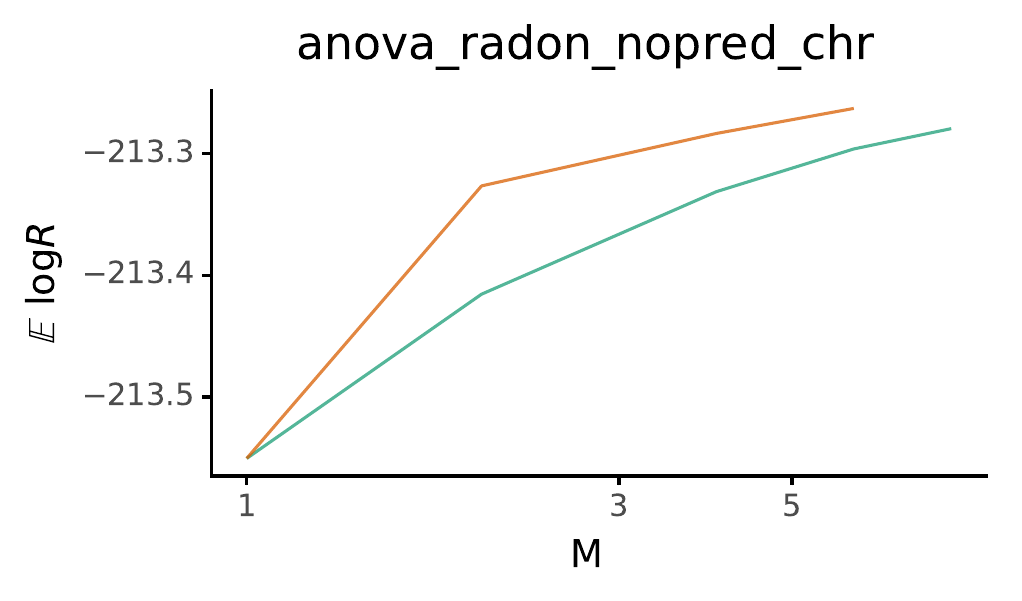}\includegraphics[width=0.33\columnwidth]{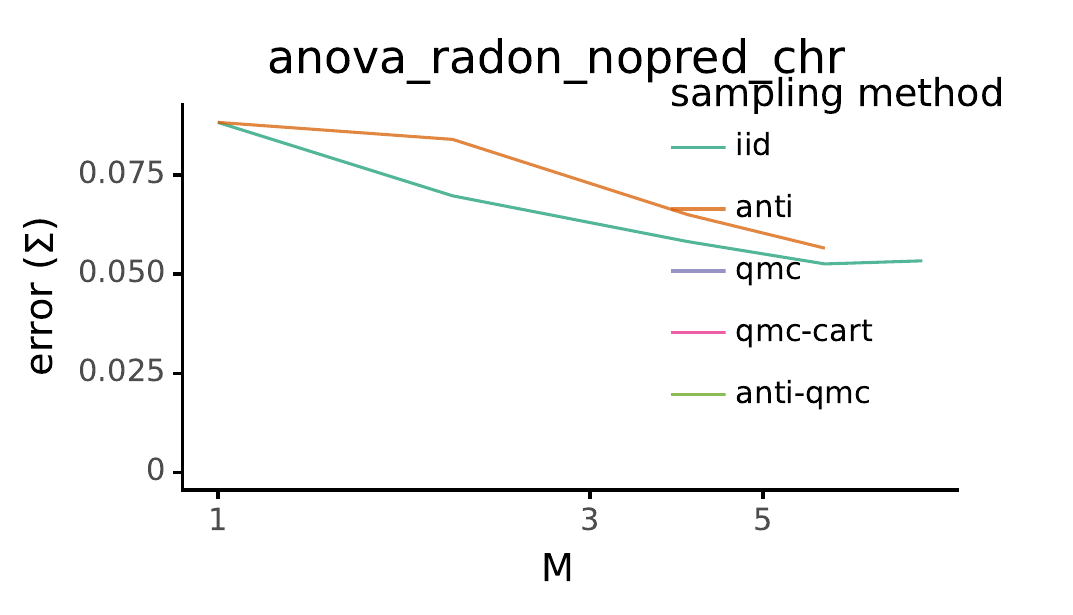}\includegraphics[width=0.33\columnwidth]{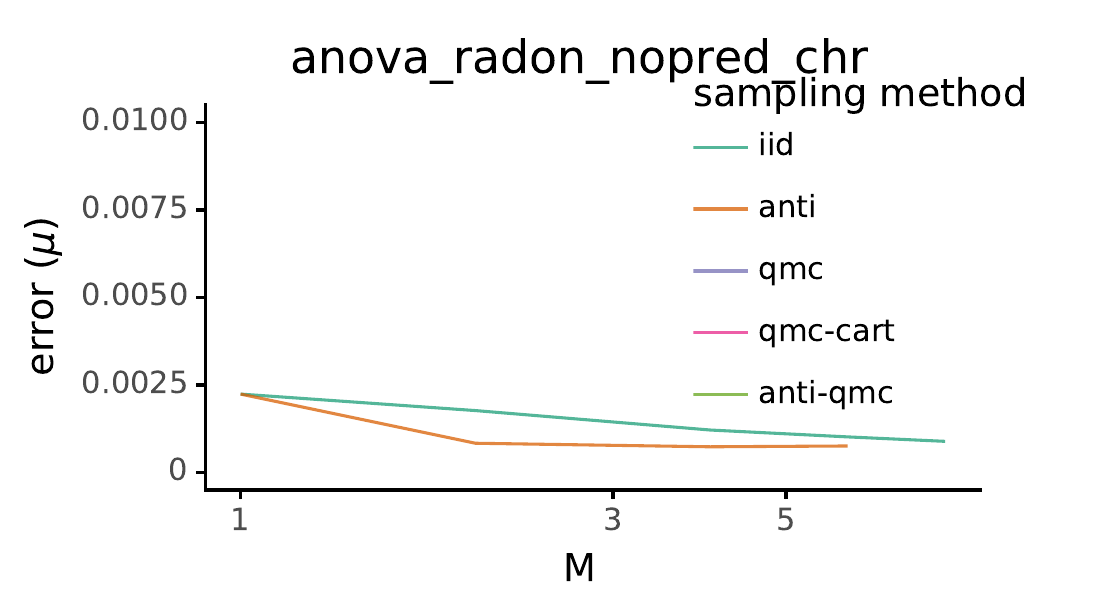}\linebreak{}

\includegraphics[width=0.33\columnwidth]{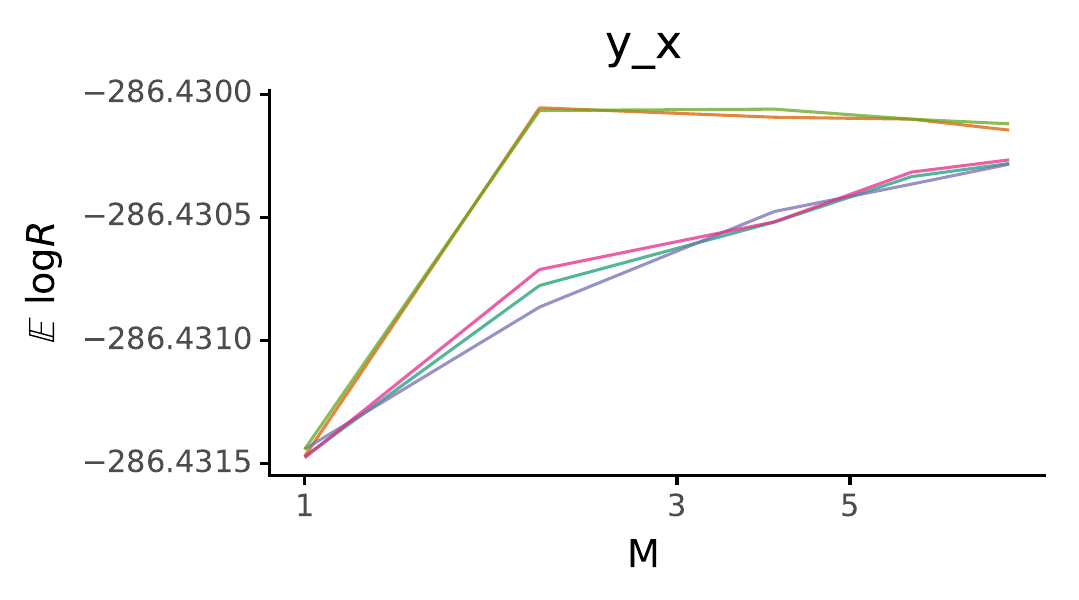}\includegraphics[width=0.33\columnwidth]{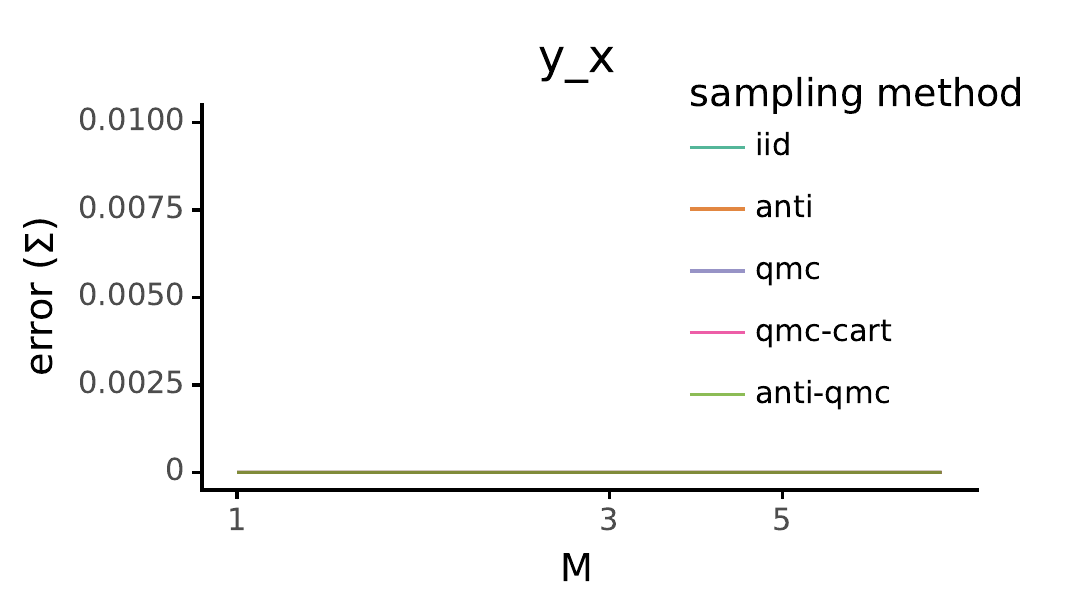}\includegraphics[width=0.33\columnwidth]{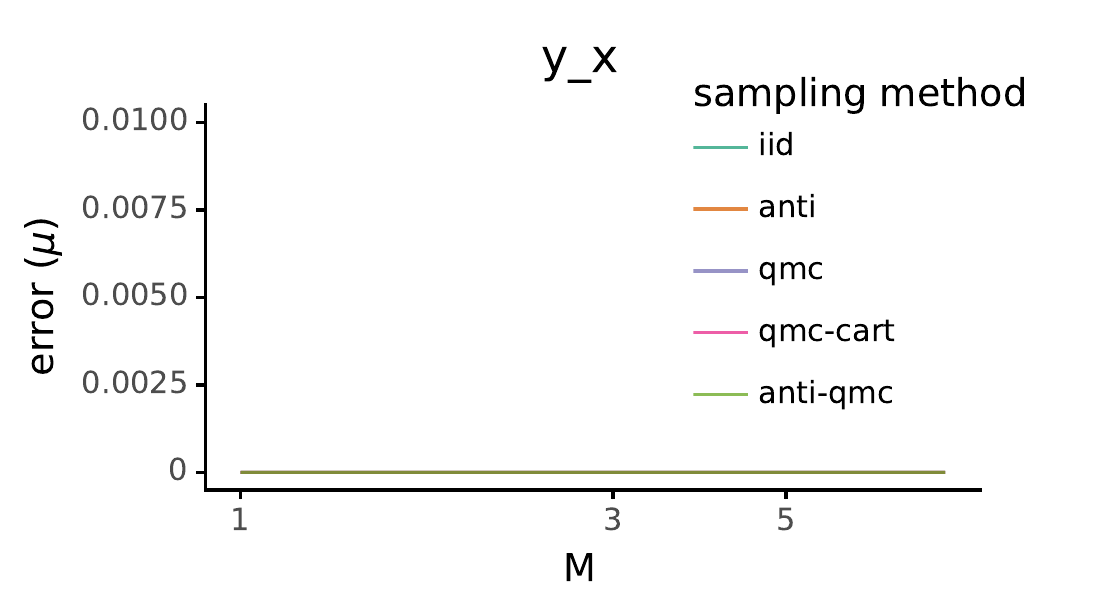}\linebreak{}

\includegraphics[width=0.33\columnwidth]{final_swarm_figs/stan64_1_elbos}\includegraphics[width=0.33\columnwidth]{final_swarm_figs/stan64_1_err_Sigma}\includegraphics[width=0.33\columnwidth]{final_swarm_figs/stan64_1_err_mu}\linebreak{}

\includegraphics[width=0.33\columnwidth]{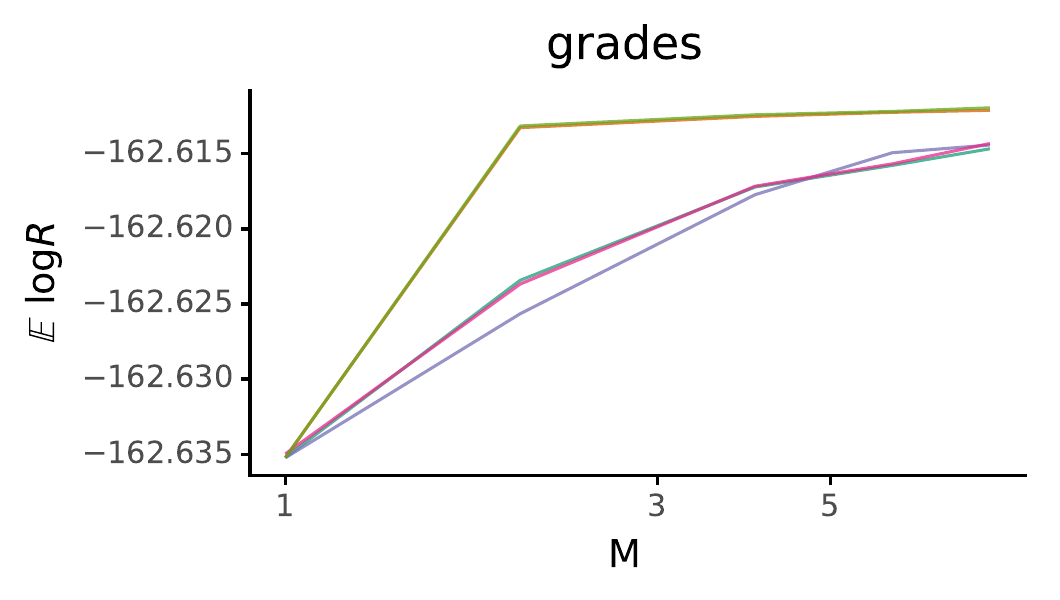}\includegraphics[width=0.33\columnwidth]{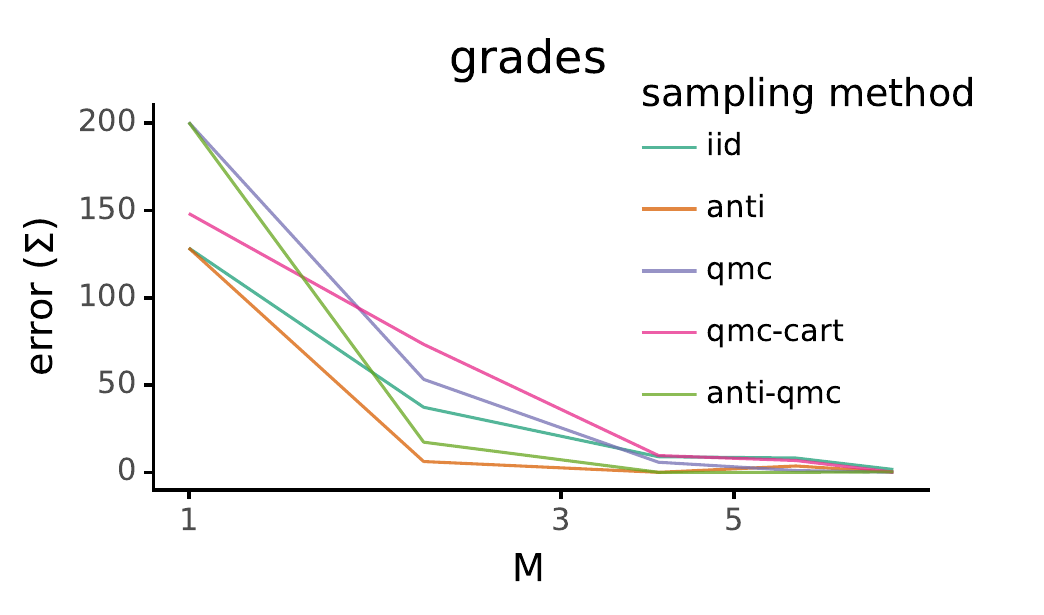}\includegraphics[width=0.33\columnwidth]{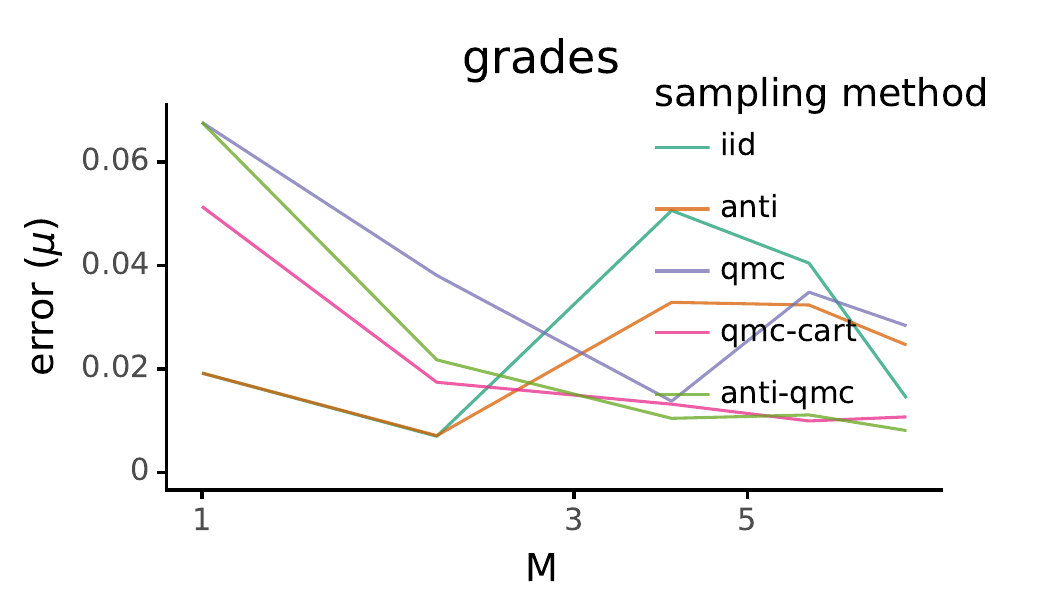}\linebreak{}

\includegraphics[width=0.33\columnwidth]{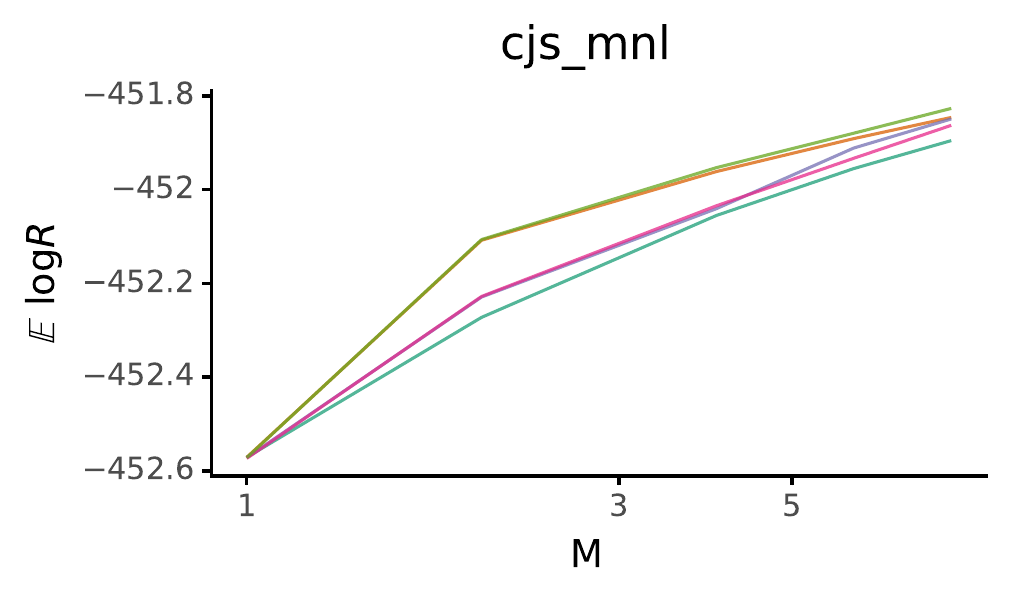}\includegraphics[width=0.33\columnwidth]{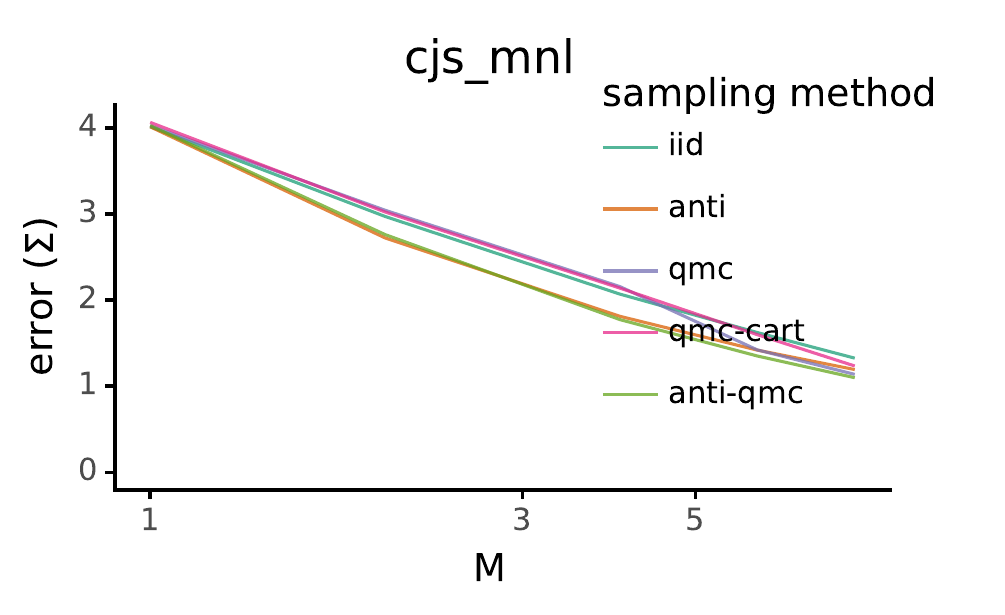}\includegraphics[width=0.33\columnwidth]{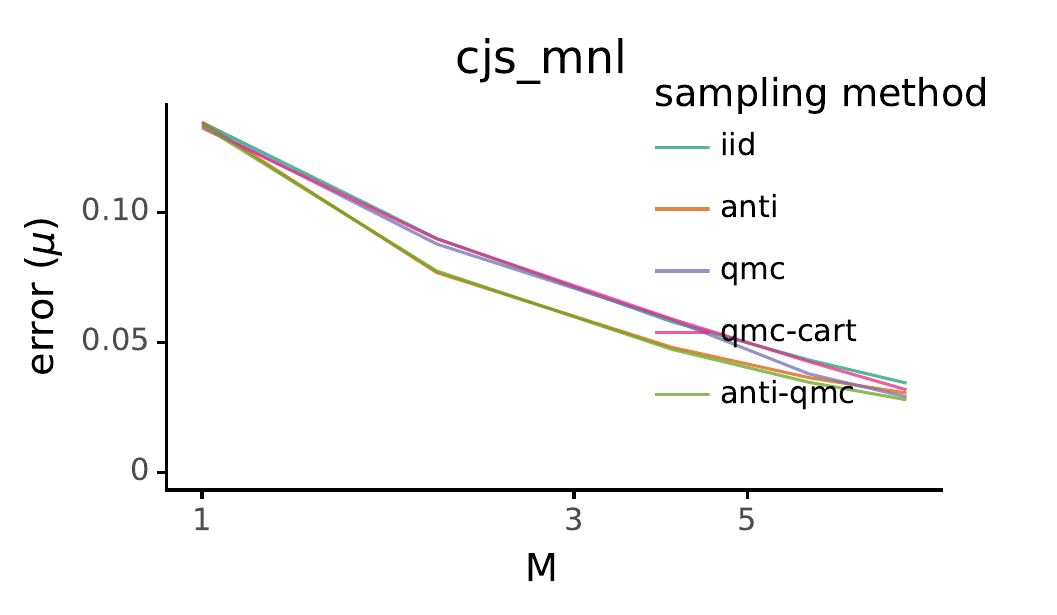}\linebreak{}

\includegraphics[width=0.33\columnwidth]{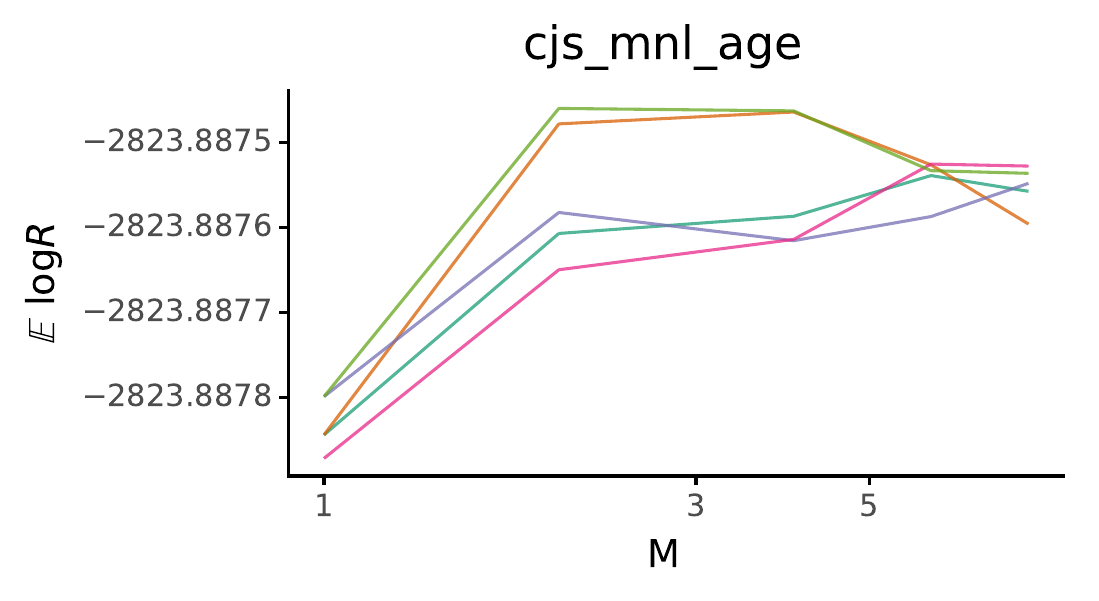}\includegraphics[width=0.33\columnwidth]{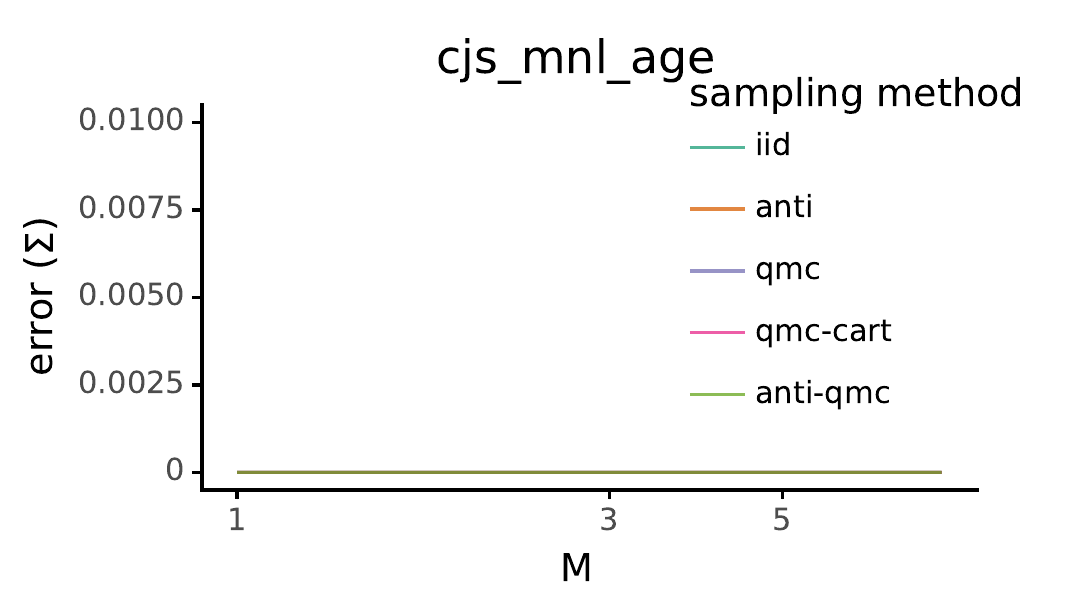}\includegraphics[width=0.33\columnwidth]{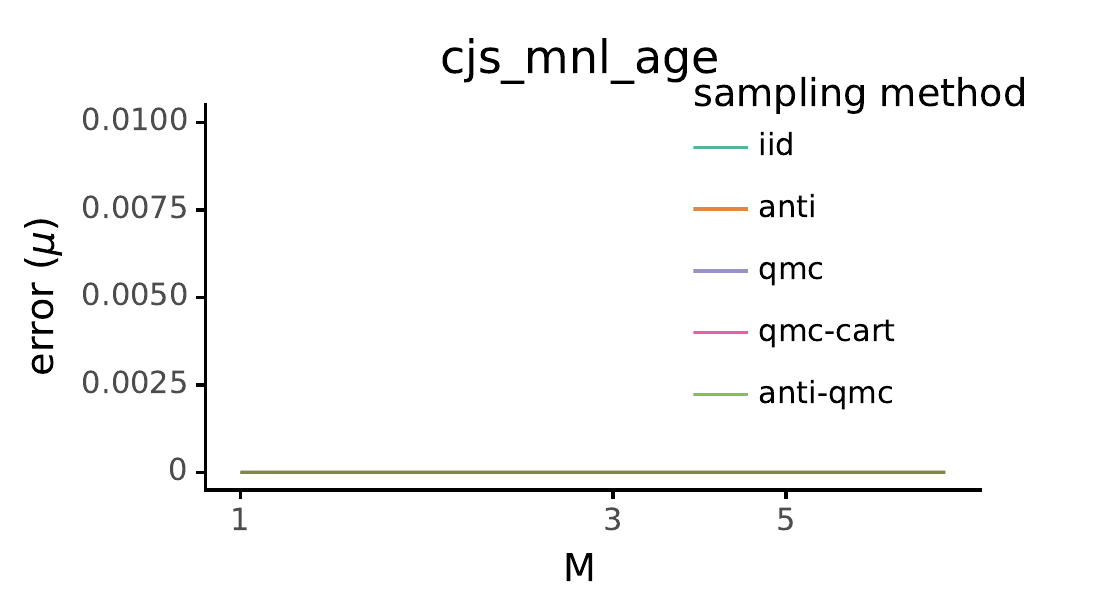}\linebreak{}

\includegraphics[width=0.33\columnwidth]{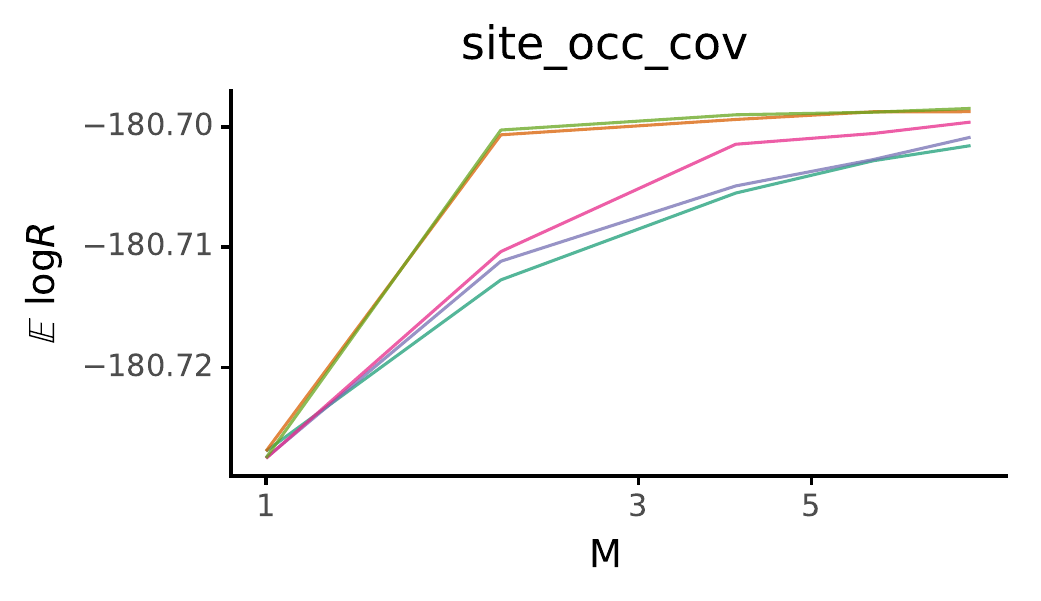}\includegraphics[width=0.33\columnwidth]{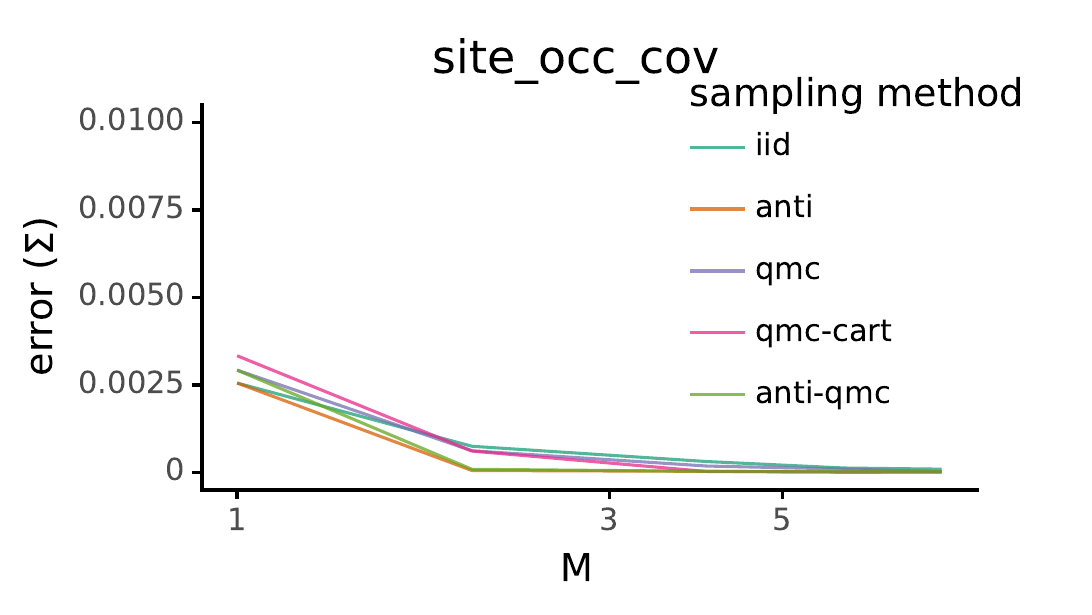}\includegraphics[width=0.33\columnwidth]{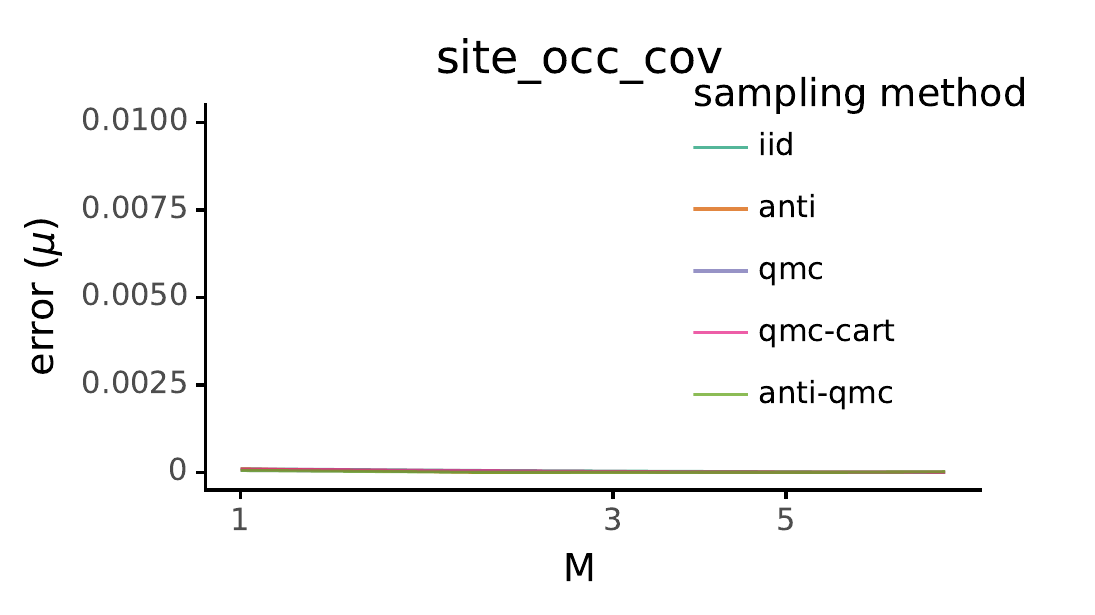}\linebreak{}

\caption{\textbf{Across all models, improvements in likelihood bounds correlate
strongly with improvements in posterior accuracy. Better sampling
methods can improve both.}}
\end{figure}

\begin{figure}
\includegraphics[width=0.33\columnwidth]{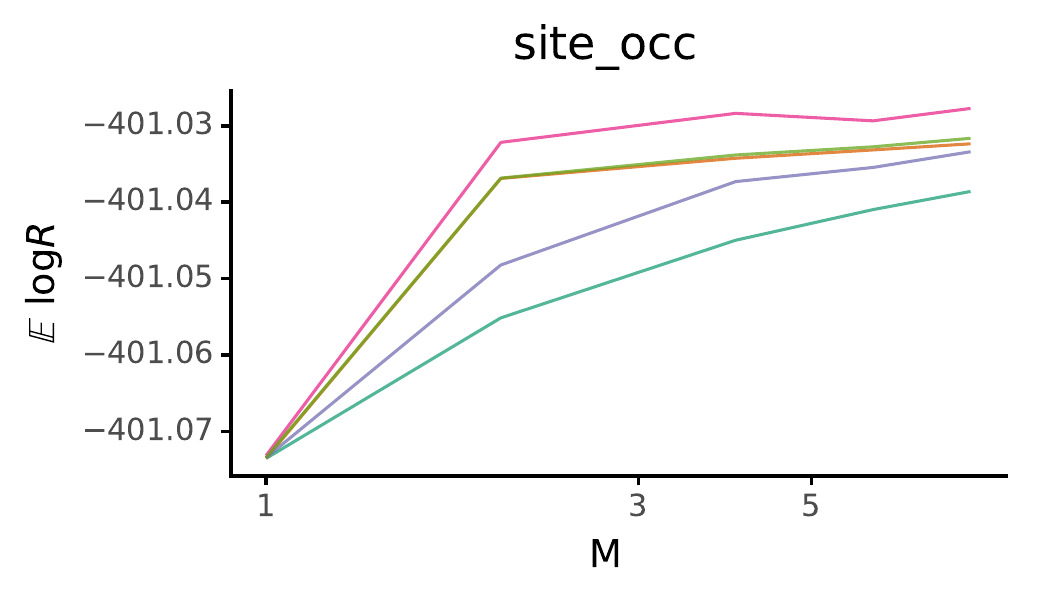}\includegraphics[width=0.33\columnwidth]{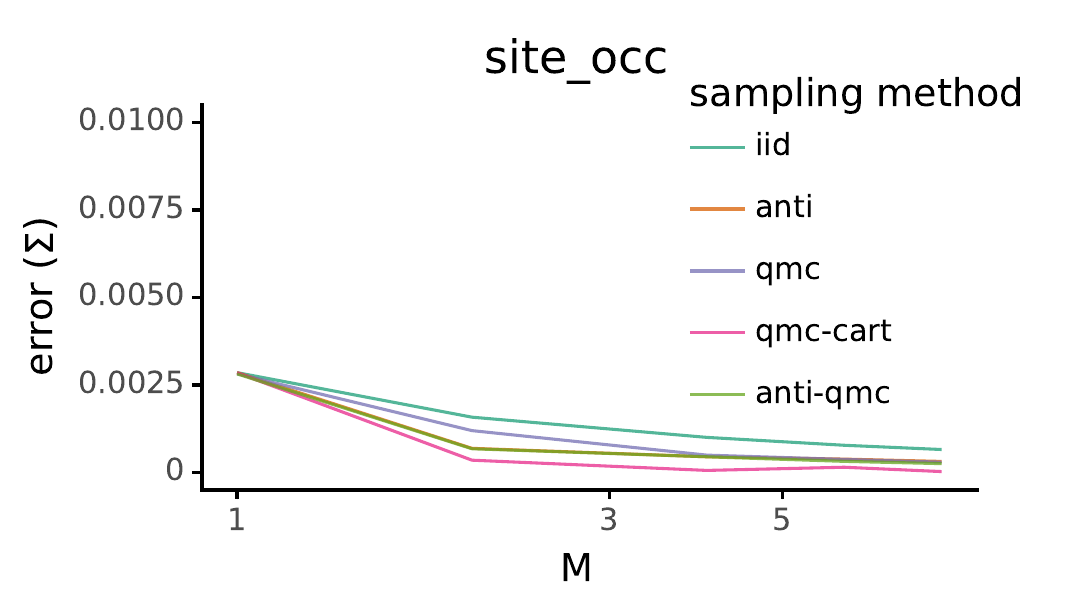}\includegraphics[width=0.33\columnwidth]{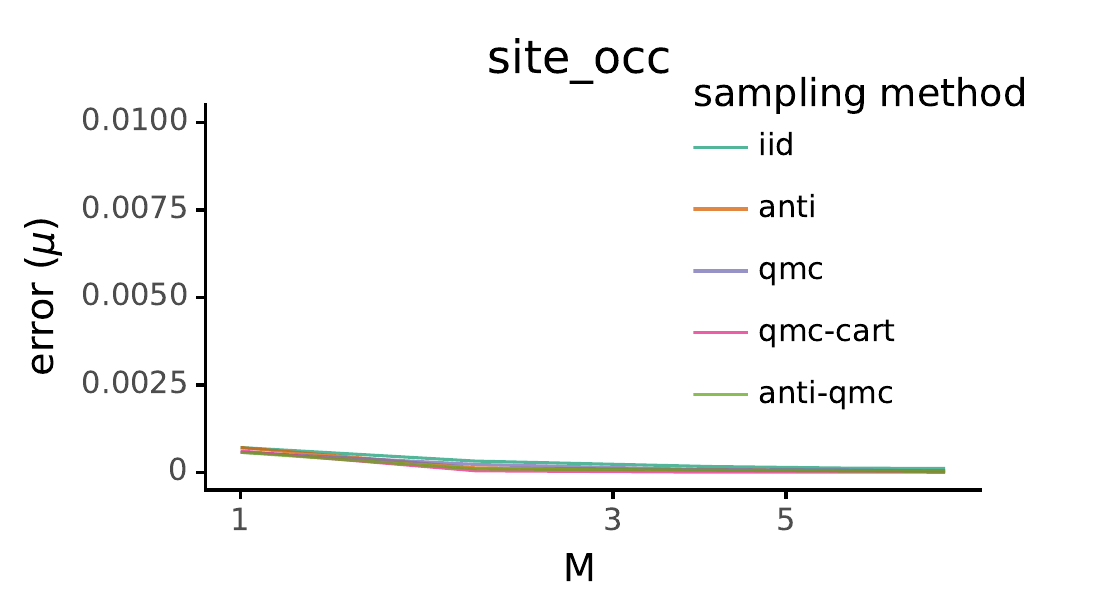}\linebreak{}

\includegraphics[width=0.33\columnwidth]{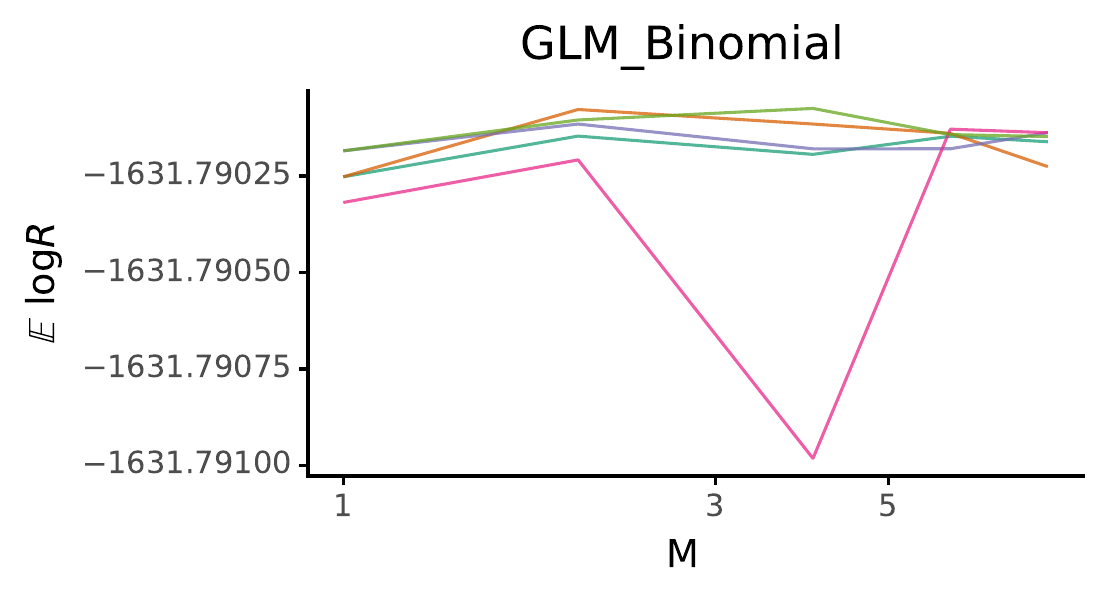}\includegraphics[width=0.33\columnwidth]{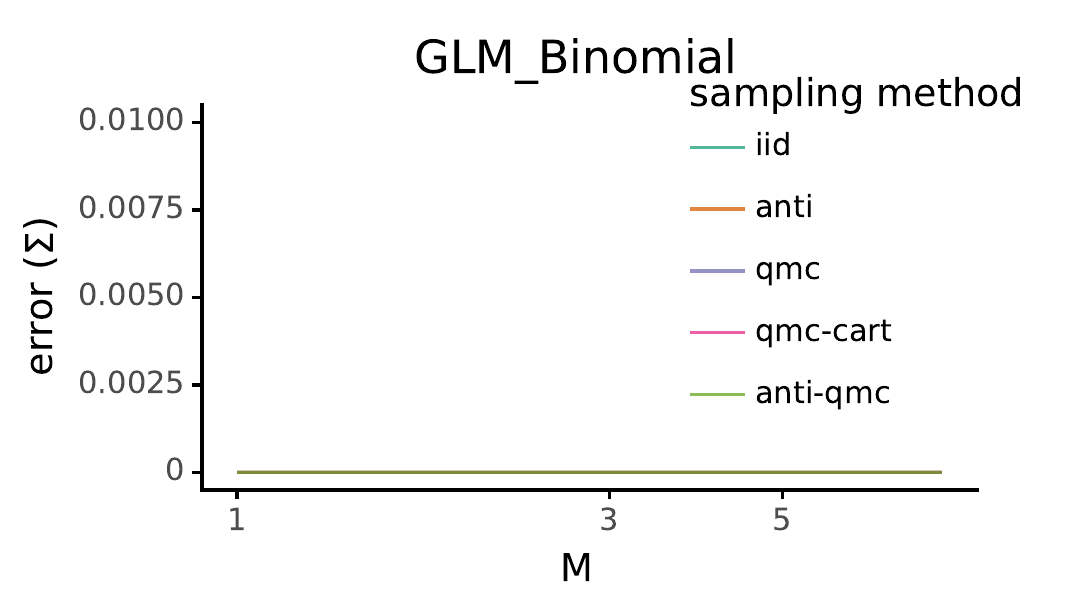}\includegraphics[width=0.33\columnwidth]{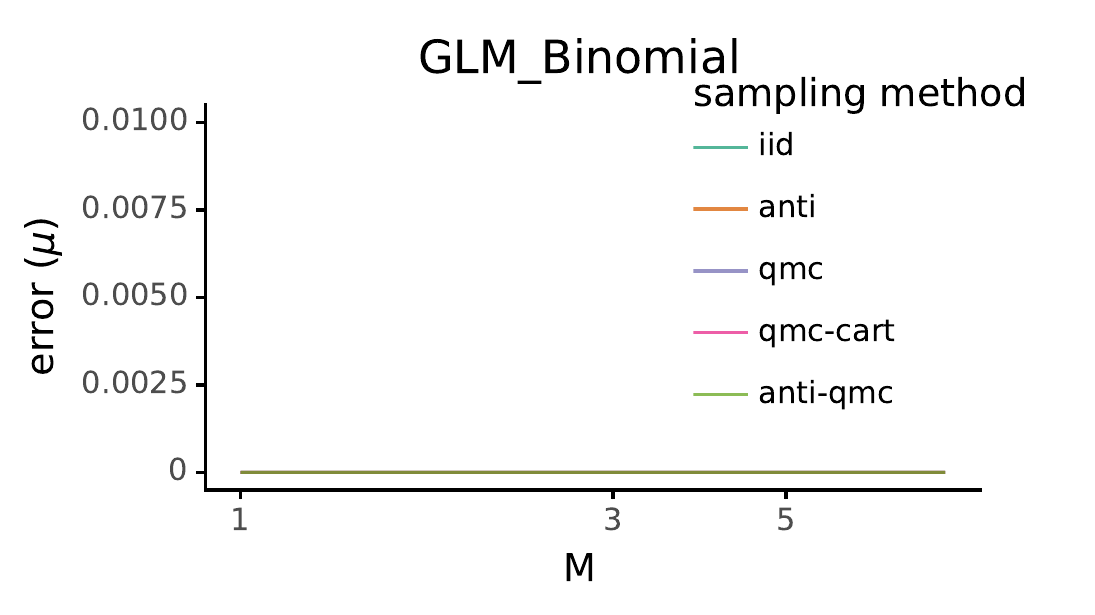}\linebreak{}

\includegraphics[width=0.33\columnwidth]{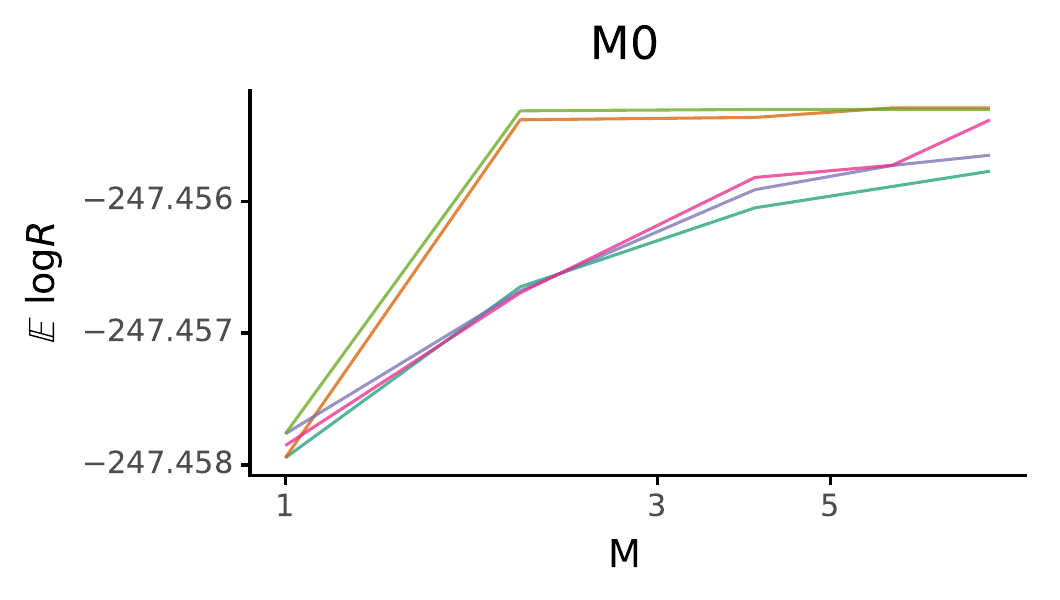}\includegraphics[width=0.33\columnwidth]{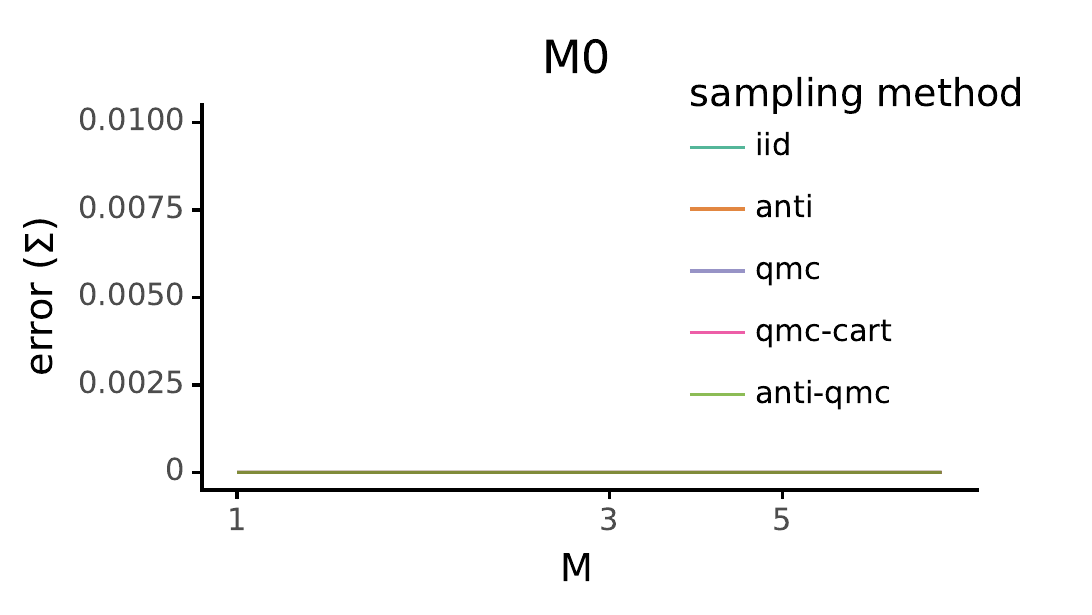}\includegraphics[width=0.33\columnwidth]{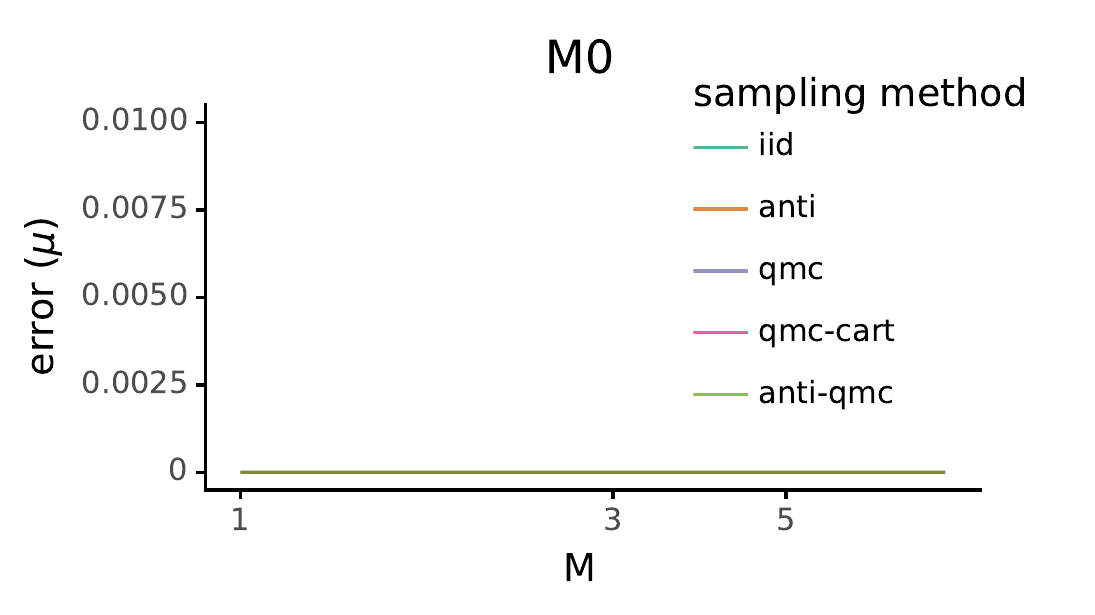}\linebreak{}

\includegraphics[width=0.33\columnwidth]{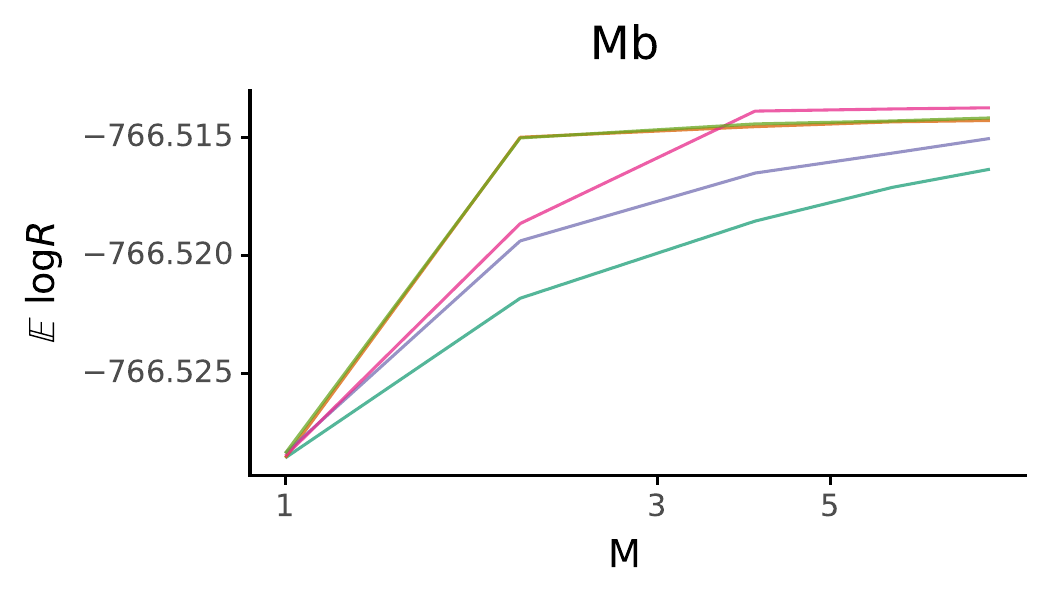}\includegraphics[width=0.33\columnwidth]{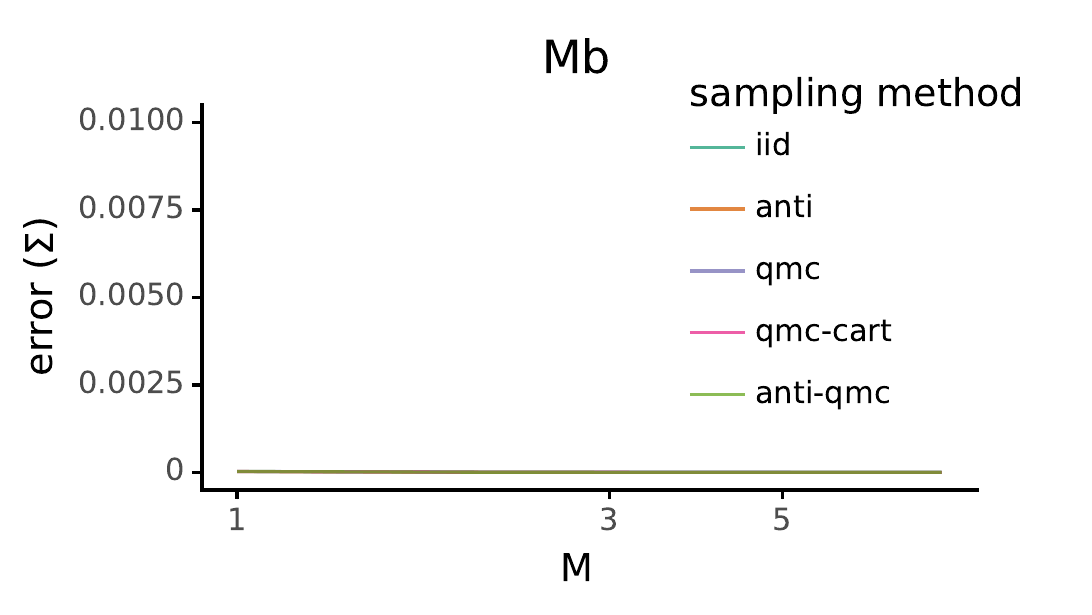}\includegraphics[width=0.33\columnwidth]{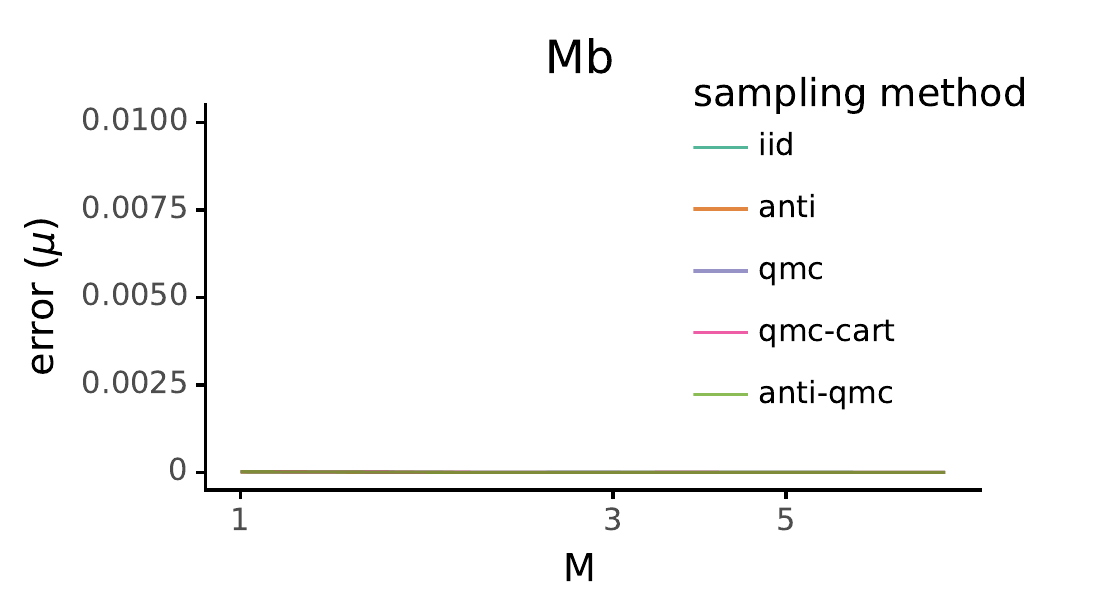}\linebreak{}

\includegraphics[width=0.33\columnwidth]{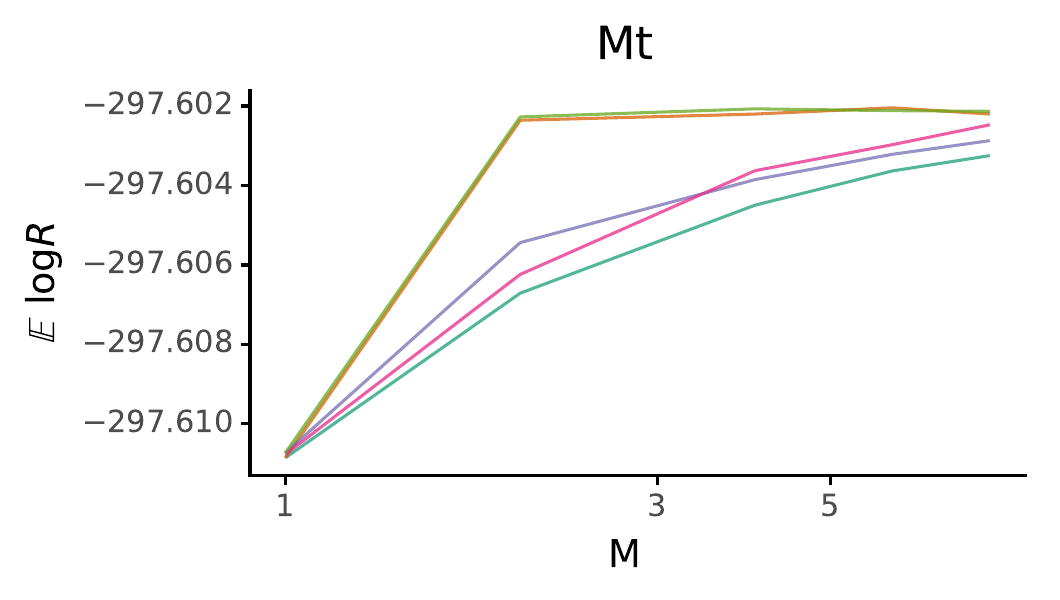}\includegraphics[width=0.33\columnwidth]{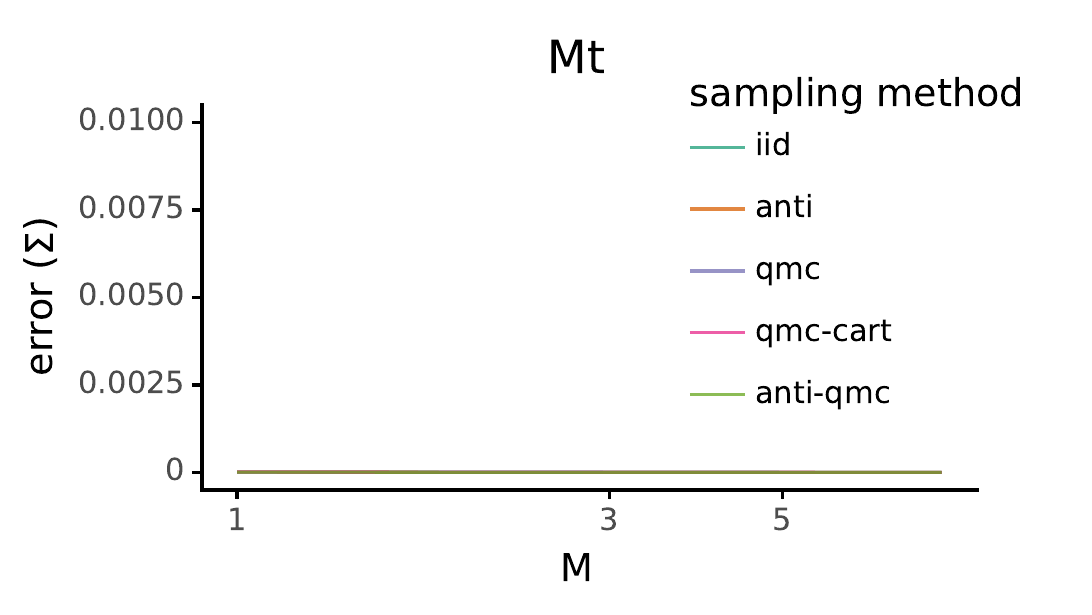}\includegraphics[width=0.33\columnwidth]{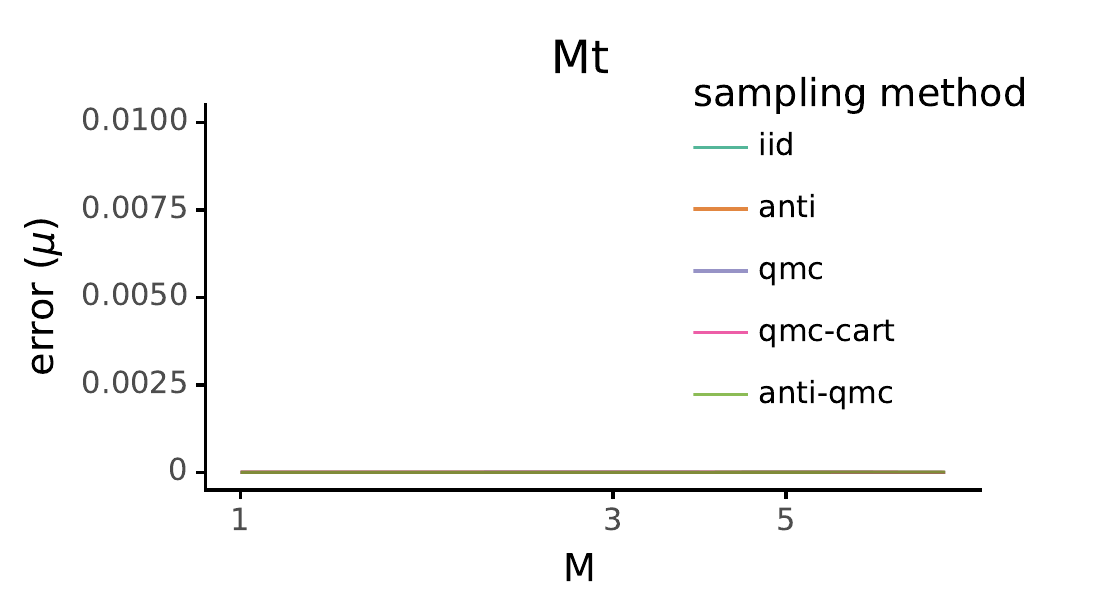}\linebreak{}

\includegraphics[width=0.33\columnwidth]{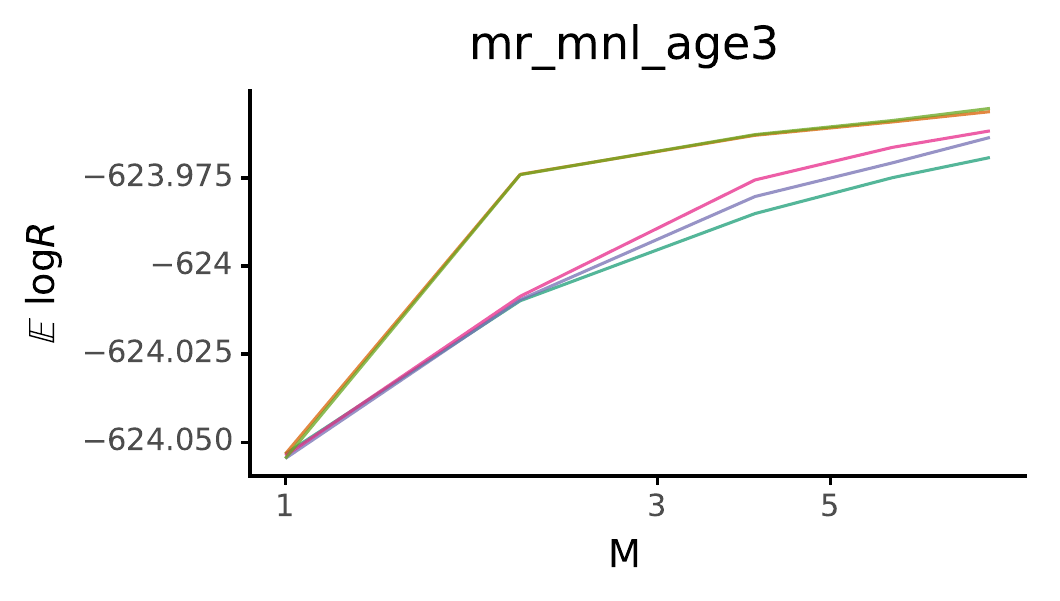}\includegraphics[width=0.33\columnwidth]{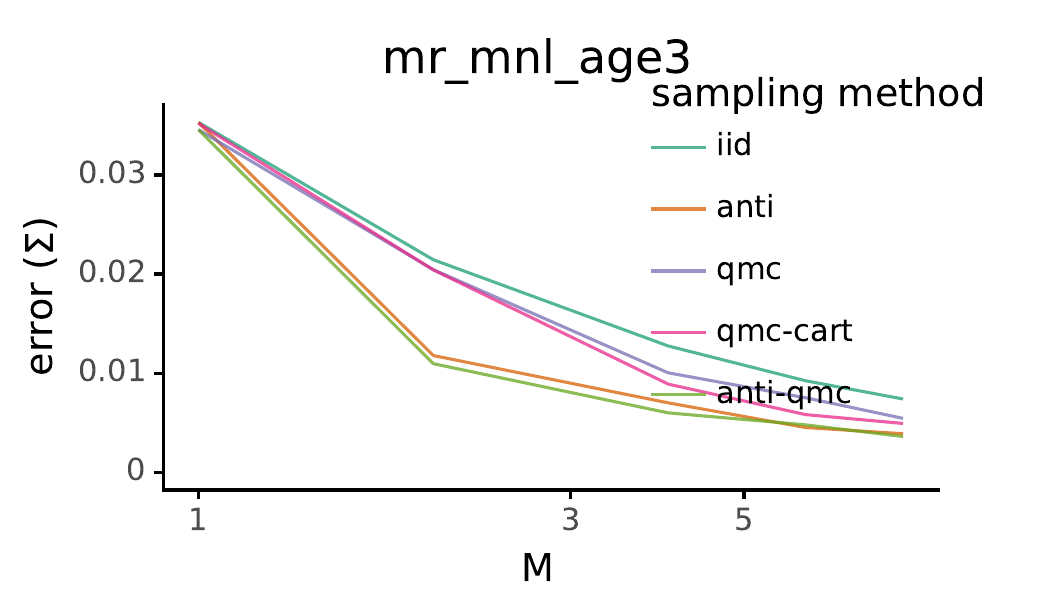}\includegraphics[width=0.33\columnwidth]{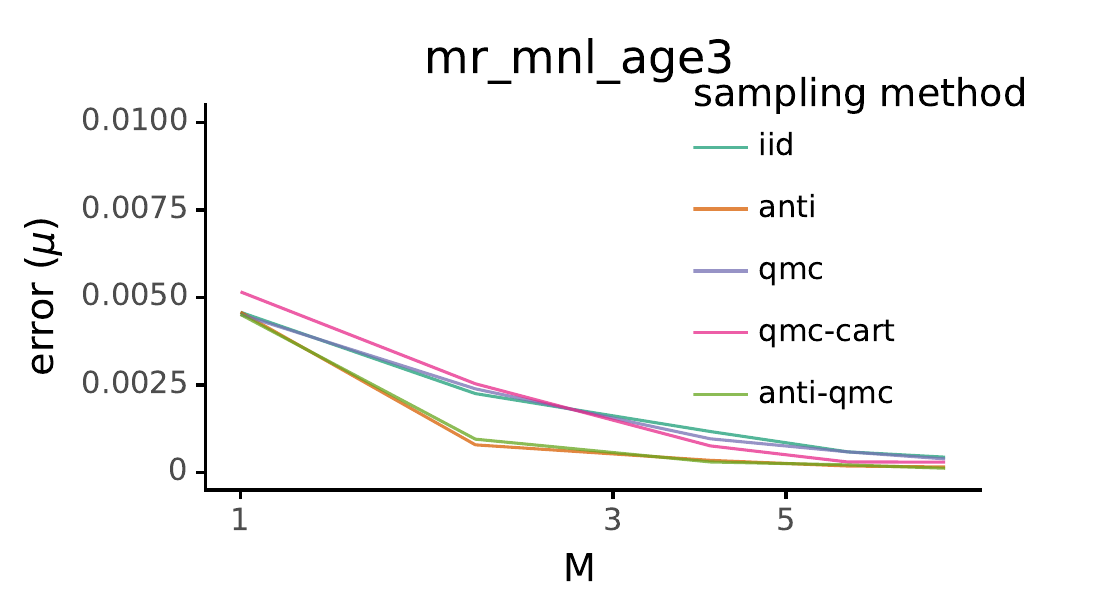}\linebreak{}

\includegraphics[width=0.33\columnwidth]{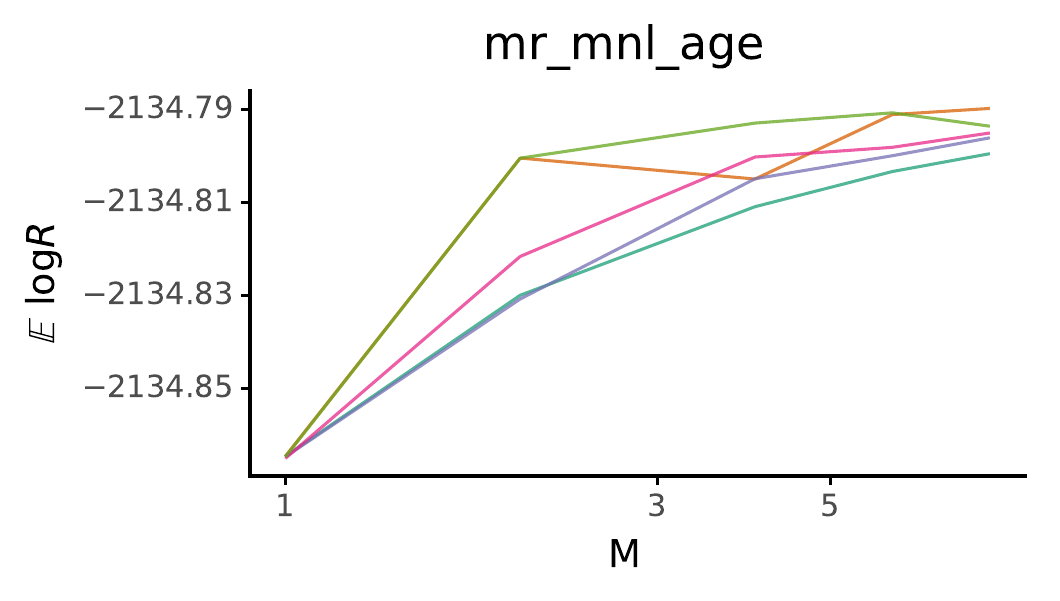}\includegraphics[width=0.33\columnwidth]{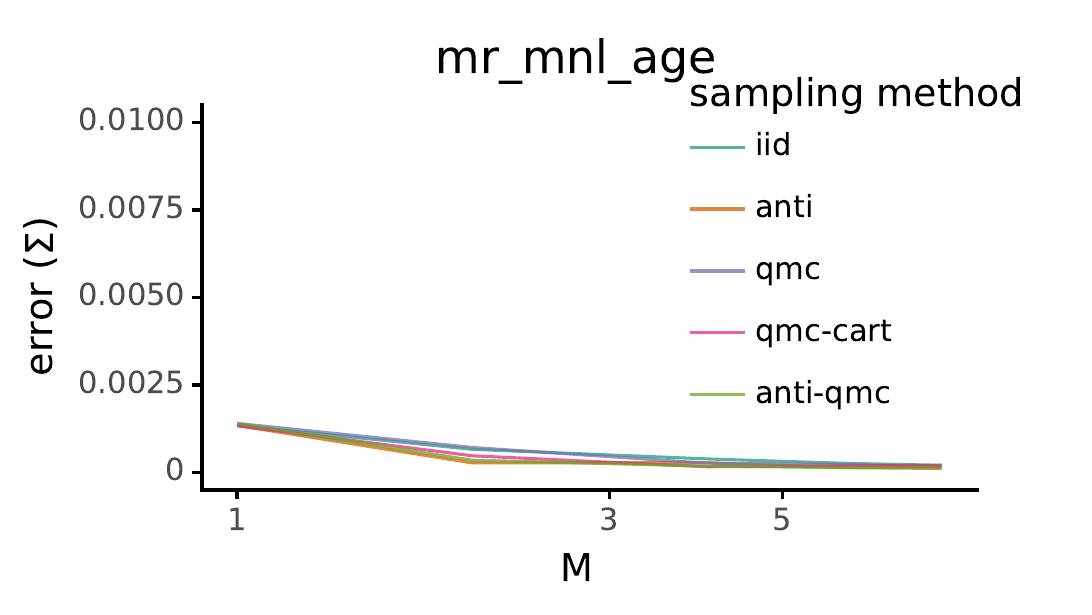}\includegraphics[width=0.33\columnwidth]{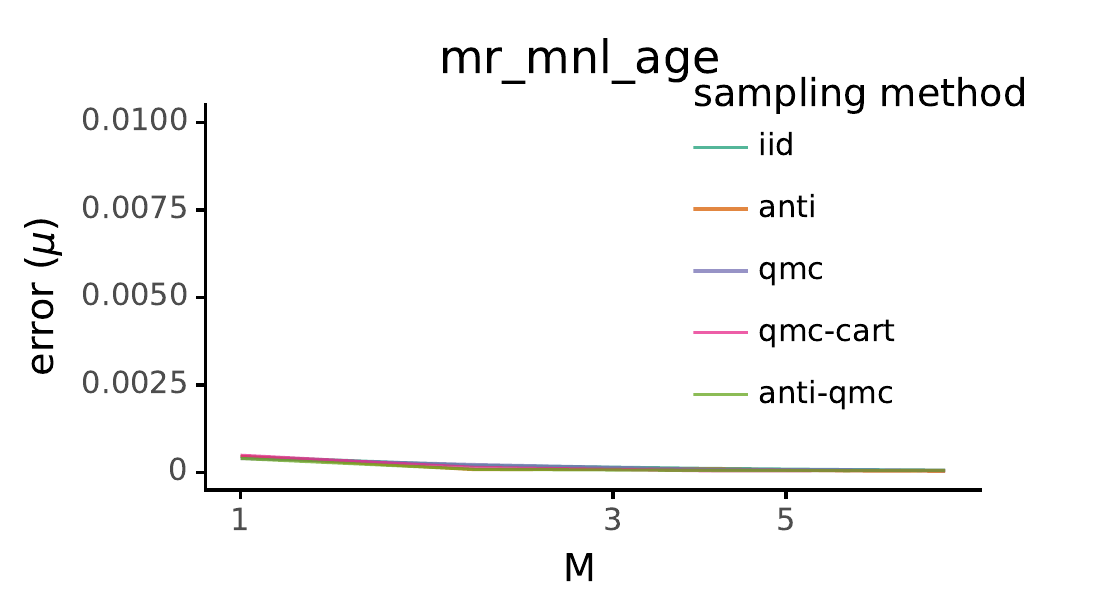}\linebreak{}

\caption{\textbf{Across all models, improvements in likelihood bounds correlate
strongly with improvements in posterior accuracy. Better sampling
methods can improve both.}}
\end{figure}

\begin{figure}
\includegraphics[width=0.33\columnwidth]{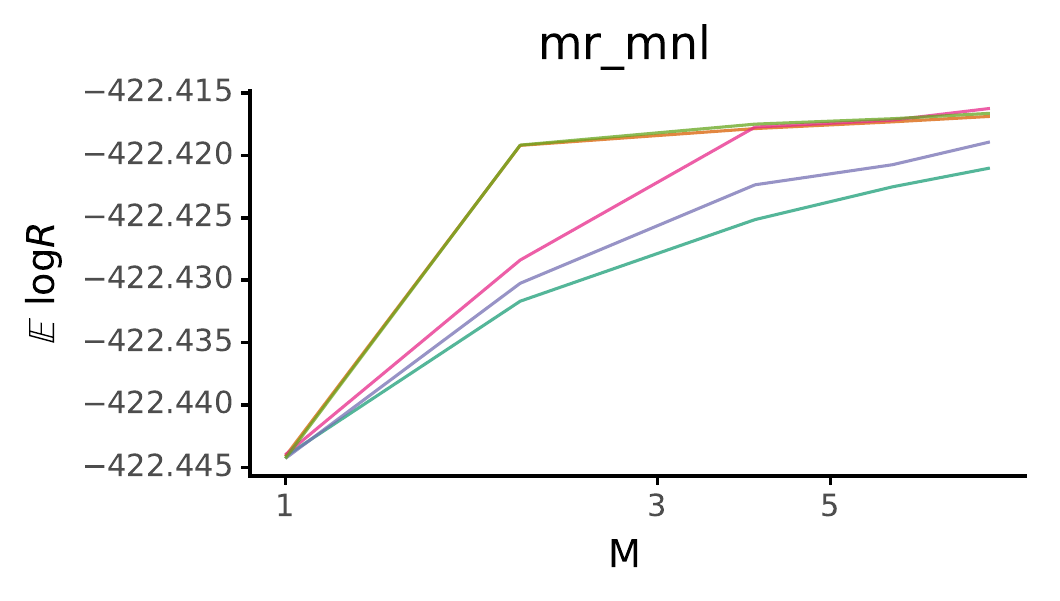}\includegraphics[width=0.33\columnwidth]{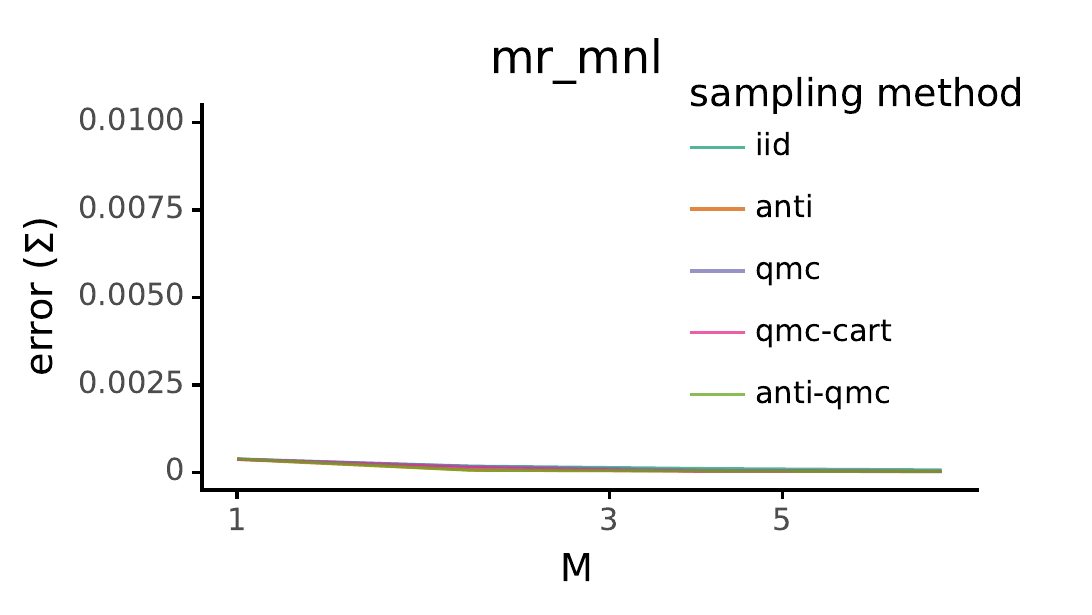}\includegraphics[width=0.33\columnwidth]{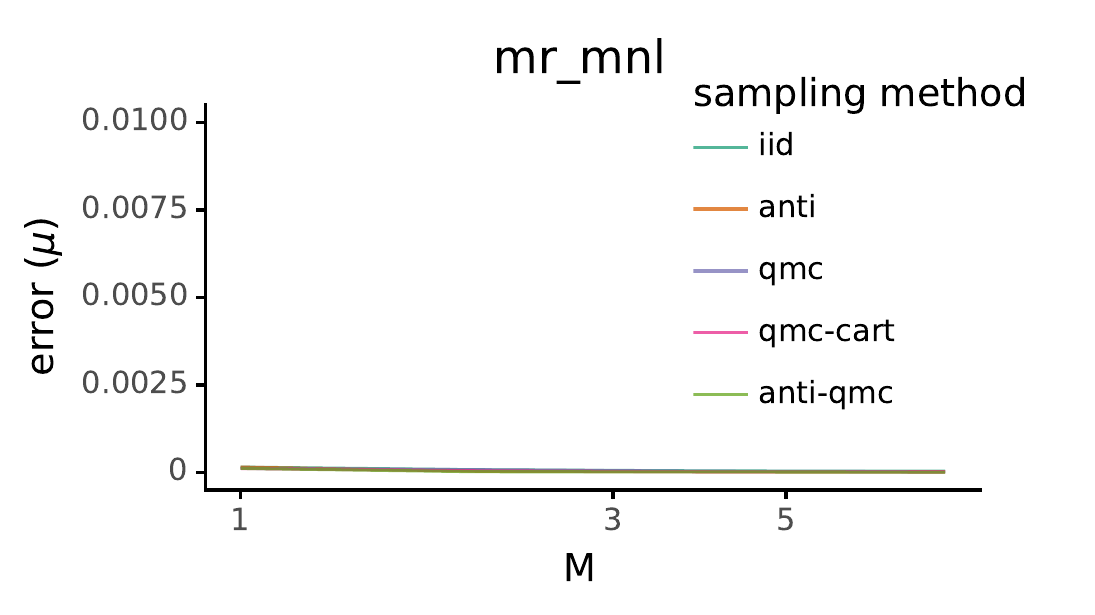}\linebreak{}

\includegraphics[width=0.33\columnwidth]{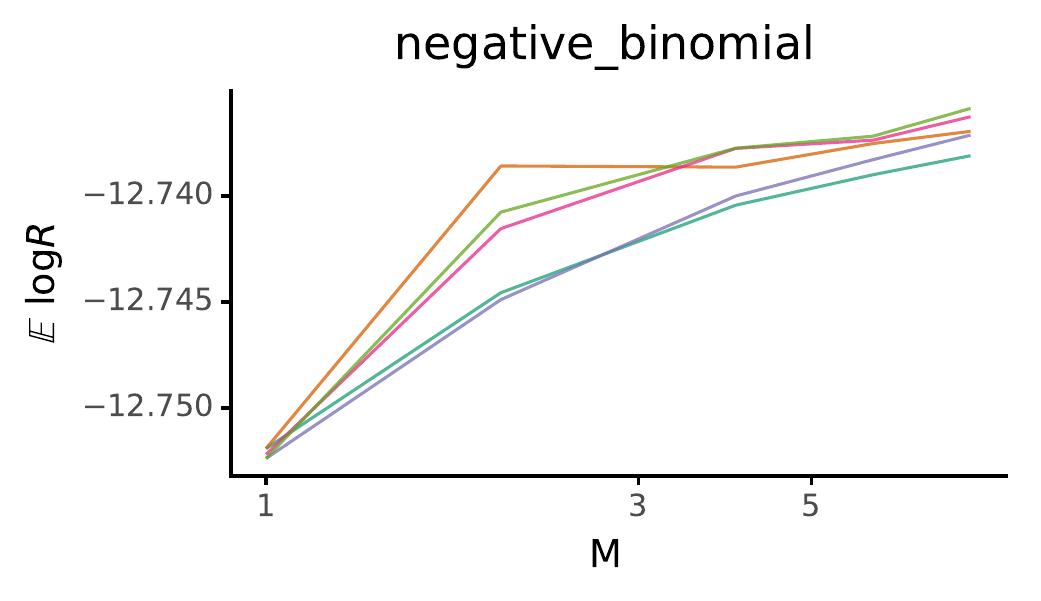}\includegraphics[width=0.33\columnwidth]{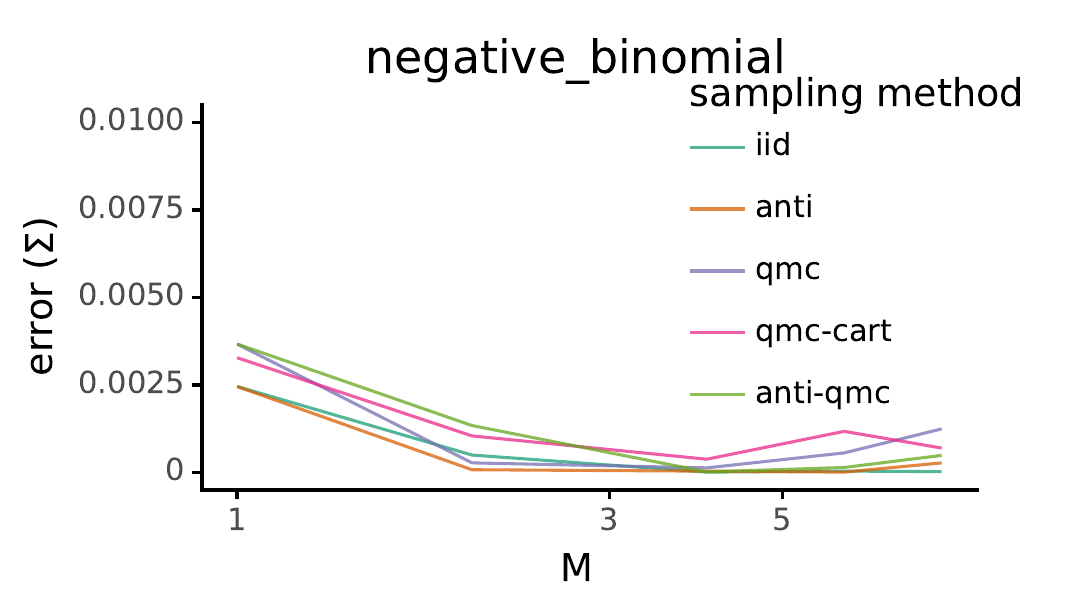}\includegraphics[width=0.33\columnwidth]{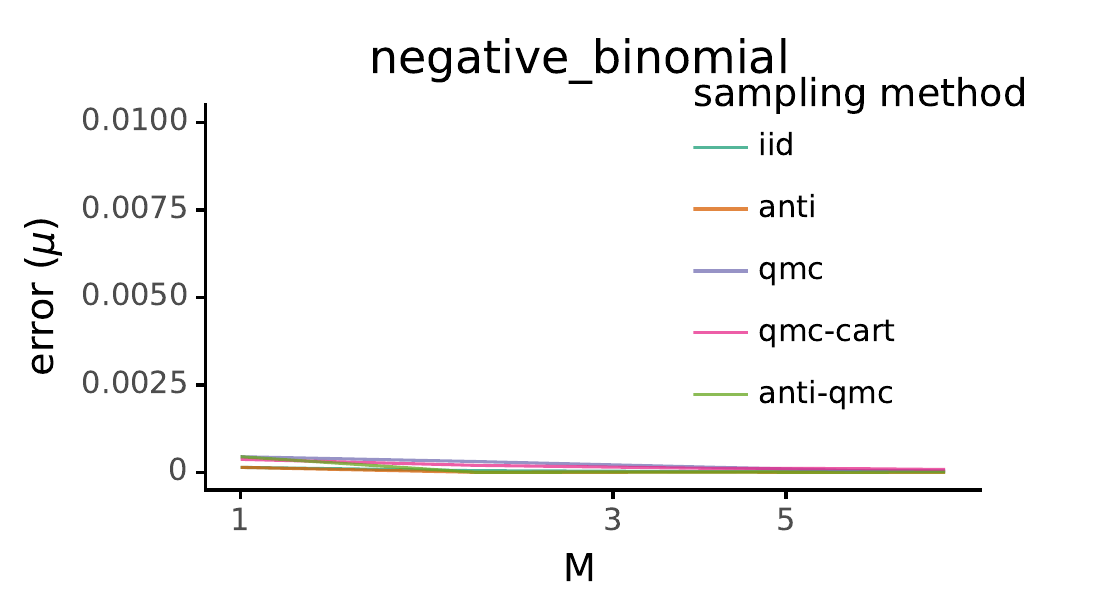}\linebreak{}

\includegraphics[width=0.33\columnwidth]{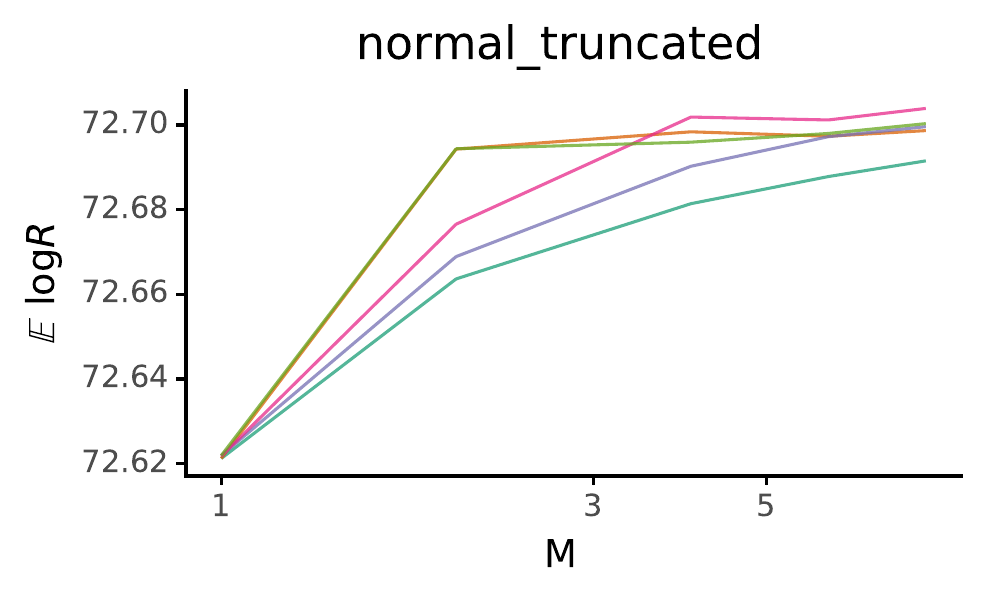}\includegraphics[width=0.33\columnwidth]{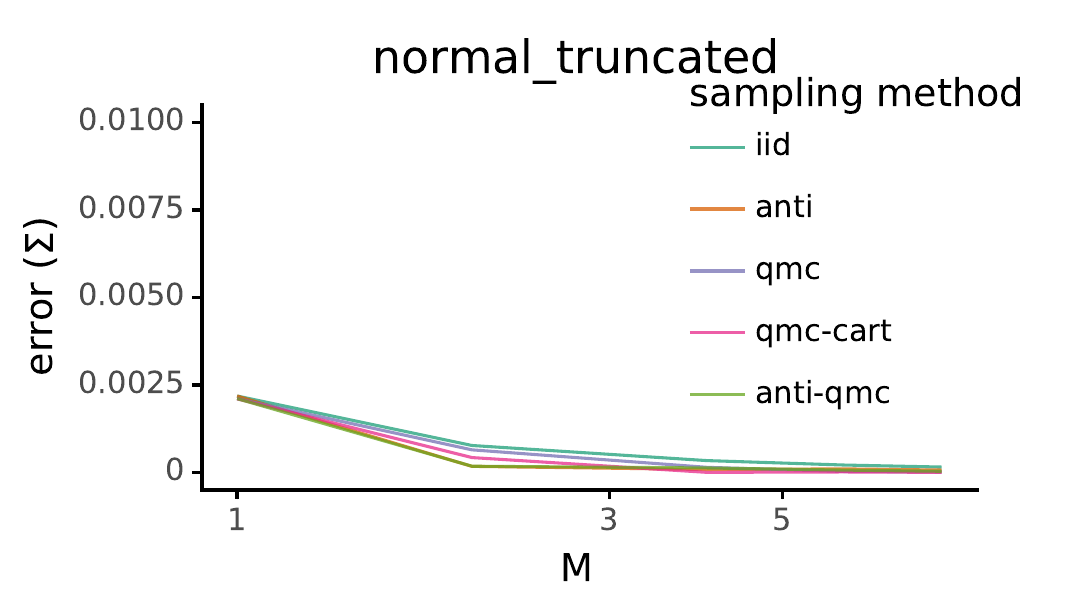}\includegraphics[width=0.33\columnwidth]{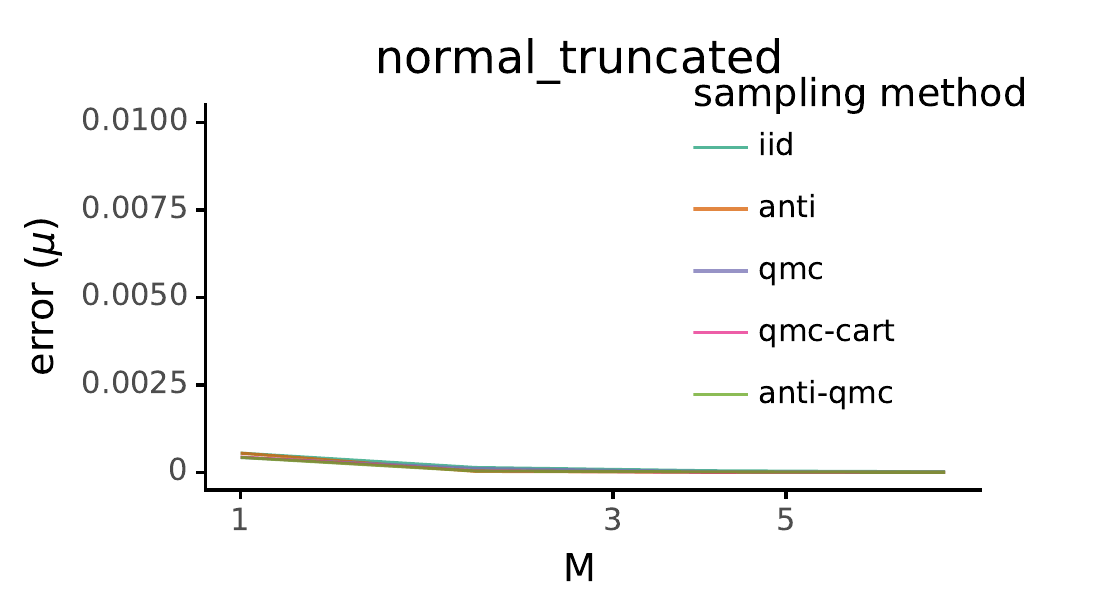}\linebreak{}

\includegraphics[width=0.33\columnwidth]{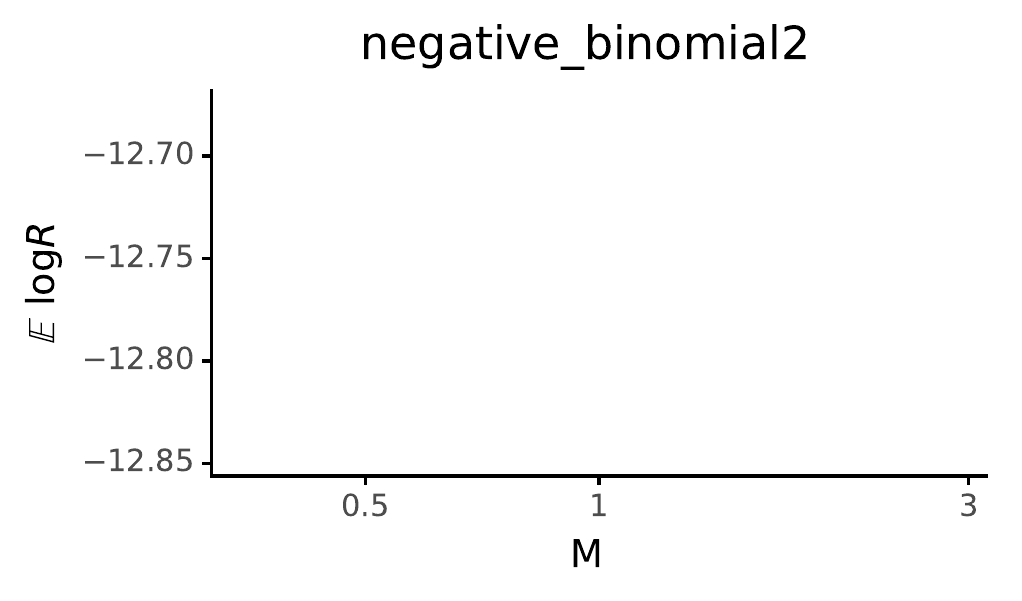}\includegraphics[width=0.33\columnwidth]{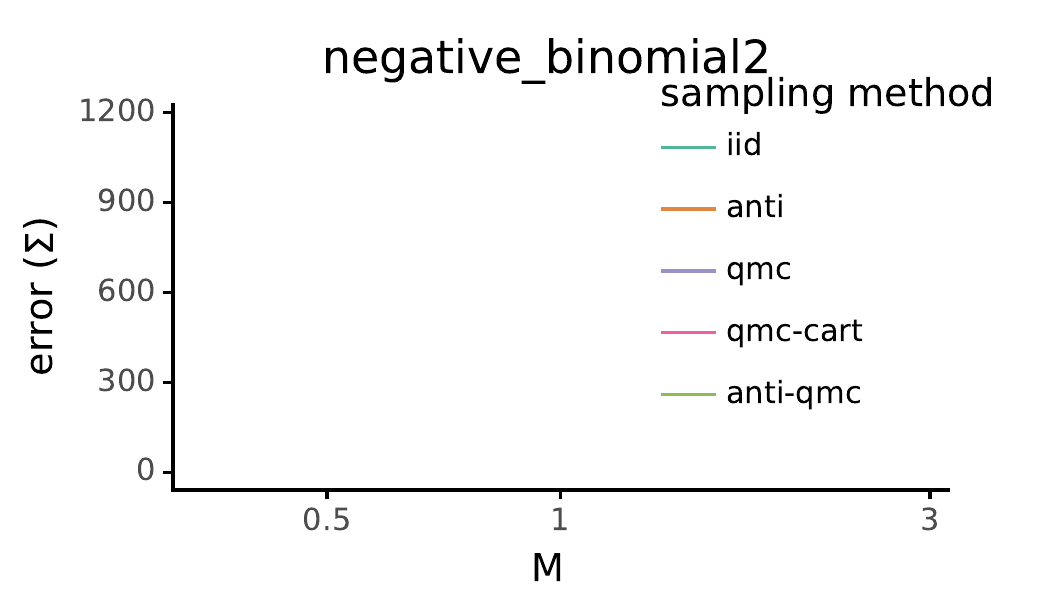}\includegraphics[width=0.33\columnwidth]{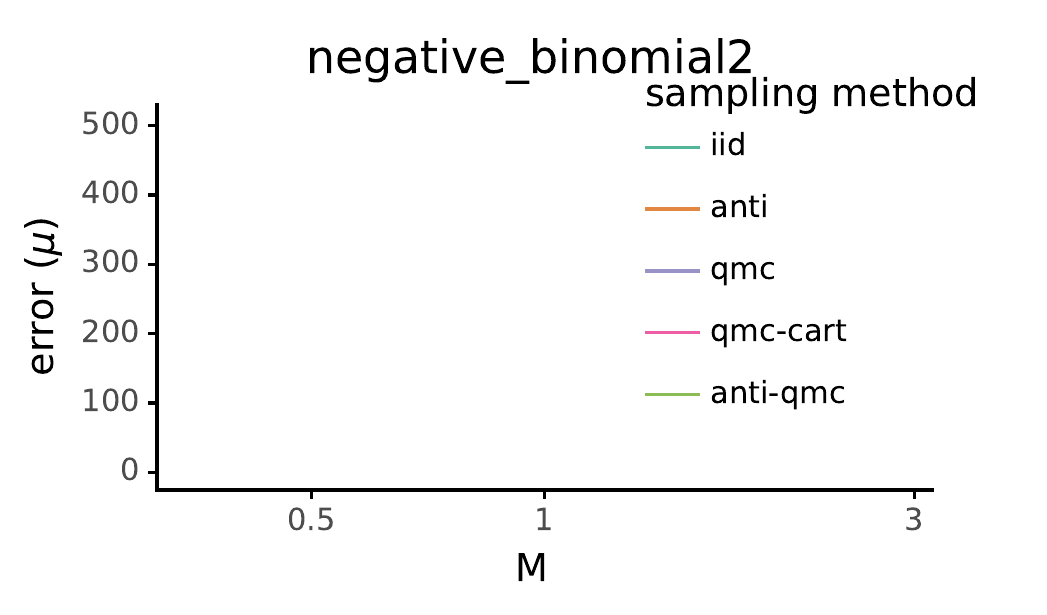}\linebreak{}

\includegraphics[width=0.33\columnwidth]{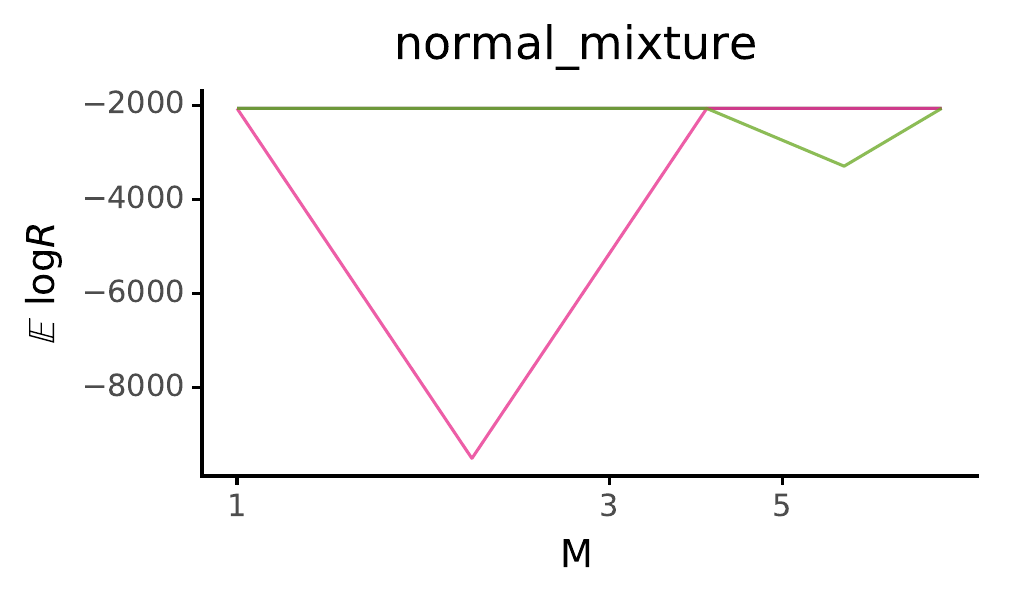}\includegraphics[width=0.33\columnwidth]{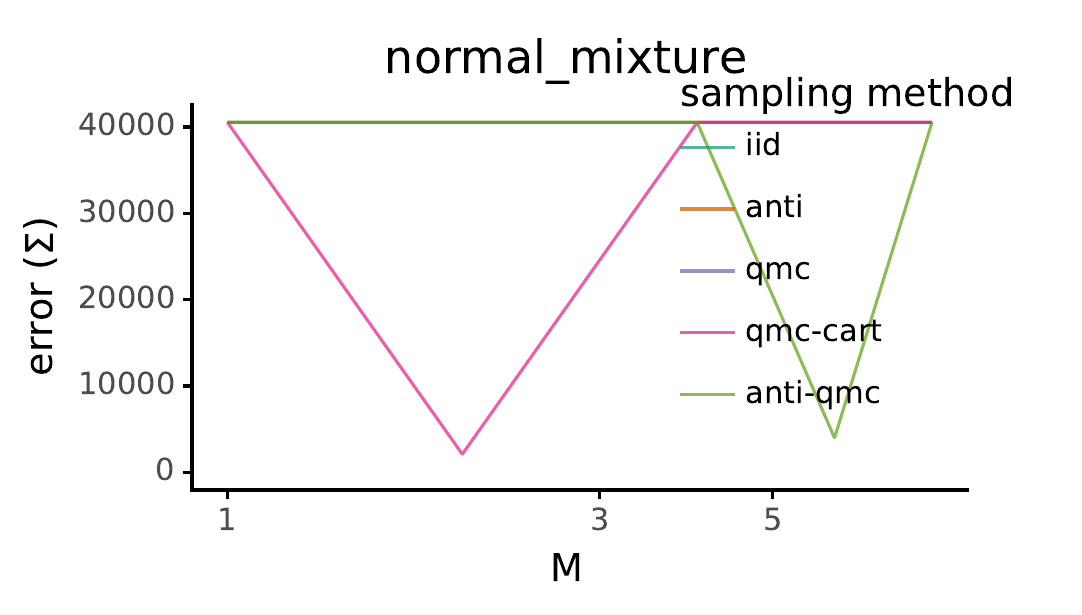}\includegraphics[width=0.33\columnwidth]{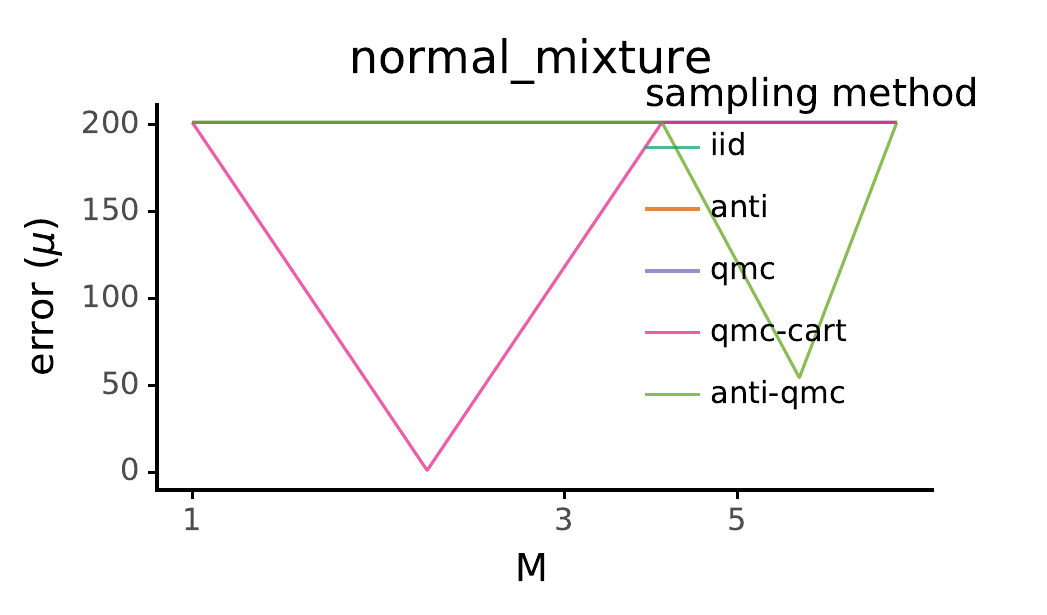}\linebreak{}

\includegraphics[width=0.33\columnwidth]{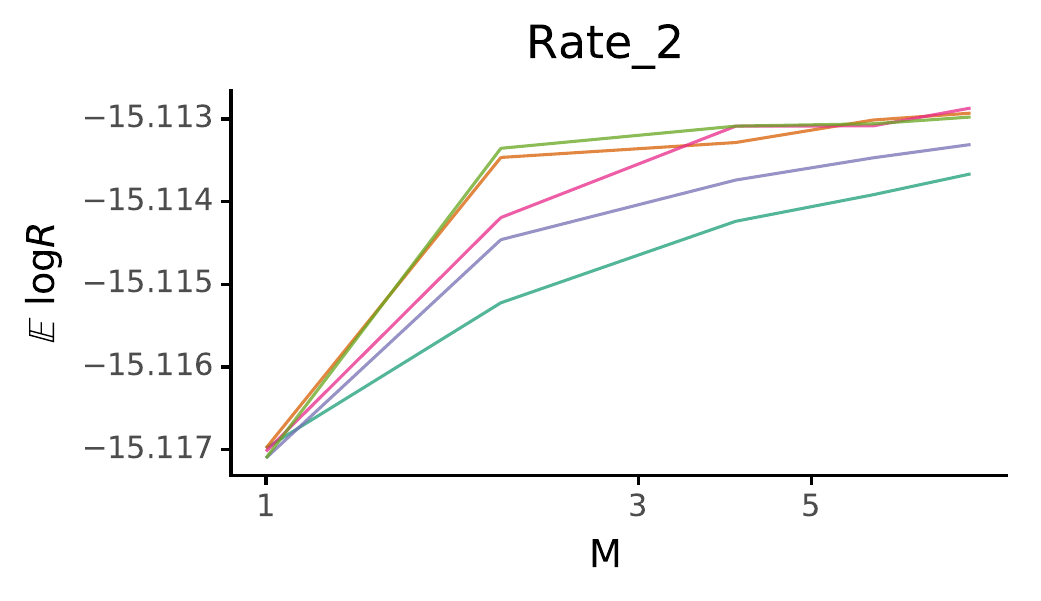}\includegraphics[width=0.33\columnwidth]{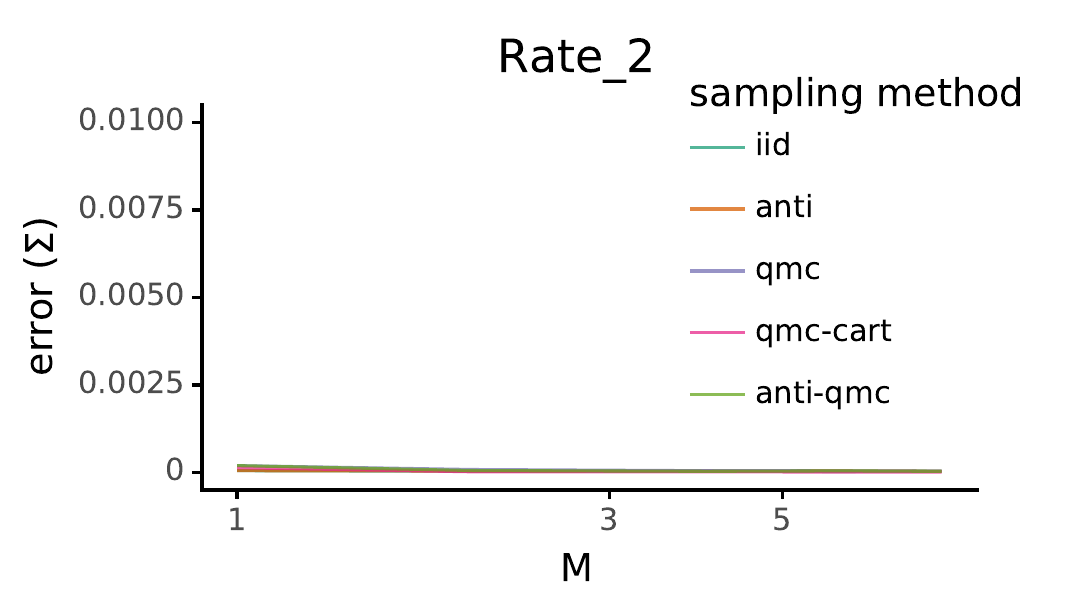}\includegraphics[width=0.33\columnwidth]{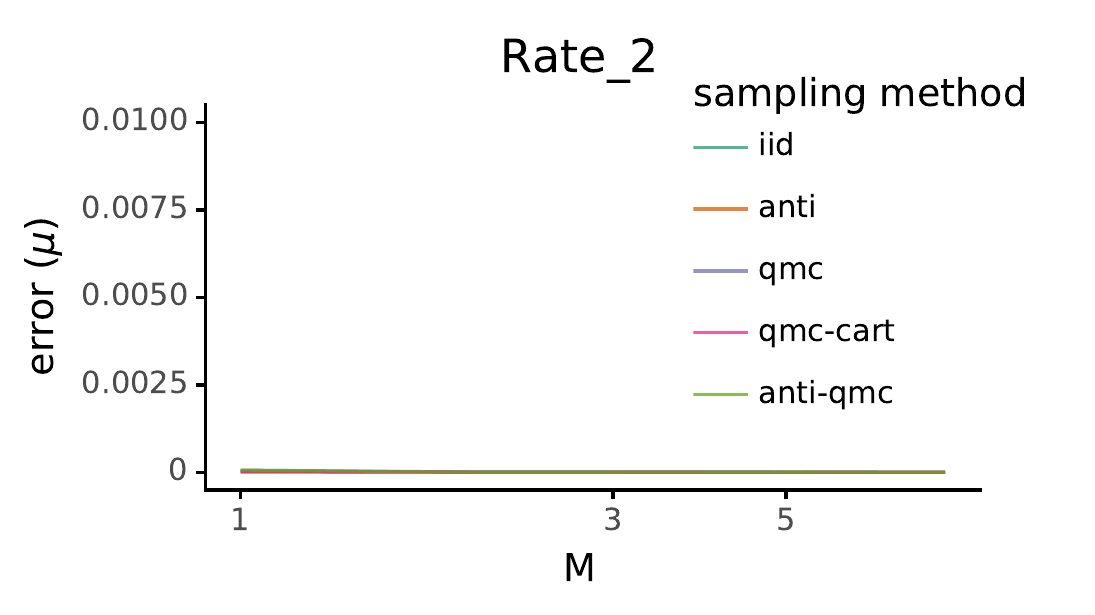}\linebreak{}

\includegraphics[width=0.33\columnwidth]{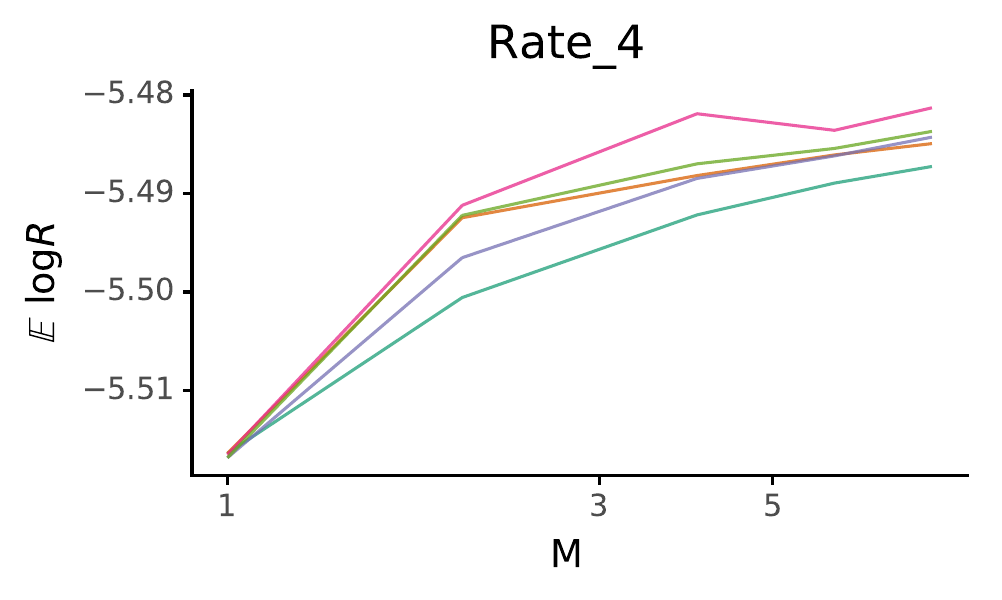}\includegraphics[width=0.33\columnwidth]{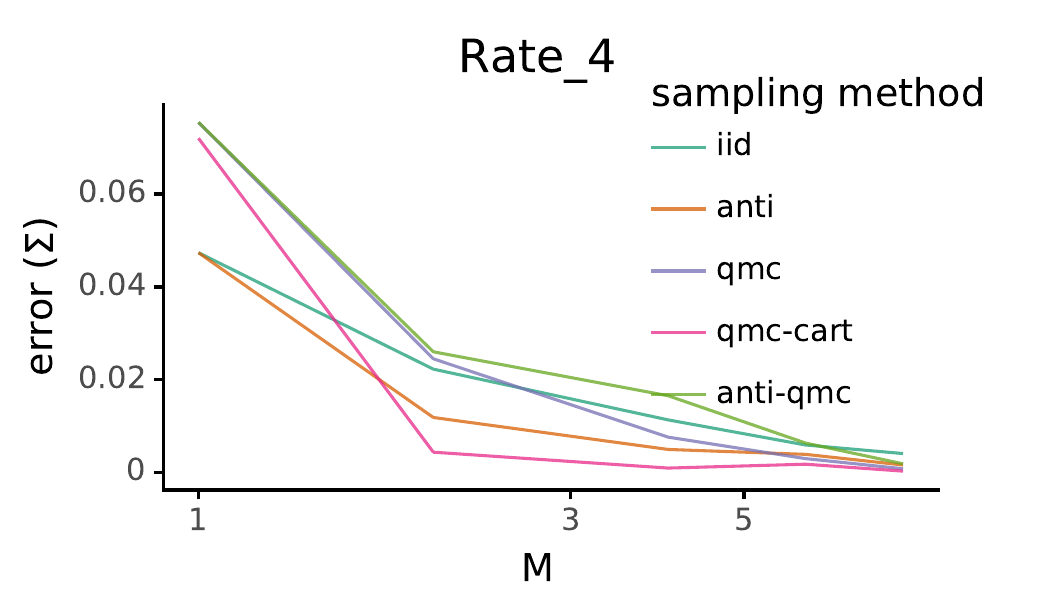}\includegraphics[width=0.33\columnwidth]{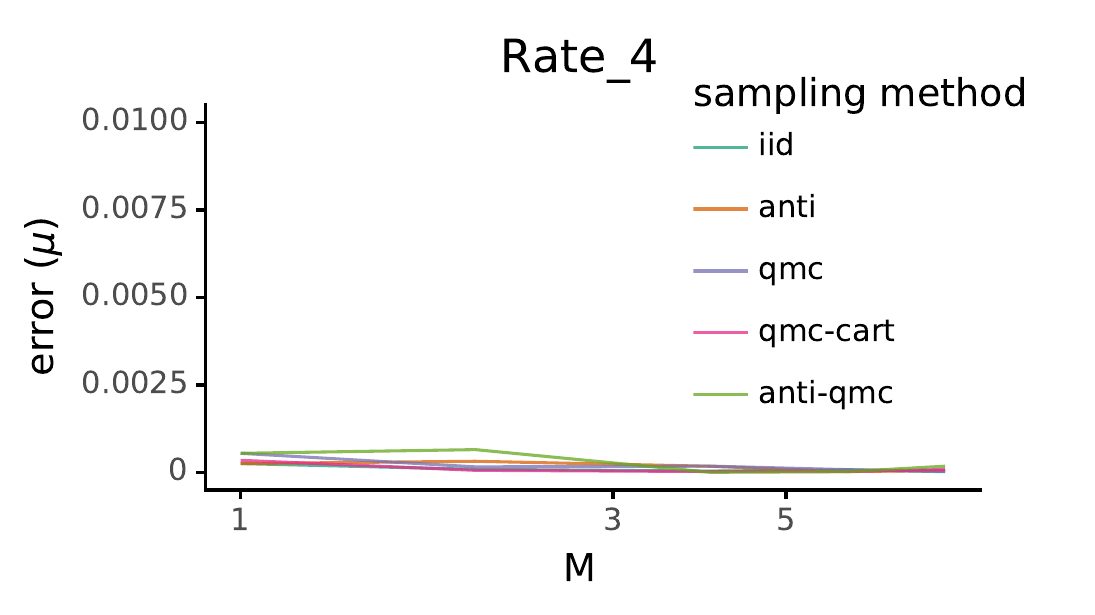}\linebreak{}

\caption{\textbf{Across all models, improvements in likelihood bounds correlate
strongly with improvements in posterior accuracy. Better sampling
methods can improve both.}}
\end{figure}

\begin{figure}
\includegraphics[width=0.33\columnwidth]{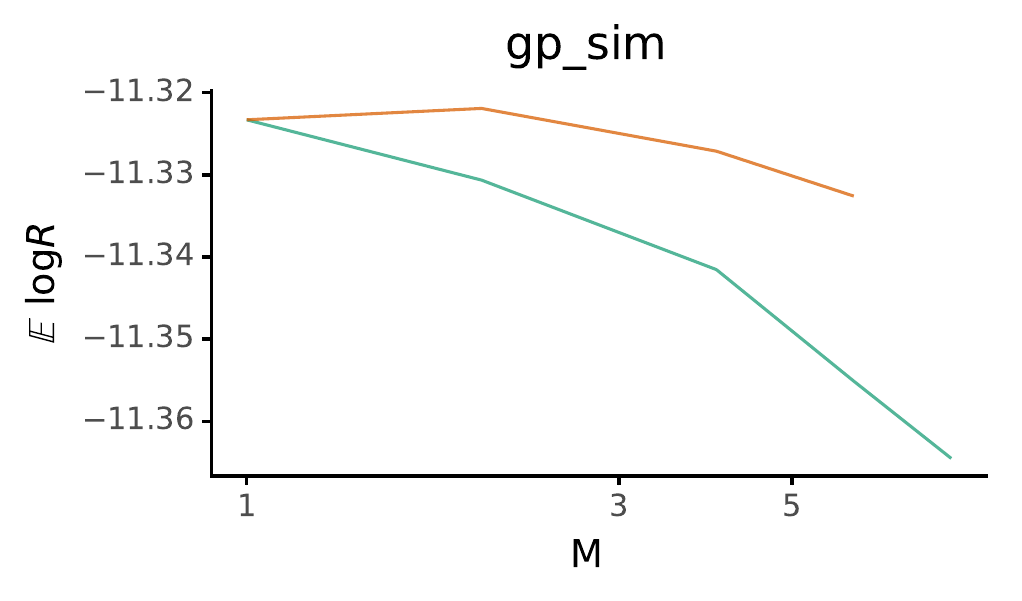}\includegraphics[width=0.33\columnwidth]{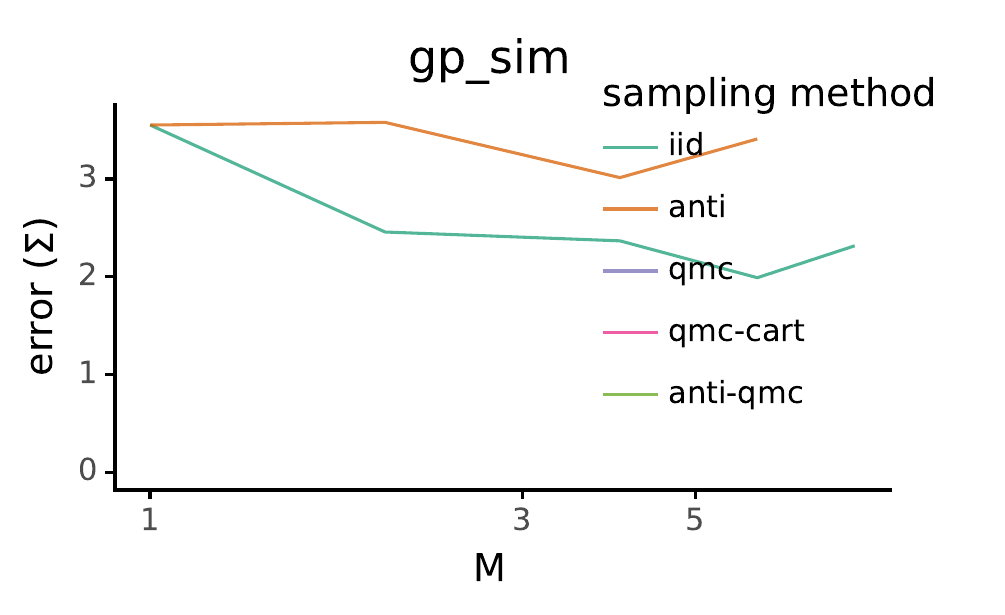}\includegraphics[width=0.33\columnwidth]{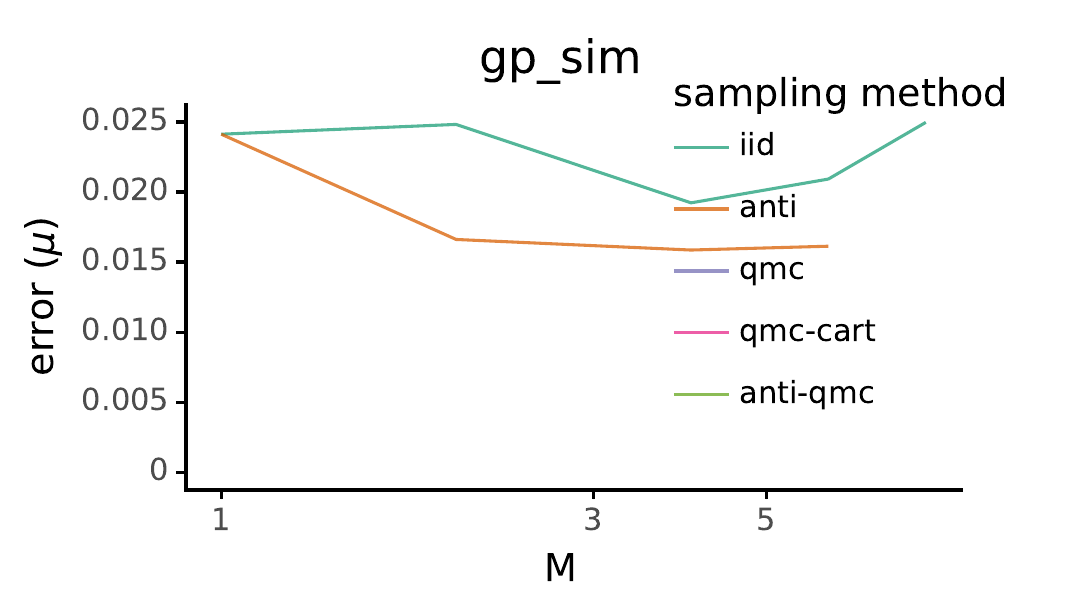}\linebreak{}

\includegraphics[width=0.33\columnwidth]{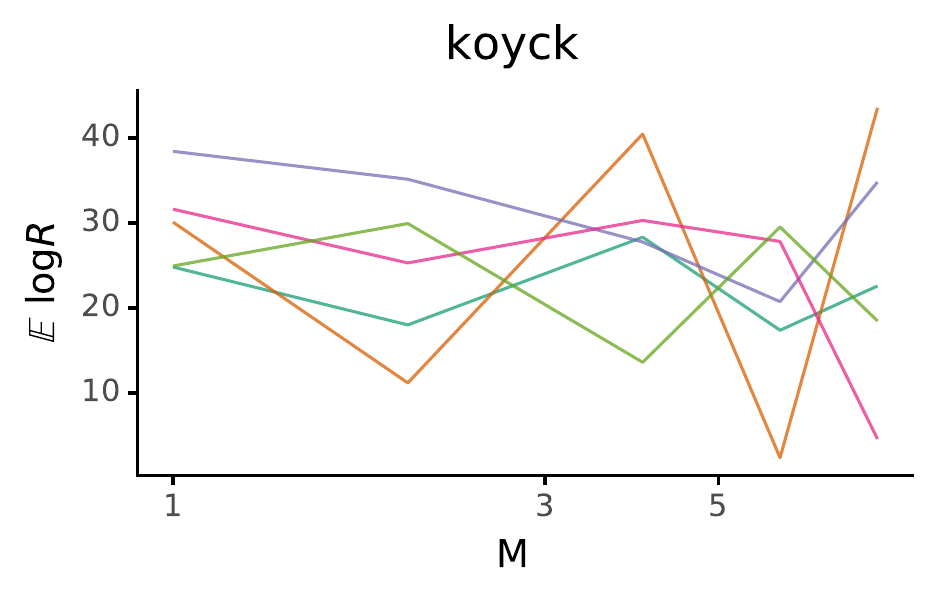}\includegraphics[width=0.33\columnwidth]{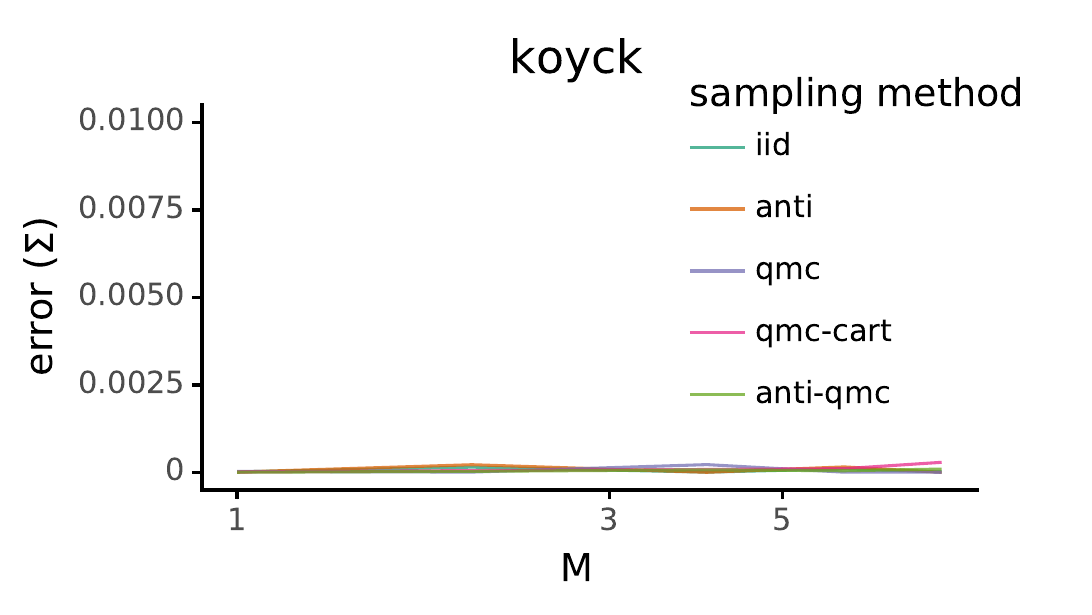}\includegraphics[width=0.33\columnwidth]{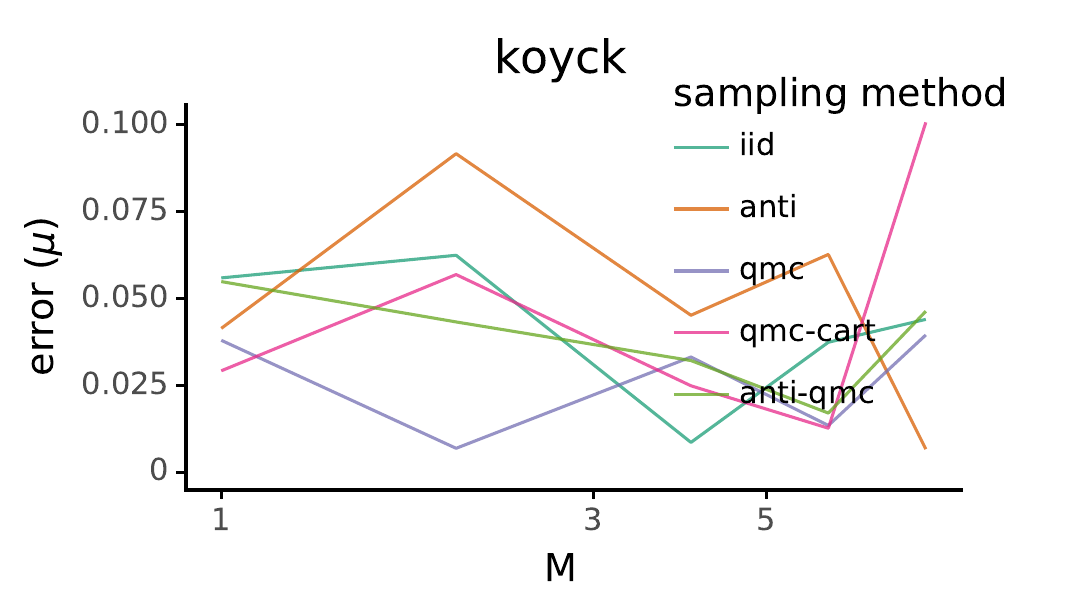}\linebreak{}

\includegraphics[width=0.33\columnwidth]{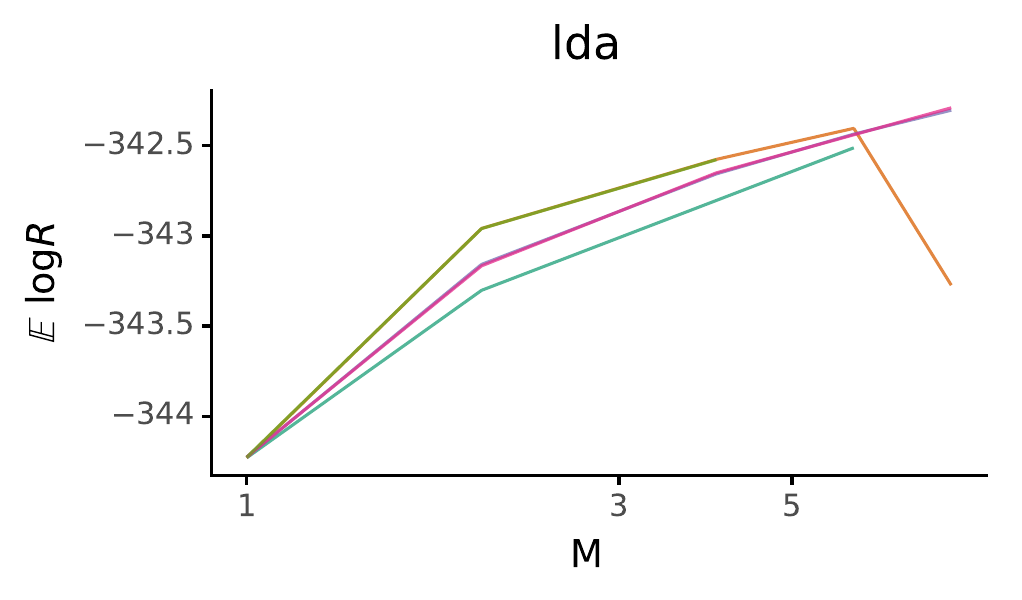}\includegraphics[width=0.33\columnwidth]{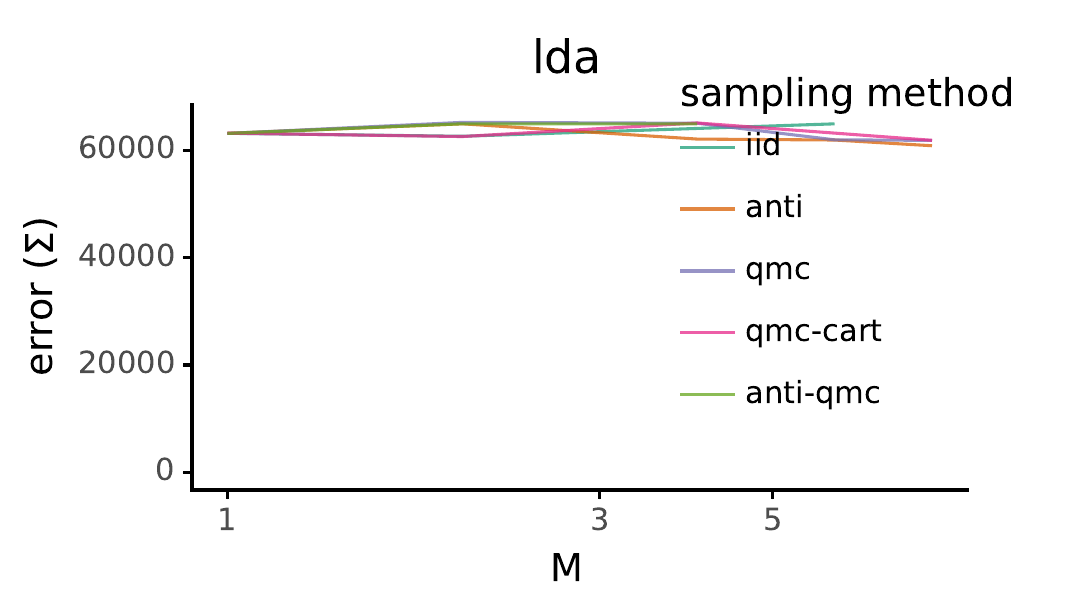}\includegraphics[width=0.33\columnwidth]{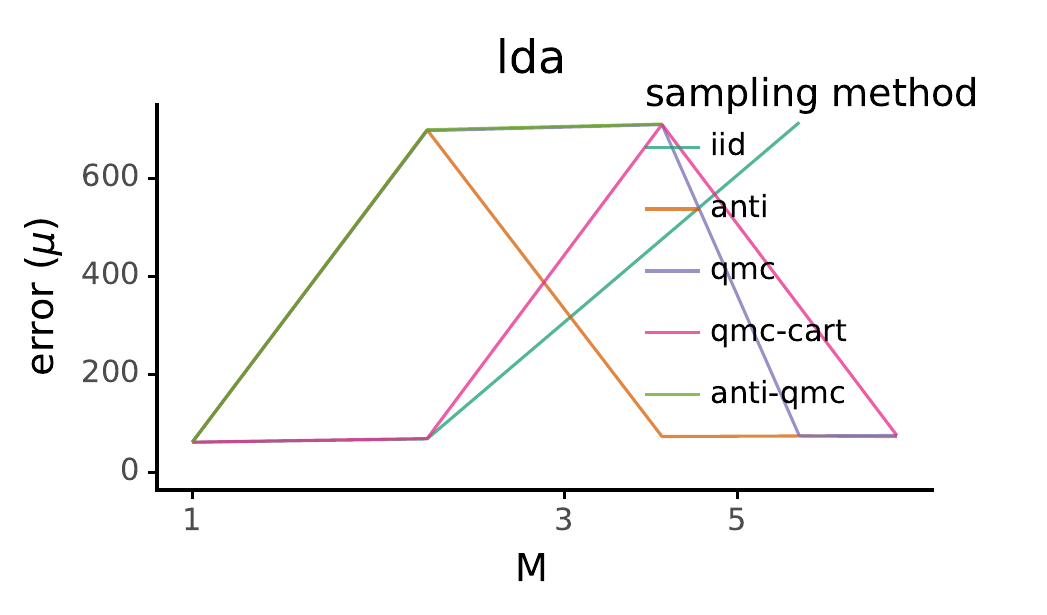}\linebreak{}

\includegraphics[width=0.33\columnwidth]{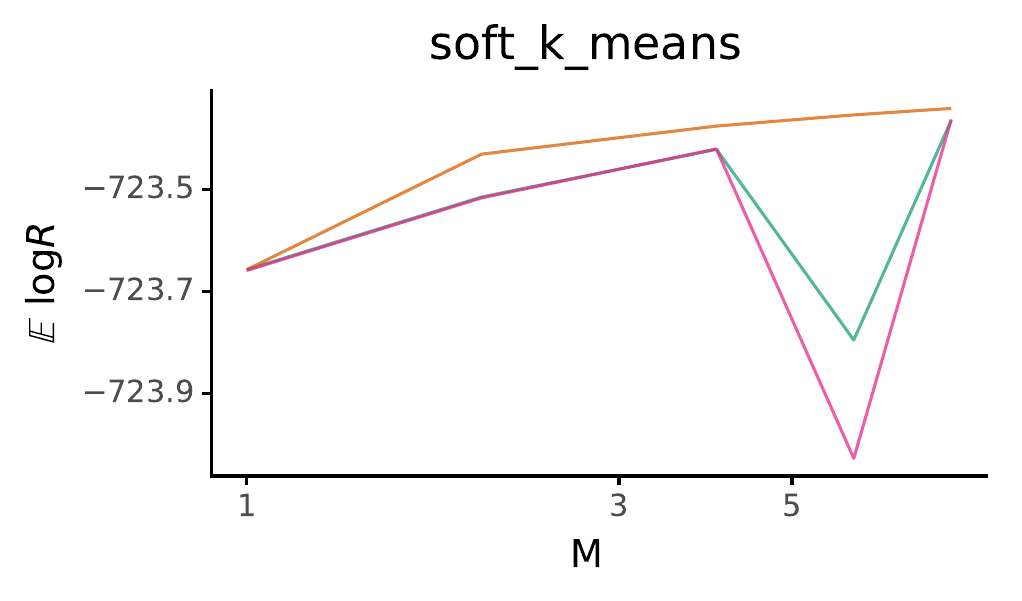}\includegraphics[width=0.33\columnwidth]{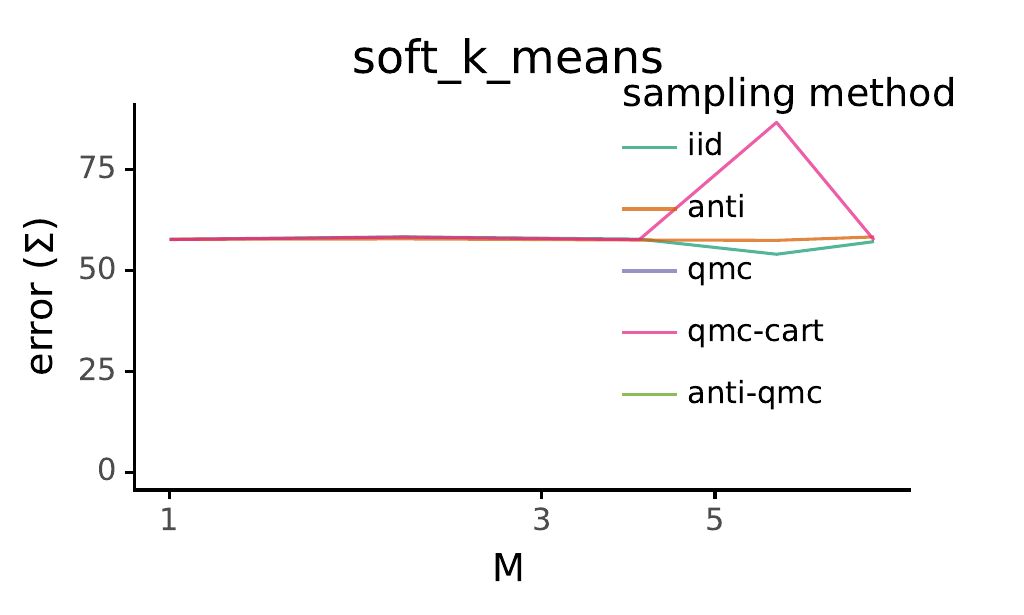}\includegraphics[width=0.33\columnwidth]{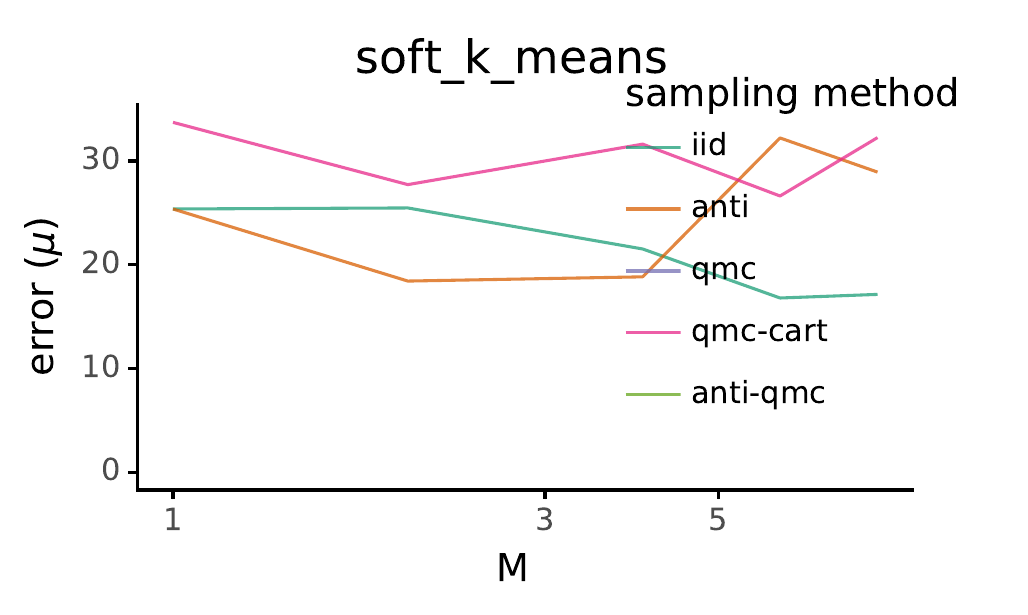}\linebreak{}

\includegraphics[width=0.33\columnwidth]{final_swarm_figs/stan123_1_elbos}\includegraphics[width=0.33\columnwidth]{final_swarm_figs/stan123_1_err_Sigma}\includegraphics[width=0.33\columnwidth]{final_swarm_figs/stan123_1_err_mu}\linebreak{}

\includegraphics[width=0.33\columnwidth]{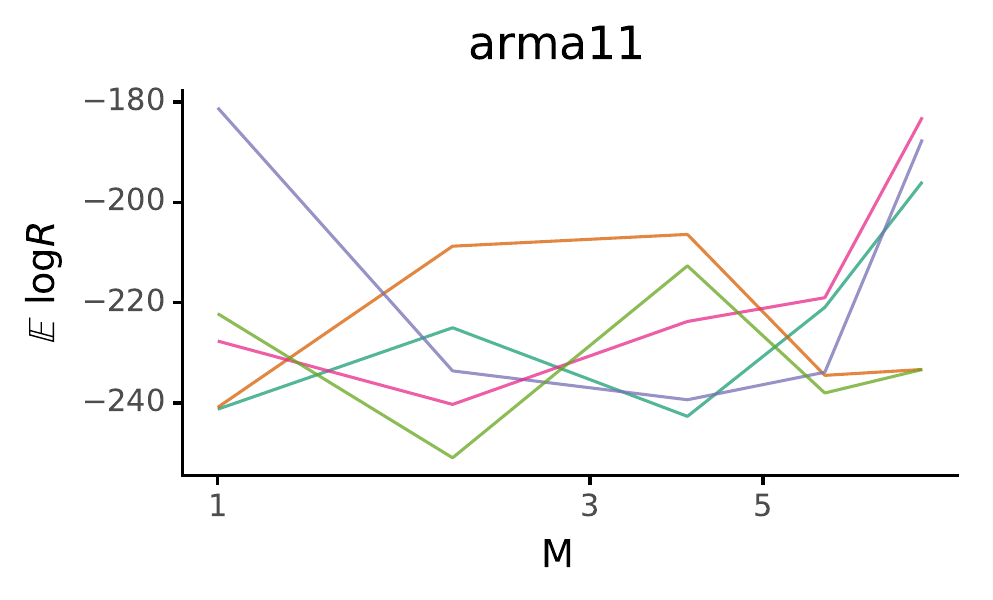}\includegraphics[width=0.33\columnwidth]{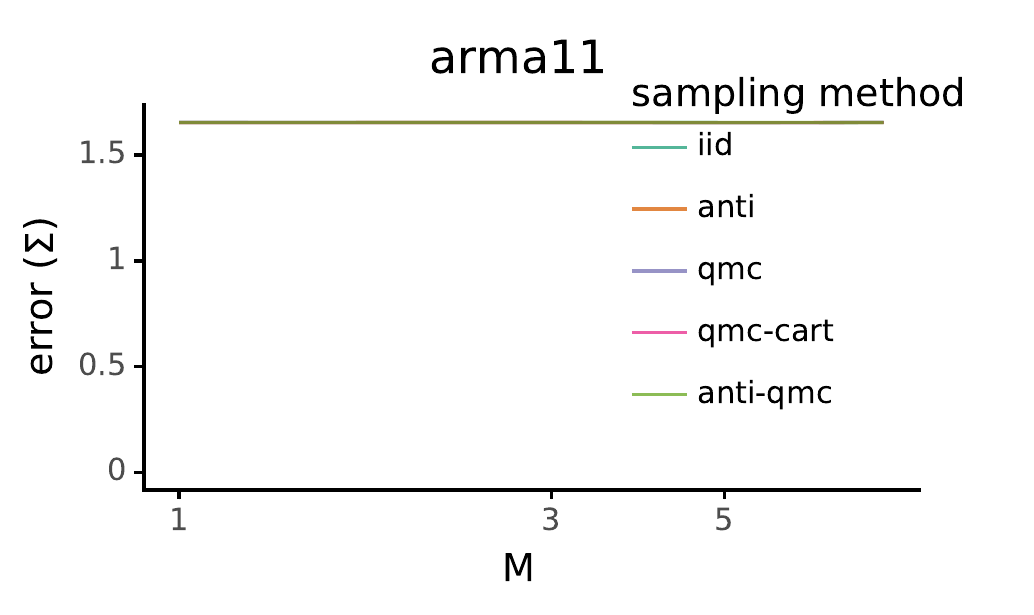}\includegraphics[width=0.33\columnwidth]{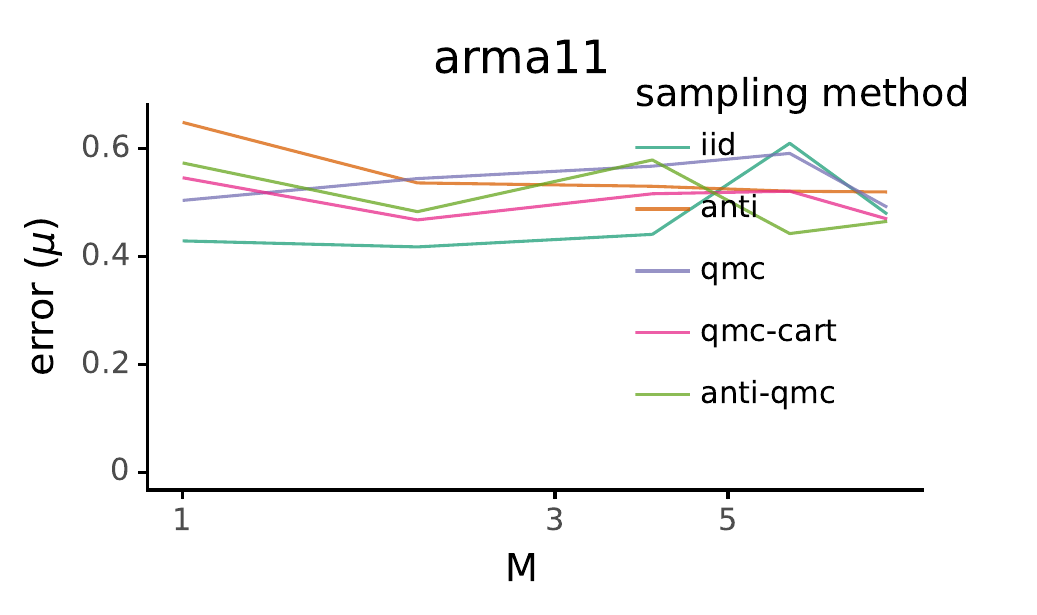}\linebreak{}

\includegraphics[width=0.33\columnwidth]{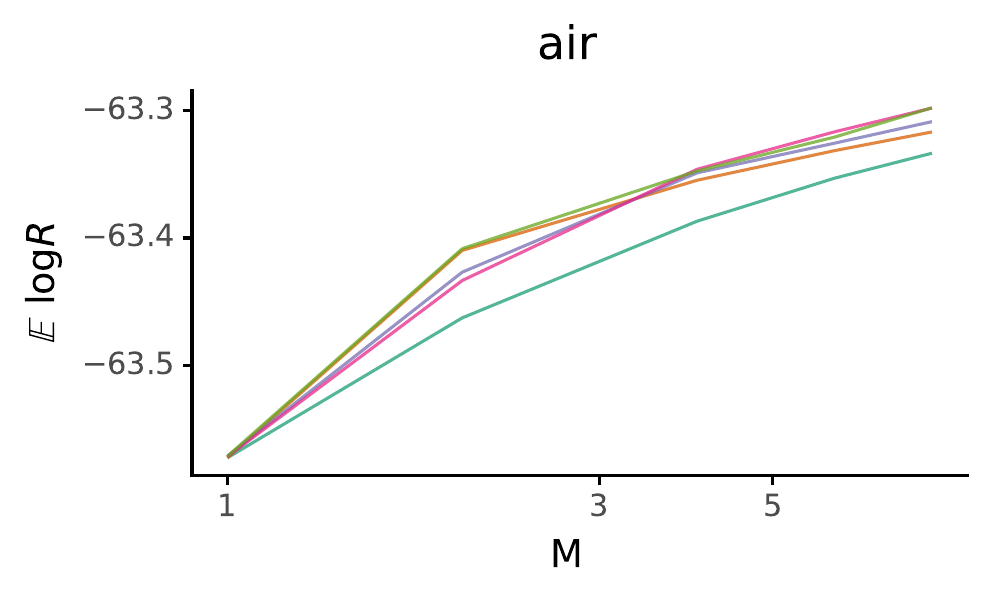}\includegraphics[width=0.33\columnwidth]{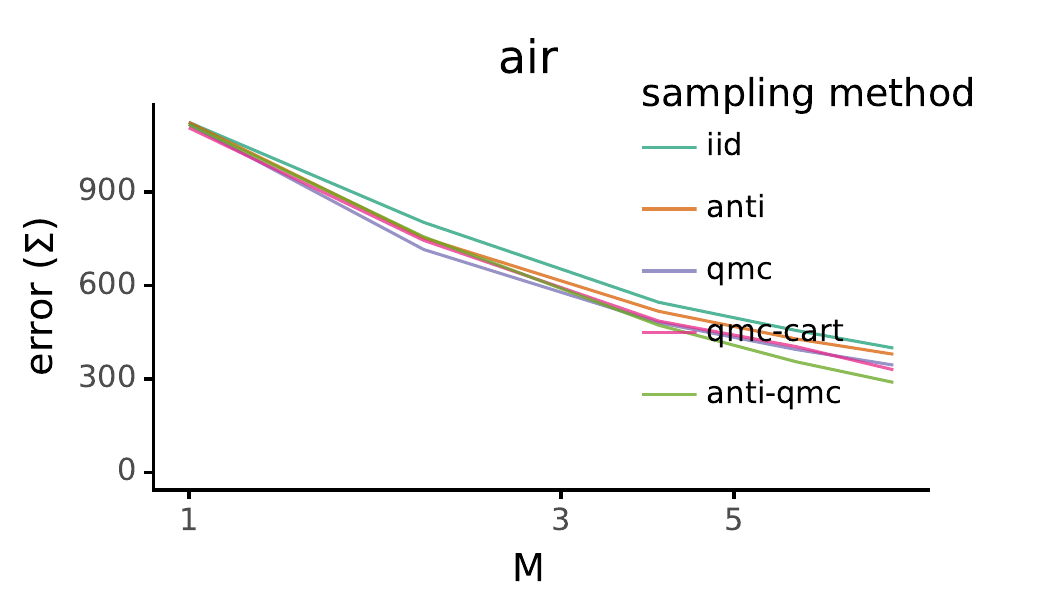}\includegraphics[width=0.33\columnwidth]{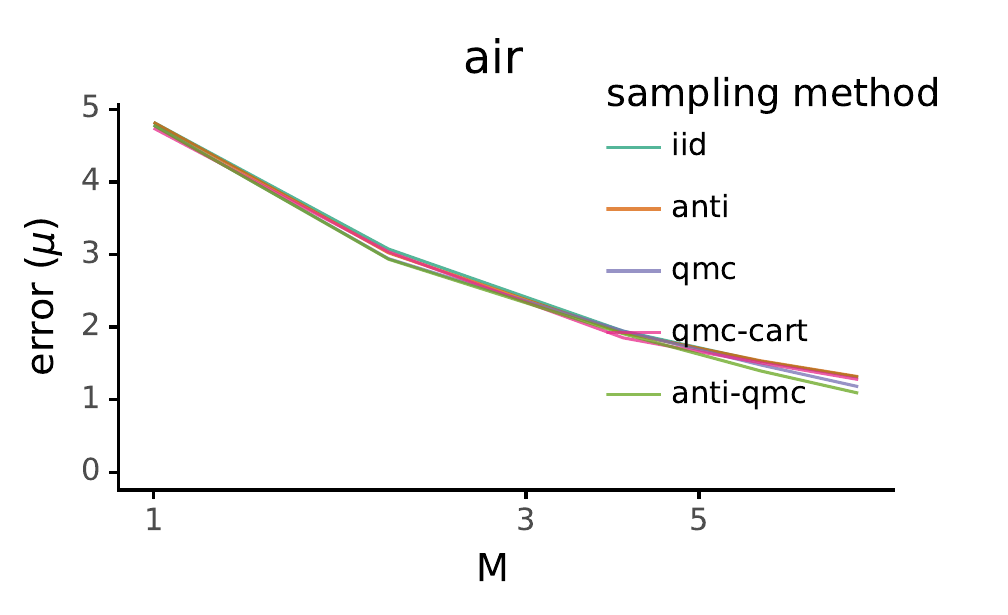}\linebreak{}

\caption{\textbf{Across all models, improvements in likelihood bounds correlate
strongly with improvements in posterior accuracy. Better sampling
methods can improve both.}}
\end{figure}

\begin{figure}
\includegraphics[width=0.33\columnwidth]{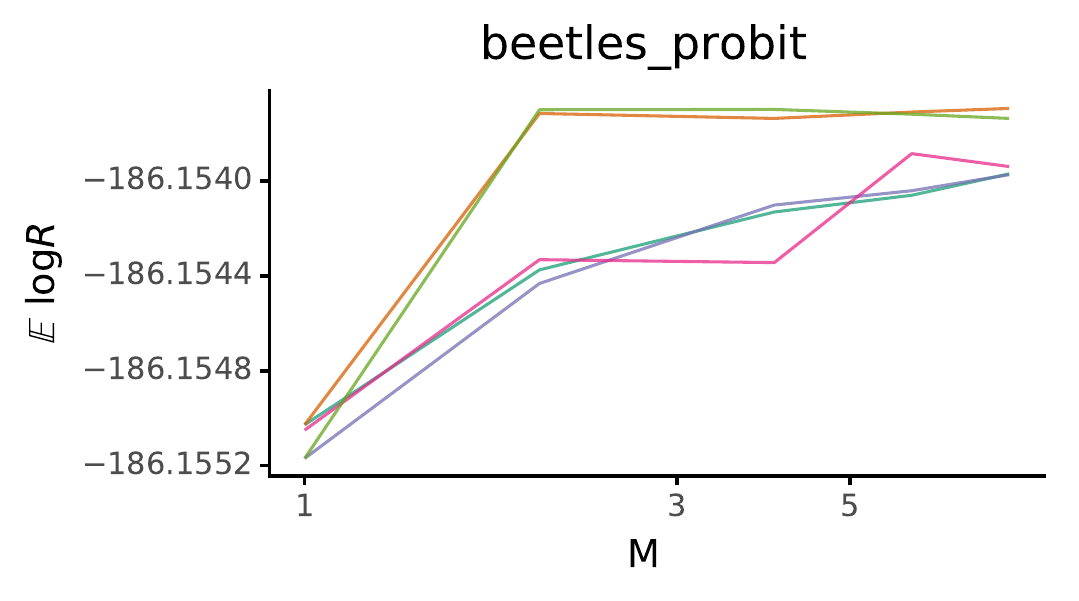}\includegraphics[width=0.33\columnwidth]{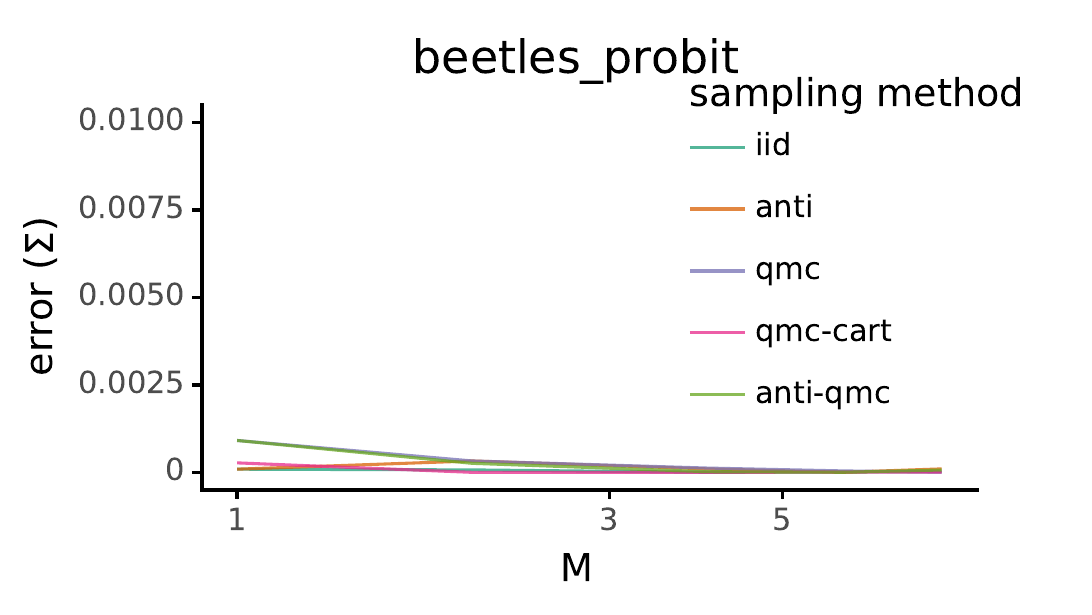}\includegraphics[width=0.33\columnwidth]{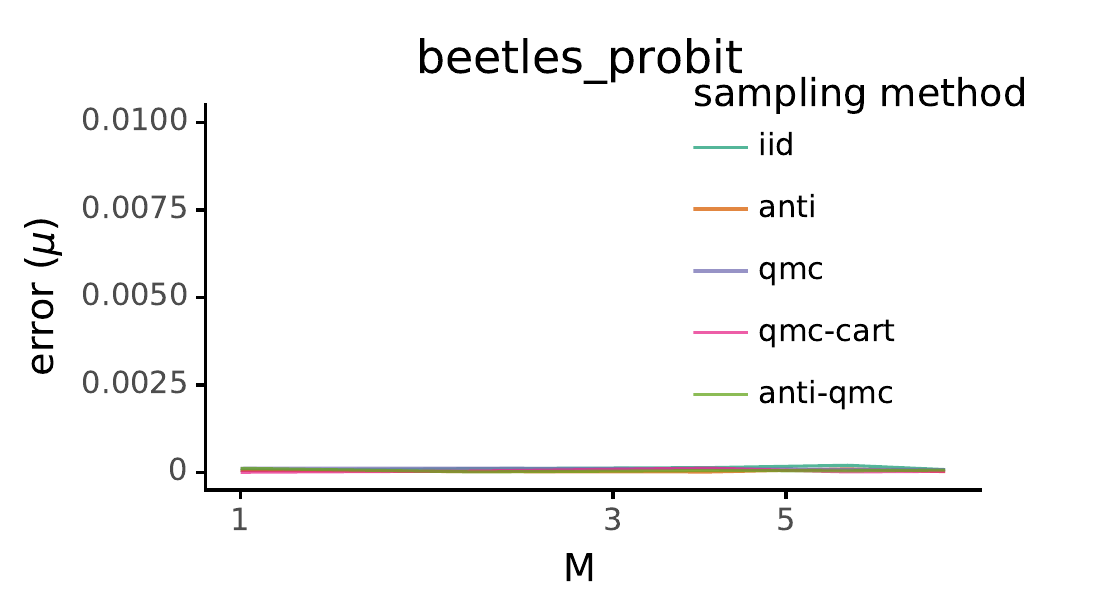}\linebreak{}

\includegraphics[width=0.33\columnwidth]{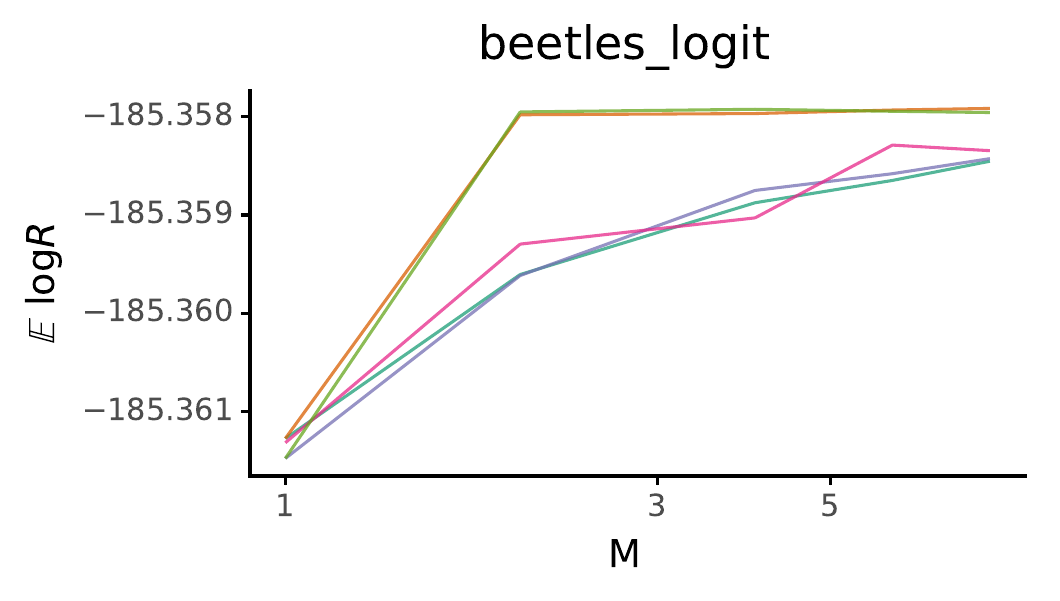}\includegraphics[width=0.33\columnwidth]{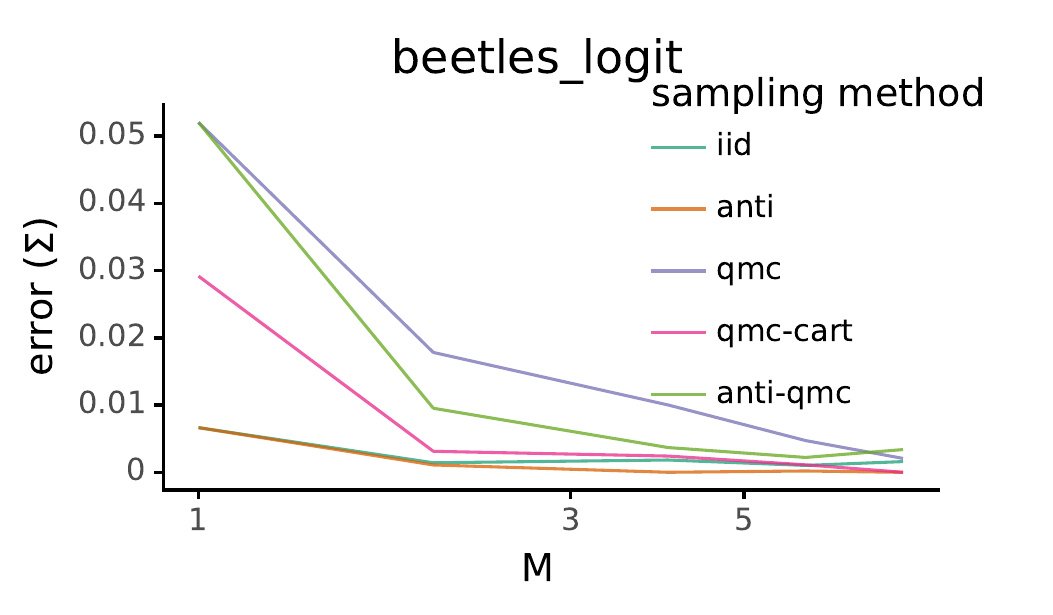}\includegraphics[width=0.33\columnwidth]{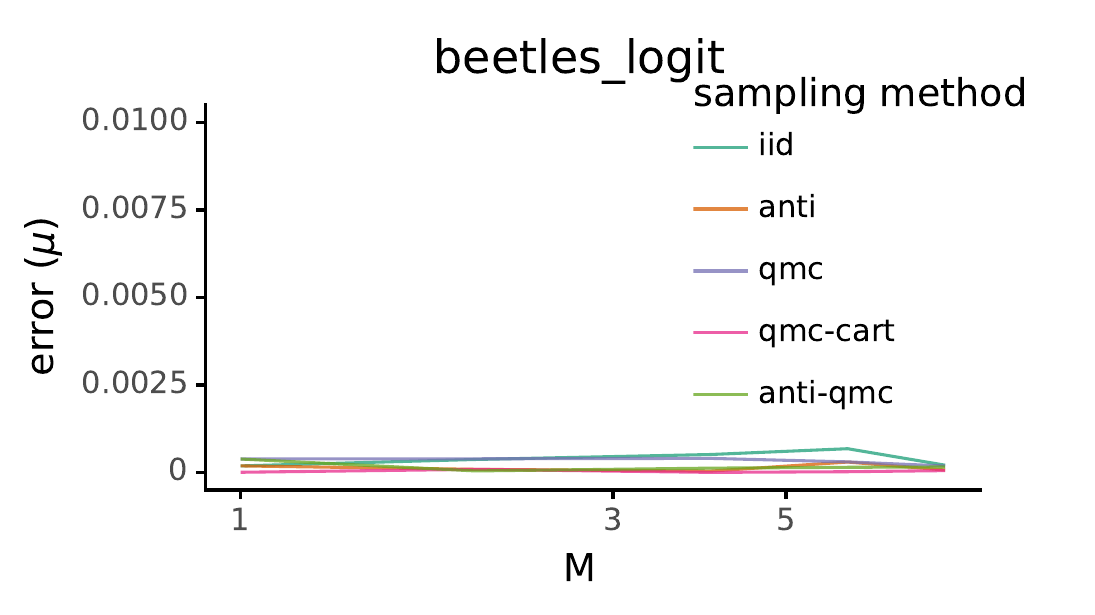}\linebreak{}

\includegraphics[width=0.33\columnwidth]{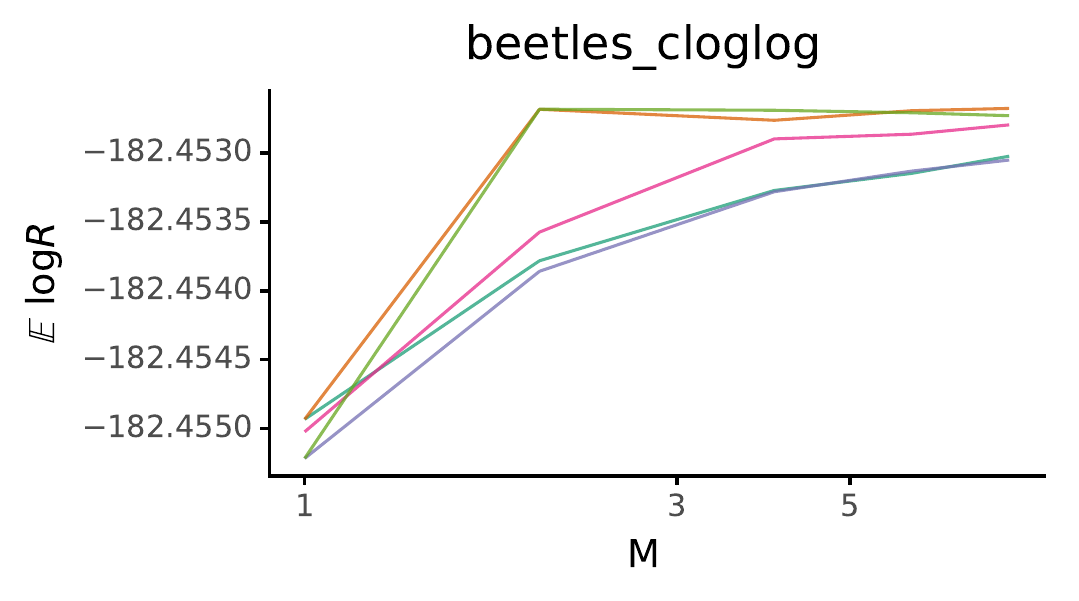}\includegraphics[width=0.33\columnwidth]{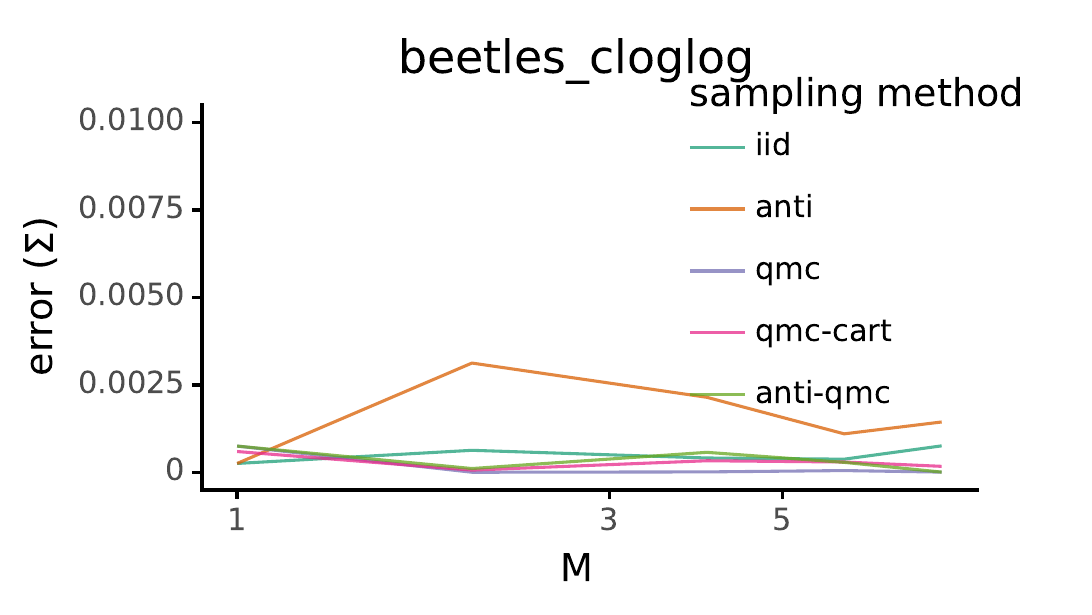}\includegraphics[width=0.33\columnwidth]{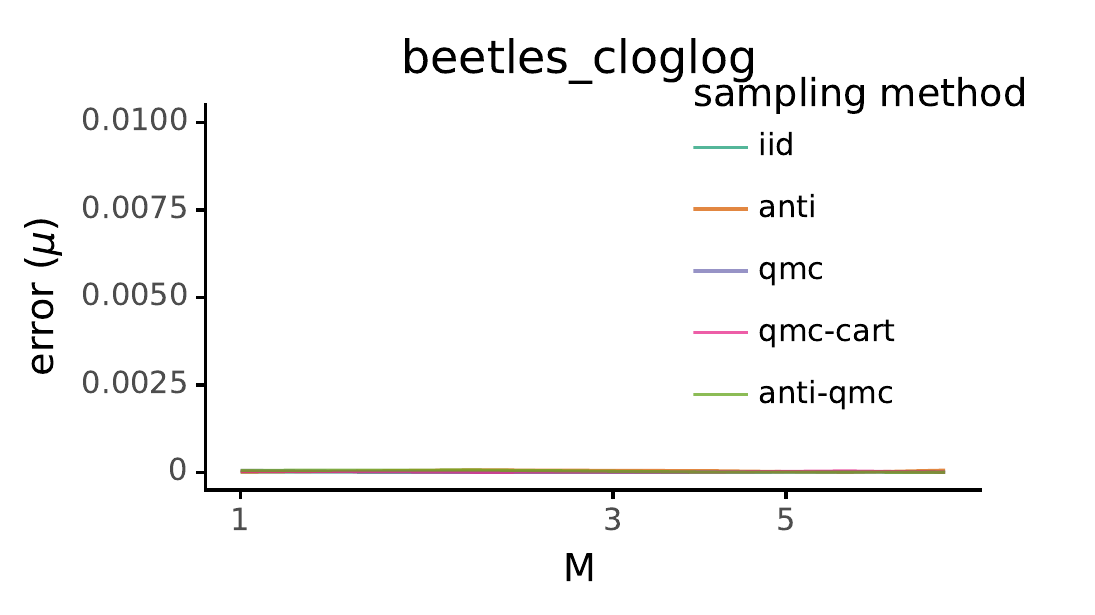}\linebreak{}

\includegraphics[width=0.33\columnwidth]{final_swarm_figs/stan140_1_elbos}\includegraphics[width=0.33\columnwidth]{final_swarm_figs/stan140_1_err_Sigma}\includegraphics[width=0.33\columnwidth]{final_swarm_figs/stan140_1_err_mu}\linebreak{}

\includegraphics[width=0.33\columnwidth]{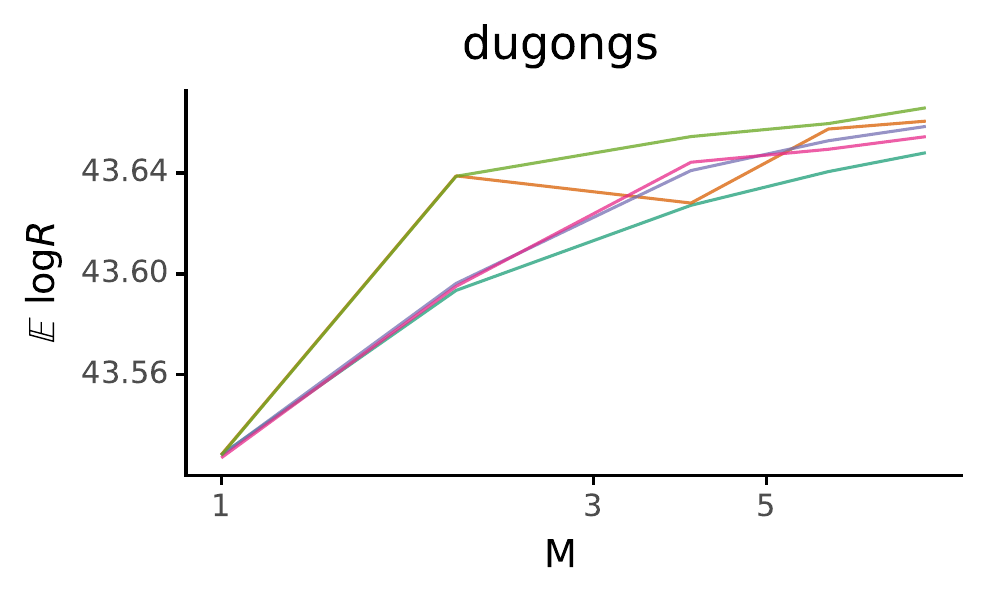}\includegraphics[width=0.33\columnwidth]{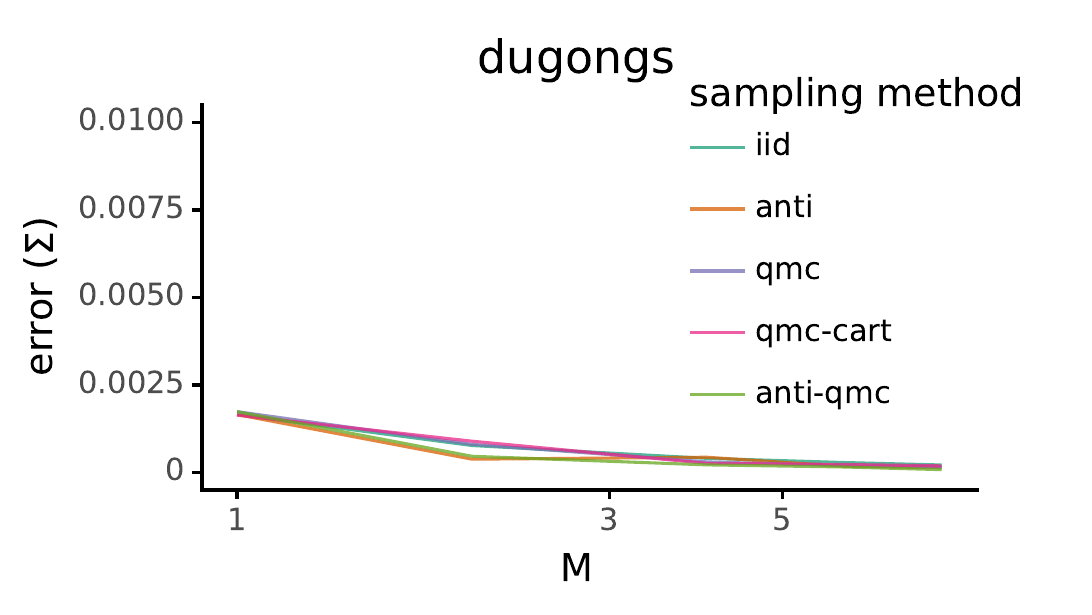}\includegraphics[width=0.33\columnwidth]{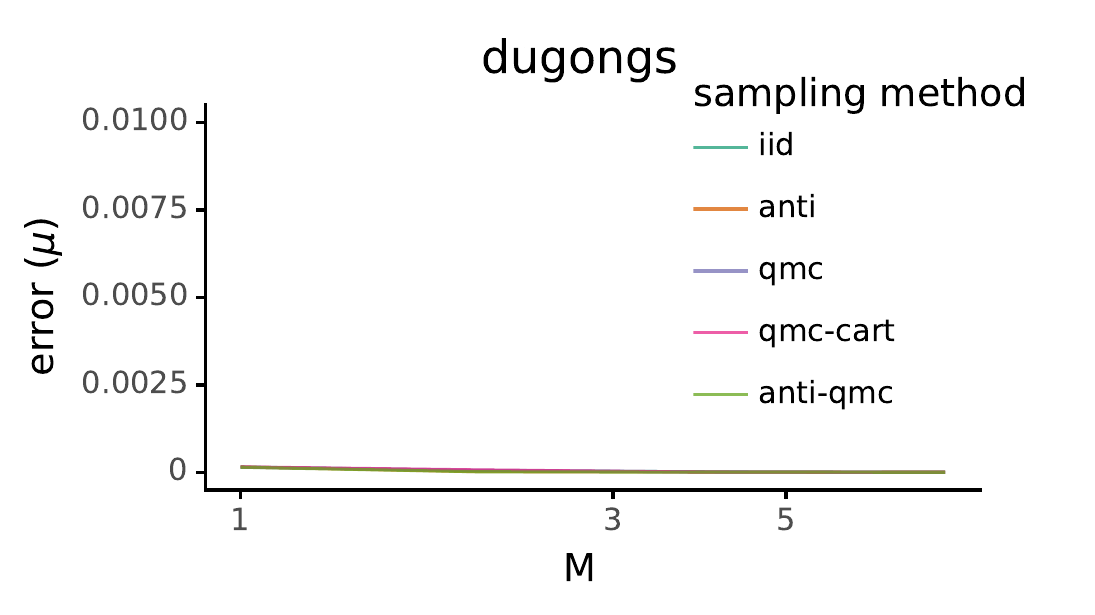}\linebreak{}

\includegraphics[width=0.33\columnwidth]{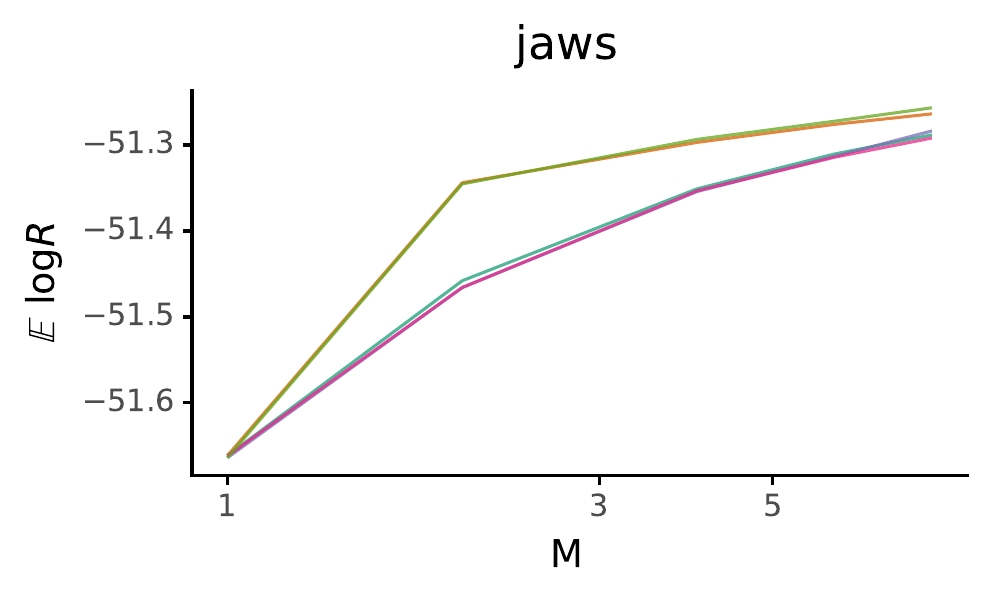}\includegraphics[width=0.33\columnwidth]{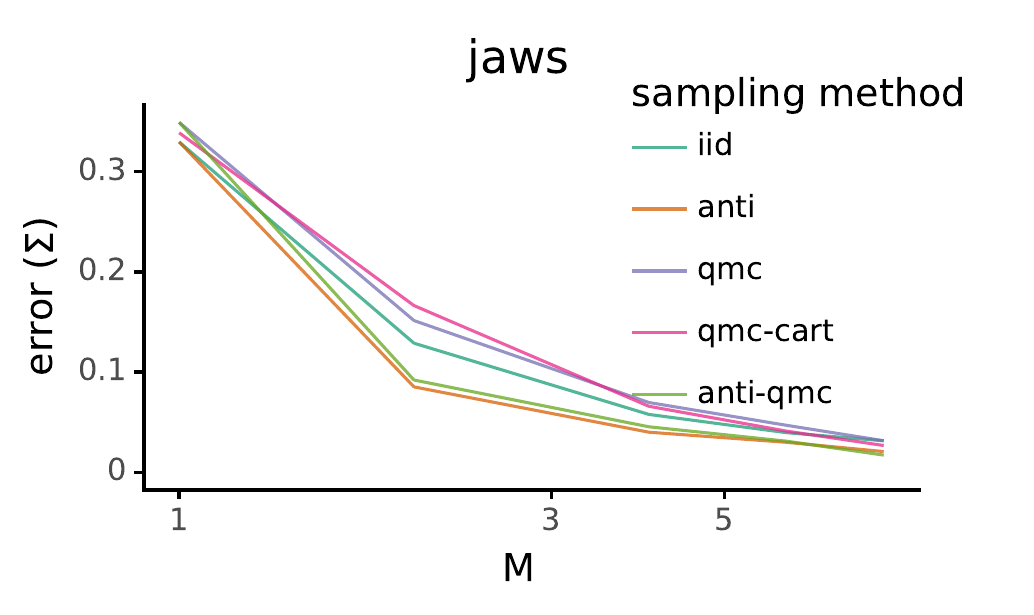}\includegraphics[width=0.33\columnwidth]{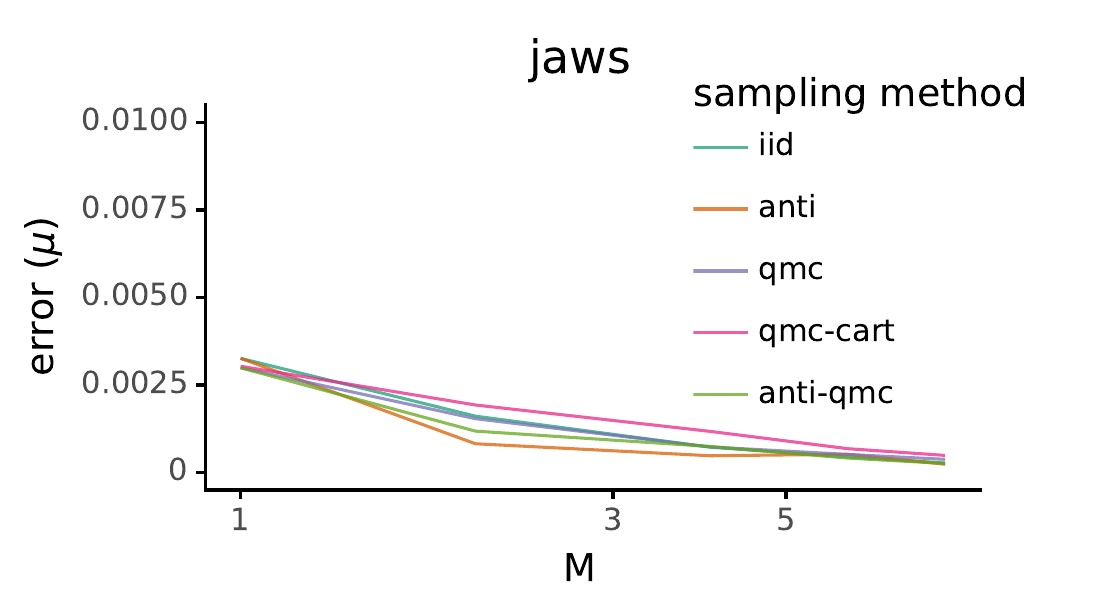}\linebreak{}

\includegraphics[width=0.33\columnwidth]{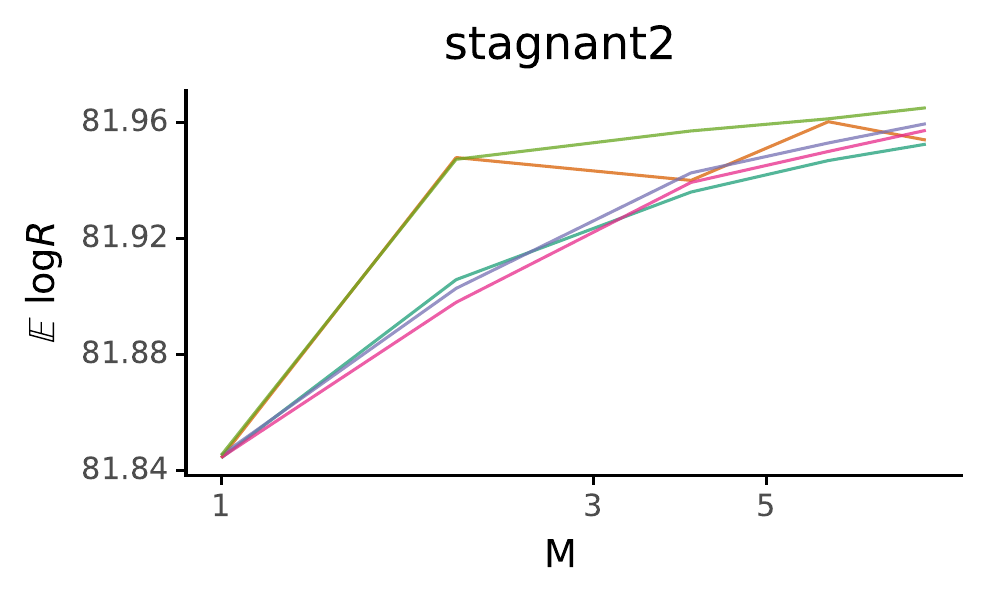}\includegraphics[width=0.33\columnwidth]{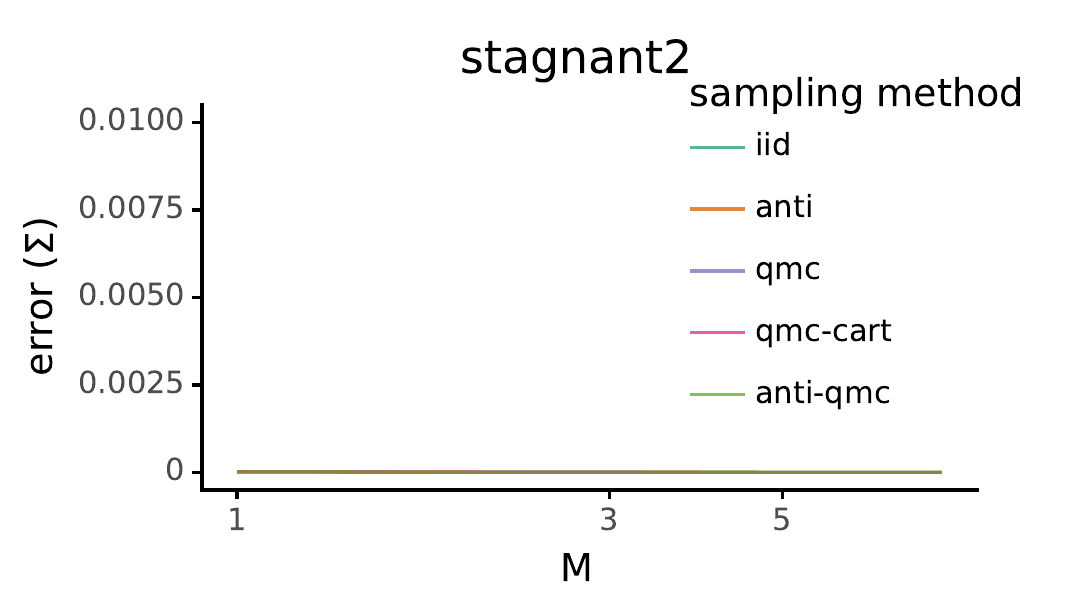}\includegraphics[width=0.33\columnwidth]{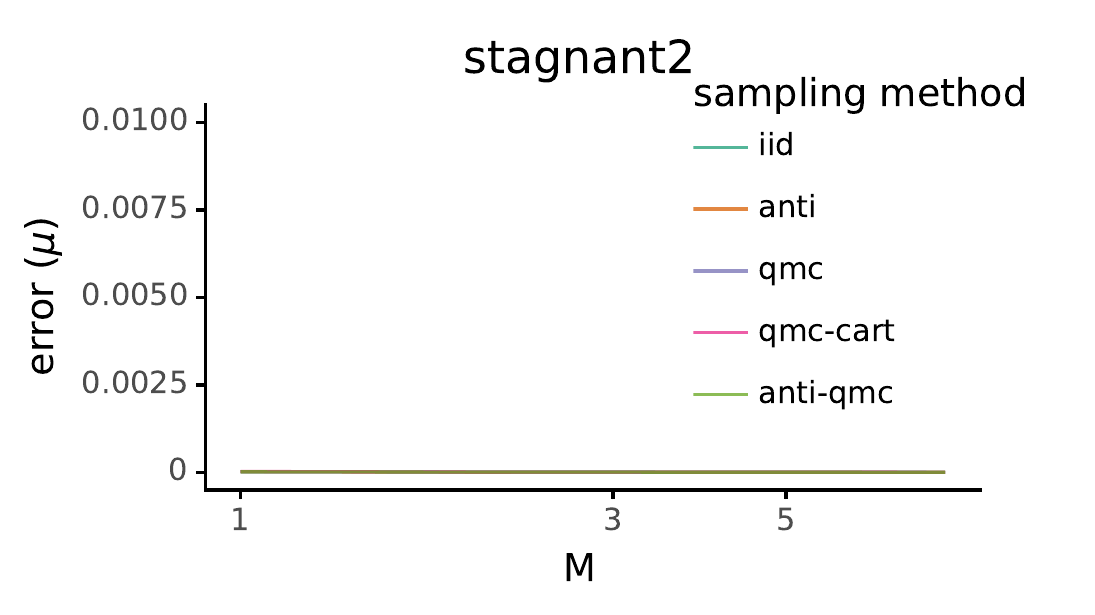}\linebreak{}

\caption{\textbf{Across all models, improvements in likelihood bounds correlate
strongly with improvements in posterior accuracy. Better sampling
methods can improve both.}}
\end{figure}

\begin{figure}
\includegraphics[width=0.33\columnwidth]{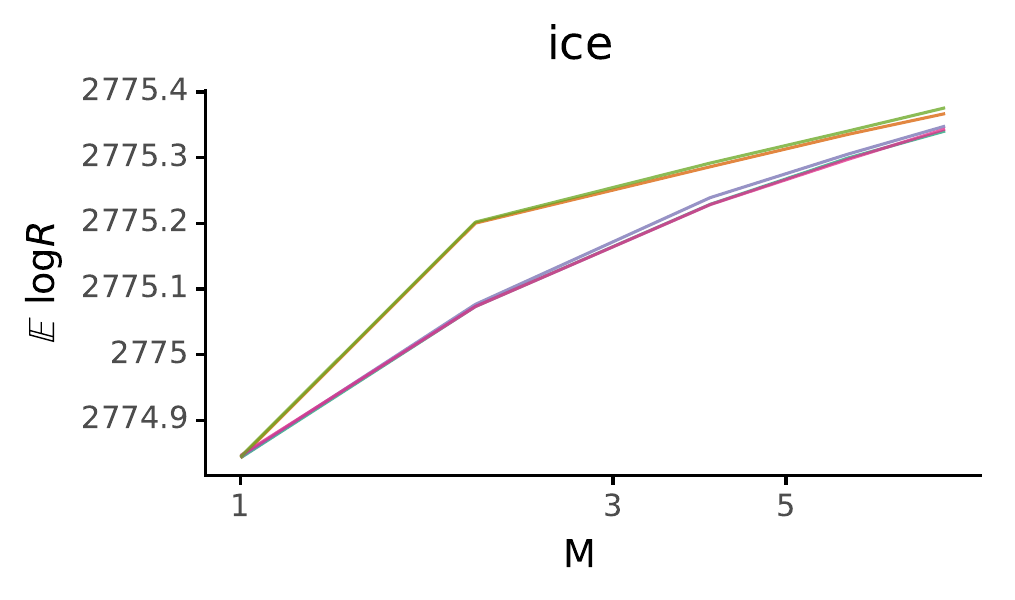}\includegraphics[width=0.33\columnwidth]{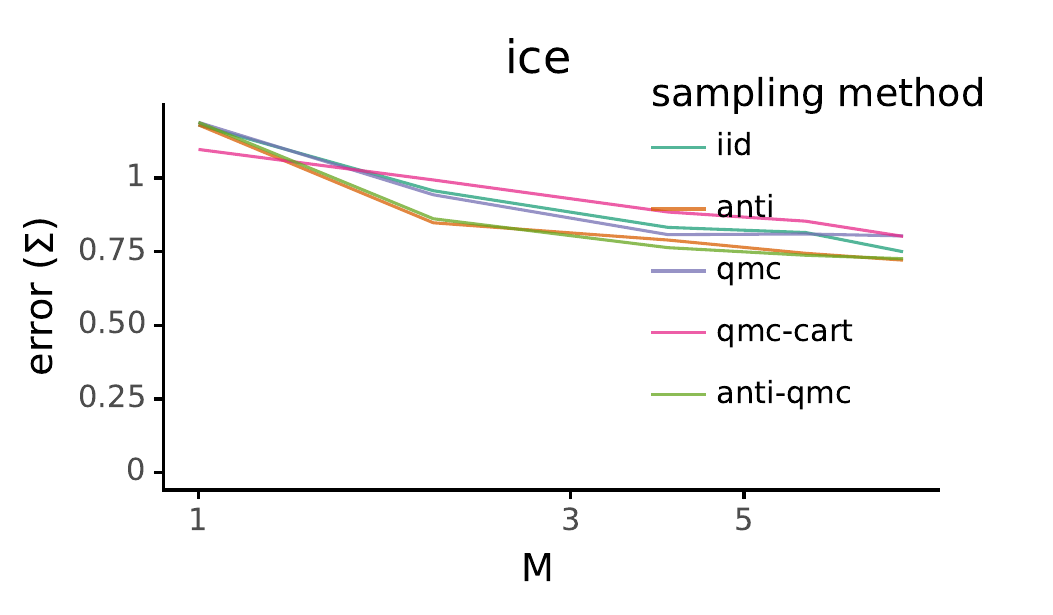}\includegraphics[width=0.33\columnwidth]{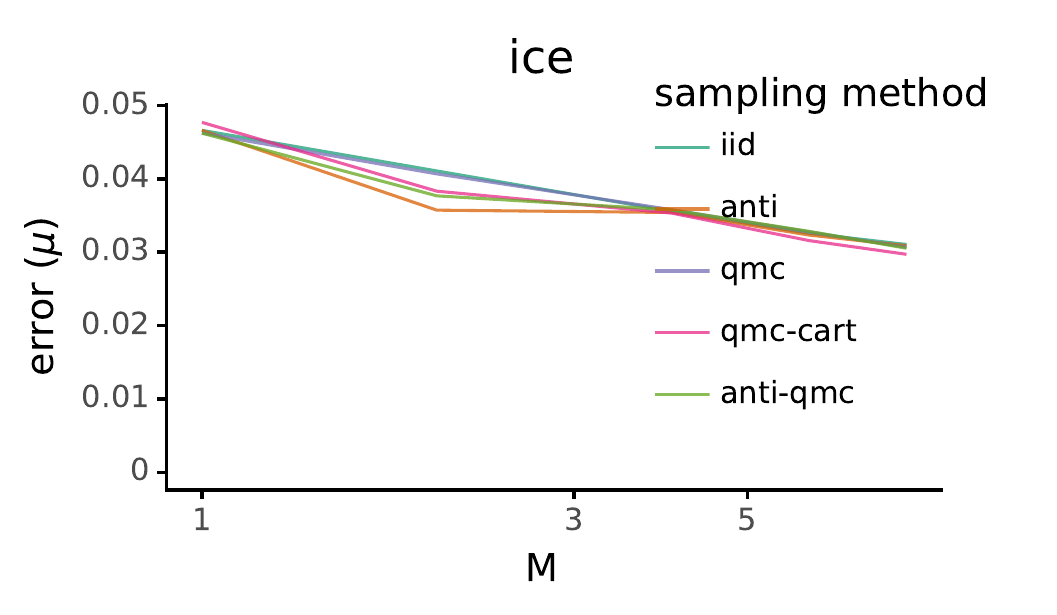}\linebreak{}

\includegraphics[width=0.33\columnwidth]{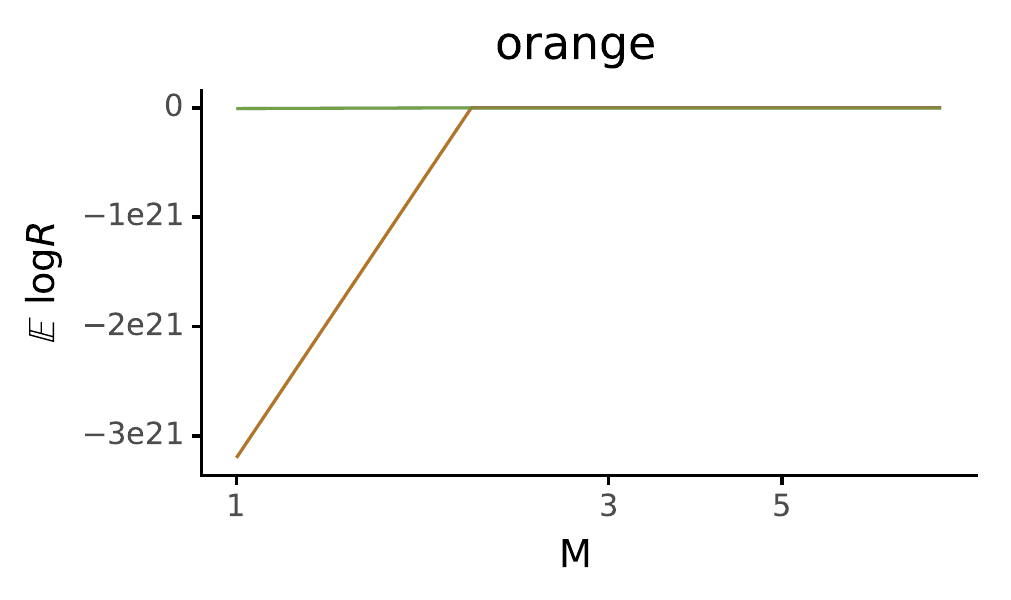}\includegraphics[width=0.33\columnwidth]{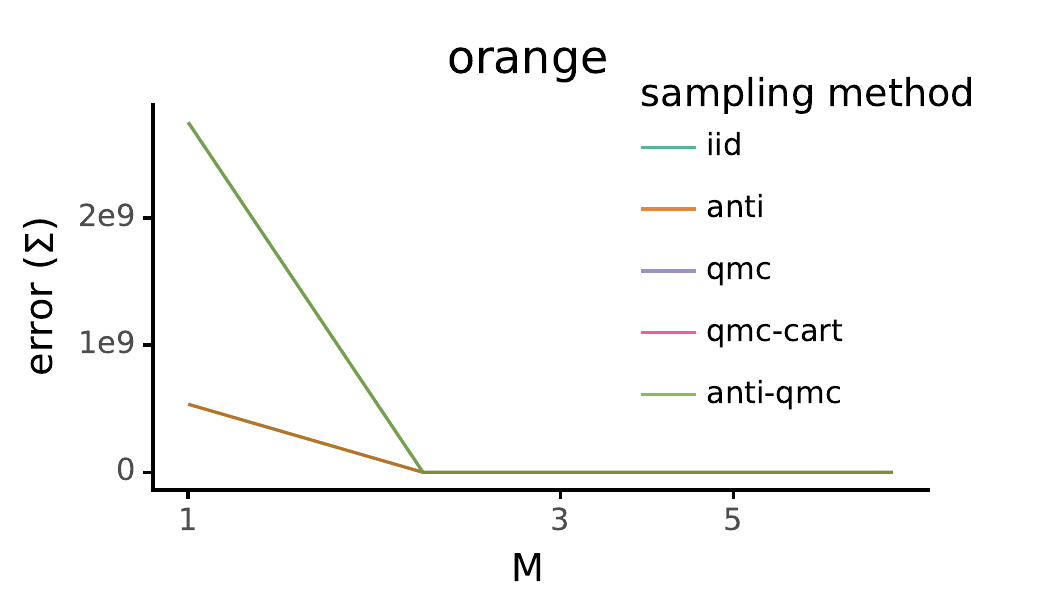}\includegraphics[width=0.33\columnwidth]{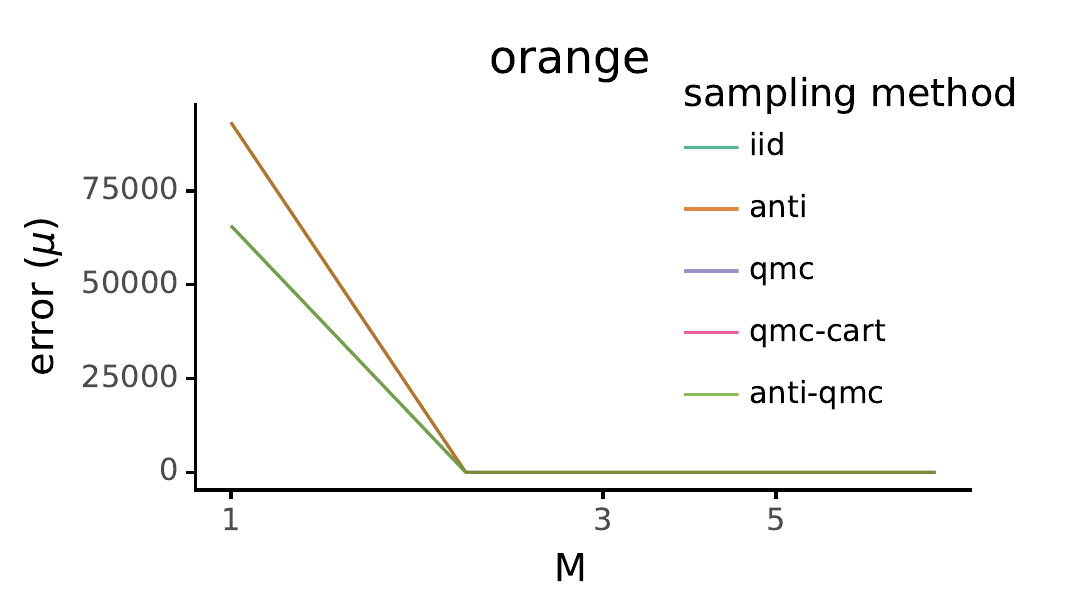}\linebreak{}

\includegraphics[width=0.33\columnwidth]{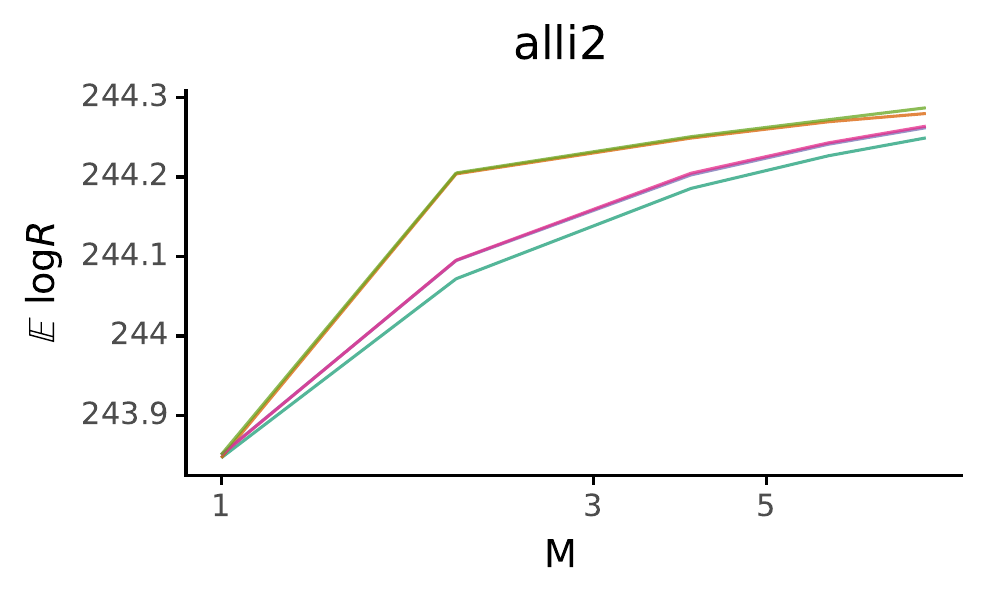}\includegraphics[width=0.33\columnwidth]{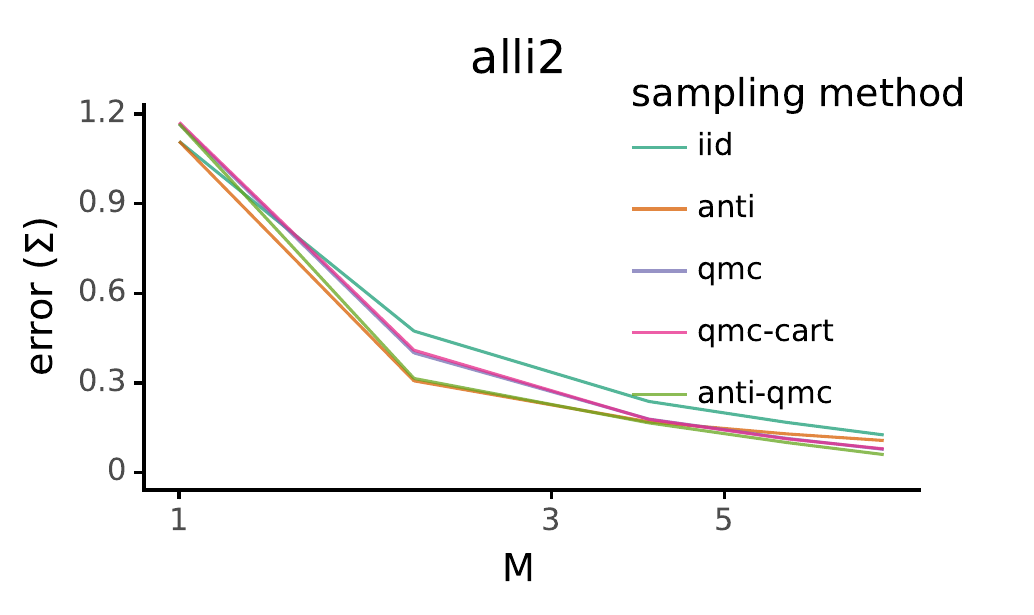}\includegraphics[width=0.33\columnwidth]{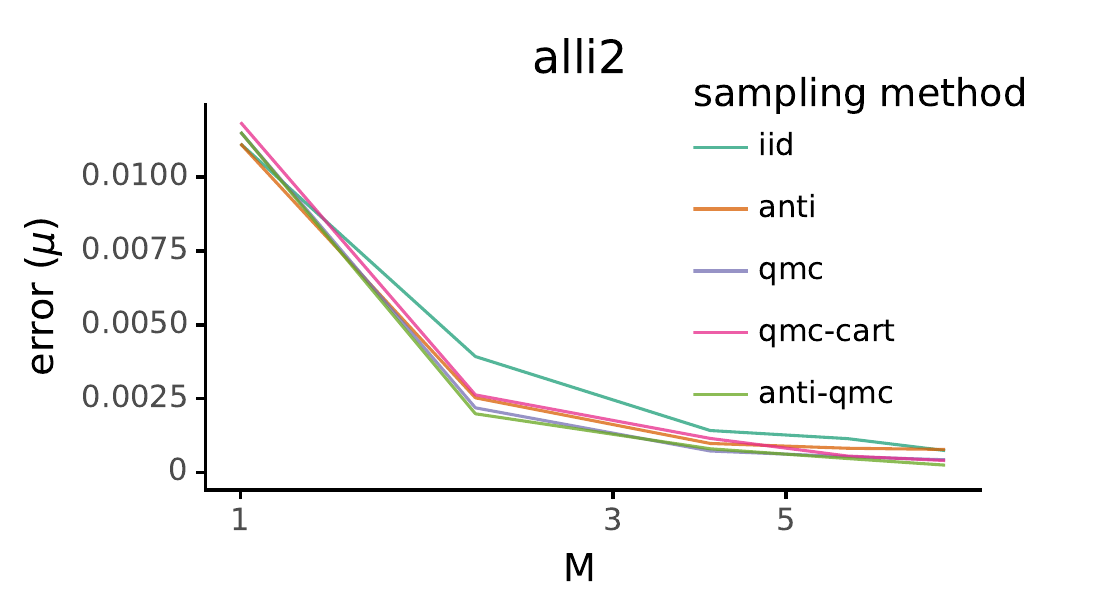}\linebreak{}

\includegraphics[width=0.33\columnwidth]{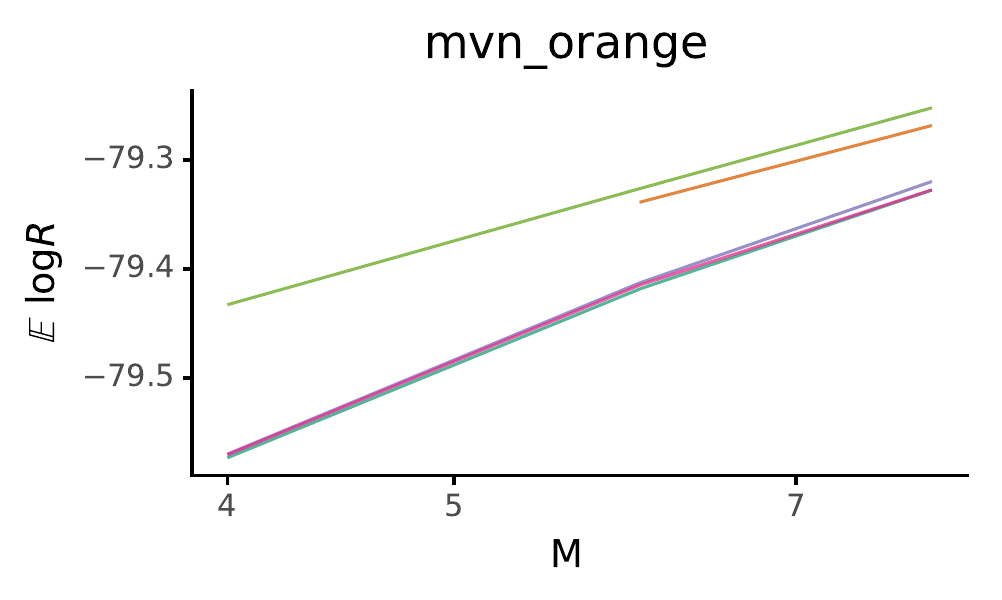}\includegraphics[width=0.33\columnwidth]{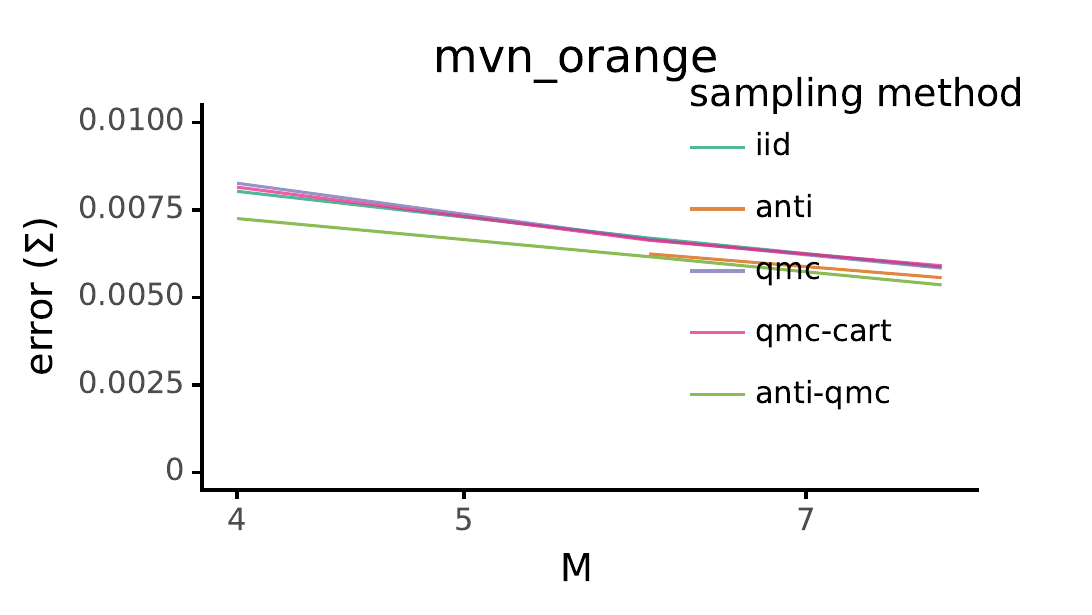}\includegraphics[width=0.33\columnwidth]{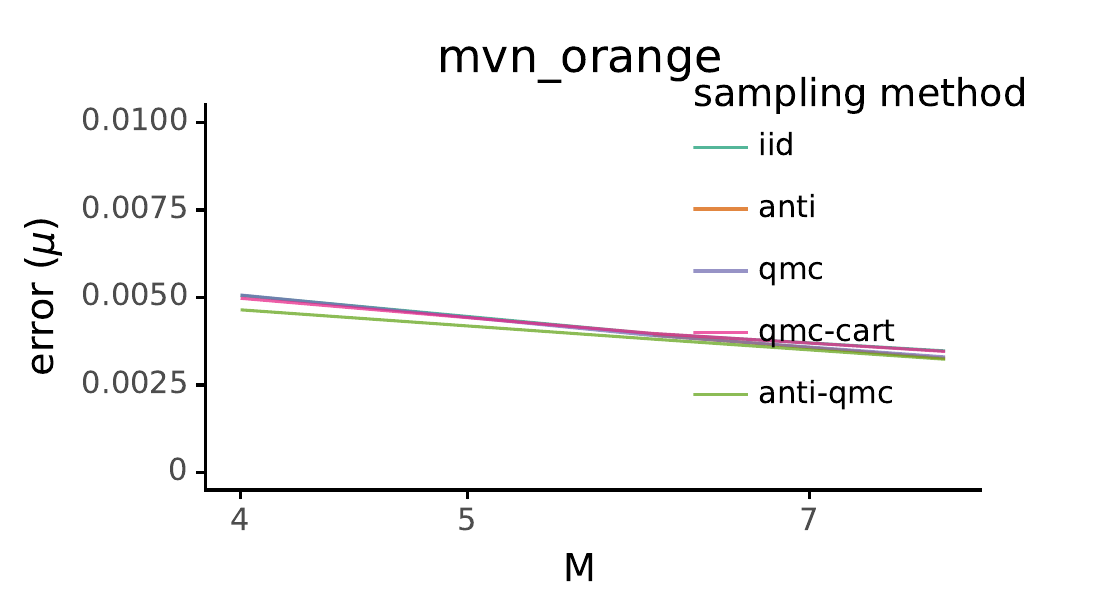}\linebreak{}

\caption{\textbf{Across all models, improvements in likelihood bounds correlate
strongly with improvements in posterior accuracy. Better sampling
methods can improve both.}}
\end{figure}

\end{document}